\pdfoutput=1
\documentclass[10pt]{article}
\makeatletter
\def\input@path{{tex/}{./}{../}}
\makeatother
\usepackage[letterpaper,hmargin=1in,vmargin=1in]{geometry}
\usepackage[T1]{fontenc}
\usepackage{lmodern}
\usepackage{amsmath,amssymb,amsthm,mathtools,mathrsfs,microtype}
\usepackage{booktabs,array,enumitem,xcolor,url,graphicx}
\usepackage[font=small,skip=4pt]{caption}
\usepackage{subcaption}
\usepackage{algorithm}
\usepackage[noend]{algpseudocode}
\algnewcommand\algorithmicinput{\textbf{Input:}}\algnewcommand\Input{\item[\algorithmicinput]}
\algnewcommand\algorithmicparams{\textbf{Parameters:}}\algnewcommand\Params{\item[\algorithmicparams]}
\algnewcommand\algorithmiczero{\textbf{Zero gradients:}}\algnewcommand\Zero{\item[\algorithmiczero]}
\graphicspath{{./}{../}}
\newcommand{\mainfigscale}{0.8826}
\definecolor{ink}{HTML}{000000}
\definecolor{citeblue}{HTML}{1F57A4}
\definecolor{eqpurple}{HTML}{7E2F8E}
\colorlet{eqrefcolor}{eqpurple}
\usepackage[colorlinks=true,linkcolor=ink,citecolor=citeblue,urlcolor=ink]{hyperref}
\makeatletter
\DeclareRobustCommand{\eqref}[1]{\textup{\begingroup\hypersetup{linkcolor=eqrefcolor}\color{eqrefcolor}%
  \tagform@{\ref{#1}}\endgroup}}
\makeatother
\hypersetup{pdftitle={Flat Loss, Evolving Features: Grokking in Modular Addition},
  pdfauthor={Akash Kumar, Karthik Sankaranarayanan}}

\pdftrailerid{}
\newtheorem{theorem}{Theorem}[section]
\newtheorem{proposition}[theorem]{Proposition}
\newtheorem{lemma}[theorem]{Lemma}
\newtheorem{corollary}[theorem]{Corollary}
\theoremstyle{definition}\newtheorem{definition}[theorem]{Definition}
\theoremstyle{remark}\newtheorem{remark}[theorem]{Remark}
\newtheoremstyle{mainthm}{12pt plus 2pt minus 2pt}{12pt plus 2pt minus 2pt}{\itshape}{}{\bfseries}{.}{\newline}{}
\theoremstyle{mainthm}\newtheorem{mainthm}[theorem]{Theorem}
\theoremstyle{plain}
\IfFileExists{tex/macros.tex}{\input{tex/macros.tex}}{\input{macros.tex}}

\newcounter{proofstep}[subsubsection]
\renewcommand{\theproofstep}{\thesubsubsection.\arabic{proofstep}}
\newcommand{\proofstepheading}[1]{%
  \par\medskip\refstepcounter{proofstep}%
  \noindent\textbf{\theproofstep\quad #1}\par\nobreak\smallskip}

\title{Flat Loss, Evolving Features: Grokking in Modular Addition}
\author{Akash Kumar\thanks{For correspondence: \texttt{akk002@ucsd.edu}}\\
  University of California, San Diego
  \and
  Karthik Sankaranarayanan\\
  Microsoft AI}
\date{}

\begin{document}
\maketitle

\begin{abstract}
Can a network keep learning while its training cross-entropy stays nearly
flat at a positive level? We show that it can. On modular addition, a
standard testbed for grokking, ReLU networks trained with supported-Polar
hidden updates, normalized output-layer updates and weight decay can keep
$100\%$ training accuracy while their hidden features change and held-out
accuracy rises from zero to above $90\%$, at relatively flat training
cross-entropy. Correct classification need not minimize cross-entropy,
which still depends on prediction confidence. We prove this coexistence
for one-hidden-layer networks in two finite-window theorems from
constructed initializations: one gives a constant gain in the Fourier
fraction of the hidden features, and the other holds at arbitrarily low
positive loss, where a readout refitted on the training data generalizes
before the network. Experiments from random initialization show the same
coexistence on a window fixed in advance, and in deeper networks held-out
accuracy also rises over post-hoc windows with a net increase in
training cross-entropy. \added{With a last-layer rate set per training
fraction and head decay scaled with the prime, deeper networks at $p=47$
and $61$ meet a criterion on post-hoc windows, fixed in advance, in nearly
all fresh runs. A one-block transformer trained
with the same updates and a slowly decaying output gain keeps every
training pair correct at every logged checkpoint while its training
cross-entropy rises and its
held-out accuracy approaches one; with two blocks this holds through the
held-out rise, after which training pairs begin to flip.} A plateau in training loss need not be a plateau in
learning. Code: \url{https://github.com/akashkumar-d/flat-loss-evolving-features}.
\end{abstract}

\ifShowAdditions
\begin{additions}
\par\medskip\noindent
\fbox{\parbox{\dimexpr\linewidth-2\fboxsep-2\fboxrule\relax}{\small
\textbf{Changes since version 1 (2026-10-13).} Text in dark orange is new.
\emph{Added:} transformers, with one block at $p=113$ and other primes and
with two blocks at $p=113$ (Section~\ref{paper:transformer},
Figure~\ref{fig:tf-contrast}, Appendix~\ref{paper:app:tf}); depths three
and four at $p=47$ and $61$ across training fractions, on fresh seeds
(Figure~\ref{fig:deep-fractions}, Appendix~\ref{paper:app:deep-fractions});
and Algorithm~\ref{alg:train}, the training loop
(Section~\ref{paper:setup:updates}).
The new trajectory figures have the rows of Figures~\ref{fig:onehl-primes}
and~\ref{fig:deep-main} (training CE, accuracies, $\AH$; every run thin,
the median thick); the appendix figures across primes and with two blocks add
the misclassified training pairs, and the two on mirror pairs and readouts
show medians.
\emph{Removed:} version 1's Figure~6 (one fixed recipe across six primes)
and Figure~9 (depths three and four across primes); their counts are now in
the text of Appendices~\ref{paper:app:studies}
and~\ref{paper:app:deep-medians}.
\emph{Equal run counts:} every panel of a trajectory figure shows as many
runs as the others in its figure, at least 10: depth 2 in
Figure~\ref{fig:deep-main} on the 40 seeds of depths 3--5 (34/40 pass;
version 1 showed 10 runs, 9 passing, with a different refit), 10 runs per
row in Figure~\ref{fig:thm41-iid}, and 10 seeds per setting in
Figures~\ref{fig:tf-primes}--\ref{fig:tf-readouts}, in
Table~\ref{paper:tab:tf-readouts} and in the rows of
Table~\ref{paper:tab:tf-primes} for those settings, with
$p=61$ on the mirror-free split in Figure~\ref{fig:tf-primes}.
Figure~\ref{fig:tf-blocks} shows the 10 two-block runs at fraction $0.2$,
and Table~\ref{paper:tab:tf-blocks-late} adds tests that change one rate
after the held-out rise. Figure~\ref{paper:fig:headline} is one row.
\emph{Readouts:} why the ridge refit leads in the one-hidden-layer network
but trails in the transformer (Section~\ref{paper:transformer},
Appendix~\ref{paper:app:tf-readouts}, and the network's own top-8 logits in
Table~\ref{paper:tab:tf-readouts}). The transformer trajectory figures
(Figures~\ref{fig:tf-contrast}, \ref{fig:tf-primes} and~\ref{fig:tf-blocks})
no longer draw the refit, which trails the network at its $50\%$ crossing in
every run they show; their accuracy rows keep the Fourier-filtered refit,
and Table~\ref{paper:tab:tf-readouts} and Figure~\ref{fig:tf-readouts} still
report the refit. \emph{Related work:} the
maximum-margin characterization of Morwani et al., and how its Fourier
features (single-frequency weights of each neuron, at the optimum) differ
from $\AH$; a comparison with explanations of grokking by initialization
scale and weight decay, and with the timescale of decoupled weight decay;
and two earlier papers on feature learning during loss plateaus in
teacher-student regression (Sections~\ref{paper:intro}
and~\ref{paper:related}).
\emph{Presentation:} Sections~\ref{paper:intro}--\ref{paper:discussion} are
reworded and condensed for readability; the rewording adds no claims and is
not marked. The introduction is shorter. Citations are printed in blue and
equation references in purple.
The statements of Theorems~\ref{paper:thm:half} and~\ref{paper:thm:random}
are set upright, with named hypotheses and conclusions and their main
bounds displayed; their mathematical content is unchanged.
Figures~\ref{fig:onehl-primes}--\ref{fig:tf-contrast}, \ref{fig:h-cohort},
\ref{fig:deep-fractions-app}--\ref{fig:deep-fractions-app3}
and~\ref{fig:tf-primes} are drawn taller. In
Figures~\ref{paper:fig:headline}--\ref{fig:tf-contrast}, the panels are
centered on the text column and all five figures are printed at one scale.
The side margins are one inch, and the appendices on experiments
(\ref{paper:app:setup-ext}--\ref{paper:app:tf}) now come before the proofs
(\ref{sec:main}--\ref{app:decay}).
Setting \texttt{\textbackslash ShowAdditionsfalse}
in \texttt{macros.tex} prints the new text in black and removes this box.}}
\par\medskip
\end{additions}
\fi

\section{Introduction}\label{paper:intro}
\begin{figure}[t]
\centering
\captionsetup{skip=6pt}
\includegraphics[scale=\mainfigscale]{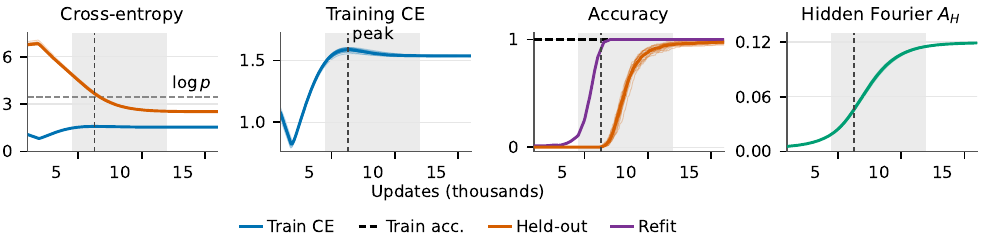}
\caption{\textbf{Coexistence on a fixed window} (1HL ReLU network, $p=31$,
width 512, iid initialization, head decay $\lambda_V=0.2$). All 20 runs,
each with a fresh initialization and training split (thin), and the median
(thick). From the left: training and held-out CE; the training CE on its
own scale; training accuracy (dashed) and the held-out accuracy of the
network and of the train-only ridge refit; the Fourier fraction $\AH$.
Gray: the locked window 4.5k--12k; dashed line: median per-run loss peak.
Loss curves show cycle means. On the window, every run keeps perfect
training accuracy, has raw per-update loss variation $6.33$--$7.81\%$ and
reaches $90\%$ held-out accuracy. In every run, a few training pairs are
misclassified again for a while after the first fit, long before the
window (Appendix~\ref{paper:app:strict}).}
\label{paper:fig:headline}\label{paper:fig:hero}
\end{figure}
The training loss is the standard measure of a neural network's progress,
and a plateau in the loss is commonly interpreted as a sign that learning
has stalled. Grokking calls this interpretation into question. On small algorithmic tasks such as modular addition, networks fit
their training data long before they generalize~\cite{power2022}. Progress
measures reveal the formation of a generalizing circuit while the training
loss remains nearly flat~\cite{nanda2023}, and the learned features change
gradually before held-out accuracy recovers~\cite{mallinar2025emergence}.
Such hidden progress has been proved for online sparse-parity
learning~\cite{barak2022}\added{ and for teacher-student regression during
loss plateaus, under orthogonalized updates~\cite{kumar2026polar} and for
the population loss~\cite{kumar2026buddha},} and analyzed after fitting
under the squared loss with coupled weight decay~\cite{pracher2026spectral}.

Classification networks, however, are trained with the cross-entropy (CE)
loss, for which fitting the training set does not end training. A
training pair is classified correctly as soon as its margin, the gap
between the logit of its true class and the largest other logit, is
positive. The CE also measures confidence: a network that classifies every
training pair correctly can have any CE below $\log p$, the loss of the
uniform prediction over $p$ classes. Without regularization, the CE can
keep decreasing as the logits grow~\cite{prieto2025numerical}; with weight
decay, it can level off at a positive value. A level CE then fixes only the
mean log-probability of the true class on the training pairs. It does not
determine whether the hidden features, or the predictions on held-out
pairs, are still changing. This leads to the following question.
\begin{quote}\itshape
\begin{minipage}{\linewidth}
Can a network that has fitted its training data continue to improve its
hidden features, and recover held-out accuracy, during a window of training
on which its CE stays relatively flat at a positive level?
\end{minipage}
\end{quote}

We study this question on a window $I$ of the actual training trajectory.
Flatness is measured by the relative variation
$v(I)=\max_{t\in I}L_t/\min_{t\in I}L_t-1$ of the raw training CE, checked
at every update, and generalization by the held-out accuracy of the
network's own head. Two observables describe the features. The Fourier
fraction $\AH$, the share of the centered hidden-activation energy that
lies in mean-zero functions of $a+b$, measures their structure. A ridge
readout fitted on the training pairs alone, and never installed in the
network, measures whether they already support held-out prediction
(Section~\ref{paper:setup}).

Figure~\ref{paper:fig:headline} shows that the answer can be affirmative
in practice. It follows a one-hidden-layer (1HL) ReLU network at $p=31$,
trained from fresh iid initializations and training splits, on a window
fixed before the runs were drawn. On this window, every run keeps all
training pairs correct, with raw training CE varying by less than $8\%$,
while held-out accuracy climbs from zero to above $90\%$, the train-only refit
reaches $90\%$ more than $3{,}000$ updates before the network, and $\AH$
rises from about $0.02$ to $0.12$ (Section~\ref{paper:onehl}).

\paragraph{Our Contributions.}
For 1HL ReLU networks started in constructed open sets of
initializations, we prove joint finite-window guarantees that place a
relatively flat positive CE, a gain in $\AH$ and held-out recovery on the
same training trajectory. The hidden weights take supported-Polar steps, an
exact and momentum-free form of the orthogonalized updates of
Muon~\cite{jordan2024muon}, and the head takes normalized gradient steps.
Apart from weight decay, both steps have a fixed length for every nonzero
gradient, so a flat loss does not imply that the weights have stopped
moving.
\begin{itemize}[leftmargin=*,itemsep=1pt,topsep=2pt,parsep=0pt]
\item \textbf{High-loss coexistence (Theorem~\ref{paper:thm:half}).} On a
structured split of nearly half the grid, for every prime $p\ge101$ and
sufficiently large width, the network keeps all training pairs correct
with CE within $0.0201$ of $\log p$, while $\AH$ gains more than $0.0998$
and held-out accuracy rises from zero to one.
\item \textbf{Low-loss, refit-first recovery (Theorem~\ref{paper:thm:random}).}
For large primes and a random split, with probability at least
$1-2/\sqrt p$, the CE stays within $5\%$ of a prescribed positive level,
which can be arbitrarily small, and $\AH$ increases at every update; the
train-only refit is accurate from the first update, and the network only
at the end of the window.
\item \textbf{Experiments (Section~\ref{paper:experiments}).} From iid
initialization at practical widths, every fresh 1HL run shows this
coexistence on a window fixed in advance, and networks with three to five
hidden layers at $p=31$ meet a pre-specified joint criterion in nearly
every run. \added{At $p=47$ and $61$, networks with three and four hidden
layers meet a criterion fixed before the runs, on post-hoc windows, in
nearly every run, but the joint criterion only rarely. A one-block
transformer with the same updates and a slowly decaying output gain meets
an analogous criterion, with a Fourier-filtered refit, in every run at
$p=89$ and $113$ while its training CE rises; with two blocks, the CE also
rises with held-out accuracy, but training pairs then begin to flip.}
\end{itemize}

Entry into the constructed initial regions from iid initialization is not
proved; the experiments probe the behavior beyond these constructions.
Section~\ref{paper:related} reviews related work,
Section~\ref{paper:setup} defines the setting, Section~\ref{paper:theory}
states the theorems, Section~\ref{paper:experiments} reports the
experiments, and Section~\ref{paper:discussion} discusses the limits of
both. The appendices give the experimental details
(Appendices~\ref{paper:app:setup-ext}--\ref{paper:app:tf}), the notation
(Appendix~\ref{app:notation}) and the proofs
(Appendices~\ref{sec:main}--\ref{app:decay}).

\section{Related Work}\label{paper:related}
\paragraph{Grokking and modular arithmetic.}
Proposed explanations of grokking include representation
formation~\cite{liu2022effective}, weight-norm
dynamics~\cite{liu2023omnigrok}, competition between memorizing and more
efficient circuits~\cite{varma2023circuit}, transitions from lazy to rich
training~\cite{kumar2024lazy}, and early and late implicit
biases~\cite{lyu2024dichotomy}. For shallow networks on modular addition,
prior work analyzes margin maximization~\cite{morwani2024margin}, the
sample complexity of kernel and rich regimes~\cite{mohamadi2024grok} and,
more recently, the training
dynamics~\cite{tian2026scaling,he2026mechanism,kunin2025agf}.

\begin{additions}
\paragraph{Initialization scale and weight decay.}
From a large initialization, a network fits its training set quickly, and
weight decay then shrinks the weights slowly until the network
generalizes~\cite{liu2023omnigrok}; with initialization scale $\alpha$ and
a small decay $\lambda$, the switch from a kernel predictor to a
max-margin one comes after a time of order
$\lambda^{-1}\log\alpha$~\cite{lyu2024dichotomy}. With decoupled weight
decay, the weights average the recent updates over about $1/(\eta\lambda)$
steps, a timescale that works best when held at a fixed number of
epochs~\cite{wang2025adamw}, so in practice the weights can keep relaxing
over much of training. Our networks share this ingredient: the
one-hidden-layer experiments start from a large first layer ($c=24$) under
weight decay, and the transformer's output gain relaxes from twice its
equilibrium (Section~\ref{paper:transformer}). These works follow the
weight norm and the test loss as the training loss becomes small; we follow
the training CE at a positive level together with the features, and bound
both on finite windows.
\end{additions}

\paragraph{Fourier features and margin maximization.}
Nanda et al.~\cite{nanda2023} identify circuit formation while the training
loss remains flat, tracked by losses restricted to, and excluding, the key
frequencies. Gromov~\cite{gromov2023} gives explicit Fourier solutions for
shallow networks, and mechanistic studies identify several algorithms for
arithmetic and group operations~\cite{chughtai2023toy,zhong2023clock}.
\added{Morwani et al.~\cite{morwani2024margin} prove that, on the full
grid, sufficiently wide bias-free one-hidden-layer networks with quadratic activation
and maximum $L_{2,3}$-margin consist of single-frequency neurons that
together use every frequency; this describes the maximum-margin solution, not
when the structure forms.} Mallinar et al.~\cite{mallinar2025emergence}
measure the circulant deviation of learned feature matrices in recursive
feature machines and shallow networks. Our $\AH$ is instead the fraction of
hidden-activation energy in a task-aligned subspace\added{, measured along
the trajectory at positive loss. Even the maximum-margin network has
$\AH=2/5$, because squaring the sum of two single-frequency inputs of equal
amplitude also produces terms in $a-b$, $2a$ and $2b$}.

\paragraph{Hidden progress and plateaus.}
Barak et al.~\cite{barak2022} prove hidden progress for online
sparse-parity learning, and plateaus that precede fitting have also been
analyzed~\cite{abbe2023leap,kunin2025agf,gopalani2025plateau}. \added{In
teacher-student regression, the leading eigenspace of the average gradient
outer product aligns with the teacher subspace while the loss plateaus:
under full-batch matrix Muon near the edge of stability, where the loss
cycles with period two and its average stays flat or
rises~\cite{kumar2026polar}, and under population gradient descent from a
small initialization, where a refit on the learned subspace lowers the
error before the loss falls~\cite{kumar2026buddha}. Here the task is
classification with the CE, the plateau comes after the training set is
fit, and its payoff is held-out accuracy.} Related mechanisms
include feature evolution after fitting under the squared loss with
coupled weight decay~\cite{pracher2026spectral} and late generalization
with fixed features~\cite{kim2026clock}. Recursive feature machines empirically
grok at zero training loss~\cite{mallinar2025emergence}, and probes reveal delays
between representations and readouts~\cite{chou2026twospeeds}. Our
constructions separate training confidence, the usefulness of the
representation and the recovery of the readout. Their bounds are joint: on
one finite window of the actual run, under fixed updates, they control the
relative CE variation, the Fourier fraction and the predictions together,
for every initialization in the constructed open sets. In these
constructions, recovery need not wait for a near-zero CE, and the
random-split theorem certifies that the refit recovers before the network
without installing the refit as the head.

\paragraph{Orthogonalized updates.}
Our feature update rescales the supported Polar factor of the negative
gradient to unit Frobenius norm and uses no momentum. Before rescaling, this factor
maximizes the first-order decrease of the loss over the unit spectral-norm
ball~\cite{bernstein2024anthology}, which connects the update to spectral
descent and Muon~\cite{carlson2015spectral,jordan2024muon,bernstein2024duality}.
Orthogonalized updates, and updates that balance the singular directions,
can shorten grokking delays~\cite{tveit2025muon,sahebpasand2026egd}. We
study a different question: the change of the representation and the
recovery of held-out accuracy during a window of relatively flat loss,
rather than acceleration.

\section{Problem Setup}\label{paper:setup}

\subsection{Task, model, and training objective}\label{paper:setup:model}
For an odd prime $p$, let $\F_p=\{0,\ldots,p-1\}$ with arithmetic modulo
$p$, and let $\Omega=\F_p^2$. The task is modular
addition~\cite{power2022}: a pair $(a,b)\in\Omega$ is encoded as the
concatenated one-hot vector $x_{ab}=(e_a,e_b)\in\R^{2p}$~\cite{gromov2023},
where $e_c$ is a standard basis vector, and its label is
$y^\star(x_{ab})=y_{ab}=a+b\pmod p$. A set $\Train\subset\Omega$ of $n$
training pairs is fixed, and $\Test=\Omega\setminus\Train$ is the
held-out set; we write $x\in\Train$ for $x=x_{ab}$ with $(a,b)\in\Train$,
and similarly for $\Test$. Each theorem specifies how $\Train$ is drawn.

The theory concerns a bias-free one-hidden-layer (1HL) ReLU network of
width $M$, with hidden weights $\mat{W}\in\R^{M\times2p}$ and output head
$\mat{V}\in\R^{p\times M}$, trained on the mean cross-entropy (CE):
\begin{equation}\label{paper:model}
\begin{aligned}
 h(x)&=\ReLU(\mat{W}x),\qquad f(x)=\mat{V}h(x),\\
 \CE(f(x),y)&=\log\sum_{c\in\F_p}e^{f_c(x)}-f_y(x),\\
 L(\mat{W},\mat{V})&=\frac1n\sum_{x\in\Train}\CE(f(x),y^\star(x)).
\end{aligned}
\end{equation}
The network predicts $\arg\max_cf_c(x)$, with a fixed tie-breaking rule,
and a subscript $t$ denotes a quantity after $t$ updates. The hidden table
$\mat{H}_t\in\R^{p^2\times M}$ has the row $h_t(x)^\top$ for each of the
$p^2$ inputs; it serves only to measure the features and does not enlarge
the training set. A subscript $B\subseteq\Omega$ restricts its rows, as in
$\mat{H}_{t,\Train}$. We write $\Id_N$ for the $N\times N$ identity matrix
and $\one_N$ for the all-ones vector in $\R^N$. A network with $D$ hidden
layers computes $h^{(\ell)}=\ReLU(\mat{W}^{(\ell)}h^{(\ell-1)}+b^{(\ell)})$
for $\ell=1,\ldots,D$, with $h^{(0)}=x$, and $f=\mat{V}h^{(D)}$; layer
$\ell$ has the table $\mat{H}_t^{(\ell)}$. The baseline networks have no
biases; the dotted medians in Figures~\ref{fig:onehl-primes}
and~\ref{fig:deep-main} come from runs with trained hidden biases
(Appendix~\ref{paper:app:bias}).

\subsection{Fixed normalized updates}\label{paper:setup:updates}
We train the hidden weights with a supported-Polar update and the head
with normalized gradient descent. Both use the negative gradients of the
training CE at the same state,
$\mat{F}_t=-\nabla_{\mat{W}}L(\mat{W}_t,\mat{V}_t)$ and
$\mat{G}_t=-\nabla_{\mat{V}}L(\mat{W}_t,\mat{V}_t)$. For a nonzero
matrix $\mat{F}$ with thin singular value decomposition
$\mat{F}=\mat{U}\mat{\Sigma}\mat{Z}^\top$ over its positive singular
values, the supported Polar factor is $\pol(\mat{F})=\mat{U}\mat{Z}^\top$,
and $\pol(0)=0$. The two normalizations are
\begin{equation}\label{paper:normalizations}
 \cN_{\rm pol}(\mat{F})=\frac{\mat{U}\mat{Z}^\top}{\sqrt{\rank\mat{F}}},
 \qquad \cN_F(\mat{G})=\frac{\mat{G}}{\norm{\mat{G}}_F},
\end{equation}
with $\cN_{\rm pol}(0)=\cN_F(0)=0$; for a nonzero argument, each returns a
matrix of unit Frobenius norm. The Polar normalization moves the weights
equally along every singular direction of the gradient, whereas $\cN_F$
preserves the ratios of the singular values. With fixed rates
$\eta_W,\eta_V>0$ and decays $\lambda_W,\lambda_V\ge0$, the updates are
\begin{equation}\label{paper:optimizer}
\begin{aligned}
 \mat{W}_{t+1}&=(1-\eta_W\lambda_W)\mat{W}_t+\eta_W\cN_{\rm pol}(\mat{F}_t),\\
 \mat{V}_{t+1}&=(1-\eta_V\lambda_V)\mat{V}_t+\eta_V\cN_F(\mat{G}_t).
\end{aligned}
\end{equation}
Each update shrinks the weights by a fixed factor and adds a step of fixed
length, with no momentum, projection or loss-dependent control. For every
nonzero gradient, however small, the step length is the same; the weights
can therefore keep moving while the loss is flat, although the fixed step
by itself guarantees neither a flat loss nor useful motion. The reported loss $L$ is the CE alone, without
decay penalties. The theorems use $\lambda_V=0$\added{, and each
experiment states its head decay}. In the experiments, the positive
singular values are selected with a numerical cutoff
(Appendices~\ref{paper:app:experiments-ext}
and~\ref{paper:app:deep-ext}).
\added{Algorithm~\ref{alg:train} states the training loop for $D$ hidden
layers, each with its own rate and decay; with one hidden layer it reduces
to~\eqref{paper:optimizer}, with $(\eta_1,\lambda_1)=(\eta_W,\lambda_W)$.
With trained biases, $[\mat{W}^{(\ell)}\mid b^{(\ell)}]$ takes the feature
step (Appendix~\ref{paper:app:bias}). In the transformer of
Section~\ref{paper:transformer}, every weight matrix except the unembedding
takes the feature step, the unembedding takes the head step, and a scalar
output gain takes a sign step with decay. The Fourier fraction and the
refit defined below are computed from the weights and never enter the
updates.}

\begin{algorithm}[!htb]
\addedcolor
\caption{\added{Training with supported-Polar feature steps and normalized head steps}}
\label{alg:train}
\begin{algorithmic}[1]
\Input training pairs $\Train$, $n=|\Train|$; initial hidden matrices
$\mat{W}^{(1)}_0,\ldots,\mat{W}^{(D)}_0$ and head $\mat{V}_0$
\Params rates $\eta_1,\ldots,\eta_D,\eta_V>0$; decays
$\lambda_1,\ldots,\lambda_D,\lambda_V\ge0$; number of updates $T$
\Zero $\cN_{\rm pol}(0)=\cN_F(0)=0$, as in~\eqref{paper:normalizations},
so a matrix whose gradient is zero only decays
\For{$t=0,1,\ldots,T-1$}
  \State $L_t\gets\frac1n\sum_{x\in\Train}\CE\bigl(f_t(x),y^\star(x)\bigr)$
    \Comment{full batch: every training pair}
  \State $\mat{F}^{(\ell)}_t\gets-\nabla_{\mat{W}^{(\ell)}}L_t$ for
    $\ell=1,\ldots,D$, and $\mat{G}_t\gets-\nabla_{\mat{V}}L_t$
    \Comment{all at the same state}
  \For{$\ell=1,\ldots,D$}
    \State $\mat{U}\mat{\Sigma}\mat{Z}^\top\gets$ thin SVD of
      $\mat{F}^{(\ell)}_t$ over its positive singular values
    \State $\mat{W}^{(\ell)}_{t+1}\gets(1-\eta_\ell\lambda_\ell)\mat{W}^{(\ell)}_t
      +\eta_\ell\,\mat{U}\mat{Z}^\top/\sqrt{\rank\mat{F}^{(\ell)}_t}$
      \Comment{$\eta_\ell\,\cN_{\rm pol}(\mat{F}^{(\ell)}_t)$, eq.~\eqref{paper:normalizations}}
  \EndFor
  \State $\mat{V}_{t+1}\gets(1-\eta_V\lambda_V)\mat{V}_t
    +\eta_V\,\mat{G}_t/\norm{\mat{G}_t}_F$
    \Comment{$\eta_V\,\cN_F(\mat{G}_t)$}
\EndFor
\end{algorithmic}
\end{algorithm}

\subsection{Fitting, plateaus, and delayed recovery}\label{paper:setup:phenomena}
The margin of an input $x$ at time $t$ is
$m_t(x)=f_{t,y^\star(x)}(x)-\max_{c\ne y^\star(x)}f_{t,c}(x)$, and
$\Acc_B(t)$ is the fraction of inputs $x\in B$ with
$\arg\max_cf_{t,c}(x)=y^\star(x)$. The network has \emph{fitted} the
training set when every training margin is positive. Fitting does not
require a small loss: correct predictions with small positive margins can
have a loss arbitrarily close to $\log p$, the loss of the uniform
prediction, so a fitted network can have any training CE in
$(0,\log p)$. Flatness on an integer window $I=\{t_0,\ldots,t_1\}$ is
measured by the relative variation of the raw training CE,
\begin{equation}\label{paper:plateau}
 v(I)=\frac{\max_{t\in I}L_t}{\min_{t\in I}L_t}-1,
\end{equation}
evaluated at every update rather than on plotted or averaged losses. A
$10\%$ plateau is a window with $v(I)\le0.1$; it allows small changes and
oscillations and claims neither stationarity nor optimality. Some experiments also use
the cycle mean $\bar L_t=(L_t+L_{t+1})/2$, the average over two
consecutive updates. For an observable $A$, $\Delta A(I)=A(t_1)-A(t_0)$ is
its change over the window. \emph{Delayed recovery} means that held-out
accuracy becomes high after a phase of fitted training with low held-out
accuracy. Our coexistence claims place the change of the features and the
recovery on the same flat window and assert nothing after it.

\subsection{The Fourier fraction}\label{paper:setup:fourier}
Since the target depends only on $a+b$, we measure the features against
the \emph{addition-Fourier subspace}, the real mean-zero functions on the
grid that depend only on the sum:
\begin{equation}\label{paper:rulespace-main}
 \cF_{\rm add}=\{g(a+b):\textstyle\sum_{y\in\F_p}g(y)=0\}.
\end{equation}
It is spanned by the centered class indicators $\one_{\{a+b=c\}}-1/p$, and
also by the real and imaginary parts of the Fourier modes
$e^{2\pi\mathrm{i}k(a+b)/p}$, $k\ne0$; its dimension is $p-1$
(Appendix~\ref{paper:fourier-defs}). Let $\Padd$ be the orthogonal
projector onto $\cF_{\rm add}$, and let
$\mat{C}_N=\Id_N-\one_N\one_N^\top/N$ and $\Cgrid=\mat{C}_{p^2}$. Then
$\Padd\mat{H}$ averages each hidden feature over the pairs with the same
sum and subtracts its mean over the grid. The \emph{Fourier fraction} of a
table $\mat{H}$ is the share of its centered energy that lies in
$\cF_{\rm add}$:
\begin{equation}\label{paper:AH}
 \AH(\mat{H})=\frac{\norm{\Padd\mat{H}}_F^2}{\norm{\Cgrid\mat{H}}_F^2}\in[0,1],
\end{equation}
defined when $\Cgrid\mat{H}\ne0$. It equals one when the centered response
of every neuron is a function of $a+b$, and zero when no response has a
component in $\cF_{\rm add}$. It is invariant to scaling $\mat{H}$ by a nonzero
constant, and it can rise because the energy in $\cF_{\rm add}$ grows or because the
remaining energy shrinks. Since it weights neurons by their amplitude, it
can also change without a change in the span of the features. A gain of
$0.10$ in $\AH$ means ten percentage points, not a $10\%$ relative
increase. For deep networks, we compute it for each table
$\mat{H}_t^{(\ell)}$.

\subsection{The train-only refit}\label{paper:setup:refit}
A high Fourier fraction does not by itself imply that the features support
correct predictions. At each snapshot $\mat{H}_t$, we therefore fit a
separate ridge head on the training pairs alone:
\begin{equation}\label{paper:expridge}
 \widehat{\mat{B}}_t=\arg\min_{\mat{B}\in\R^{M\times p}}
 \left\{\norm{\mat{H}_{t,\Train}\mat{B}-\mat{T}_\Train}_F^2
             +\tau_t\norm{\mat{B}}_F^2\right\},
\end{equation}
where $\mat{T}$ has the rows $T_x=e_{y^\star(x)}-\one_p/p$ and $\tau_t>0$;
the features are raw, and there is no intercept. The refitted logits are
$\widehat f_t(x)=\widehat{\mat{B}}_t^\top h_t(x)$, computed from the last
hidden table in deep networks. The refit is evaluated on $\Test$ but never
installed in the network, and the training run is not modified. It
measures what ridge regression extracts from the features, which can
differ from what the network's own head predicts. In the sum-loss units
of~\eqref{paper:expridge}, the penalty rule is fixed in advance: with
$R_t=\norm{\mat{H}_{t,\Train}}_F^2$,
\begin{equation}\label{paper:refit-penalties}
 \tau_t=
 \begin{cases}
  10^{-3}n, & \text{canonical experiments},\\
  R_t/(32n), & \text{Theorem~\ref{paper:thm:half}},\\
  \kappa R_t/[4p(p-1)],
       & \text{Theorem~\ref{paper:thm:random}},
 \end{cases}
\end{equation}
where $\kappa$ is the constant of Theorem~\ref{paper:thm:random}.
Appendix~\ref{paper:prediction-defs} gives the Gram form and the
experimental variants. Theorem~\ref{paper:thm:random}(iii) measures the
refit by its normalized held-out squared error,
\begin{equation}\label{paper:refiterror-main}
 E_{\rm ref}(t)=
 \frac{\norm{\mat{H}_{t,\Test}\widehat{\mat{B}}_t-\mat{T}_\Test}_F^2}
      {|\Test|(p-1)/p},
\end{equation}
which equals one for the zero predictor, since $\norm{T_x}_2^2=(p-1)/p$.

\section{Main Theoretical Results}\label{paper:theory}

The two theorems below prove that coexistence of the kind seen in
Figure~\ref{paper:fig:headline} can occur under the actual CE training
dynamics of a 1HL network, started from constructed initializations. In
both, the two layers follow the updates~\eqref{paper:optimizer} with
$\lambda_V=0$. Under each theorem's conditions on the split, every
trajectory from a specified, possibly small, open set of initializations
keeps perfect training accuracy and recovers held-out accuracy during a CE
plateau; entry into this set from iid initialization is not established.
The first theorem holds at a loss near $\log p$ on a structured split, and
the second at a prescribed positive loss, which can be made arbitrarily
small, on a random split.

\subsection{A constant Fourier gain on a structured half-grid}
\label{paper:theory:half}
The first construction isolates coexistence at high training loss: the
network classifies every training pair correctly while its CE stays just
below $\log p$.

\paragraph{The split.}
The training set consists of the pairs whose difference is a nonzero
quadratic residue, that is, a value $u^2\bmod p$ with $u\ne0$:
\begin{equation}\label{eq:split}
 \Train=\{(a,b):b-a\in\mathrm{QR}(\F_p)\},
\end{equation}
where $\mathrm{QR}(\F_p)$ is the set of nonzero quadratic residues. It
contains $p(p-1)/2$ pairs, just under half the grid, and each operand and
each label occurs $(p-1)/2$ times in it. This balance supports the Fourier
analysis, but the split selects whole difference classes rather than
independent pairs, and it holds out every diagonal pair. The theorem
requires width at least $M_{\min}(p)=p(p^2+3)/2$.

\paragraph{Idea of the construction.}
We first analyze a reference run in which the head is frozen and the
hidden weights take Polar steps with decay. Its hidden units form three
groups. An \emph{arithmetic core} classifies every training pair correctly
by a small margin, and its responses and training logits stay fixed along
the run. \emph{Helper units} have zero head weights, so they affect no
prediction, and they have no component in $\cF_{\rm add}$. \emph{Query
units} are active only on held-out inputs, where their logits are
initially harmful. As the helper and query units decay, the training
logits do not move, so every training pair stays correct and the CE stays
near $\log p$. On held-out inputs, the harmful query logits decay until the
correctly signed contribution of the core dominates, and held-out accuracy
goes from zero to one. The decay of the helper units removes energy outside
$\cF_{\rm add}$, which, with bounds for the other two groups, yields a net
gain in $\AH$. These bounds then carry over to the actual run, in which
both layers train.

\noindent\begin{minipage}{\linewidth}\vspace{8pt}
\begin{mainthm}[Structured half-grid]\label{paper:thm:half}
\upshape
\setlength{\parindent}{0pt}
\setlength{\parskip}{4pt}
\textbf{Parameters.} Let $p\ge101$ be prime, and let the width satisfy
$M\ge p(p^2+3)/2$.

\textbf{Split.} The training set is
$\Train=\{(a,b):b-a$ is a nonzero quadratic residue modulo $p\}$.

There exist fixed rates and decay $(\eta_W,\lambda_W,\eta_V)$, a horizon
$J$, sets $\cI_{\rm early},\cI_{\rm late}\subseteq\{1,\ldots,J\}$ of size at
least $16$ with $\max\cI_{\rm early}<\min\cI_{\rm late}$, and a nonempty open
set of initializations in $\R^{3pM}$ such that every trajectory
of~\eqref{paper:optimizer} with $\lambda_V=0$ started in this set satisfies
all four conclusions below.
\begin{enumerate}[label=(\roman*),topsep=6pt,itemsep=6pt,parsep=0pt,leftmargin=2.4em,labelsep=0.5em]
\item \textbf{Training fit and CE plateau.}
For every $0\le t\le J$, training accuracy is one and
\begin{equation*}
 \log p-0.0201<L_t<\log p.
\end{equation*}
Consequently, $v([0,J])<0.0045$.

\item \textbf{Early held-out failure.}
For every $t\in\cI_{\rm early}$, the held-out accuracies of the network and
of the refit are zero.

\item \textbf{Late network recovery and Fourier gain.}
For every $t\in\cI_{\rm late}$, the actual trained network's held-out
accuracy is one and
\begin{equation*}
 \AH(\mat{H}_t)-\AH(\mat{H}_0)>0.0998.
\end{equation*}

\item \textbf{Late refit margin.}
If moreover $p\ge1764$, then for every $t\in\cI_{\rm late}$ the refit has
margin greater than $11/50$ against every rival class at every held-out
pair whose difference is a nonsquare, so its held-out accuracy is at least
$(p-1)/(p+1)$.
\end{enumerate}
\end{mainthm}
\end{minipage}\par
\paragraph{Interpretation.}
Recovery and a net gain of more than $9.98$ percentage points in $\AH$ thus
occur within one plateau that is checked at every update. The theorem
asserts neither monotone growth of $\AH$ nor persistence after $J$, and it
does not order the recoveries of the refit and the network; clause (iv)
says nothing about the refit on the diagonal pairs.

\paragraph{Proof outline.}
The proof is in Appendix~\ref{sec:main} and is assembled in
Appendix~\ref{sec:width}. In the reference run, every training contrast
$m_{t,c}(x)=f_{t,y^\star(x)}(x)-f_{t,c}(x)$ stays in $(0,0.0201)$, so the
identity
\begin{equation}\label{paper:ce-contrast-main}
 \CE(f_t(x),y^\star(x))
 =\log\!\left(1+\sum_{c\ne y^\star(x)}e^{-m_{t,c}(x)}\right)
\end{equation}
gives perfect training classification with CE near $\log p$. Quantitative
bounds on the three groups of units control the switch of the held-out
predictions and the net gain in $\AH$, and the refit is analyzed
separately, from the full ridge problem with the prescribed penalty.
Strict signs of the training gates, a locally constant rank of the
supported gradient and a nonzero head gradient give continuity near the
reference trajectory. Over the finite horizon, the strict bounds therefore
carry over to the actual run for a small enough positive head step and an
open set of initializations.

\subsection{Random splitting and refit-first recovery}
\label{paper:theory:random}
The second construction removes both restrictions of the first, the
engineered split and the loss near $\log p$. Each pair enters the training
set independently with probability $4/5$, a value set by our proof
estimates rather than a known threshold, and the training CE stays near a
prescribed positive level $L_*$, set by $a_{\rm conf}$ below, which can be
made arbitrarily small and is not the minimum of the CE. The refit becomes accurate early, and
the network recovers late.

\paragraph{Idea of the construction.}
The initialization pairs a fixed rectangular dictionary of hidden units,
which already spans every function on the grid, with a head calibrated on
the training data. Along a reference feature trajectory with fixed decay
and a rule-directed update, the rates balance decay against the growth of
the useful component, which keeps the leading training confidence, and
with it the CE, stable. On held-out inputs, however, the competing logit
contributions change their balance, and the network's held-out accuracy
becomes high only at a prescribed time $J$. Ridge fitting on the training
rows selects an accurate predictor from the first update on. Since the
initial features already span every function on the grid, the refit
improves because it selects a better predictor, not because the features
gain new span.

\noindent\begin{minipage}{\linewidth}\vspace{8pt}
\begin{mainthm}[Random split]\label{paper:thm:random}
\upshape
\setlength{\parindent}{0pt}
\setlength{\parskip}{4pt}
\textbf{Parameters.} Fix an integer $J\ge30$, constants $\alpha,\theta,a_{\rm conf}>0$,
$\beta\in[1/2,1)$, $\kappa\in(0,1/8]$, and tolerances
$\epsilon_{\rm lag},\epsilon_{\rm test},\epsilon_{\rm obs}\in(0,1)$.
Set $\rho=4/5$ and $L_*=\log(1+\rho^2e^{-a_{\rm conf}})$.

\textbf{Sampling and initialization.} There is a threshold $p_0$ such that
the following holds for every odd prime $p\ge p_0$.
Use width $M=p^2(p-1)^2$ and the prescribed rates and decay
$(\eta_W,\lambda_W,\eta_V)$, chosen before sampling.
Place each pair in $\Train$ independently with probability $\rho$.

With probability at least $1-2/\sqrt p$ over the split, the split is proper
and there is a nonempty open ball in $\R^{3pM}$ around a prescribed
initialization built from the training data such that every trajectory
of~\eqref{paper:optimizer} with $\lambda_V=0$ started in this ball satisfies
all four conclusions below.
\begin{enumerate}[label=(\roman*),topsep=6pt,itemsep=6pt,parsep=0pt,leftmargin=2.4em,labelsep=0.5em]
\item \textbf{Training fit and CE plateau.}
For every $0\le t\le J$, training accuracy is one and
\begin{equation*}
 \frac{20}{21}L_*<L_t<\frac{22}{21}L_*.
\end{equation*}
Consequently, $v([0,J])<0.1$.

\item \textbf{Delayed network recovery.}
The network's held-out accuracy is $0$ at $t=0$, below
$\epsilon_{\rm lag}$ for $0\le t<J$, and above
$1-\epsilon_{\rm test}$ at $t=J$.

\item \textbf{Refit accuracy and error.}
The refit's held-out accuracy is $0$ at $t=0$ and at least
$1-\epsilon_{\rm obs}$ for $1\le t\le J$. Its normalized error satisfies
\begin{equation*}
 E_{\rm ref}(0)-E_{\rm ref}(t)\ge\frac{401}{1568}>\frac14,
 \qquad t\in\{J-1,J\}.
\end{equation*}

\item \textbf{Stepwise Fourier improvement.}
For every $0\le t<J$,
\begin{equation*}
 \AH(\mat{H}_{t+1})-\AH(\mat{H}_t)>\frac{1}{24J(3p+1)}.
\end{equation*}
\end{enumerate}
\end{mainthm}
\end{minipage}\par
\paragraph{Parameter choices.}
Here $\alpha$ scales the initial dictionary, $\beta$ its offsets and
$\theta$ the head step. For each sufficiently large prime, the width and
the rates~\eqref{eq:r62optimizer} are chosen before sampling, and the
initial ball depends on the realized training data. The probability is
over the split only and says nothing about entry from iid weights.
Increasing $a_{\rm conf}$ makes $L_*$ arbitrarily small but may change the
uncalibrated threshold $p_0$, and changing $J$ changes the initialization
and the rates rather than extending one fixed run.

\paragraph{Interpretation.}
Both predictors start with zero held-out accuracy. If
$\epsilon_{\rm lag}+\epsilon_{\rm obs}<1$, clauses (ii) and (iii) give
\begin{equation}\label{paper:refit-chronology-main}
 \Acc_{\Test}(t)<\epsilon_{\rm lag}
 <1-\epsilon_{\rm obs}\le\Acc_{{\rm ref},\Test}(t),
 \qquad 1\le t<J,
\end{equation}
where $\Acc_{{\rm ref},\Test}$ is the refit's held-out accuracy: the refit
is accurate from the first update, and the network recovers only at $J$,
within the same CE plateau. The error gain in (iii) compares the times
$J-1$ and $J$ with the initial error of the same trajectory; it is neither
a per-step decrease of the error nor a $25$-percentage-point accuracy gain.
Clause (iv) gives a Fourier gain at every update, although $\AH$ remains
$O(1/p)$ in this construction.

\paragraph{Proof outline.}
Appendix~\ref{sec:r62} constructs the reference feature trajectory. Its
rates keep the leading training confidence stable, signed-margin estimates
on held-out inputs place the network's recovery at the prescribed time
$J$, and an analysis of the complete feature kernel with the prescribed
ridge penalty gives accurate refit predictions from time $1$. Uniform
error bounds transfer these estimates to the actual updates for every
initialization in the prescribed ball (Appendix~\ref{sec:bwc-radii}).
Both layers train, but the small head step makes the recovery mainly
feature-driven.

\begin{figure*}[t]
\centering
\includegraphics[scale=\mainfigscale]{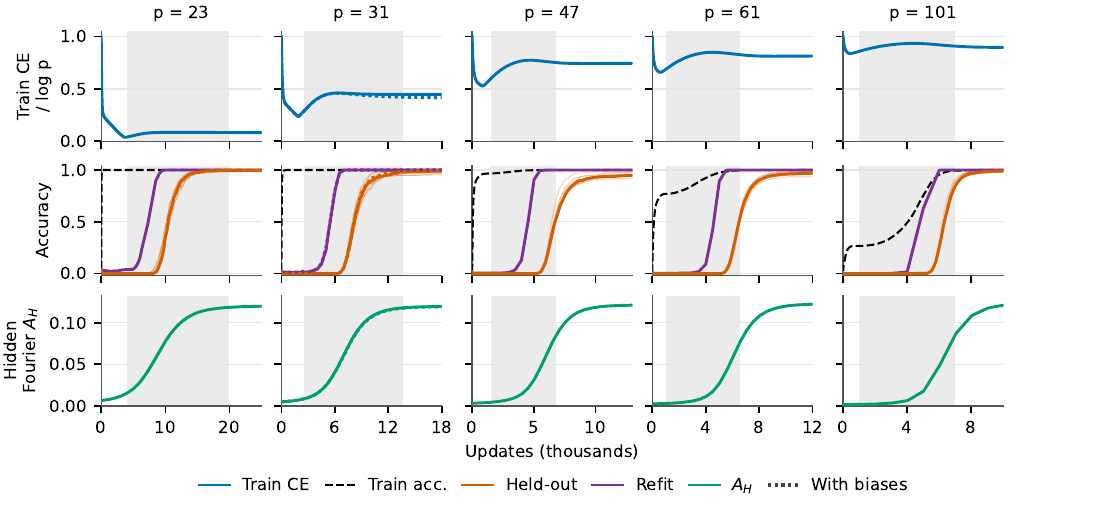}
\caption{\proofrevtext{\textbf{One hidden layer across primes} (width 512, 10 runs per prime;
$\lambda_W=0.01,0.015,0.02,0.02,0.02$ at $p=23,31,47,61,101$, chosen on
development seeds, for $p=101$ at $p=97$). Rows: cycle-mean training CE
divided by $\log p$; median training accuracy (dashed), held-out accuracy
of the network (runs and median) and of the train-only refit (median);
$\AH$. Gray: from the median start to the median end of the runs'
post-hoc windows (Appendix~\ref{paper:app:windows}). The primary endpoint
of Section~\ref{paper:deep} is met by 9, 4, 0, 0 and 0 runs; at $p\ge47$
some training pairs stay misclassified inside the window. Runs stopped
early by the stopping rule keep their last value. Dotted ($p=31$): medians
with a hidden bias (Appendix~\ref{paper:app:bias}).}}
\label{fig:onehl-primes}
\end{figure*}

\begin{figure*}[t]
\centering
\includegraphics[scale=\mainfigscale]{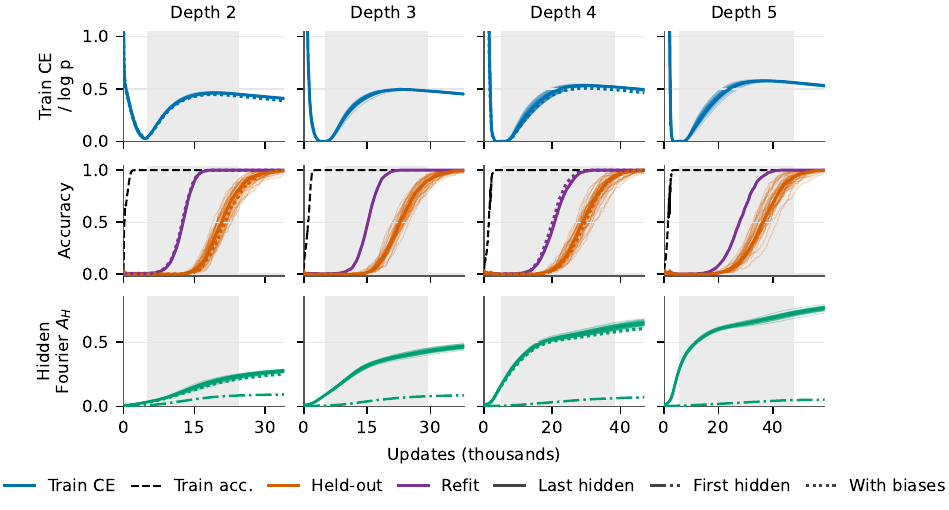}
\caption{\proofrevtext{\textbf{Two to five hidden layers at $p=31$: every run (thin) and the median (thick).}
Width 256. Depth 2 uses a fixed rate set on development runs; depths 3--5
tune the last-layer rate on a preliminary run
(Appendix~\ref{paper:app:deep-settings})\added{; every depth shows the same
40 evaluation seeds}. The primary endpoint is met by \added{34}, 40, 39 and
38 runs; \added{the six depth-2 misses fail the share threshold, and one of
them also ends with held-out accuracy below $0.99$}. Rows as in
Figure~\ref{fig:onehl-primes}, with the last-hidden-layer $\AH$ (solid) and
the median first-hidden-layer $\AH$ (dash-dot). Gray: from the median start
to the median end of the runs' post-hoc windows
(Appendix~\ref{paper:app:windows}), which contain the rise in loss after
the knee. One hidden layer at $p=31$ is the second column of
Figure~\ref{fig:onehl-primes}. Dotted (depths 2, 4): medians with hidden
biases.}}
\label{fig:deep-main}
\end{figure*}

\begin{figure*}[t]
\centering
\includegraphics[scale=\mainfigscale]{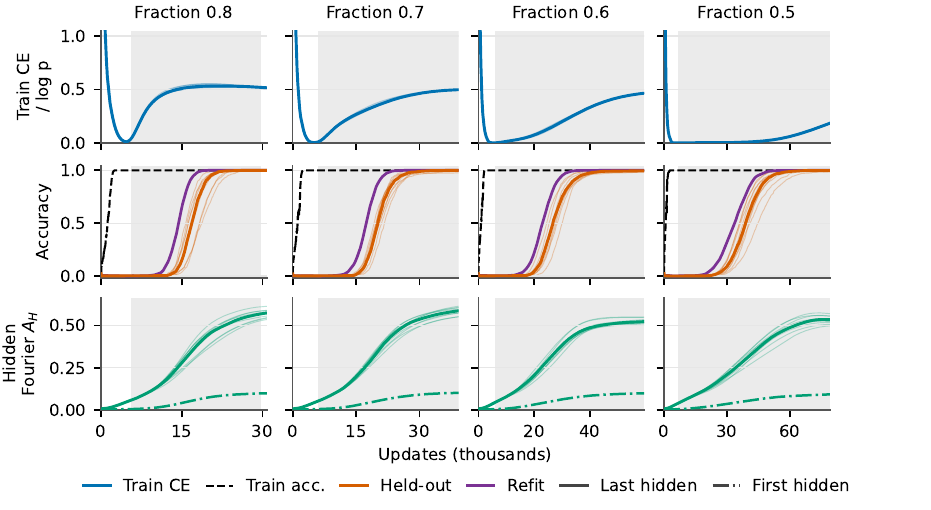}
\caption{\added{\textbf{Depth 3 at $p=47$ across training fractions: every run
(thin) and the median (thick).} Width 256, 10 fresh runs per fraction,
last-layer multiplier $r=0.45$, $0.30$, $0.18$ and $0.08$ chosen per
fraction on development runs, head decay $\lambda_V=0.2\cdot31/p$
(Appendix~\ref{paper:app:deep-fractions}). Rows and gray bands as in
Figure~\ref{fig:deep-main}. The criterion of
Appendix~\ref{paper:app:deep-fractions}, fixed before the runs, is met by
10, 10, 10 and 8 runs; the two misses end at held-out accuracy $0.959$ and
$0.971$ when they reach the update budget. No training pair is
misclassified at any log inside any window. Each column ends no later than
its shortest run.}}
\label{fig:deep-fractions}
\end{figure*}

\begin{figure*}[t]
\centering
\includegraphics[scale=\mainfigscale]{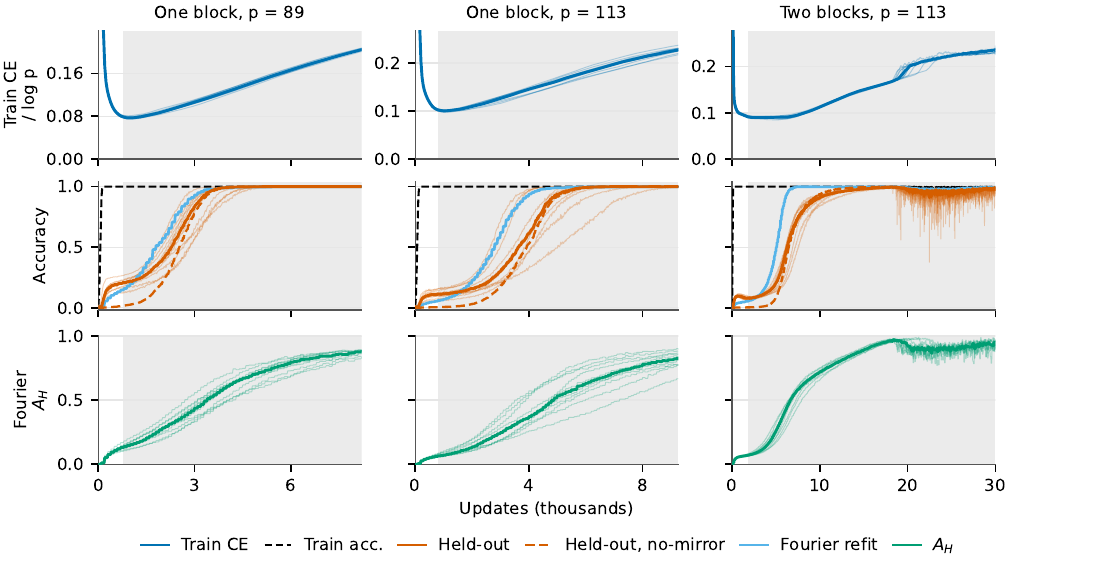}
\caption{\added{\textbf{Transformers with one and two blocks: every run (thin)
and the median (thick).} One block at $p=89$ (training fraction $0.25$,
$\lambda_V=0.33$) and at $p=113$ ($0.2$, $\lambda_V=0.275$), with the
output gain decaying as $g_t\approx1+e^{-t/5000}$, and two blocks at
$p=113$ ($0.2$, $\lambda_V=0.275$, $g_t\approx1+e^{-t/10000}$); 10 runs
per column (seeds $i=11$--$20$). Rows as in Figure~\ref{fig:onehl-primes},
with the training CE averaged over 201 updates (each column on its own
scale) and $\AH$ of the normalized head input $\hat r$. The accuracy row
adds the two medians used by the pass rule of
Section~\ref{paper:transformer}: held-out accuracy on the no-mirror pairs
(dashed) and the Fourier-filtered refit, which replaces the refit of
Figure~\ref{fig:onehl-primes}. The refit, not drawn, trails the network at
its $50\%$ crossing in all 30 runs, by a median of 0.7k, 1.2k and 0.5k
updates in the three columns. Gray: from the median start to the median
end of the post-hoc windows on all held-out pairs
(Appendix~\ref{paper:app:tf-window}). Each column ends at its shortest
run. One block passes in 10 of 10 runs at both primes; its runs stop at
8.2k--10.2k updates ($p=89$) and 9.3k--14k ($p=113$). Two blocks pass in
none. Through the held-out rise they behave like one block: the CE rises
and every training pair is correct at every 20-update log. The first
training errors are logged at 17.8k--22.0k updates, 1.4k to 5.7k updates
after held-out accuracy first reaches $0.99$; held-out accuracy then falls
back and jitters, and 6 of the 10 runs end below $0.99$
(Appendix~\ref{paper:app:tf-gain}, Figure~\ref{fig:tf-blocks}).}}
\label{fig:tf-contrast}
\end{figure*}

\section{Experiments}\label{paper:experiments}
In this section, we discuss our experimental verification of whether
qualitative counterparts of the coexistence proved in
Section~\ref{paper:theory} appear outside the assumptions of its theorems.
The theorems start from constructed initializations, at widths that grow
with $p$; the experiments start from iid initialization at practical
widths, in 1HL networks (Section~\ref{paper:onehl}), in deeper ReLU
networks (Section~\ref{paper:deep})\added{ and in transformers with one and
two blocks (Section~\ref{paper:transformer})}.
Appendix~\ref{paper:app:thm41-iid} also runs the split and the updates of
Theorem~\ref{paper:thm:half} from iid initialization at width 512.
Appendices~\ref{paper:app:setup-ext}--\ref{paper:app:deep-ext} give the
window definitions, protocols and further results, including stricter
evaluations that require the first training fit to persist
(Appendices~\ref{paper:app:strict} and~\ref{paper:app:deep-protocol})\added{;
Appendix~\ref{paper:app:tf} details the transformer}.

The \added{runs of Figures~\ref{paper:fig:headline}--\ref{fig:deep-main}}
use head decay $\lambda_V=0.2$\added{; the deeper networks across training
fractions use $\lambda_V=0.2\cdot31/p$, and the transformers a value set
per prime}. The theory allows head decay only under the construction-specific
sufficient conditions of Remark~\ref{rem:readoutdecay}, which in both
constructions keep the cumulative decay $J\eta_V\lambda_V$ below one; the
cumulative decay on the window of Figure~\ref{paper:fig:headline},
$7500\cdot0.03\cdot0.2=45$, lies outside them.

\paragraph{Windows fixed in advance and post-hoc windows.}
Since a statement about a window is only as reliable as the rule that
chooses it, the primary counts use a window, or a band-selection rule,
fixed before evaluation: the locked interval of
Figure~\ref{paper:fig:headline} for one hidden layer, and for deeper
networks a band that a loss-only rule selects in each run
(Section~\ref{paper:deep}). Figures~\ref{fig:onehl-primes}
and~\ref{fig:deep-main} instead shade looser post-hoc non-falling windows.
Such a window lies after the head norm has nearly reached the ceiling
$1/\lambda_V$ set by its decay, and it is the stretch over which the
cycle-mean loss rises the most while never falling more than a factor
$1.03$ below its running maximum (Appendix~\ref{paper:app:windows}). These
windows make no flatness claim, enter only the secondary analysis
(Appendix~\ref{paper:app:posthoc}) and do not determine the primary
counts.\added{ Figure~\ref{fig:deep-fractions} shades the same windows,
which there also enter a criterion fixed before its runs
(Appendix~\ref{paper:app:deep-fractions}); the transformer has no locked
window and is scored on a post-hoc window of its own
(Appendix~\ref{paper:app:tf-window}).} Every
plotted cohort shows all its runs\added{, except in the two appendix
figures on the transformer's mirror pairs and readouts, which show
medians}.
\subsection{One-hidden-layer ReLU networks}\label{paper:onehl}
\paragraph{Protocol.}
Figure~\ref{paper:fig:headline} shows 20 runs at $p=31$ and width $M=512$,
each with a fresh initialization and a fresh uniform random $80\%$ split
(768 training and 193 held-out pairs). The weights are iid Gaussian,
$W_{ij}=cZ_{ij}/\sqrt p$ and $V_{kj}=\widetilde Z_{kj}/(c\sqrt M)$, with
$c=24$ and independent standard Gaussian arrays
$\mat{Z},\widetilde{\mat{Z}}$. Training applies the
updates~\eqref{paper:optimizer}, full batch in float64, for 40{,}000
updates with fixed $\eta_W=\eta_V=0.03$, $\lambda_W=0.015$ and
$\lambda_V=0.2$. A second development search held $c$, both rates and the head
decay fixed and swept the prime, the training fraction and the feature
decay. A predeclared rule that uses held-out accuracy selected this
configuration and the window $I=[4500,12000]$, and both were locked before
the 20 initializations and splits were drawn
(Appendix~\ref{paper:app:experiments-ext}), so the evaluated runs played no
part in choosing them.

\paragraph{Results.}
All 20 runs classify every training pair correctly at every update of
$I$. The training loss first dips to about $0.8$ and rises to about $1.5$;
on $I$, it stays within the declared $10\%$ plateau tolerance, with minimum
$1.456$--$1.498$ and $v(I)=6.33$--$7.81\%$. Over $I$, $\AH$ rises from
$0.018$--$0.020$ to $0.115$--$0.116$, and held-out accuracy climbs from
zero to $91.7$--\added{$98.4$}$\%$. The features support an accurate train-only
readout before the network's own head is accurate: the refit first records
$90\%$ at update 6{,}500 in every run, whereas the network first crosses
$90\%$ at updates 9{,}631--11{,}665. The window is not anchored at the
first fit of the training set. In every run, a few training pairs are
misclassified again for a while after the first fit, long before the
window, so the stricter endpoint that requires the fit to persist from its
first attainment fails (Appendix~\ref{paper:app:strict}). Pre-registered
tests of recovery without a decrease of the CE, at $p=23$, are in
Appendix~\ref{paper:app:nodecrease}.

\paragraph{Other primes.}
Figure~\ref{fig:onehl-primes} shows separate 1HL cohorts of 10 runs at
$p=23$, 31, 47, 61 and 101. The feature decay was chosen on one or two
development seeds per prime (for $p=101$, at $p=97$), so these cohorts are
not a pre-specified evaluation. The primary endpoint of
Section~\ref{paper:deep} is met by 9, 4, 0, 0 and 0 runs; at $p\ge47$, some
training pairs stay misclassified inside the window.

\subsection{Extension to deeper ReLU networks}\label{paper:deep}
\paragraph{Protocol.}
We train networks with $D=2$ to $5$ hidden layers of width 256 at $p=31$,
from iid Gaussian initialization (Appendix~\ref{paper:app:deep-ext}), with
supported-Polar hidden updates and fixed decays $0.2$ for the head,
$0.015$ for the first hidden layer and $0.07$ for the later layers. For
$D=3,4,5$, only the last-layer rate is tuned, by a pre-specified rule
applied to a preliminary run, after which training restarts from the same
initialization and split with that rate fixed
(Appendix~\ref{paper:app:deep-settings}). Depth 2 uses a fixed rate chosen
on development runs. The evaluation seeds are disjoint from development and
shared across depths\added{, depth 2 included}, so the cohorts are
dependent.

\paragraph{The primary endpoint.}
Deeper networks have no locked window. Instead, a rule that uses no
held-out outcome selects a band $I_C$ around the CE peak: after the knee,
the first checkpoint at which the head norm reaches $99\%$ of its ceiling
$1/\lambda_V$, it takes the largest cycle-mean loss and the maximal
interval around it on which the cycle-mean loss stays at least that peak
divided by $1.03$ (Appendix~\ref{paper:app:deep-endpoint}). The
pre-specified primary endpoint is a joint criterion with four parts: $I_C$
contains at least
$75\%$ of the rise of held-out accuracy from $10\%$ to $90\%$; the final
held-out accuracy is at least $0.99$; the Fourier fraction of the last
hidden layer gains at least $0.02$ across $I_C$; and every training margin
is
positive at every 20-update log in $I_C$. A depth succeeds if the lower end
of the two-sided $95\%$ Wilson interval for its pass rate is at least
$0.80$, that is, if at least 37 of its 40 runs pass.

\paragraph{Results at $p=31$.}
The primary endpoint is met by 40/40, 39/40 and 38/40 runs at $D=3,4,5$,
so these depths succeed, and by \added{34/40} at depth 2\added{, whose
Wilson lower end, $0.71$, is below $0.80$} (Figure~\ref{fig:deep-main}). The selected bands of these cohorts
have cycle-mean loss variation at most $3\%$
(Appendix~\ref{paper:app:deep-endpoint}). All 120 runs at depths 3--5
finish with perfect held-out accuracy, and 118 have raw-loss variation at
most $10\%$ across the first logged $10\%$-to-$90\%$ transition.

The recipe does not transfer unchanged: with the same recipe, changing the
prime to 23, the width to 512 or the feature update to a normalized
gradient gives no joint successes. The normalized-gradient runs still
generalize, but their loss keeps falling during the recovery
(Appendix~\ref{paper:app:deep-protocol}).
Appendix~\ref{paper:app:deep-medians} reports further fixed settings,
selected on development runs, with depths 3--4 at $p=23,47,61,97,101$ and
depth 2 at $p=23,47$, and separates primary from secondary post-hoc
outcomes.

\paragraph{\added{Other primes and training fractions.}}
\added{A re-tuned recipe carries the post-hoc behavior to other primes and
smaller training sets. With the last-layer multiplier $r$ set per training
fraction on development runs and the head decay scaled with the prime,
$\lambda_V=0.2\cdot31/p$, depths 3 and 4 at $p=47$ and $61$, with training
fractions $0.5$ to $0.8$, meet a criterion on the post-hoc non-falling
window, fixed before the runs, in 155 of 160 fresh runs
(Appendix~\ref{paper:app:deep-fractions}). A run passes if the window
holds at least $75\%$ of both the network's and the refit's
$10\%$-to-$90\%$ held-out rises, the refit gains at least $0.2$ and the
last-layer $\AH$ at least $0.02$ inside it, and held-out accuracy gains at
least $0.2$ from the start of the window to the end of the run and ends at
$0.99$ or above; the training fit is not required. For depth 3 at $p=47$
(Figure~\ref{fig:deep-fractions}), no training pair is misclassified at
any log inside any window. The windows contain the CE's climb after the
knee, which at fraction $0.5$ starts from near zero. From the refit's
$20\%$ crossing to the network's $90\%$ crossing, an interval that at
fraction $0.8$ contains the network's whole $10\%$-to-$90\%$ held-out rise,
the CE changes by $-2.9\%$ to $+9.5\%$ at that fraction (setting medians);
over the same interval, it rises at the smaller fractions. The primary
endpoint holds in only 9 of the 160 runs, so these runs extend the post-hoc
picture rather than the primary endpoint; $p=23$ was not rerun in this
way.}

\subsection{Transformers}\label{paper:transformer}
\begin{additions}
To test the behavior beyond fully connected networks, we train a one-block
transformer at $p=113$ in the layout of Nanda et al.~\cite{nanda2023}, with
no LayerNorm and, here, no biases (Appendix~\ref{paper:app:tf-setup}), on a
random $20\%$ of the grid (2{,}553 pairs); it reads $(a,b,{=})$ and predicts
at the last token. The logits are $f=g\,\mat{V}\hat r$, where
$\hat r=r/\mathrm{rms}(r)$ is the final residual vector $r$ at that token
divided by its root mean square~\cite{zhang2019rmsnorm}, the unembedding
$\mat{V}\in\R^{p\times128}$ is the head and $g$ is a trained scalar gain.
The other weight matrices take the supported-Polar step
of~\eqref{paper:optimizer} with $\eta_W=0.1$ and $\lambda_W=0.01$, and
$\mat{V}$ takes the normalized step with $\eta_V=0.03$ and
$\lambda_V=0.275$, full batch in float64.

\paragraph{The output gain.}
Without the gain, the logit scale is set by $\mat{V}$ alone, whose norm
settles within a few hundred updates, and in all 10 runs the training CE
falls by $7$--$10\%$ during the held-out rise. The gain starts at $g_0=2$
and takes a sign step with decay,
$g\leftarrow g-\eta_g\operatorname{sign}(\partial L/\partial g)-\eta_g\lambda_g g$,
with $\eta_g=2\cdot10^{-4}$ and $\lambda_g=1$. Once the training set is
fitted, $\partial L/\partial g<0$, so $g$ relaxes toward 1 with time
constant $1/(\eta_g\lambda_g)=5{,}000$ updates; the logit scale then
shrinks slowly, and the training CE can rise while every training pair
stays correct.

\paragraph{Pass rule and results.}
A held-out pair whose mirror $(b,a)$ is a training pair can be answered from
that pair, so we also score the \emph{no-mirror} pairs, whose mirrors are
held out too. Each run is scored on a post-hoc window that starts at or after the
first full fit, on which the averaged training CE never falls more than
$3\%$ below its running maximum, with the start chosen to hold the largest
share of the held-out rise (Appendix~\ref{paper:app:tf-window}). A run
passes if, on both sets of held-out pairs, the window holds at least $75\%$
of the $10\%$-to-$90\%$ held-out rise and, inside it, a train-only
Fourier-filtered refit (below) gains at least $0.2$, $\AH$ of $\hat r$ at
least $0.02$, and no training pair is misclassified at any 20-update log,
and if the final accuracy on all held-out pairs is at least $0.99$. All 10
runs pass (Figure~\ref{fig:tf-contrast}, middle). After its minimum of
about $0.1\log p$, the averaged training CE rises by a median $33\%$
($18$--$55\%$) across the no-mirror held-out rise; inside the window,
$\AH$ of $\hat r$ gains $0.78$--$0.83$, and the final held-out accuracy is
at least $0.999$. With the gain held fixed at $1.5$, $2$ or $3$, the CE
falls during the rise and no run passes (0/9), so the decay of the gain is
what lifts the CE.

\paragraph{Readouts.}
Fitted to $\hat r$, the train-only ridge refit~\eqref{paper:expridge},
which leads the network in Section~\ref{paper:onehl}, trails the
transformer in every run, by a median of 1.2k updates at $50\%$ accuracy on
all held-out pairs: it weights all 56 frequencies of the label equally,
while $\hat r$ carries only 7 to 13 of them at the network's $50\%$
crossing. Restricted to the 8 frequencies with the most training energy, a
choice that uses training data only, the same readout leads the network in
all 10 runs, by a median of 0.7k updates on all held-out pairs; the
network's own logits, restricted to the same 8 frequencies, lead it by
about as much, so the restriction removes memorized terms rather than
adding signal (Appendix~\ref{paper:app:tf-readouts}).

\paragraph{Other primes and two blocks.}
With the head decay and the training fraction set per prime, the fraction
so that the held-out rise starts after the CE minimum, the recipe passes in
10/10 runs at $p=89$ (Figure~\ref{fig:tf-contrast}, left), 7/10 at $p=149$
(10/10 if up to 10 training pairs may be misclassified) and 5/10 at $p=61$
(10/10 on the no-mirror pairs, and 9/10 with a split that keeps every
training pair with its mirror). With two blocks at the main runs' fraction
(Figure~\ref{fig:tf-contrast}, right), the CE still rises during the
held-out rise and no training pair is misclassified at any 20-update log,
but the first training errors are logged 1.4k to 5.7k updates after
held-out accuracy first reaches $0.99$, and held-out accuracy then falls
back. No run at the four planned two-block settings passes (0/10 at that
fraction, Figure~\ref{fig:tf-blocks}, and 0/15 at the other three);
lowering the feature rate after the held-out rise kept the fit in three
exploratory runs (Table~\ref{paper:tab:tf-blocks-late}). These are
development runs: the gain and the rates were chosen on pilot runs, and the
pass thresholds were fixed before the runs with the gain
(Appendix~\ref{paper:app:tf}).
\end{additions}

\section{Discussion}\label{paper:discussion}
Relatively flat training cross-entropy can coexist with evolving features
and delayed generalization. Our theorems establish this for 1HL networks
from constructed initial regions, and the experiments show related
behavior from iid initialization and in deeper networks\added{ and a
one-block transformer}, outside these guarantees. The loss itself explains
why this is possible: once every training pair is correct, the CE measures
confidence rather than correctness. A flat training loss is therefore not
by itself evidence that learning has stopped, and diagnostics of the
representation, such as the Fourier fraction or a train-only refit, can
reveal progress that the loss does not show.

The two diagnostics answer different questions. The Fourier fraction
summarizes the structure of the representation, while signed margins
underlie our accuracy guarantees. Since the fraction ignores decoder signs
and class coverage, we pair it with the train-only refit, which tests
predictive usefulness directly. Refit-first recovery shows earlier
decodability under the specified ridge rule; it does not show that the
actual head follows the refit, and it does not separate the direction of
the features from their scale. \added{In the transformer runs of
Appendix~\ref{paper:app:tf-readouts}, the same ridge rule trails the
network in every run, while its restriction to the 8 Fourier frequencies
with the most training energy leads at the $50\%$ crossing in nearly all
of them.
The rule weights all frequencies equally, which suits the 1HL features,
where every frequency carries held-out signal, but not the transformer's,
where about 8 do.}

The theory has clear limits. Theorem~\ref{paper:thm:half} gives its
initialization basin, positive head step and small decay allowance only
qualitatively, and its width extension uses inactive padding. Neither
theorem shows how training from iid initialization reaches the constructed
regions. Quantifying these possibly small neighborhoods and rates,
calibrating the prime threshold of Theorem~\ref{paper:thm:random} and
allowing training-active memory copies are natural next steps that require
further trajectory estimates.

The deeper-network experiments show qualitative counterparts of our 1HL
results, but their configurations rely on hyperparameter search and tuning
(Appendix~\ref{paper:app:deep-settings})\added{; at $p=47$ and $61$, a
last-layer rate set per training fraction, with head decay scaled with the
prime, suffices for the post-hoc criterion of
Appendix~\ref{paper:app:deep-fractions} but not for the primary
endpoint}. Understanding how these choices support relatively flat
training loss, feature evolution and delayed generalization is a step
toward guarantees for deeper networks from iid initialization at practical
widths and learning rates.

\added{In the one-block transformer, the behavior depends on a slowly
decaying output gain, without which the training CE falls during the
held-out rise. A second block, at the main runs' training fraction, keeps
the rising CE during the held-out rise but loses the training fit later,
except in a few exploratory runs in which the feature rate is lowered after
the rise (Appendix~\ref{paper:app:tf-gain}). We have no theory for these
cases.}

\typeout{MAIN-TEXT-END-PAGE: \thepage}

\section*{AI use statement}
We used large language model assistants to support the development of the initial ideas of the work, in formulating mathematical claims and hypotheses, and aspects of proof development and
writing. These tools also assisted with experimental design, implementation, and interpretation of results. Additionally, we used generative AI tools to check literature and references, and to draft and edit parts of
the paper. We have reviewed all AI-assisted work. We take
responsibility for the final content of this work, including text, claims or artifacts produced with the
aid of generative AI.

\clearpage
\appendix

\begin{proofrevision}
\section*{Appendices}

\begingroup
\setlength{\parskip}{0pt}
\makeatletter
\newcommand{\appentry}[3]{%
  \csname l@#1\endcsname
    {\hyperref[#2]{\numberline{\ref*{#2}}#3}}{\pageref*{#2}}}
\appentry{section}{paper:app:setup-ext}{Representation diagnostics, windows and endpoints}
\appentry{subsection}{paper:fourier-defs}{Fourier features, projection, and energy}
\appentry{subsection}{paper:prediction-defs}{Actual prediction, train-only refitting, and signed rule margins}
\appentry{subsection}{paper:app:windows}{Windows and endpoints of the experiments}
\appentry{section}{paper:app:experiments-ext}{One-hidden-layer experiments}
\appentry{subsection}{paper:app:onehl-protocol}{Protocol and selection}
\appentry{subsection}{paper:app:onehl-headline}{The headline cohort}
\appentry{subsection}{paper:app:nodecrease}{Pre-registered tests of recovery without CE decrease}
\appentry{subsection}{paper:app:studies}{Results across primes}
\appentry{subsection}{paper:app:onehl-other}{Replays and other one-hidden-layer studies}
\appentry{subsection}{paper:app:thm41-iid}{Theorem-inspired split and update settings from iid initialization}
\appentry{section}{paper:app:deep-ext}{Deeper-network experiments}
\appentry{subsection}{paper:app:deep-settings}{Network settings and learning-rate selection}
\appentry{subsection}{paper:app:deep-protocol}{Depths three to five}
\appentry{subsection}{paper:app:deep-medians}{Results across primes}
\added{\appentry{subsection}{paper:app:deep-fractions}{Depths three and four across training fractions}}
\appentry{subsection}{paper:app:posthoc}{Post-hoc window counts}
\appentry{subsection}{paper:app:bias}{Biases in the hidden layers}
\added{\appentry{section}{paper:app:tf}{A one-block transformer}}
\added{\appentry{subsection}{paper:app:tf-setup}{Model, updates and readouts}}
\added{\appentry{subsection}{paper:app:tf-window}{Window and pass criterion}}
\added{\appentry{subsection}{paper:app:tf-gain}{The gain and the recipe at $p=113$}}
\added{\appentry{subsection}{paper:app:tf-primes}{Other primes and mirror pairs}}
\added{\appentry{subsection}{paper:app:tf-readouts}{Readouts}}
\appentry{section}{app:notation}{Cross-appendix notation}
\appentry{section}{sec:main}{Structured half-grid theorem: construction and estimates}
\appentry{subsection}{sec:setup}{Statement and conventions}
\appentry{subsection}{sec:a-claim-dynamics}{Claim I: construction and uniformly controlled reference dynamics}
\appentry{subsection}{sec:a-claim-representation}{Claim II: useful-representation consequences on the same trajectory}
\appentry{subsection}{sec:width}{Assembly: positive head step and an initial ball at every width}
\appentry{section}{sec:r62}{A random-split low-CE theorem with refit-first recovery}
\appentry{subsection}{sec:b-prelude}{Statement and construction}
\appentry{subsection}{sec:b-claim-dynamics}{Claim I: actual dynamics, plateau and delayed network recovery}
\appentry{subsection}{sec:b-claim-representation}{Claim II: refit-first recovery and Fourier growth}
\appentry{subsection}{sec:bwc-radii}{Assembly: the prescribed ball and one common threshold}
\appentry{section}{app:decay}{Small decoupled readout decay}
\appentry{subsection}{sec:sc26-c-half}{The structured half-grid construction}
\appentry{subsection}{sec:sc26-c-random}{The random-split construction}
\makeatother
\endgroup
\end{proofrevision}
\clearpage

\section{Representation diagnostics, windows and endpoints}\label{paper:app:setup-ext}
\subsection{Fourier features, projection, and energy}\label{paper:fourier-defs}
\paragraph{The addition-Fourier features.}
\begin{proofrevision}
For the Euclidean inner product on full-grid functions, the orthonormal
complex modes and the transform of hidden column $m$ are
\begin{equation}\label{paper:fourier-main}
\begin{aligned}
 \psi_{r,s}(a,b)&=\frac1p
       \exp\!\left(\frac{2\pi\mathrm{i}(ra+sb)}p\right),
       \qquad r,s\in\F_p,\\
 \widehat H_m(r,s)&=\frac1p\sum_{a,b\in\F_p}
 H_{(a,b),m}\exp\!\left(-\frac{2\pi\mathrm{i}(ra+sb)}p\right).
\end{aligned}
\end{equation}
The addition modes are $\phi_k=\psi_{k,k}$ for $k\ne0$. Each centered
class indicator has the expansion
\begin{equation}\label{paper:target-fourier-main}
 \one_{\{a+b=c\}}-\frac1p
 =\frac1p\sum_{k=1}^{p-1}
       \exp\!\left(\frac{2\pi\mathrm{i}k(a+b-c)}p\right).
\end{equation}
Thus these targets span $\cF_{\rm add}$ of~\eqref{paper:rulespace-main}.
For real features, let $\mat{Q}_{\rm add}\in\R^{p^2\times(p-1)}$ have columns
\begin{equation}\label{paper:fourier-real-main}
 \frac{\sqrt2}{p}\cos\!\left(\frac{2\pi k(a+b)}p\right),\quad
 \frac{\sqrt2}{p}\sin\!\left(\frac{2\pi k(a+b)}p\right)
\end{equation}
for $1\le k\le(p-1)/2$. Fourier orthogonality gives
$\mat{Q}_{\rm add}^\top\mat{Q}_{\rm add}=\Id$, $\dim\cF_{\rm add}=p-1$,
$\Padd=\mat{Q}_{\rm add}\mat{Q}_{\rm add}^\top$ and $\Padd\Cgrid=\Padd$.
Summing the nonzero addition modes gives the equivalent projector formula
\end{proofrevision}
\begin{equation}\label{paper:projector}
 (\Padd\mat{H})_{(a,b),:}
 =\frac1p\sum_{u\in\F_p}\mat{H}_{(u,a+b-u),:}
  -\frac1{p^2}\sum_{u,v\in\F_p}\mat{H}_{(u,v),:}.
\end{equation}

\paragraph{Absolute energies and their fraction.}
We normalize energies per input pair, summing over all physical hidden
columns. Parseval's identity gives
\begin{proofrevision}
\begin{equation}\label{paper:energies}
 \begin{aligned}
 E_{\rm add}(\mat{H})&=\frac{\norm{\Padd\mat{H}}_F^2}{p^2}
   =\frac1{p^2}\sum_m\sum_{k\ne0}|\widehat H_m(k,k)|^2,\\
 E_{\rm ctr}(\mat{H})&=\frac{\norm{\Cgrid\mat{H}}_F^2}{p^2}
   =\frac1{p^2}\sum_m\sum_{(r,s)\ne(0,0)}|\widehat H_m(r,s)|^2,\\
 E_{\rm other}(\mat{H})&=E_{\rm ctr}(\mat{H})-E_{\rm add}(\mat{H})
   =\frac{\norm{(\Cgrid-\Padd)\mat{H}}_F^2}{p^2}.
 \end{aligned}
\end{equation}
\end{proofrevision}
\proofrevtext{The zero-frequency energy is
$E_{\rm DC}(\mat{H})=\norm{p^{-2}\sum_{a,b}\mat{H}_{(a,b),:}}_2^2$.
The uncentered energy is $p^{-2}\norm{\mat{H}}_F^2=E_{\rm ctr}+E_{\rm DC}$.}
These per-input units match the empirical observers; unnormalized
Frobenius energies in the proof appendices differ by the common factor
$p^2$. \proofrevtext{The fraction and window change are those of
\eqref{paper:AH}, with $t_0=\min I$, $t_1=\max I$ and $E_{\rm ctr}>0$.}
\proofrevtext{These identities distinguish absolute addition-energy growth
from a fraction gain caused by decreasing other energy. The same
calculation gives the difference-Fourier control by projecting onto
$\{g(a-b):\sum_y g(y)=0\}$, replacing $(k,k)$ with $(k,-k)$ and retaining
the centered-energy denominator. It applies separately at every hidden layer.}

\proofrevtext{For the orthogonal projector $\mat{P}_H$ onto
$\operatorname{col}(\Cgrid\mat{H})$, hidden-span availability is
$A_+=\norm{\mat{P}_H\mat{Q}_{\rm add}}_F^2/(p-1)$. It measures how much of the rule
space the centered hidden span contains, without weighting feature
amplitudes; it can already be one while $\AH$ and readout performance change.}
It is an oracle full-grid diagnostic, not the train-only refit below.
For an isotropic hidden-feature direction in the
$(p^2-1)$-dimensional centered grid space, the expected $\AH$ is
$(p-1)/(p^2-1)=1/(p+1)$. This is a geometric reference, not the expectation
under our nonlinear network initialization.
Full-grid projections use the known arithmetic task only for observation;
held-out targets do not enter training-gradient or refit objectives.

\subsection{Actual prediction, train-only refitting, and signed rule margins}
\label{paper:prediction-defs}\label{paper:app:observers}
\begin{proofrevision}
Section~\ref{paper:setup:refit} defines the diagnostic train-only head.
Its full score table is $\widehat{\mat{Y}}_j=\mat{H}_j\widehat{\mat{B}}_j$. The complete
Gram form of its held-out scores is
\begin{equation}\label{eq:ridgekernel}
 \widehat{\mat{Y}}_{j,\Test}
 =\mat{H}_{j,\Test}\mat{H}_{j,\Train}^{\top}
 (\mat{H}_{j,\Train}\mat{H}_{j,\Train}^{\top}+\tau_j\Id)^{-1}\mat{T}_\Train.
\end{equation}
These are the scores of the unique ridge solution of~\eqref{paper:expridge}; all hidden
columns enter the fit. In the canonical experiments, $\alpha=10^{-3}$
is the penalty in mean-loss units. Equivalently, for the last hidden layer,
\begin{equation}\label{paper:expridge-ext}
 \widehat{\mat{B}}_t=\arg\min_{\mat{B}}\left\{
 \frac1n\norm{\mat{H}_{t,\Train}^{(D)}\mat{B}-\mat{T}_\Train}_F^2+
 \alpha\norm{\mat{B}}_F^2\right\}.
\end{equation}
The empirical MSE is
\begin{equation}\label{paper:refitmetrics}
 \operatorname{MSE}_{\rm ref}(t)
 =|\Test|^{-1}\norm{\mat{H}_{t,\Test}\widehat{\mat{B}}_t-\mat{T}_\Test}_F^2
 =\frac{p-1}{p}E_{\rm ref}(t).
\end{equation}
For depth $D$, $\mat{H}_t$ in this equation means $\mat{H}_t^{(D)}$.
\end{proofrevision}
The saved empirical diagnostics use $\operatorname{MSE}_{\rm ref}$ with the empirical head;
the theory uses $E_{\rm ref}$ with its separately prescribed head.
The normalization is the squared norm of a centered one-hot target.
\proofrevtext{The theoretical penalties in~\eqref{paper:refit-penalties}
scale with training-feature energy, unlike the fixed empirical $\alpha$.}
Legacy deep-development refits instead subtract the training-row feature
mean and fit training-mean-centered one-hot targets. With
$\mat{H}_c$ the centered training-feature matrix, their ridge penalty in
sum-loss units is
$10^{-3}\operatorname{tr}(\mat{H}_c^\top\mat{H}_c)/M_D+10^{-12}$.
Held-out features are centered by the same training mean before applying
the fitted coefficients, without adding a label-mean intercept.
These refits remain separately labeled.
\added{The} representation-first flags of
Appendix~\ref{paper:app:windows} \added{use this legacy refit; the depth-two
runs of Figure~\ref{fig:deep-main} log the same refit as depths three to five}.
A better training refit objective alone does not imply better held-out
accuracy, and fixed-$\alpha$ refits are not invariant to feature scale.

\paragraph{Signed rule versus residual logits.}
\begin{proofrevision}
The actual logit table and its class-centered rule/residual split are
\begin{equation}\label{paper:logit-decomposition}
\begin{aligned}
 \mat{\mathsf F}_t&=\mat{H}_t^{(D)}\mat{V}_t^\top,\qquad
 \mat{\mathsf R}_t=\Padd\mat{\mathsf F}_t\mat{C}_p,\\
 \mat{\mathsf N}_t&=\mat{\mathsf F}_t\mat{C}_p-\mat{\mathsf R}_t.
\end{aligned}
\end{equation}
For $y=y^\star(x)$, define the signed rule and adverse residual contrasts
\begin{equation}\label{paper:signed-margins}
\begin{aligned}
 \gamma_t(x)&=\min_{c\ne y}(\mathsf R_{t,xy}-\mathsf R_{t,xc}),\\
 \beta_t(x)&=\max_{c\ne y}(\mathsf N_{t,xc}-\mathsf N_{t,xy}).
\end{aligned}
\end{equation}
Their names are local to this diagnostic; they are unrelated to the
construction parameters $\gamma,\beta$ in the proofs.
\end{proofrevision}
The residual includes input-constant, class-dependent logits; centering
hidden features in $\AH$ does not remove their effect on classification.
Then $m_t(x)\ge\gamma_t(x)-\beta_t(x)$; $\beta_t(x)$ may be negative
when the residual helps classification. This elementary bound explains
which signed information a scalar energy discards. A gain $\Delta\AH>0$
alone does not guarantee a larger margin, a successful refit, or better
test accuracy. The theorems require their separate
signed-prediction arguments, not a universal alignment threshold.

\paragraph{Observation cadence.}
In the headline study, refit observations every 500 updates do not
resolve crossings or dips between observations. Network accuracy is saved
every update there and every 20 updates in the deeper-network experiments. A first
$90\%$ crossing need not be permanent.

\subsection{Windows and endpoints of the experiments}\label{paper:app:windows}
\begin{proofrevision}
Plateau variation $v(I)$, the cycle mean $\bar L_t$ and window gains such as
$\Delta\AH(I)$ are those of Section~\ref{paper:setup:phenomena}.

\paragraph{Locked windows.}
In the fresh one-hidden-layer cohorts (the \emph{pre} rows of
Table~\ref{paper:tab:1hl-main} and the cohorts of
Appendix~\ref{paper:app:nodecrease}), the window $I=[t_0,t_1]$ is fixed on
development runs and locked, with the
recipe, before the evaluated pairs are drawn. On it we report $v(I)$ against
the declared $10\%$ tolerance, training correctness at every update,
$\Delta\AH(I)$, and the held-out accuracy of the network and of the
train-only refit.

\paragraph{Loss-selected band and primary endpoint.}\label{paper:app:deep-endpoint}
The deeper-network evaluations, and the one-hidden-layer runs of
Figure~\ref{fig:onehl-primes}, select a CE peak after a head-norm knee,
without using held-out outcomes to locate the band.
The knee is the first logged head-norm checkpoint with
$\norm{\mat{V}_t}_F\ge0.99/\lambda_V$. The band $I_C=[a,b]$ is the maximal
interval containing the largest post-knee $\bar L$ on which $\bar L$ stays
at least that peak divided by $1.03$. With
$\operatorname{clip}(z)=\min(0.9,\max(0.1,z))$, the captured share of the
held-out rise is
\begin{equation}
 q(I_C)=\bigl[\operatorname{clip}(\Acc_{\Test}(b))
                  -\operatorname{clip}(\Acc_{\Test}(a))\bigr]/0.8,
\end{equation}
using the nearest logged accuracies. A run meets the \emph{primary
endpoint} if $q(I_C)\ge0.75$, its final held-out accuracy is at least
$0.99$, $\Delta\AH^{(D)}(I_C)\ge0.02$ for the last hidden layer, and every
logged minimum training margin in $I_C$ is positive. Accuracies and margins
are logged every 20 updates, and margins between logs are not certified;
head-norm, Fourier and refit diagnostics are logged every 500 updates
(1000 for the $p=97,101$ across-prime runs). Fourier and refit values
at band endpoints use the nearest observation.
A depth succeeds if the lower end of the two-sided
$95\%$ Wilson interval is at least $0.80$, that is, at least 37 of 40
passes. Selection guarantees a $3\%$ cycle-mean range if the band's maximum
is the selected post-knee peak, as in the reported primary and selected
across-prime cohorts. The band can extend before the knee, so this bound
is not universal across development runs and is not an additional pass
condition. A short flat band can occur
away from recovery, which is why the share condition is needed.

\paragraph{Strict first-fit-anchored endpoints.}
These stronger endpoints anchor at the first fit $T_{\rm fit}$: the first
update at which every training margin exceeds $10^{-10}$ (one hidden layer)
or the first positive logged minimum training margin (deeper networks).
For one hidden layer the strict joint endpoint requires $T_{\rm fit}\le
t_0$; the $10\%$ plateau and $\Delta\AH\ge0.02$ on $[t_0,t_1]$; strict fit
from $T_{\rm fit}$ through the window end and the recovery; held-out
accuracy at most $20\%$ at $t_0$ and during a delay of
$\max(1000,0.02\times\text{horizon})$ updates after $T_{\rm fit}$; and a
first $90\%$ crossing inside the window. For deeper networks it uses $I_C$
and requires, besides the fit, $\Delta\AH^{(D)}(I_C)\ge0.02$, raw band
variation of at most $10\%$, band-start held-out accuracy of at most
$20\%$, a first $90\%$ crossing inside the band, no fit relapse through the
band end, and held-out accuracy of at most $20\%$ for
$\max(1000,\lceil0.02T\rceil)$ updates after the fit. Its persistent-fit
variant starts the delay after the last nonpositive margin. Two
representation-first flags combine these joints with a useful legacy refit
(Appendix~\ref{paper:prediction-defs}) on $I_C$: from at most $0.70$ at the
start of the band, it crosses $90\%$ by its end, gains at least $0.20$, and
crosses before the network, whose accuracy is then at most $0.20$.

\paragraph{Post-hoc non-falling window.}
After every run had been seen, we also scored runs with a looser, loss-only
rule. With $\bar L_t$ and the knee as above, call an interval $[a,b]$ after
the knee \emph{non-falling} if $\bar L_t\ge\max_{a\le s\le t}\bar L_s/1.03$
for every $t\in[a,b]$. The window is the maximal non-falling interval with
the largest rise $\max_{[a,b]}\bar L-\bar L_a$; ties go to the longer. A run
passes if its held-out accuracy, or that of the train-only refit, rises by
at least $0.2$ between the first and the last log inside the window; the
captured share, final accuracy, training fit and $\Delta\AH$ are not
required. For the one-hidden-layer studies of Table~\ref{paper:tab:1hl-main}\added{,
Figures~\ref{fig:p23-confirm} and~\ref{fig:h-cohort} and the fixed-recipe
sweep of Appendix~\ref{paper:app:studies}}, the knee is instead
the first logged update, rather than a 500- or 1000-update checkpoint, with
$\norm{\mat{V}_t}_F\ge0.99/\lambda_V$. Of these, the fixed-recipe, replay and
quadratic studies do not log the head norm and start instead at the CE
minimum before their largest relative CE climb; on the 80 headline and
per-prime runs, which log both, this start lies within
230 updates of the knee and gives the same verdicts. The gray bands of
Figures~\ref{fig:onehl-primes}, \ref{fig:deep-main},\added{ \ref{fig:deep-fractions},}
\ref{fig:h-cohort}\added{ and~\ref{fig:deep-fractions-app}--\ref{fig:deep-fractions-app3}} run from the median start to
the median end of these windows, and Appendix~\ref{paper:app:posthoc}
gives the counts\added{, except for
Figures~\ref{fig:deep-fractions} and~\ref{fig:deep-fractions-app}--\ref{fig:deep-fractions-app3}, whose counts are in
Appendix~\ref{paper:app:deep-fractions}}.
\end{proofrevision}

\section{One-hidden-layer experiments}\label{paper:app:experiments-ext}
\begin{proofrevision}
This appendix supports Section~\ref{paper:onehl}. \emph{Fresh} pairs were
drawn after a recipe and its window had been locked. \emph{Earlier} runs
predate this protocol; they are reanalyzed from saved states or replayed,
and are reported descriptively. Windows and endpoints are defined in
Appendix~\ref{paper:app:windows}, and Table~\ref{paper:tab:1hl-main}
summarizes the headline, per-prime, replay and quadratic studies. The experiments are evidence for coexistence, not
replications of either initial-region theorem.
\end{proofrevision}

\begin{table}[b]
\centering\footnotesize\setlength{\tabcolsep}{2.4pt}
\caption{One-hidden-layer studies (ReLU rows: width 512, $80\%$ random
split, $\eta_W=0.03$, $\lambda_V=0.2$). Window: \emph{pre} = fixed on
development runs before the fresh pairs; \emph{per run} = persistent
held-out $10\%\!\to\!90\%$ crossings (descriptive). Flat: $v\le10\%$. Rec.:
held-out $90\%$ within the horizon. First-fit: the strict joint endpoint of
Appendix~\ref{paper:app:windows}; in particular, every training margin must stay
above $10^{-10}$ from $T_{\rm fit}$ through recovery and the window end.
--: not computed. The fixed-recipe sweep is in
Appendix~\ref{paper:app:studies}.}
\label{paper:tab:1hl-main}
\begin{tabular}{@{}lllrrrrrrrr@{}}
\toprule
Study & $c/\eta_V/\lambda_W$ & Window & $p$ & Runs & Flat & Rec. & First-fit & Med.\ $v$ & Med.\ $\Delta\AH$ & Med.\ min CE\\
\midrule
Headline (fresh) & 24/.03/.015 & pre $[4.5,12]$k & 31 & 20 & 20 & 20 & 0 & 7.06\% & .0965 & 1.48\\
\midrule
Per-prime (fresh) & 24/.01/.005 & pre $[13,21]$k & 17 & 20 & 16 & 20 & 15 & 6.3\% & .039 & $1.8\cdot10^{-4}$\\
 & 48/.03/.005 & pre $[15,21]$k & 23 & 20 & 20 & 20 & 0 & 5.4\% & .041 & $5.1\cdot10^{-3}$\\
 & 96/.03/.005 & pre $[18,24]$k & 29 & 20 & 20 & 20 & 0 & 6.5\% & .043 & $3.9\cdot10^{-2}$\\
\midrule
Replays & 24/.03/.01 & per run & 23 & 5 & 5 & 5 & -- & 2.4\% & .038 & 0.26\\
 & 24/.03/.015 & per run & 31 & 5 & 5 & 5 & -- & 2.7\% & .049 & 1.5\\
\midrule
Quadratic $^\dagger$ & -- & per run & 13 & 40 & 0 & 40 & -- & 23.6\% & .108 & --\\
\bottomrule
\end{tabular}\\[2pt]
\parbox{\textwidth}{\footnotesize $^\dagger$Width 64, 135 training pairs,
row-sphere features, steps $5\cdot10^{-4}$ and $0.5$, head decay
$10^{-3}$, loss logged every 500 updates.}
\end{table}

\begin{proofrevision}
\subsection{Protocol and selection}\label{paper:app:onehl-protocol}
The headline configuration and initialization are those of
Section~\ref{paper:onehl}. The supported Polar step keeps singular values
above $10^{-10}$ times the largest. CE and accuracies are saved at all
40{,}001 states of every pair, $\AH$ every 100 updates, and the refit every
500. No run is excluded; plotted loss uses two-update cycle means without
additional smoothing. The refit uses the raw-feature, fixed-$\alpha=10^{-3}$
rule~\eqref{paper:expridge}. A uniform shrinkage of the features strengthens
its penalty, unlike the training-RMS-normalized rules of the theory.

A first search of 96 runs at $p=17$ found no qualifying recipe. A second tested 48
runs across $p\in\{23,31\}$, training fractions $0.5$ and $0.8$ and
$\lambda_W\in\{0.01,0.015,0.02\}$, with four initialization/split pairs per
recipe, $c=24$, $\eta_W=\eta_V=0.03$, $\lambda_V=0.2$ and 40k updates. A
rule fixed before that search scanned windows on the 500-update grid, with
start at least 1000 and length at least 2000. Its strict tier required, for
every development pair of a recipe, $v\le8\%$, minimum CE at least $0.12$,
$\Delta\AH\ge0.07$, perfect window training accuracy, held-out accuracy at
most $20\%$ at the start, and a first $90\%$ crossing in $(t_0,t_1-500]$.
Qualifying recipes were ordered by worst-pair Fourier gain, with fixed
tie-breaks. The rule selected the headline recipe and $[4500,12000]$, and
both were locked before 20 fresh pairs were drawn with NumPy's
\texttt{default\_rng(20260930)}. Auxiliary selection and evaluation
criteria, not all of which were met, are recorded with the code. Selection used held-out accuracy, so
the evaluation tests this locked choice, not an untuned optimizer family.

\subsection{The headline cohort}\label{paper:app:onehl-headline}
Section~\ref{paper:onehl} gives the window outcomes. The median of $v(I)$
is $7.06\%$: these are positive-loss plateaus at the predeclared $10\%$
tolerance, not $3\%$ plateaus. Over $I$, $\AH$ gains $0.0955$--$0.0972$
(median $0.0965$). When the refit first records $90\%$, at update 6{,}500,
the previous checkpoint is below it in every run. This is earlier
availability under a fixed rule on a 500-update grid, not a proof that the
head tracks the ridge solution.

\paragraph{Two phases.}
The CE peak separates two phases. While CE rises, $\AH$ makes $25$--$39\%$
of its window gain and the refit exceeds $70\%$ by update 6{,}000, whereas
the network stays at most $5.7\%$. The network's own rise follows the peak,
with cycle-mean CE $\bar L$ falling by $1.5$--$3.4\%$ between its first
$10\%$ and $90\%$ crossings. Even so, $\bar L$
ends $2.7$--$11.4\%$ higher at the network's first $90\%$ than at the update
at which the refit starts rising, in every run. Here and below, the refit's
rise starts at its last 500-update observation strictly before its first
observed $10\%$ crossing.

\paragraph{Norms and energies.}
Before the window, CE falls to $0.792$--$0.854$ at updates 1{,}801--1{,}926
and then rises to $1.456$--$1.498$ at update 4{,}500. During the window,
$\norm{\mat{V}_t}_F=4.997$--$5.000$, at the cap $1/\lambda_V=5$, while
$\norm{\mat{W}}_F$ falls from $128.9$--$130.5$ to $57.4$--$57.7$ and the
centered training-logit RMS from $0.720$--$0.740$ to $0.464$--$0.480$. For
$\lambda_V>0$ and $q=1-\eta_V\lambda_V\in[0,1]$, the normalized head
update~\eqref{paper:optimizer} gives the envelope
\begin{equation}\label{paper:vcap}
 \norm{\mat{V}_t}_F\le q^t\norm{\mat{V}_0}_F+(1-q^t)/\lambda_V ,
\end{equation}
an upper bound, not a proof that norm saturation causes recovery. Over $I$,
$E_{\rm ctr}$ falls to $0.204$--$0.209$ of its initial value while
$E_{\rm add}$ grows by a factor $1.22$--$1.31$, so the fraction gain is
mostly suppression of energy outside the addition subspace. With fixed
ridge regularization on raw features, the refit-first observation does not
separate directional change from scale.

\paragraph{Retrospective suffix windows.}\label{paper:app:suffix}
Let $t_\epsilon$ be the earliest update after which CE stays within a
factor $1+\epsilon$ of its minimum until update 40{,}000. The rule uses
training CE only but needs the whole future trace, so it is retrospective,
not a stopping rule. At $\epsilon=10\%$, $t_\epsilon$ is 4{,}170--4{,}311,
and at $5\%$ it is 4{,}653--6{,}421, both with zero held-out accuracy. At
$3\%$ it is 4{,}946--8{,}376, after improvement has begun in some runs.
Training stays fitted on every suffix, held-out accuracy at 40k is
$94.30$--$99.48\%$, and $\AH$ ends in $[0.1194,0.1203]$.

\paragraph{The strict first-fit-anchored endpoint.}\label{paper:app:strict}
No headline run meets the strict endpoint of
Appendix~\ref{paper:app:windows}. The first strict fit, at updates
203--796, is followed by intermittent misclassification of a few training
pairs until a permanent fit at updates 239--1{,}085, so fit is not
persistent from its first attainment. Held-out accuracy stays zero through
the delay. Reanchoring at the persistent fit would change the definition
post hoc.

\subsection{Pre-registered tests of recovery without CE decrease}\label{paper:app:nodecrease}
In the headline cohort, $\bar L$ decreases by $1.5$--$3.4\%$ between the
network's first $10\%$ and $90\%$ held-out crossings. After this evaluation we
screened the 12 recipes of the second development search for recovery
without such a decrease. Only one passed the screen in all four of its
pairs: $p=23$ with the headline hyperparameters. A window rule applied to
these four pairs, relaxed once after its first version admitted no window,
chose $I=[6500,16500]$.

Before any new run we locked the recipe, the window, 20 fresh pairs, a
20{,}000-update horizon, the code and a decision rule. The joint endpoint
requires $v(I)\le10\%$, perfect training accuracy on $I$, held-out
accuracy at most $20\%$ at $t_0$, $\Delta\AH\ge0.04$, and no CE decrease
during recovery: the first $90\%$ held-out crossing $T_{90}$ occurs by
$t_1-500$, $\bar L$ never falls more than $0.5\%$ below its running maximum
on $[t_0,T_{90}]$, and $\bar L(T_{90})\ge\bar L(t_0)$. At least 18 of the
20 runs had to pass. It holds in 10 of the 20 runs, so this test fails
(Figure~\ref{fig:p23-confirm}, top). Every run meets the other four
clauses; the failures concern the end of the recovery, a late or missing
$90\%$ crossing or a small loss decrease before it.

In these runs $\bar L$ stays within $0.001\%$ of its running maximum up to
the first $50\%$ crossing. This pattern motivated a second lock, written
before 20 new pairs were launched, that keeps the joint endpoint with
$T_{90}$ replaced by the first $50\%$ crossing. The joint endpoint then holds in 20
of 20 runs ($95\%$ Wilson interval $[0.84,1.00]$;
Figure~\ref{fig:p23-confirm}, bottom, and Figure~\ref{fig:h-cohort}a).
Held-out accuracy rises from at most $1.9\%$ to $50\%$ at updates
8{,}278--9{,}191, while $\bar L$ never falls more than $0.05\%$ below its
running maximum and ends $3.8$--$5.9\%$ above $\bar L(t_0)$. As secondary
endpoints, the joint criterion with the registered $0.5\%$ drawdown
tolerance holds up to the first $70\%$ crossing in 20 runs and up
to the first $90\%$ crossing in 9, and 17 runs reach $90\%$ within
20{,}000 updates. The confirmed statement covers recovery to $50\%$, not to
$90\%$.
\end{proofrevision}

\begin{figure}[tbp]
\centering
\includegraphics[width=\textwidth]{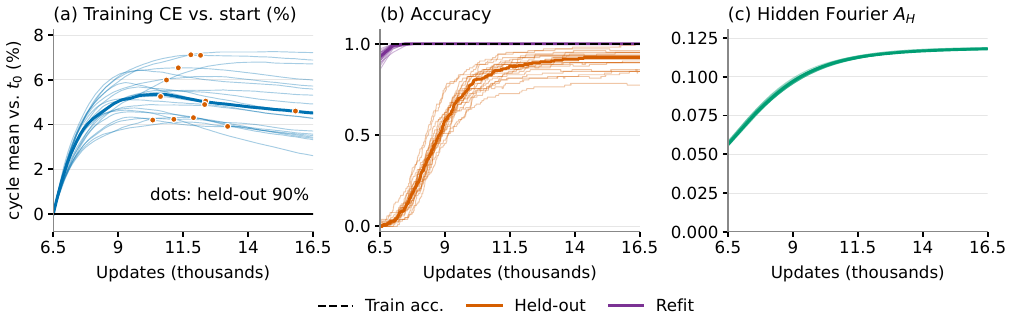}
\includegraphics[width=\textwidth]{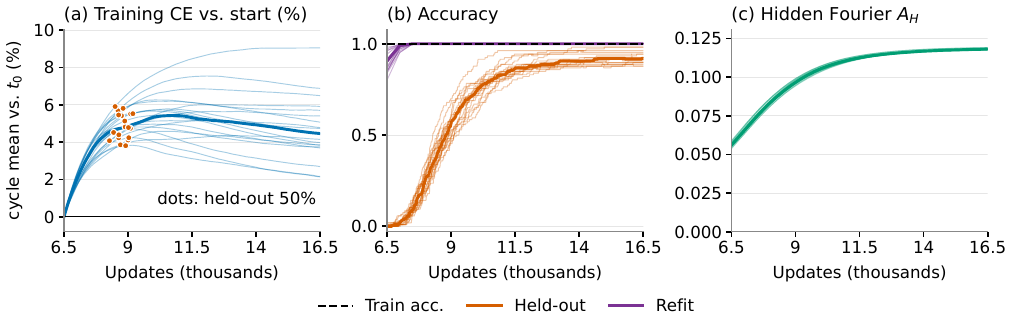}
\caption{\textbf{Pre-registered $p=23$ cohorts} (20 fresh pairs each; window
$I=[6500,16500]$ chosen on four development pairs and locked before these
runs). (a)~Cycle-mean training CE relative to its value at $t_0$; dots mark
each run's first $90\%$ (top) or $50\%$ (bottom) held-out crossing on $I$.
(b)~Median training accuracy (dashed) and held-out accuracy of the network
and of the train-only refit (500-update grid). (c)~$\AH$. Thin curves are
runs and thick curves medians. The joint endpoint, which limits cycle-mean
drawdown to $0.5\%$ before the dot and requires the loss at the dot to be no
smaller than the loss at $t_0$, holds in 10 of 20 runs (top) and 20 of 20
(bottom); 18 were required.}
\label{fig:p23-confirm}
\end{figure}

\begin{figure}[tbp]
\centering
\includegraphics[width=\textwidth]{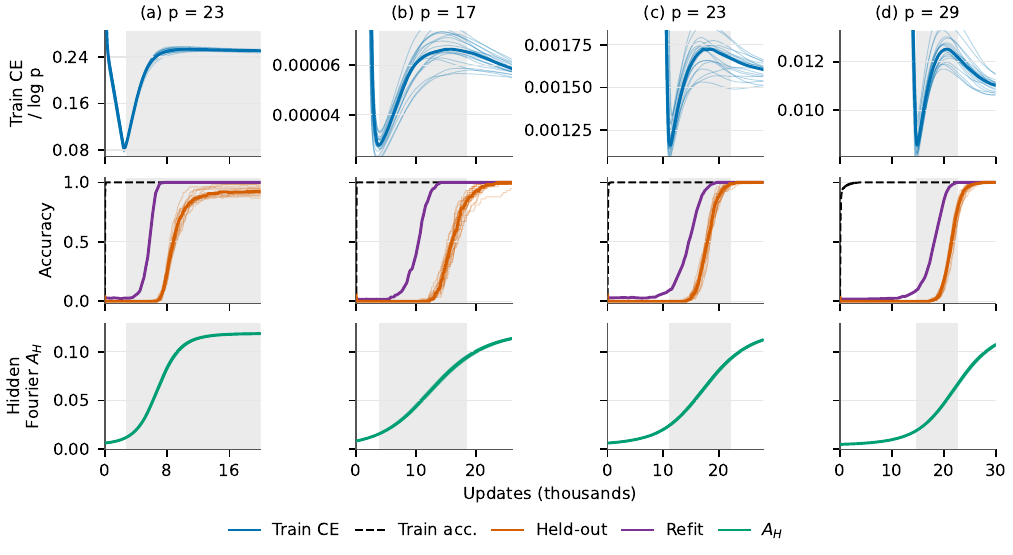}
\caption{\proofrevtext{\textbf{Locked and pre-registered one-hidden-layer cohorts} (20 fresh pairs per column).
(a)~$p=23$ with the headline hyperparameters: the second pre-registered
cohort, window $[6500,16500]$. (b)--(d)~The per-prime recipes at $p=17,23,29$,
with windows $[13000,21000]$, $[15000,21000]$ and $[18000,24000]$. Rows:
cycle-mean training CE divided by $\log p$ (each column on its own scale);
median training accuracy (dashed), held-out accuracy of the network (runs
and median) and median train-only refit; $\AH$. Gray: from the median start
to the median end of the runs' post-hoc windows. The headline runs at
$p=31$ are in Figure~\ref{paper:fig:headline}.}}
\label{fig:h-cohort}
\end{figure}

\begin{proofrevision}
\subsection{Results across primes}\label{paper:app:studies}
\paragraph{One hidden layer across primes.}
The runs of Figure~\ref{fig:onehl-primes} have one hidden layer of width
512 and use the deeper-network implementation of the updates
(Appendix~\ref{paper:app:deep-settings}). At each prime, $\lambda_W$ is the
best of a development sweep on one or two seeds. It was fixed before any
new run and then run on 8 or 9 fresh initialization/split seeds, for 10
runs per prime; at $p=101$ the setting chosen at $p=97$ was run unchanged
on 10 fresh seeds. Runs stopped early by the stated rule (held-out accuracy
of at least $0.99$, maintained) keep their last value. The figure's caption
gives the primary-endpoint counts. These are development settings on a few
seeds each, not a pre-specified evaluation.

\paragraph{Per-prime recipes.}
Three recipes at $p=17,23,29$ were fixed, each with its own window, after
54 development runs and before 20 fresh pairs per prime were drawn
(Figure~\ref{fig:h-cohort}b--d). Each run lasts 50k updates, with width
512, training fraction $0.8$, $\eta_W=0.03$, $\lambda_W=0.005$,
$\lambda_V=0.2$ and $(c,\eta_V)=(24,0.01),(48,0.03),(96,0.03)$. The
$10\%$-plateau counts are 16/20, 20/20 and 20/20, with median $\Delta\AH$
of $0.039$, $0.041$ and $0.043$, and all 60 runs reach $90\%$ held-out
accuracy. The median minimum CE is $1.8\cdot10^{-4}$, $5.1\cdot10^{-3}$ and
$3.9\cdot10^{-2}$, far below the headline study, and the strict joint
endpoint holds in 15/20, 0/20 and 0/20 runs, so the plateau and recovery
counts are not joint successes. From the refit-start checkpoint defined
above to the network's first $90\%$, $\bar L$ ends above its
starting value in 20, 13 and 15 runs (medians $+36\%$, $+37\%$ and
$+33\%$).

\paragraph{One fixed recipe across primes.}
Earlier runs hold one recipe fixed across $p=11,13,17,19,23,29$, with 20
pairs each: $c=24$, $\eta_W=\eta_V=0.03$,
$\lambda_W=0.005$, $\lambda_V=0.2$, width 512 and 40k updates. Observers
and the refit were recomputed from saved states every 1{,}000 updates. On
the common window $[15000,20000]$ the $10\%$ plateau holds in most runs
only at $p=17$\added{: in 0, 0, 17, 2, 0 and 0 of the 20 runs at
$p=11,13,17,19,23,29$}. Success of a retuned family is thus
no evidence that one practical recipe works across moduli.
\end{proofrevision}

\begin{proofrevision}
\subsection{Replays and other one-hidden-layer studies}\label{paper:app:onehl-other}
\paragraph{Replayed trajectories.}
We replay five earlier pairs at each of $p=23$ ($\lambda_W=0.01$) and
$p=31$ (the headline hyperparameters) for 30k updates, with a
Gram-eigendecomposition Polar factor (relative cutoff $10^{-7}$). Relative
raw-CE differences from the originals stay below $2\cdot10^{-14}$ through
15k updates, first exceed $10^{-9}$ at updates 17{,}470--28{,}134, and
are at most $0.069\%$ at the 30k horizon; every scored window ends by
13.5k. The windows come from each run's persistent
$10\%\!\to\!90\%$ held-out transition, snapped to the observer grid, so
they are descriptive. All ten runs have $10\%$ plateaus, with median $v$ of
$2.4\%$ and $2.7\%$, and recover, with median $\Delta\AH$ of $0.038$ and
$0.049$. These trajectories also belong to the deeper study's development
inventory and are not counted twice.

\paragraph{Quadratic-activation contrast.}
A separate study uses 40 earlier pairs at $p=13$, width 64, with a square
activation, row-sphere feature updates and their own head dynamics (feature
step $5\cdot10^{-4}$, head step $0.5$, head decay $10^{-3}$). All reach
$90\%$ at least once and gain at least $0.02$ in Fourier fraction, but none
has a $10\%$ CE plateau (median variation $23.6\%$, with loss logged every
500 updates). Activation and optimizer both differ, so this is not an
activation-only ablation.

\paragraph{Other settings.}
The second search's recipes with a $50\%$ training fraction did not recover
within 40k updates,
and a later search at $p=31$ with no head decay or a quadratic-residue
split met the selection rule in no arm. These are limits of the tested recipes and
horizons, not counterexamples to the theorems.
\end{proofrevision}

\begin{proofrevision}
\subsection{Theorem-inspired split and update settings from iid initialization}\label{paper:app:thm41-iid}
\begin{figure}[t]
\centering
\includegraphics[width=\linewidth]{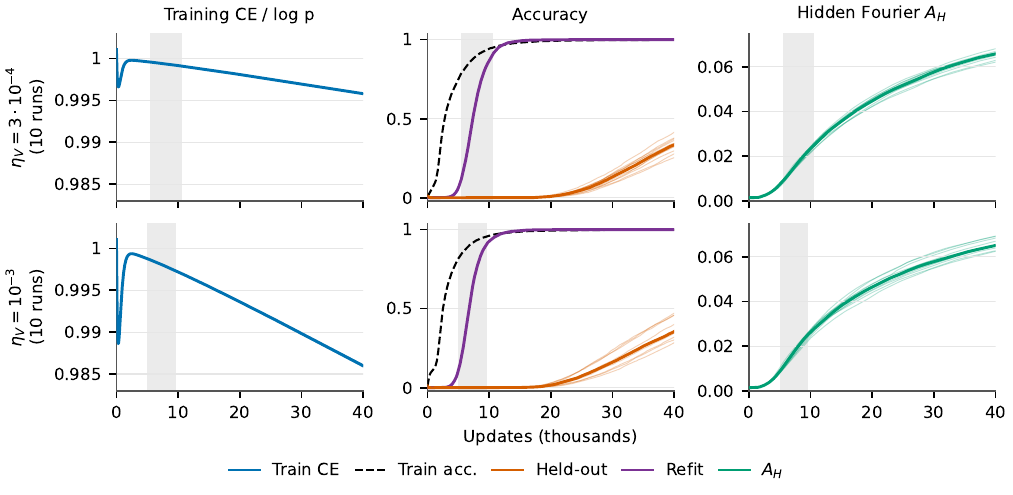}
\caption{\proofrevtext{\textbf{Theorem-inspired settings at $p=101$ from iid
initialization} (width 512; all runs in the two displayed arms). Rows: head steps $\eta_V=3\cdot10^{-4}$ (\added{10} runs; all
criteria of Appendix~\ref{paper:app:thm41-iid} hold in each) and the primary $\eta_V=10^{-3}$
(10 runs; the training loss varies by $1.34$--$1.37\%$ over $[1000,40000]$, above the $1\%$
criterion). Columns: training cross-entropy loss divided by $\log p$ (the axis spans $2\%$ of
$\log p$); median training accuracy (dashed), held-out accuracy of the network (runs and median)
and median held-out accuracy of the theorem's train-only refit; $\AH$. Gray: from the median
first observed checkpoint at which the refit's held-out accuracy reaches $0.1$ to the median
first at which it reaches $0.9$; these medians need not define any individual run's interval.}}
\label{fig:thm41-iid}
\end{figure}
We trained a bias-free 1HL ReLU network at $p=101$ using
Theorem~\ref{paper:thm:half}'s split~\eqref{eq:split} ($n=5{,}050$ training
and $5{,}151$ held-out pairs), the clock~\eqref{eq:clock} with
$\lambda_W=\lambda=1$ and $\eta_W=s=1-(9/10)^{1/32}$, $\lambda_V=0$, and
the refit penalty $\tau_t=R_t/(32n)$ of~\eqref{paper:refit-penalties}.
This is not a theorem replication: the width is $M=512$ instead of at least
$M_{\min}(101)=515{,}302$, and the initialization is iid as in
Section~\ref{paper:onehl} ($c=24$) instead of the engineered ball. The head
steps are not checked against the theorem's existential cap, and the
observed windows are not its certified windows. In a plan fixed after one
development run and before these runs, we chose four head steps $\eta_V$,
with $10^{-3}$ as the primary one, and four per-run criteria for
$T=40{,}000$ updates: (a) $v([1000,T])\le0.01$; (b) $\AH$ rises by at least
$0.02$ over $[1000,T]$; (c) the refit's held-out accuracy is at most $0.1$
at update $1{,}000$ and reaches $0.9$; (d) the refit reaches $0.5$ at least
$1{,}000$ updates before the network, or the network does not reach $0.5$
by $T$. Network accuracy is logged every 50 updates; the refit and Fourier
fraction are observed every 250. Crossings refer to these checkpoints.

With $\eta_V=3\cdot10^{-4}$ all four criteria hold in \added{$10/10$} runs
(Figure~\ref{fig:thm41-iid}, top\added{; the plan's 6 runs and 4 more seeds run
for this version}). With the primary $\eta_V=10^{-3}$,
(b)--(d) hold in every run but (a) in none, with $v([1000,T])=1.34$--$1.37\%$
(bottom); the larger steps $3\cdot10^{-3}$ and $10^{-2}$ give $v$ of about
$4\%$ and $15\%$ (0/6 each). In the two displayed arms, the training loss varies by at
most $0.05\%$ and $0.17\%$ over each run's own interval from refit accuracy
$0.1$ to $0.9$. Over these intervals $\AH$ rises from about $0.009$ to
$0.025$, and the network's held-out accuracy stays zero at recorded
checkpoints. By update $40{,}000$ the refit is at one in every run and the
network at $0.25$--$0.47$.

The displayed runs do not reproduce Theorem~\ref{paper:thm:half}'s clauses~(i)
and~(iii): training accuracy is \added{$0.76$}--$0.84$ when the refit reaches $0.1$,
$2$--$14$ training pairs are still misclassified at update $40{,}000$, the
network has not recovered by then, and the loss stays
within about $1.5\%$ of $\log p$, a plateau near uniform prediction rather
than a fitted one.

The plan also fixed two controls with two runs each. With a nearly frozen
head, $\eta_V=10^{-6}$, the runs stall: at update $40{,}000$ training accuracy
is $0.42$--$0.46$, the refit's held-out accuracy is below $0.01$, and $\AH$
has risen by $0.002$. With $\eta_V=10^{-3}$ and the experiments' slower clock,
$\eta_W=0.03$ and $\lambda_W=0.02$, criteria (b)--(d) hold but
$v([1000,T])$ is $105$--$106\%$.
\end{proofrevision}

\clearpage
\section{Deeper-network experiments}\label{paper:app:deep-ext}
\begin{proofrevision}
This appendix supports Section~\ref{paper:deep}: networks with three to
five hidden layers at $p=31$ and width 256, changes of the modulus, width
and feature optimizer, settings across primes, and hidden biases. The 164
evaluated runs use 64 seed indices, separate from the 199 development runs
that chose the settings using held-out accuracy and windows maximizing the
captured accuracy rise under a flat cycle-mean-loss condition. The
development runs, which include the ten replays of
Appendix~\ref{paper:app:onehl-other}, are not counted in the evaluation
rates. Seed indices recur across depths and settings, so these observations
are not independent. Windows and endpoints are defined in
Appendix~\ref{paper:app:windows}.

\subsection{Network settings and learning-rate selection}\label{paper:app:deep-settings}
The task, bias-free ReLU architecture, random $80\%$ split and full-batch
float64 CE are as in Sections~\ref{paper:setup}
and~\ref{paper:experiments}. The primary configuration uses $p=31$ and width
256 in every hidden layer. With independent standard Gaussian arrays
$\mat{Z}_\ell,\widetilde{\mat{Z}}$, the initial weights are
\begin{equation}
 \mat{W}^{(1)}=\frac{24}{\sqrt p}\mat{Z}_1,\qquad
 \mat{W}^{(\ell)}=8\sqrt{\frac{2}{M}}\mat{Z}_\ell\ (\ell\ge2),\qquad
 \mat{V}=\frac{\widetilde{\mat{Z}}}{24\cdot8^{D-1}\sqrt M}.
\end{equation}
Every layer trains throughout the main run. Feature updates use a
Gram-eigendecomposition implementation of supported Polar normalization.
It keeps eigenvalues of $\mat{F}^\top\mat{F}$ (or $\mat{F}\mat{F}^\top$ for a wide $\mat{F}$) above
$10^{-14}$ times the largest, a relative singular-value cutoff of
$10^{-7}$, and divides by the square root of the retained rank. This
differs numerically from the headline SVD implementation, and
near-threshold rank sensitivity remains a robustness limitation.

The head uses $\eta_V=0.03$ and $\lambda_V=0.2$. The first layer has step
$0.03$ and decay $0.015$, so $K_1=\eta_1\lambda_1=4.5\cdot10^{-4}$. For
$\ell\ge2$, $\lambda_\ell=0.07$ and $\eta_\ell=r_\ell K_1/0.07$, where the
shrink-rate multiplier $r_\ell$ is 1 for intermediate layers and $r$ for
the last. To select $r$, we first train at multiplier $r_0$ and record
$t_{10}^{\rm sel}$, the first logged update with $10\%$ held-out accuracy.
We then solve
\begin{equation}\label{paper:recipe}
 rK_1t_{10}^{\rm sel}\exp\{0.9(r-r_0)\}=x_* ,
\end{equation}
then round $r$ to a multiple of $0.005$ and clip it to $[0.10,1.20]$.
For $D=3,4,5$, $r_0=0.50,0.35,0.35$ and $x_*=3.85$. We restart from the
same initialization on the same split and hold the selected $r$ fixed throughout training, so there is no within-run
loss controller.

\subsection{Depths three to five}\label{paper:app:deep-protocol}
The evaluation criteria, rate-selection rule and seeds were fixed before
training, and all 164 evaluated runs completed.
Table~\ref{paper:tab:deep-design} lists the settings and results.
Initialization and split seeds are $1000+i$ and $2000+i$: indices 1--40
recur at each primary depth, the normalized-gradient arm reuses depth-3
indices 1--20, the $p=23$ arm uses 41--60 and the width-512 arm 61--64.

\paragraph{Primary endpoint.}
The 40/40, 39/40 and 38/40 passes at depths 3, 4 and 5 have per-depth (not
simultaneous) Wilson intervals $[0.91,1.00]$, $[0.87,1.00]$ and
$[0.83,0.99]$, whose lower ends are at least $0.80$, so each depth meets the
criterion (Figure~\ref{fig:deep-campaign-summary}). The three misses fail only the
share threshold. The median band-peak CE is about $0.50$, $0.54$ and
$0.58\log p$.

\paragraph{Fixed versus selected learning rates.}
On 20 paired seeds per depth, runs at the fixed multiplier $r_0$ pass
19/20, 5/20 and 19/20, against 20/20 for the recipe (exact McNemar $p=1$,
$6.1\cdot10^{-5}$ and $1$). The recipe's benefit is thus clear at depth 4,
not at every depth.

\paragraph{Changing the modulus, width or optimizer.}
None of the three changed settings meets the primary endpoint
(Table~\ref{paper:tab:deep-design}): $p=23$ at depth 4 (0/20), width 512 at
depth 4 (0/4; final held-out accuracy $0.92$--$0.95$ at 100k updates), and
Frobenius-normalized gradient (NGD) feature updates at depth 3 (0/20,
against 20/20 for the paired Polar recipe; exact McNemar
$p=1.9\cdot10^{-6}$). The NGD runs generalize, but their CE keeps falling
during recovery. They reuse the Polar-calibrated recipe, so they are not a
separately tuned baseline. \added{Other primes are not ruled out by these
counts: at $p=47$ and $61$, with $r$ set per training fraction and
$\lambda_V=0.2\cdot31/p$, a post-hoc-window criterion holds in 155 of 160
fresh runs, though the primary endpoint holds in only 9
(Appendix~\ref{paper:app:deep-fractions}); $p=23$ was not rerun that
way.}

\paragraph{Secondary endpoints.}
Of the 120 primary runs, the pre-specified strict first-fit joint of
Appendix~\ref{paper:app:windows} passes in none, its pre-specified
persistent-fit variant in 49, and the two representation-first flags in
none. These checks do not
alter the primary verdict, and no multiplicity correction is applied. The
experiments do not establish a uniform refit-before-head chronology or a
causal role for scalar Fourier growth.
\end{proofrevision}

\begin{figure}[tbp]
\centering
\includegraphics[width=\textwidth]{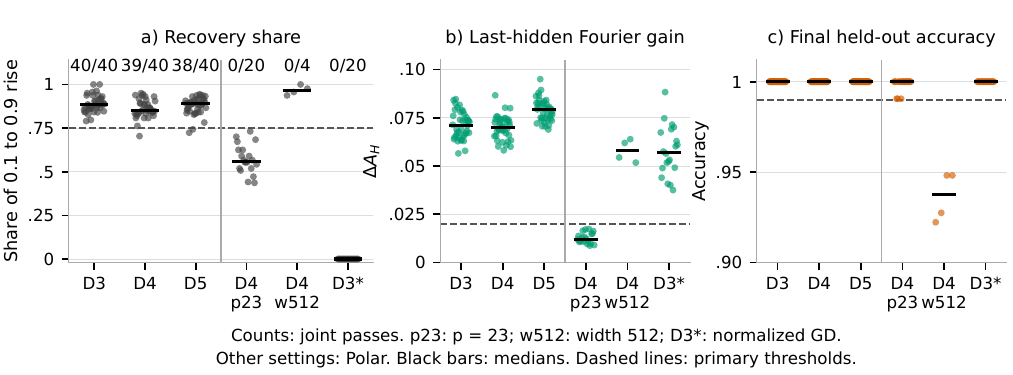}
\caption{\textbf{Results for all 164 evaluated deep-network runs.}
Dots are runs and bars medians. Left of each divider: the primary cohorts
($p=31$, width 256, $D=3,4,5$); right: $D=4$ with $p=23$, $D=4$ with width
512, and $D=3$ with Frobenius-normalized hidden updates (NGD). (a) Share of
the $10\%$-to-$90\%$ held-out rise inside the loss-only band $I_C$, with
primary-endpoint pass counts above; (b) last-hidden-layer Fourier-fraction gain $\Delta\AH$ on
$I_C$; (c) held-out accuracy at the saved horizon.
Dashed lines mark the primary thresholds in (a)--(c); the fourth primary
criterion, positive logged training margins in $I_C$, holds in all 164
runs. Rates were selected by rule~\eqref{paper:recipe}.}
\label{fig:deep-campaign-summary}
\end{figure}

\begin{table}[tbp]
\centering\footnotesize\setlength{\tabcolsep}{3pt}
\caption{\proofrevtext{Deep-network settings and outcomes, retaining all evaluated runs.
Rate-selection runs stop at $t_{10}^{\rm sel}$,
except those also used for fixed-rate comparisons, which continue to the
full horizon. Each evaluated run restarts from the same initialization
and split used to select its rate. Primary: the per-run primary endpoint.
$J_1/J_{\rm pers}$: first-fit and persistent-fit strict joints;
$P_{10}/P_3$: raw-CE range at most $10\%$ or $3\%$ from the first $10\%$
to the first $90\%$ held-out accuracy; RF: legacy-refit
representation-first flag. These
secondary tests do not alter the primary verdict.}}
\label{paper:tab:deep-design}\label{paper:tab:deep-secondary}
\begin{tabular}{llrrrrrrrr}
\toprule
Setting & $(D,M,p)$ & Indices $i$ & $r_0$ & $x_*$ & Horizon & Primary & $J_1/J_{\rm pers}$ & $P_{10}/P_3$ & RF\\
\midrule
Primary & $(3,256,31)$ & 21--40 / 1--20 & .50 & 3.85 & 80k & 40/40 & 0/17 & 40/2 & 0\\
Primary & $(4,256,31)$ & 21--40 / 1--20 & .35 & 3.85 & 100k & 39/40 & 0/13 & 39/0 & 0\\
Primary & $(5,256,31)$ & 21--40 / 1--20 & .35 & 3.85 & 80k & 38/40 & 0/19 & 39/0 & 0\\
Modulus & $(4,256,23)$ & 41--60 & .35 & 3.85 & 120k & 0/20 & 0/0 & 0/0 & 0\\
Width & $(4,512,31)$ & 61--64 & .40 & 4.50 & 100k & 0/4 & 0/1 & 4/1 & 0\\
NGD & $(3,256,31)$ & 1--20 & .50 & 3.85 & 80k & 0/20 & 0/0 & 0/0 & 0\\
\bottomrule
\end{tabular}
\end{table}

\begin{proofrevision}
\subsection{Results across primes}\label{paper:app:deep-medians}
\added{Depths 3 and 4 were also run at $p=23,47,61,101$, with width 256 and
last-layer multiplier $r=0.45$ (depth 4 at $p=23$: $0.32$).}
Each setting is the best of a development sweep on one or two seeds, with a
tuned last-layer multiplier $r$ and all other layers at $r_\ell=1$. It was
fixed before any new run and then run on 8 or 9 fresh initialization/split
seeds, for 10 runs per setting; at $p=101$ the settings chosen at $p=97$
were run unchanged on 10 fresh seeds. Depth 2 was run at $p=23,31,47$ with its development
configurations; at $p=31$ (5 development and 5 fresh seeds) it passes the
primary endpoint in 9 of 10 runs. \added{Rerun with the same fixed rate on the
40 seed indices of the primary depths, as the first column of
Figure~\ref{fig:deep-main}, it passes in 34/40 (Wilson interval
$[0.71,0.93]$, whose lower end is below the $0.80$ of the criterion that
depths 3--5 meet); the six misses fail the share threshold, and one also
ends at held-out accuracy $0.984$.} Off $p=31$, among depths 2 to 4 only depth 3 at $p=47$
passes (2/10)\added{; the other depth-3 and depth-4 settings at
$p=23,47,61,101$ pass in none of their 10 runs, and at $p\ge47$ some
training pairs are misclassified inside the post-hoc window}. The four settings chosen at $p=97$ (one hidden layer of width
512 or 2048, and depths 3 and 4) were also run at $p=97$, 10 runs each, and
none passes. These are development settings on a few seeds each, not a
pre-specified evaluation like the depth-three-to-five experiment. Under the
looser post-hoc rule every run of these settings passes
(Appendix~\ref{paper:app:posthoc}). \added{Appendix~\ref{paper:app:deep-fractions}
reruns depths 3 and 4 at $p=47$ and $61$ with $r$ chosen per training
fraction, on fresh seeds and with a criterion fixed in advance.}
\end{proofrevision}

\subsection{Depths three and four across training fractions}\label{paper:app:deep-fractions}
\begin{additions}
This experiment asks whether a re-tuned recipe carries the deeper-network
behavior to other primes and smaller training sets. It runs depths 3 and 4
at $p=47$ and $61$ with training fractions $0.8$, $0.7$, $0.6$ and $0.5$,
in the configuration of Appendix~\ref{paper:app:deep-settings} (width 256,
the same initialization scales, rates and decays), with two changes. The
head decay scales with the prime, $\lambda_V=0.2\cdot31/p$ ($0.132$ at
$p=47$ and $0.102$ at $p=61$), and the last-layer multiplier $r$ is fixed
for each depth and fraction, without the per-run
rule~\eqref{paper:recipe} (Table~\ref{paper:tab:deep-fractions}). Each $r$
was chosen on 6 development runs per candidate value (3 seeds at each
prime) by a fixed order: the most passes of the criterion below, then the
most passes with training accuracy at least $0.97$ throughout the window,
then the larger median share, then the earlier median $90\%$ held-out
crossing. The criterion, the values of $r$ and the analysis were fixed
before ten fresh initialization/split seeds per setting were run, 160 runs
in all; their seed indices were not used in any development run. One split
seed per index nests the training sets across fractions, so each seed is
a paired ladder. Runs stop at 100k, 150k, 200k and 250k updates at
fractions $0.8$, $0.7$, $0.6$ and $0.5$, or once held-out accuracy has been
at least $0.99$ at every log for $\max(0.5t,10^4)$ updates since the start
$t$ of that unbroken stretch.

\begin{table}[!htbp]
\centering\footnotesize\setlength{\tabcolsep}{4pt}\addedcolor
\caption{\added{Depths 3 and 4 at $p=47,61$ across training fractions (10
fresh runs per setting; $\lambda_V=0.2\cdot31/p$). Pass: the criterion of
Appendix~\ref{paper:app:deep-fractions}. $\bar L/\log p$: training CE (cycle
mean) when held-out accuracy crosses $0.5$. $\Delta\bar L$: its relative
change from the refit's $0.2$ crossing to the held-out $0.9$ crossing, with
the number of runs in which it rose. $\Delta\AH(W)$: last hidden layer.
Lead: held-out minus refit $0.5$ crossing. $t_{90}$: held-out $0.9$
crossing. Mis.: the most training pairs misclassified at any log inside $W$,
over all 10 runs. Medians over the 10 runs unless stated; updates in
thousands.}}
\label{paper:tab:deep-fractions}
\begin{tabular}{@{}llrrrrrrrrr@{}}
\toprule
$D$ & $p$ & Fraction & $r$ & Pass & $\bar L/\log p$ & $\Delta\bar L$ (rose) & $\Delta\AH(W)$ & Lead & $t_{90}$ & Mis.\\
\midrule
3 & 47 & 0.8 & .45 & 10/10 & 0.52 & $+9.5\%$ (10) & .51 & 2.4 & 19.5 & 0\\
3 & 47 & 0.7 & .30 & 10/10 & 0.35 & $+49\%$ (10) & .54 & 2.8 & 24.2 & 0\\
3 & 47 & 0.6 & .18 & 10/10 & 0.15 & $+233\%$ (10) & .48 & 3.7 & 34.1 & 0\\
3 & 47 & 0.5 & .08 & 8/10 & 0.007 & $+1117\%$ (10) & .48 & 4.7 & 48.1 & 0\\
\midrule
3 & 61 & 0.8 & .45 & 10/10 & 0.58 & $+0.7\%$ (8) & .53 & 1.2 & 17.6 & 18\\
3 & 61 & 0.7 & .30 & 10/10 & 0.43 & $+13\%$ (10) & .63 & 1.3 & 22.0 & 1\\
3 & 61 & 0.6 & .18 & 10/10 & 0.21 & $+51\%$ (10) & .64 & 1.5 & 28.3 & 0\\
3 & 61 & 0.5 & .08 & 10/10 & 0.014 & $+136\%$ (10) & .69 & 1.8 & 40.7 & 0\\
\midrule
4 & 47 & 0.8 & .45 & 9/10 & 0.63 & $-0.5\%$ (5) & .40 & 2.7 & 26.4 & 2\\
4 & 47 & 0.7 & .28 & 10/10 & 0.52 & $+29\%$ (10) & .56 & 4.9 & 36.4 & 0\\
4 & 47 & 0.6 & .13 & 10/10 & 0.17 & $+303\%$ (10) & .62 & 4.8 & 47.5 & 0\\
4 & 47 & 0.5 & .06 & 9/10 & 0.044 & $+1312\%$ (10) & .59 & 9.5 & 88.3 & 0\\
\midrule
4 & 61 & 0.8 & .45 & 9/10 & 0.67 & $-2.9\%$ (0) & .38 & 1.5 & 23.3 & 14\\
4 & 61 & 0.7 & .28 & 10/10 & 0.56 & $+11\%$ (10) & .61 & 2.1 & 31.0 & 4\\
4 & 61 & 0.6 & .13 & 10/10 & 0.23 & $+69\%$ (10) & .68 & 2.0 & 41.0 & 0\\
4 & 61 & 0.5 & .06 & 10/10 & 0.045 & $+238\%$ (10) & .69 & 3.3 & 67.8 & 0\\
\bottomrule
\end{tabular}
\end{table}

\paragraph{Criterion.}
The criterion uses the post-hoc non-falling window $W$ of
Appendix~\ref{paper:app:windows}. A run passes if the train-only refit
gains at least $0.2$ inside $W$; held-out accuracy gains at least $0.2$
from the start of $W$ to the end of the run and ends at least $0.99$; $W$
holds at least $75\%$ of both the held-out and the refit $10\%$-to-$90\%$
rises; and $\AH$ of the last hidden layer gains at least $0.02$ across $W$.
A setting works if at least 8 of its 10 runs pass. Training fit is not
required, but adding training accuracy at least $0.97$ at every logged
update of $W$ changes no count.

\paragraph{Results.}
All 16 settings work, and 155 of the 160 runs pass
(Table~\ref{paper:tab:deep-fractions},
Figures~\ref{fig:deep-fractions} and~\ref{fig:deep-fractions-app}--\ref{fig:deep-fractions-app3}).
The runs are not independent: each seed index fixes the split and the
initialization in all 16 settings, so at a given prime the training sets
are nested across fractions and the same at both depths, and at a given
depth and prime the initialization is the same across fractions. We
therefore give counts per setting, not an interval for the pooled rate.
Three of the five failures
are at fraction $0.5$ and $p=47$: the whole held-out rise lies inside $W$,
but the runs reach their 250k-update budget with held-out accuracy $0.959$,
$0.971$ and $0.988$. The other two are at fraction $0.8$ and depth 4: the
CE peaks and turns down during the held-out rise, so $W$ closes early and
holds $0.749$ and $0.736$ of it. In every setting the refit crosses $50\%$
held-out accuracy before the network, by a median of 1.2k to 9.5k updates.
The last hidden layer's $\AH$ gains $0.38$ to $0.69$ across $W$ (setting
medians), while the first hidden layer's rises only to about $0.1$.

\paragraph{The training loss.}
The training fraction sets the CE level at which the held-out rise
happens; the values below are setting medians. At fraction $0.8$ the
training CE climbs after the fit to $0.52$--$0.67\log p$ and stays near
that level: from the refit's $20\%$
crossing to the network's $90\%$ crossing it changes by $-2.9\%$ to
$+9.5\%$. In every run at this fraction the refit's $20\%$ crossing comes
before the network's $10\%$ crossing, so the interval contains the
network's whole $10\%$-to-$90\%$ held-out rise.
Training accuracy stays at or near one; at $p=61$, 3 to 18 of the
2{,}976 training pairs are misclassified at some log inside $W$, and at
$p=47$ at most 2 of 1{,}767. At smaller fractions the rise comes earlier
on the CE's climb: at fraction $0.5$ the CE is $0.007$--$0.045\log p$ when
held-out accuracy crosses $50\%$, and from the refit's $20\%$ crossing
to the network's $90\%$ crossing it grows by $136\%$ to $1312\%$, large
relative changes at a small level.

\paragraph{Primary endpoint.}
On the same runs the primary endpoint of Section~\ref{paper:deep}, which
asks for the rise inside a flat band at the CE peak, holds in only 9 of
160 (4 of 10 for each depth at $p=47$ and fraction $0.8$, and 1 of 10 for
depth 4 at $p=61$ and fraction $0.7$). At fractions $0.5$ to $0.7$ the CE
peaks after the held-out rise, so the band holds too little of it; at
fraction $0.8$ and $p=61$ the band holds the rise, but some training pairs
are misclassified inside it. These runs thus extend the post-hoc picture
of Appendix~\ref{paper:app:posthoc} to other primes and training fractions,
under a criterion fixed in advance; they do not show the fitted, flat band
of the primary endpoint at these primes.
\end{additions}

\begin{figure}[tbp]
\centering
\includegraphics[width=5.6in]{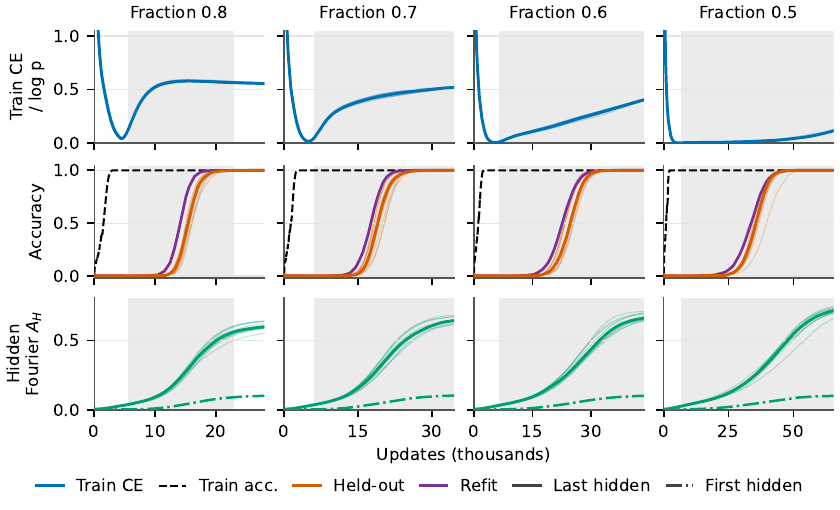}
\caption{\added{\textbf{Depth 3 at $p=61$ across training fractions: every run
(thin) and the median (thick)}, in the format of
Figure~\ref{fig:deep-fractions} (depth 3 at $p=47$); 10 fresh runs per
fraction. Rows as in Figure~\ref{fig:deep-main}; gray: from the median start
to the median end of the post-hoc windows $W$.
Table~\ref{paper:tab:deep-fractions} gives $r$, the passes and the
misclassified training pairs. Each column ends at its shortest run, so every
run contributes at every plotted point; in this setting and in those of
Figures~\ref{fig:deep-fractions-app2} and~\ref{fig:deep-fractions-app3}, the
shortest run ends after the last run's $90\%$ held-out crossing.}}
\label{fig:deep-fractions-app}
\end{figure}

\begin{figure}[tbp]
\centering
\includegraphics[width=5.6in]{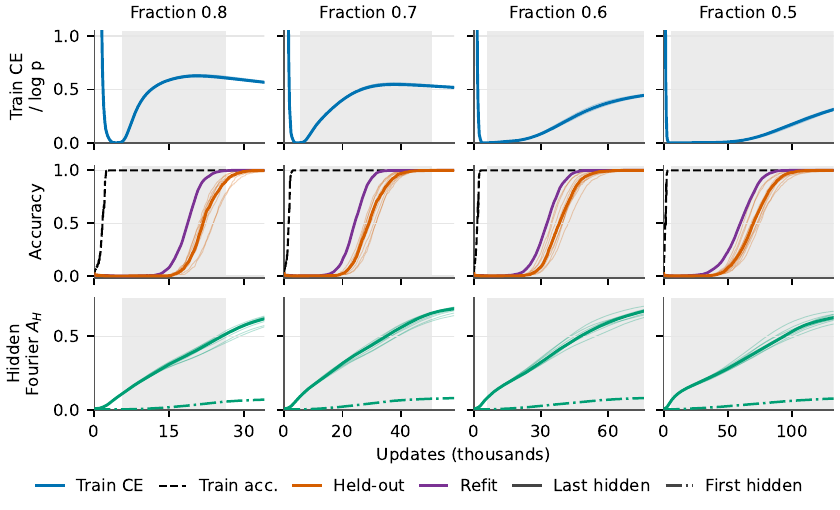}
\caption{\added{\textbf{Depth 4 at $p=47$ across training fractions: every run
(thin) and the median (thick)}, as in Figure~\ref{fig:deep-fractions-app}.}}
\label{fig:deep-fractions-app2}
\end{figure}

\begin{figure}[tbp]
\centering
\includegraphics[width=5.6in]{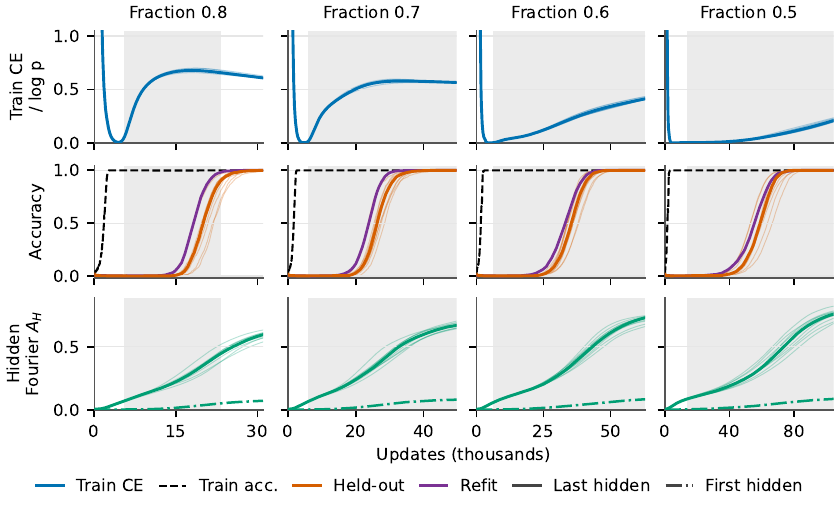}
\caption{\added{\textbf{Depth 4 at $p=61$ across training fractions: every run
(thin) and the median (thick)}, as in Figure~\ref{fig:deep-fractions-app}.}}
\label{fig:deep-fractions-app3}
\end{figure}

\clearpage
\begin{proofrevision}
\subsection{Post-hoc window counts}\label{paper:app:posthoc}
After every run had been seen, we also scored the runs with the post-hoc
non-falling window of Appendix~\ref{paper:app:windows}. The following pass:
\begin{itemize}[nosep,leftmargin=*]
\item with one hidden layer, all 20 headline runs and the 20 runs of the
second $p=23$ cohort; all 20 per-prime runs at each of $p=17,23,29$; 20,
20, 20, 9, 0 and 0 fixed-recipe runs at $p=11,13,17,19,23,29$; all five
replayed trajectories at each of $p=23,31$; and none of the 40
quadratic-activation runs;
\item all 170 runs of the 17 development-selected settings with one to four
hidden layers, all 30 runs at $p=101$ and all 120 primary runs;
\item the fixed-rate arms (60/60), $D=4$ at $p=23$ (20/20) and width 512
(4/4), but none of the 20 normalized-gradient runs, whose CE keeps falling.
\end{itemize}
On the 127-run development sweep that chose 14 of the 17 settings, 115
runs pass. Outside that sweep and the one-hidden-layer cohorts of
Appendix~\ref{paper:app:experiments-ext}, moving the threshold anywhere from $0.1$ to $0.5$ changes
no count. A pass requires only a
rise of $0.2$ in held-out or refit accuracy, not recovery to $90\%$.
Because the window is chosen to contain the largest CE rise, raw CE varies
by orders of magnitude inside it at depths 3 to 5, and at $p\ge47$ nearly
every run misclassifies some training pairs inside it. The rule thus shows
held-out or refit accuracy rising while the training CE does not fall, not
the fitted, flat window of the primary endpoint.
\end{proofrevision}

\subsection{Biases in the hidden layers}\label{paper:app:bias}
\begin{proofrevision}
The baseline networks are bias-free. The dotted curves of
Figures~\ref{fig:onehl-primes} and~\ref{fig:deep-main} come from a small
first experiment with trained hidden biases. We reran the $p=31$ runs of
Figure~\ref{fig:onehl-primes} (one hidden layer, 10 runs) and
Figure~\ref{fig:deep-main} (depth 2, \added{40 runs}; depth 4, 40 runs, each with
its preliminary run that selects $r$) with a bias in every hidden layer,
$h^{(\ell)}=\ReLU(\mat{W}^{(\ell)}h^{(\ell-1)}+b^{(\ell)})$. The bias starts
at zero and is one more column of $\mat{W}^{(\ell)}$: the matrix
$[\mat{W}^{(\ell)}\mid b^{(\ell)}]$ receives the supported Polar step and the
decay of~\eqref{paper:optimizer}. The head stays bias-free. Each run starts
from the same network as its bias-free counterpart and keeps its split,
fixed hyperparameters, rate-selection rule and stopping rule. At depth 4
the preliminary run, which now has biases, reselects $r$, and the selected
value differs from the bias-free one in 37 of the 40 runs. Depths 3 and 5
were not rerun. With
one-hot inputs a first-layer bias adds no functions, because it can be
folded into the first $p$ columns of $\mat{W}^{(1)}$; it changes only the
update. In Figures~\ref{fig:onehl-primes} and~\ref{fig:deep-main} the dotted
medians stay close to the bias-free ones; the most visible difference is a
lower training loss during and after the held-out rise. For one hidden
layer, the relative decrease of $\bar L$ in
Figure~\ref{fig:bias-ablation} (left) has median $5.3\%$ with the bias and
$2.6\%$ without, and is larger in all 10 pairs. With the bias, the captured
share $q(I_C)$ has median $0.62$ (range $0.40$--$0.74$), below the $0.75$
that the primary endpoint requires; without it, the median is $1$.
\added{At depth 2 the primary endpoint holds in 18 of the 40 runs with
biases and in 34 without: 16 seeds lose it with the biases and none gains it (exact McNemar
$p=3.1\cdot10^{-5}$), and 20 of the 22 misses with biases fail the share
threshold.}

With one hidden layer we also tried three other ways of updating the bias,
on the same 10 initializations and splits (Figure~\ref{fig:bias-ablation}):
(i)~the joint step until the knee, then the bias frozen (no step and no
decay); (ii)~the joint step until the knee, then the bias only decays, which
is the bias-free optimizer from the knee on, applied to the network with the
bias folded into $\mat{W}$; (iii)~from the start, $\mat{W}$ takes the
bias-free step of~\eqref{paper:optimizer} and the bias a step of its own,
$b_{t+1}=(1-\eta_W\lambda_W)b_t+\eta_W g_t/(\norm{g_t}_2\sqrt{\rank\mat{F}_t})$
with $g_t=-\nabla_bL$, of about the length of the bias part of the joint
step. In
these runs the training loss falls about as much as without a bias in (i)
and (ii), and about as much as with the joint step in (iii). These
experiments are preliminary: they use one prime and one width per depth,
and we do not yet draw conclusions from them.
\end{proofrevision}

\begin{figure}[htbp]
\centering
\includegraphics[width=\linewidth]{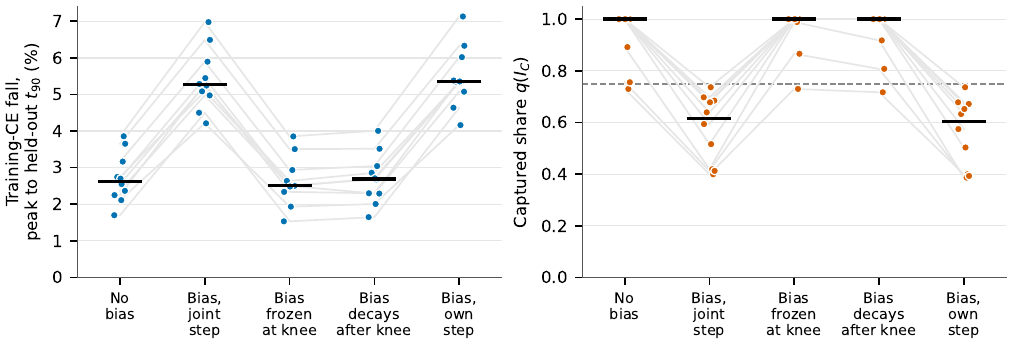}
\caption{\proofrevtext{\textbf{Updating a hidden bias in different ways}
(one hidden layer, $p=31$, 10 initializations and splits; gray lines join
the same run, bars are medians). Left: the relative decrease of $\bar L$
from its largest post-knee value to the first update after which held-out
accuracy stays at least $0.9$. Right: the captured share $q(I_C)$ of the
primary endpoint, which requires at least $0.75$ (dashed). Columns: no
bias; the bias with the joint $[\mat{W}\mid b]$ step; the joint step until
the knee, then the bias frozen or only decaying; the bias updated by its
own step throughout.}}
\label{fig:bias-ablation}
\end{figure}

\clearpage
\section{A one-block transformer}\label{paper:app:tf}
\begin{additions}
This appendix supports Section~\ref{paper:transformer}. It reports every
run of the settings below. The settings were chosen on earlier development
runs with other seeds, the gain on five pilot runs of one seed.
Initialization and split seeds are $7100+i$ and $8100+i$, with
$i=11$--$20$ for the main runs at $p=113$, the runs without the gain,
the settings of
Figures~\ref{fig:tf-primes}--\ref{fig:tf-readouts} and two blocks at
$\phi=0.2$, $i=11$--$13$ for the recipe sweep and the two-block tests
(Table~\ref{paper:tab:tf-blocks-late} and the fractions $0.175$ and
$0.225$), and $i=11$--$15$ elsewhere, so
settings at the same prime and with the same number of blocks share their
initial weights, whatever their gain, rates or split (one- and two-block
runs share only the embeddings and the first block), and settings at the
same prime share their splits when the training fraction agrees. These runs test settings found on
earlier runs. The pass thresholds were fixed before all of them except the
Polar runs without the gain, which had finished before and were scored
afterwards (Appendix~\ref{paper:app:tf-window}). Seeds $i=16$--$20$ of the
settings across primes and of two blocks were run last, after the settings
had been chosen on seeds $i=11$--$15$.

\subsection{Model, updates and readouts}\label{paper:app:tf-setup}
\paragraph{Model.}
Token and position embeddings $\mat{W}_E\in\R^{(p+1)\times d}$ and
$\mat{W}_{\rm pos}\in\R^{3\times d}$, $d=128$, feed one block: causal
attention with four heads ($\mat{W}_Q,\mat{W}_K,\mat{W}_V,\mat{W}_O\in\R^{d\times d}$)
and a ReLU MLP ($\mat{W}_{\rm in}\in\R^{512\times d}$,
$\mat{W}_{\rm out}\in\R^{d\times512}$), each added to the residual stream,
with no LayerNorm and no biases. The token $p$ stands for ``$=$''. Only the
last position is computed, since the others cannot reach its logits. Its
final residual $r\in\R^d$ gives
$\hat r=r/(\frac1d\sum_i r_i^2+10^{-12})^{1/2}$ and the logits
$f=g\,\mat{V}\hat r$, with the unembedding $\mat{V}\in\R^{p\times d}$
($\mat{W}_U^\top$ in~\cite{nanda2023}). Initial entries are independent
Gaussians with standard deviation $1/\sqrt d$ for $\mat{W}_E$,
$\mat{W}_{\rm pos}$, the attention matrices and $\mat{V}$, $\sqrt{2/d}$ for
$\mat{W}_{\rm in}$ and $1/\sqrt{512}$ for $\mat{W}_{\rm out}$; the gain
starts at $g_0$.

\paragraph{Updates.}
Training uses the full-batch training CE in float64. Every matrix except
$\mat{V}$ takes the $\mat{W}$ step of~\eqref{paper:optimizer}, with the
supported Polar factor computed by an SVD that keeps singular values above
$10^{-7}$ times the largest; $\mat{V}$ takes the $\mat{V}$ step. The gain
takes $g\leftarrow g+\eta_g(-\operatorname{sign}(\partial L/\partial g)-\lambda_g g)$.
Once the training set is fitted, $\partial L/\partial g<0$, so with
$\lambda_g=1$ the gain relaxes toward $1$ with time constant
$\tau=1/(\eta_g\lambda_g)$ ($\eta_g=2\cdot10^{-4}$ and $\tau=5{,}000$
unless stated). Runs stop after 30k updates, or once held-out accuracy
on all held-out pairs has been at least $0.99$ at every log from the start
$t$ of an unbroken stretch to $\max(1.5t,t+5000)$.

\paragraph{Splits and held-out sets.}
A random split keeps $\lfloor\phi p^2\rfloor$ of the $p^2$ pairs for
training, with $\phi$ the training fraction (2{,}553 at $p=113$ and
$\phi=0.2$). About a fraction $\phi$ of the held-out pairs have their
mirror $(b,a)$ in the training set. They can be answered by memorizing the
mirror, which the shared token embedding makes easy: at the first full fit
at $p=113$ they are $28\%$ correct (median over runs), against $0.06\%$ for
the other held-out pairs. We therefore report every quantity on all
held-out pairs and on the \emph{no-mirror} pairs, whose mirror is also held
out (8{,}133--8{,}193 of the 10{,}216 held-out pairs at $p=113$). A
\emph{mirror-free} split instead adds unordered pairs $\{a,b\}$ in a random
order, each with both orders, until it holds $\lfloor\phi p^2\rfloor$
ordered pairs, so no held-out pair has its mirror in the training set.

\paragraph{Observers.}
Accuracies, margins and misclassified training pairs are logged every 20
updates, $\AH$ and the refits every 100, and the probe of
Appendix~\ref{paper:app:tf-readouts} every 200. $\AH$ is~\eqref{paper:AH}
on the full-grid table of $\hat r$, taken before the gain. This table is
the residual stream after the MLP's output map $\mat{W}_{\rm out}$, one
linear map after its neurons, so it corresponds to the logits of the
networks of Section~\ref{paper:onehl} rather than to their hidden layer; the transformer's MLP neurons end the main runs with
$\AH$ of only $0.03$--$0.04$. The refit is~\eqref{paper:expridge} on the same
table, with $\tau_t=10^{-3}n$.

\paragraph{Fourier-filtered refit.}
With $\omega_k=2\pi k/p$ and $K=(p-1)/2$, the target $T_x=e_{y}-\one_p/p$,
$y=y^\star(x)$, has entries
$\frac2p\sum_{k=1}^{K}[\cos(\omega_k y)\cos(\omega_k c)+\sin(\omega_k y)\sin(\omega_k c)]$
for $c\in\F_p$. The ridge solution is linear in the targets, so the refit's
logits are
\begin{equation}\label{paper:fourier-refit}
 \widehat f_c(x)=\frac2p\sum_{k=1}^{K}
   \bigl[C_k(x)\cos(\omega_k c)+S_k(x)\sin(\omega_k c)\bigr],
\end{equation}
where $C_k$ and $S_k$ are the ridge predictions, from the same table and
penalty, of $\cos(\omega_k y)$ and $\sin(\omega_k y)$. The refit thus
weights every frequency equally. The Fourier-filtered refit keeps in
\eqref{paper:fourier-refit} only the 8 frequencies with the largest
training energy
$E_k=\bigl\|\frac1n\sum_{x\in\Train}(\hat r(x)-\bar r_\Train)\,e^{-\mathrm{i}\omega_k y^\star(x)}\bigr\|_2^2$,
where $\bar r_\Train$ is the training mean of $\hat r$, so it uses training
data only; with all $K$ frequencies kept it is the refit itself. The number
8 was fixed at $p=113$ on development runs, before the runs reported
here.

\subsection{Window and pass criterion}\label{paper:app:tf-window}
Let $\tilde L_t$ be the training CE averaged over the 201 updates centered
at $t$ (the series is padded with its first and last values), and
$T_{\rm fit}$ the first log at which every training pair is correctly
classified. For each start $a\in\{T_{\rm fit},T_{\rm fit}+20,\ldots\}$, let
$b(a)$ be the last update before $\tilde L_t$ first falls below
$0.97\max_{a\le s\le t}\tilde L_s$, or the end of the run. The window is
the $[a,b(a)]$ with the largest captured share $q$ of the held-out
$10\%$-to-$90\%$ rise, with $q$ as in Appendix~\ref{paper:app:windows} and
the earliest start on ties. Unlike the post-hoc window of
Appendix~\ref{paper:app:windows}, it has candidate starts at or after the first
fit rather than after a head-norm knee, uses a 201-update average and is placed using held-out
accuracy, so it is post hoc. It is computed separately for each held-out
definition. A run passes on a definition if inside its window $q\ge0.75$,
the Fourier-filtered refit's held-out accuracy gains at least $0.2$,
$\Delta\AH\ge0.02$ for $\hat r$, and no training pair is misclassified at
any log (\emph{strict}; \emph{loose} allows up to 10). A run passes if it
passes on both definitions and its final accuracy on all held-out pairs is
at least $0.99$; unless stated, counts are strict. These thresholds were
fixed before the runs at $p=113$ with the gain, on the no-mirror
definition; the requirement on both definitions was added after those runs
and before every other run reported here, except the Polar runs without the
gain, which predate both and were scored afterwards.

\subsection{The gain and the recipe at \texorpdfstring{$p=113$}{p=113}}\label{paper:app:tf-gain}
Table~\ref{paper:tab:tf-contrast} compares the runs with and without the
gain. Without the gain, the runs
fit the training set later (median 1{,}190 updates) and their CE falls
during the held-out rise, so the window holds only $0.35$ to $0.51$ of it
on all held-out pairs ($0.43$ to $0.61$ on the no-mirror pairs).
A gain held fixed behaves like no gain. The relaxing gain lifts the CE
during the rise: in the 10 main runs, the averaged training CE never falls
more than $0.05\%$ below its running maximum across the no-mirror held-out
rise.

\begin{table}[tbp]
\centering\footnotesize\setlength{\tabcolsep}{3pt}\addedcolor
\caption{\added{One block at $p=113$, $\phi=0.2$: the gain
(seeds $i=11$--$20$; the 3-run rows use $i=11$--$13$). Fit: first full fit
(median; updates). $\Delta\tilde L$: relative change of the averaged
training CE from the no-mirror held-out $0.1$ crossing to the $0.9$
crossing (median and range). $q$: share of the all-pairs rise inside the
all-pairs window. Final: held-out accuracy on all pairs.}}
\label{paper:tab:tf-contrast}
\begin{tabular}{@{}lrrrrrrl@{}}
\toprule
Setting & Runs & Pass & Fit & $\Delta\tilde L$ & $q$ & Final & Failing condition\\
\midrule
Polar, gain $2\to1$ & 10 & 10 & 200 & $+33\%$ ($+18$, $+55$) & 0.93--1.00 & 0.999--1.000 & --\\
Polar, no gain & 10 & 0 & 1{,}190 & $-8\%$ ($-10$, $-7$) & 0.35--0.51 & 0.999--1.000 & share (10), refit (9)\\
Polar, fixed $g=1.5$ & 3 & 0 & 240 & $-17\%$ ($-18$, $-16$) & 0.20--0.26 & 0.999--1.000 & share, refit (3)\\
Polar, fixed $g=2$ & 3 & 0 & 200 & $-27\%$ ($-29$, $-24$) & 0.13--0.17 & 1.000 & share, refit (3)\\
Polar, fixed $g=3$ & 3 & 0 & 160 & $-44\%$ ($-49$, $-42$) & 0.07--0.10 & 1.000 & share, refit (3)\\
\bottomrule
\end{tabular}
\end{table}

\paragraph{The recipe.}
On seeds $i=11$--$13$ we varied the initial gain $g_0$ and the time
constant $\tau$ (Table~\ref{paper:tab:tf-recipe}). Every run with
$g_0\in\{2,3\}$ and $\tau$ from 2{,}500 to 10{,}000 passes (18/18), with a
median CE rise of $24\%$ to $104\%$ per setting across the held-out rise;
with $\tau=20{,}000$ the CE is about flat and 6 of 9 runs pass. Changing one
setting at a time around the main runs ($g_0=2$, $\tau=5{,}000$): head
decay $\lambda_V=0.25$ or $0.3$, and feature rate $\eta_W=0.05$ or $0.2$,
pass in 3/3 each ($\eta_W=0.2$ with a CE rise of only $10\%$); feature decay
$\lambda_W=0.005$ passes in 1/3 (two runs end below $0.99$) and
$\lambda_W=0.02$ in 2/3; training fraction $0.3$ in 0/3, because the
held-out rise then comes before the CE minimum. The recipe is thus a gain
of 2 to 3 that relaxes to 1 on about the time scale of the held-out rise,
supported-Polar feature steps, and a training fraction that puts the rise
after the CE minimum.

\begin{table}[tbp]
\centering\footnotesize\addedcolor
\caption{\added{The gain at $p=113$, $\phi=0.2$ (seeds $i=11$--$13$): passes
of 3 and the median change of the averaged training CE across the
no-mirror held-out rise. The $g_0=2$, $\tau=5{,}000$ entry uses the first
three main runs. With $g_0=1.5$ and $\tau=2{,}500$, all three runs pass
the loose criterion.}}
\label{paper:tab:tf-recipe}
\begin{tabular}{@{}lrrrrr@{}}
\toprule
& $\tau=2{,}500$ & $5{,}000$ & $10{,}000$ & $20{,}000$ & fixed\\
\midrule
$g_0=1.5$ & 1/3, $+10\%$ & 3/3, $+13\%$ & 2/3, $+6\%$ & 2/3, $-4\%$ & 0/3, $-17\%$\\
$g_0=2$ & 3/3, $+28\%$ & 3/3, $+38\%$ & 3/3, $+24\%$ & 2/3, $-1\%$ & 0/3, $-27\%$\\
$g_0=3$ & 3/3, $+84\%$ & 3/3, $+104\%$ & 3/3, $+58\%$ & 2/3, $+9\%$ & 0/3, $-44\%$\\
\bottomrule
\end{tabular}
\end{table}

\paragraph{Two blocks.}
Two blocks pass in none of 25 runs at four settings. At $\phi=0.2$
($\lambda_V=0.275$, $\tau=10{,}000$; seeds $i=11$--$20$,
Figures~\ref{fig:tf-contrast} and~\ref{fig:tf-blocks}) they behave like one block during the held-out
rise: the no-mirror held-out accuracy goes from $10\%$ at 4.5k--6.1k updates
to $90\%$ at 8.1k--10.8k, the CE rises by $9\%$ to $29\%$ across that rise,
no training pair is misclassified at any log, the window holds the whole no-mirror
rise, and the Fourier-filtered refit gains $0.64$--$0.98$ inside it.
Held-out accuracy on all pairs reaches $0.99$ at 14.9k--18.2k updates. At
17.8k--22.0k updates, 1.4k to 5.7k updates later, training pairs begin to
flip: 45 to 438 are misclassified at some log inside the window, which runs
to the end of the run, held-out accuracy falls back, and 6 of the 10 runs
end below $0.99$. At $\phi=0.25$ ($\lambda_V=0.4$, $\tau=10{,}000$) and $0.3$
($\lambda_V=0.275$ or $0.5$, $\tau=5{,}000$; seeds $i=11$--$15$) the rise
comes before the CE minimum: the window holds at most $0.45$ of the
no-mirror rise in 14 of the 15 runs, and the fifteenth has no window.

On seeds $i=11$--$13$ at $\phi=0.2$ we then changed one rate at update
12k, after the held-out rise and before the late flips
(Table~\ref{paper:tab:tf-blocks-late}). The three changed runs of a seed
share their first 12k updates; through floating-point differences between
thread counts they part from the runs of the same seed above after 2k--3k
updates. Freezing the gain at $1.30$ ends the CE's rise: the CE falls
slowly from 12k on, so the window closes at 13.5k--13.7k updates, and pairs
still flip, after the window, so the two runs that pass lose their fit as
well. A smaller head step delays the flips to 21.6k--23.0k. A smaller
feature step keeps every training pair correct at every log to the end in all three
runs; its CE dips by about $8\%$ just after the change and then rises
again, so the window closes at the change, and held-out accuracy on all
pairs reaches $0.99$ later (20.1k and 23.8k updates; the third run ends at
$0.986$). Training fractions $0.175$ and $0.225$ (3 runs each, base rates)
pass in none: at $0.175$ the no-mirror rise ends only at 14.3k--16.2k,
pairs begin to flip at 17.7k--22.4k and held-out accuracy ends at
$0.82$--$0.91$; at $0.225$ the
rise starts (2.6k--3.3k) before the CE minimum (4.9k--5.2k) and the CE falls
by $13\%$ to $17\%$ across it. Lowering the feature rate after the held-out
rise is thus the one change that kept the fit; it has been tried on three
seeds only.

\begin{table}[tbp]
\centering\footnotesize\setlength{\tabcolsep}{4pt}\addedcolor
\caption{\added{Two blocks at $\phi=0.2$ ($\lambda_V=0.275$, $\tau=10{,}000$):
one rate changed at update 12k (seeds $i=11$--$13$; no change: $i=11$--$20$).
Window ends: end of the all-pairs window, in updates. First flip: first log
after 12k with a misclassified training pair; Most: the most misclassified
training pairs at a log after 12k. Final: held-out accuracy on all
pairs.}}
\label{paper:tab:tf-blocks-late}
\begin{tabular}{@{}lrrrrr@{}}
\toprule
Change at 12k & Pass & Window ends & First flip & Most & Final\\
\midrule
None & 0/10 & 30k (end) & 17.8k--22.0k & 45--438 & 0.941--0.994\\
Gain frozen ($\eta_g=0$, $g=1.30$) & 2/3 & 13.5k--13.7k & 18.8k--19.7k & 57--217 & 0.983--0.997\\
Head step $\eta_V$: $0.03\to0.01$ & 0/3 & 30k (end) & 21.6k--23.0k & 59--395 & 0.950--0.970\\
Feature step $\eta_W$: $0.1\to0.03$ & 2/3 & 12.0k & none & 0 & 0.986--0.997\\
\bottomrule
\end{tabular}
\end{table}

\subsection{Other primes and mirror pairs}\label{paper:app:tf-primes}
At other primes we keep the recipe ($g_0=2$, $\tau=5{,}000$ and the rates
of the main runs) and choose the head decay $\lambda_V$ and the training
fraction $\phi$ (Table~\ref{paper:tab:tf-primes}). In a first set of
runs ($\phi=0.3$ at $p=61$ and $0.2$ at $p=89$ and $149$), the held-out
rise was too slow at $p=61$ and $89$ (most runs end below $0.99$ after 30k
updates) and came before the CE minimum at $p=149$. A second set of runs set $\phi$ for each prime by a rule, fixed
before those runs, for when the rise starts as a function of
$n/(p\ln p)$; it predicted the median start of each new setting within a
factor $0.63$--$1.34$. The random-split settings chosen on seeds
$i=11$--$15$, the most strict passes at each prime with ties broken by loose
passes and then by the smaller fraction, pass in 3/5, 5/5 and 3/5 runs at $p=61$, $89$ and $149$,
and in 2/5, 5/5 and 4/5 on the five further seeds, so in 5/10, 10/10 and
7/10 (10/10 at $p=113$). At $p=149$ all ten pass the loose criterion (at
most 7 of the 3{,}552 training pairs are misclassified inside the window).
At $p=61$ all ten pass on the no-mirror pairs, and the five failures come
from the mirror pairs; with the mirror-free split below, which
Figure~\ref{fig:tf-primes} shows at $p=61$, 9 of 10 pass.

\begin{table}[tbp]
\centering\footnotesize\setlength{\tabcolsep}{4pt}\addedcolor
\caption{\added{One block across primes (gain $2\to1$, $\tau=5{,}000$; seeds
$i=11$--$15$, and $i=11$--$20$ in the rows with 10 runs). Pass: strict, on both
definitions; Loose: at most 10 misclassified training pairs; No-mirror:
strict on the no-mirror definition alone, with the final all-pairs
accuracy at least $0.99$. $\Delta\tilde L$: median change of the averaged
training CE across the no-mirror rise, over the runs with a full rise.
Final: held-out accuracy on all pairs. $^\dagger$Mirror-free split.}}
\label{paper:tab:tf-primes}
\begin{tabular}{@{}rrrrrrrrr@{}}
\toprule
$p$ & $\phi$ & $\lambda_V$ & Training pairs & Pass & Loose & No-mirror & $\Delta\tilde L$ & Final\\
\midrule
61 & 0.30 & 0.45 & 1{,}116 & 1/5 & 1/5 & 1/5 & $+87\%$ & 0.979--0.994\\
61 & 0.30 & 0.60 & 1{,}116 & 1/5 & 1/5 & 1/5 & $+58\%$ & 0.964--0.991\\
61 & 0.33 & 0.45 & 1{,}227 & 5/10 & 5/10 & 10/10 & $+106\%$ & 0.994--0.999\\
61 & 0.35 & 0.45 & 1{,}302 & 0/5 & 0/5 & 5/5 & $+87\%$ & 0.998--0.999\\
61 & $0.55^\dagger$ & 0.45 & 2{,}046 & 9/10 & 9/10 & 9/10 & $+161\%$ & 0.9899--0.997\\
\midrule
89 & 0.20 & 0.33 & 1{,}584 & 1/5 & 1/5 & 1/5 & $+33\%$ & 0.39--0.996\\
89 & 0.20 & 0.40 & 1{,}584 & 0/5 & 0/5 & 0/5 & $+22\%$ & 0.55--0.98\\
89 & 0.25 & 0.33 & 1{,}980 & 10/10 & 10/10 & 10/10 & $+35\%$ & 1.000\\
89 & 0.275 & 0.33 & 2{,}178 & 2/5 & 2/5 & 5/5 & $+4\%$ & 1.000\\
\midrule
113 & 0.20 & 0.275 & 2{,}553 & 10/10 & 10/10 & 10/10 & $+33\%$ & 0.999--1.000\\
\midrule
149 & 0.15 & 0.23 & 3{,}330 & 0/5 & 4/5 & 0/5 & $+28\%$ & 0.12--0.999\\
149 & 0.16 & 0.23 & 3{,}552 & 7/10 & 10/10 & 7/10 & $+19\%$ & 0.992--1.000\\
149 & 0.20 & 0.23 & 4{,}440 & 0/5 & 0/5 & 0/5 & $-23\%$ & 0.999--1.000\\
149 & 0.20 & 0.275 & 4{,}440 & 0/5 & 0/5 & 0/5 & $-18\%$ & 0.996--1.000\\
\bottomrule
\end{tabular}
\end{table}

\paragraph{Mirror pairs.}
A window that opens when the all-pairs held-out accuracy is $a_0$ can hold
at most $(0.9-\max(a_0,0.1))/0.8$ of the all-pairs rise, so the share falls
below $0.75$ whenever $a_0>0.3$. At $p=61$ the training fraction that lets
the rise finish within 30k updates also puts the mirrors of about a third
of the held-out pairs in the training set. These held-out pairs are learned
early: at $\phi=0.33$, $90\%$ of them are correct by update 240--520 in 8 of
the 10 runs (by 711 and 1{,}615 in the other two), and the all-pairs window
opens at 560--700. At $\phi=0.33$ the window therefore opens at
$a_0=0.27$--$0.34$, and the all-pairs share sits near the threshold:
$0.751$--$0.787$ in the five passing runs, $0.694$--$0.744$ in the five
failing ones (Figure~\ref{fig:tf-mirror}). A mirror-free split at
$\phi=0.55$ has about as many distinct unordered training pairs
(1{,}038--1{,}043) as the random split at $0.33$ (1{,}021--1{,}039). With
it, $a_0$ is $0.007$--$0.036$, the share is $1.00$, and 9 of 10 runs pass;
the tenth ends at held-out accuracy $0.9899$. The CE minimum (median update
724, against 741 with the random split) and the $10\%$ crossing of the
no-mirror held-out accuracy (1.9k, against 1.8k) are close to those of the
random split, but the rise is slower: the $50\%$ and $90\%$ crossings come
at 4.3k and 7.6k, against 3.5k and 5.4k. At $\phi=0.33$ a mirror-free split,
with 622 distinct pairs, did not generalize within 15k updates (one
development run). The number of distinct unordered training pairs, rather
than of ordered ones, thus appears to set when the held-out rise starts,
and in a random split the held-out pairs whose mirror is a training pair
behave almost like training pairs.

\subsection{Readouts}\label{paper:app:tf-readouts}
Table~\ref{paper:tab:tf-readouts} and Figure~\ref{fig:tf-readouts} compare
the readouts of $\hat r$ with the network on the no-mirror pairs. The
refit trails the network in all 50 runs of the appendix settings, at the
$50\%$ and at the $90\%$ crossing. The Fourier-filtered refit leads at the
$50\%$ crossing in 49 of 50 runs and at the $90\%$ crossing in all 30 runs
at $p\ge89$, but at $p=61$ it levels off near $0.96$ (median final $0.957$
with the random split and $0.962$ with the mirror-free split) and crosses
$90\%$ after the network. At the network's $90\%$ crossing, 11 to 19 of the
30 frequencies at $p=61$ carry held-out signal (held-out $R^2$ of $(C_k,S_k)$
above $0.05$, a diagnostic that uses held-out labels), against 9 to 15 of
44 to 74 at the larger primes. The top 8 keep 7 or 8 of them; the top 16
keep 11 to 16 at $p=61$ and all of them in 29 of the 30 runs at the larger
primes, and the top-16 refit leads at both crossings in 49 of the 50
runs. The choice $K=16$ was made after
seeing $p=61$, so this readout is post hoc. A train-only multinomial
logistic probe of $\hat r$ tracks the network (median leads of 0.02k to
0.19k updates). Its logits are $\mat{B}\hat r$ with $\mat{B}\in\R^{p\times d}$,
without bias or feature standardization; every 200 updates it is refitted
from $\mat{B}=0$ to minimize the mean training CE plus
$5\cdot10^{-5}\norm{\mat{B}}_F^2$ by at most 100 L-BFGS iterations
(strong-Wolfe line search, float64), and it is scored like the network, by
the held-out accuracy of its argmax, with crossings interpolated between
these logs. Whether a train-only readout leads the network thus
depends on which frequencies it sums: the refit weights all frequencies
equally, while at the network's $50\%$ crossing only 6 to 13 of them carry
held-out signal in any run.

\paragraph{Why the same refit leads with one hidden layer and trails here.}
By~\eqref{paper:fourier-refit} the refit's logits are an equal-weight sum
over all $K$ frequencies, so it tests whether the features reproduce the
one-hot label at every frequency. In replays of three runs of
Figure~\ref{fig:onehl-primes} ($p=31$, 61 and 101, one seed each, with the
parameters saved every 200 updates), every frequency of the hidden layer
carries held-out signal by update 6.4k, 5.2k and 5.6k, while the network's
held-out accuracy is still below $0.05$; the network crosses $50\%$ at 8.2k,
6.6k and 6.3k. There the full sum is the right readout: keeping only the top
8 frequencies delays it (leads at the network's $50\%$ crossing of 2.4k,
1.3k and 0.7k updates, against 2.8k, 1.9k and 1.4k with all frequencies). In
the transformer, $\hat r$ carries 6 to 13 frequencies at the network's
$50\%$ crossing (medians 7 to 9 per setting of
Table~\ref{paper:tab:tf-readouts}), and the frequencies outside the top 8
have negative held-out $R^2$ (medians of $-0.27$ to $-0.79$ per setting).
The ridge fits them on the training pairs, where the features encode the
labels: in replays of one run of Figure~\ref{fig:tf-contrast} and one of
Figure~\ref{fig:tf-blocks}, at the network's $50\%$ crossing, their training
$R^2$ is $0.38$ and $0.53$ and their held-out $R^2$ is $-0.42$ and $-0.49$.
The network's own logits, restricted to the same top 8 frequencies without
any fitting, lead its held-out accuracy by about as much as the
Fourier-filtered refit (Table~\ref{paper:tab:tf-readouts}), so the
restriction removes memorized terms rather than adding signal that the
network lacks. This matches what is known of the two kinds of solution: the
sufficiently wide bias-free one-hidden-layer network with quadratic
activation and maximum $L_{2,3}$-margin on the full grid uses every frequency~\cite{morwani2024margin}, and the one-layer transformer
of~\cite{nanda2023} uses five.
\end{additions}

\begin{table}[tbp]
\centering\footnotesize\setlength{\tabcolsep}{3pt}\addedcolor
\caption{\added{Readout leads over the network on the no-mirror pairs: the
network's crossing minus the readout's, in
thousands of updates (positive: the readout crosses first), at the
network's $0.5$ and $0.9$ crossings; medians, with the number of leading
runs in parentheses. Settings of Figure~\ref{fig:tf-primes}, and $p=61$
with the random split at $\phi=0.33$; $^\dagger$mirror-free split,
$\phi=0.55$. Network's logits, top 8: the network's own centred logits
projected, in class space, onto the 8 frequencies of the Fourier refit,
without fitting.}}
\label{paper:tab:tf-readouts}
\begin{tabular}{@{}llrrrrr@{}}
\toprule
Readout & Crossing & $p=61^\dagger$ & $p=61$ & $p=89$ & $p=113$ & $p=149$\\
\midrule
Refit, all frequencies & 0.5 & $-0.70$ (0) & $-0.68$ (0) & $-0.78$ (0) & $-1.31$ (0) & $-1.51$ (0)\\
 & 0.9 & $-0.68$ (0) & $-0.72$ (0) & $-1.26$ (0) & $-2.15$ (0) & $-1.86$ (0)\\
Fourier refit, top 8 & 0.5 & $+0.78$ (9) & $+0.80$ (10) & $+0.56$ (10) & $+0.86$ (10) & $+1.18$ (10)\\
 & 0.9 & $-2.46$ (0) & $-1.16$ (0) & $+0.45$ (10) & $+0.90$ (10) & $+1.21$ (10)\\
Network's logits, top 8 & 0.5 & $+0.79$ (9) & $+0.76$ (10) & $+0.52$ (10) & $+0.81$ (10) & $+1.06$ (10)\\
 & 0.9 & $-3.99$ (0) & $-1.70$ (0) & $+0.34$ (10) & $+0.74$ (10) & $+1.04$ (10)\\
Fourier refit, top 16 & 0.5 & $+1.16$ (10) & $+0.73$ (10) & $+0.41$ (10) & $+0.50$ (10) & $+0.65$ (10)\\
 & 0.9 & $+0.87$ (9) & $+0.61$ (10) & $+0.34$ (10) & $+0.47$ (10) & $+0.73$ (10)\\
Logistic probe & 0.5 & $+0.05$ (8) & $+0.02$ (8) & $+0.02$ (9) & $+0.03$ (7) & $+0.06$ (9)\\
 & 0.9 & $+0.18$ (9) & $+0.08$ (9) & $+0.04$ (9) & $+0.04$ (10) & $+0.19$ (10)\\
\midrule
Runs & & 10 & 10 & 10 & 10 & 10\\
\bottomrule
\end{tabular}
\end{table}

\begin{figure}[tbp]
\centering
\includegraphics[width=\textwidth]{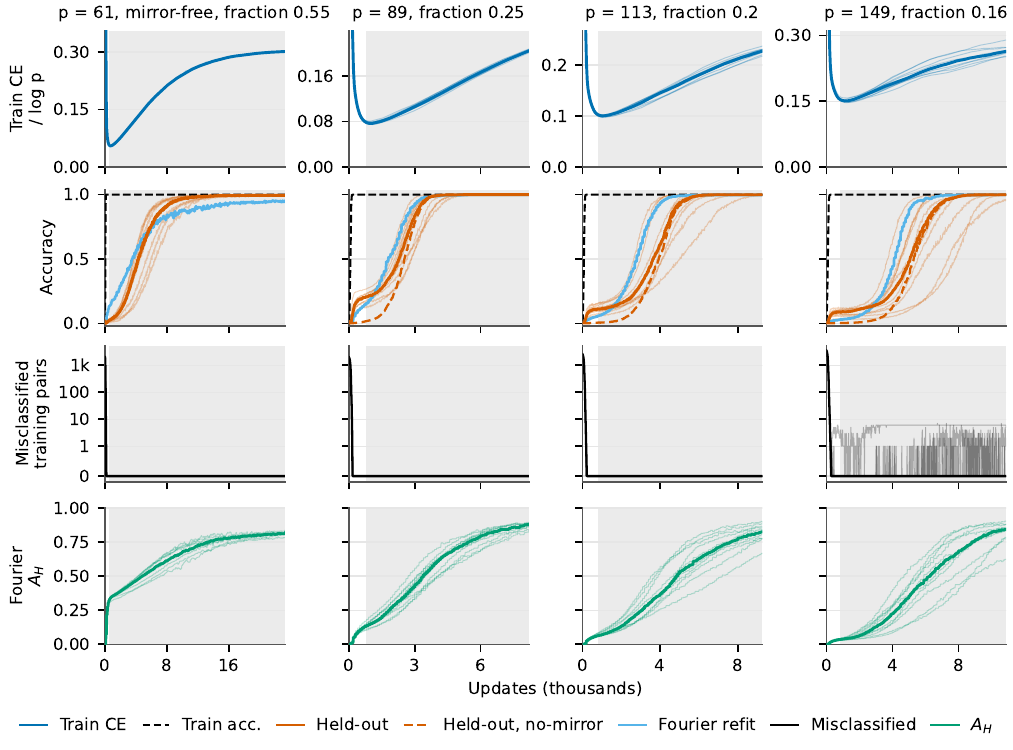}
\caption{\added{\textbf{One block across primes: every run (thin) and the
median (thick)} (gain $2\to1$, $\tau=5{,}000$; $\lambda_V=0.45$, $0.33$,
$0.275$ and $0.23$; 10 runs per prime, seeds $i=11$--$20$). At $p=61$ the
training set is the mirror-free split of Appendix~\ref{paper:app:tf-primes}
(fraction $0.55$), so no held-out pair has its mirror in it and the two
held-out accuracies coincide. At each $p$ the training fraction puts the
start of the held-out rise after the CE minimum. Rows as in
Figure~\ref{fig:tf-contrast}, plus the misclassified training pairs
(symmetric log scale): the pass rule requires none inside the window, and
the accuracy row cannot show a few pairs out of thousands. Gray: from the
median start to the median end of the windows on all held-out pairs. Each
column ends at its shortest run. Strict passes: 9/10, 10/10, 10/10 and
7/10; the run that fails at $p=61$ ends at held-out accuracy $0.9899$, and at
$p=149$ all 10 runs pass the loose criterion. With a random split at
fraction $0.33$, $p=61$ passes in 5/10, and all five failures come from the
mirror pairs (Figure~\ref{fig:tf-mirror}).}}
\label{fig:tf-primes}
\end{figure}

\begin{figure}[tbp]
\centering
\includegraphics[width=\textwidth]{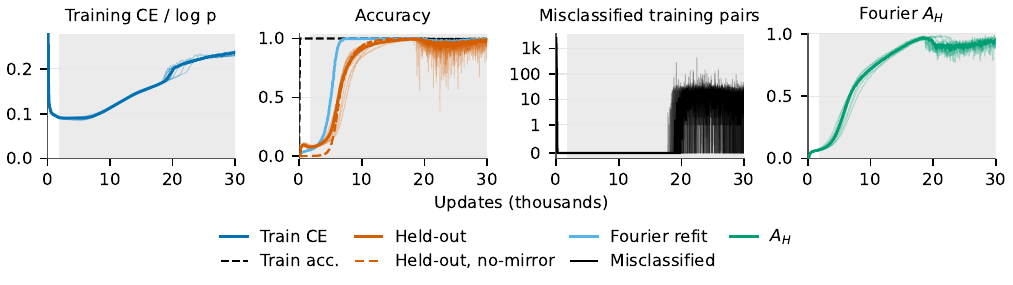}
\caption{\added{\textbf{Two blocks: every run (thin) and the median (thick)}
(training fraction $0.2$, $\lambda_V=0.275$, gain $2\to1$ with
$\tau=10{,}000$; 10 runs, seeds $i=11$--$20$: the two-block runs of
Figure~\ref{fig:tf-contrast}). Panels: the rows of
Figure~\ref{fig:tf-primes}. No run passes: the CE rises during the held-out
rise and no training pair is misclassified at any 20-update log inside the window until
17.8k--22.0k updates, when
pairs begin to flip and held-out accuracy falls back. The three other
two-block settings (5 runs each) and the runs of
Table~\ref{paper:tab:tf-blocks-late} are not shown.}}
\label{fig:tf-blocks}
\end{figure}

\begin{figure}[tbp]
\centering
\includegraphics[width=0.92\textwidth]{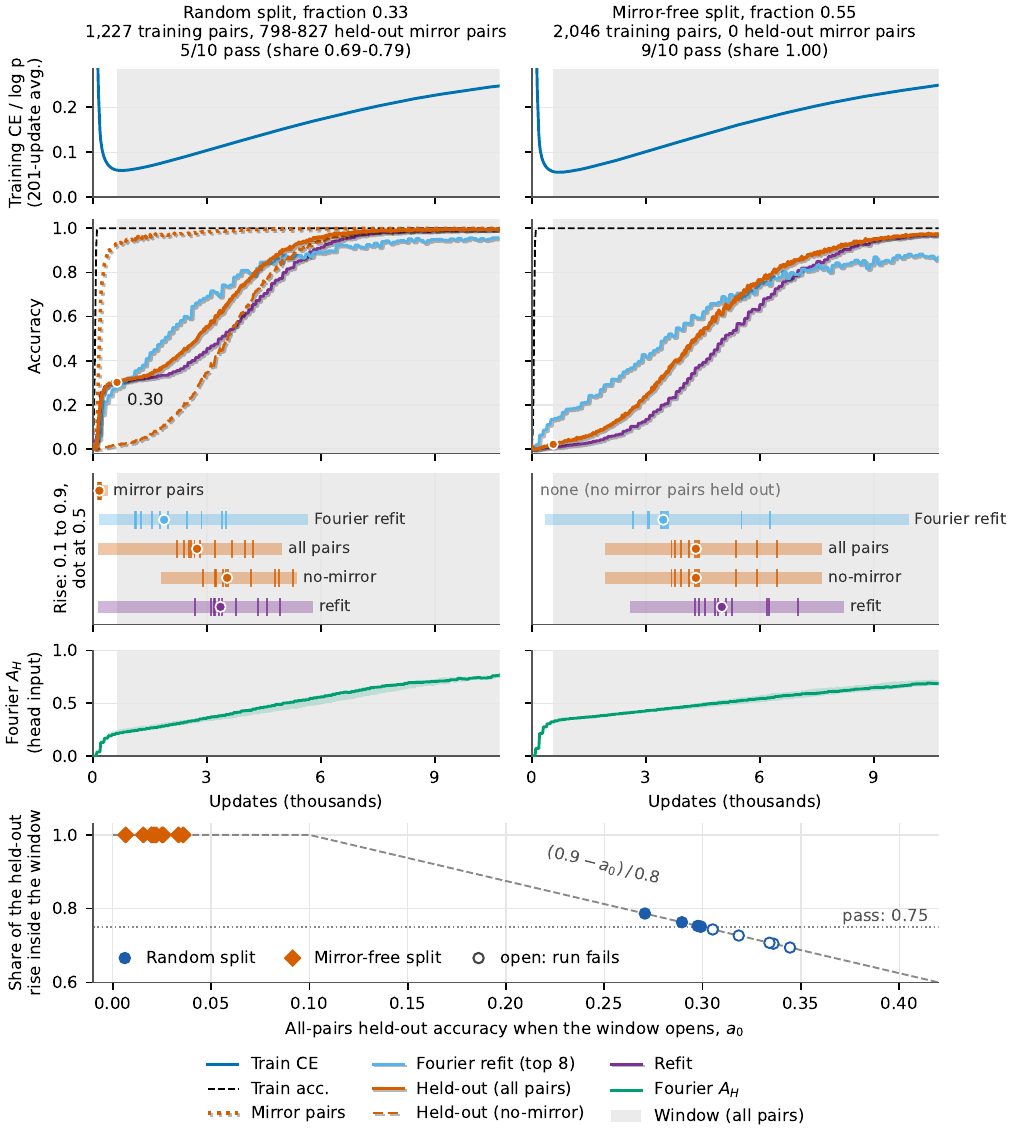}
\caption{\added{\textbf{Mirror pairs at $p=61$}: the random split at
$\phi=0.33$ (left; 798 to 827 of the 2{,}494 held-out pairs have their
mirror in the training set) and a mirror-free split at $\phi=0.55$ (right;
none), 10 runs each from
the same initializations. Rows: training CE; held-out accuracy on all pairs
(solid), the no-mirror pairs (dashed) and the mirror pairs (dotted), the
refit and the Fourier-filtered refit on all held-out pairs, and training
accuracy, medians (dot: median all-pairs accuracy when the window opens);
where each accuracy rises (bar from the median $0.1$ to the median $0.9$
crossing, dot at the median $0.5$ crossing, ticks at each run's $0.5$
crossing); $\AH$ of $\hat r$. Gray: the all-pairs window (median). Bottom:
the all-pairs share of the held-out rise inside the window against $a_0$,
the all-pairs accuracy when the window opens, one point per run. Every run
lies on the bound $(0.9-\max(a_0,0.1))/0.8$; open symbols fail (the
mirror-free run that fails, on its final accuracy, is hidden among the others
at $a_0=0.02$).}}
\label{fig:tf-mirror}
\end{figure}

\begin{figure}[tbp]
\centering
\includegraphics[width=0.92\textwidth]{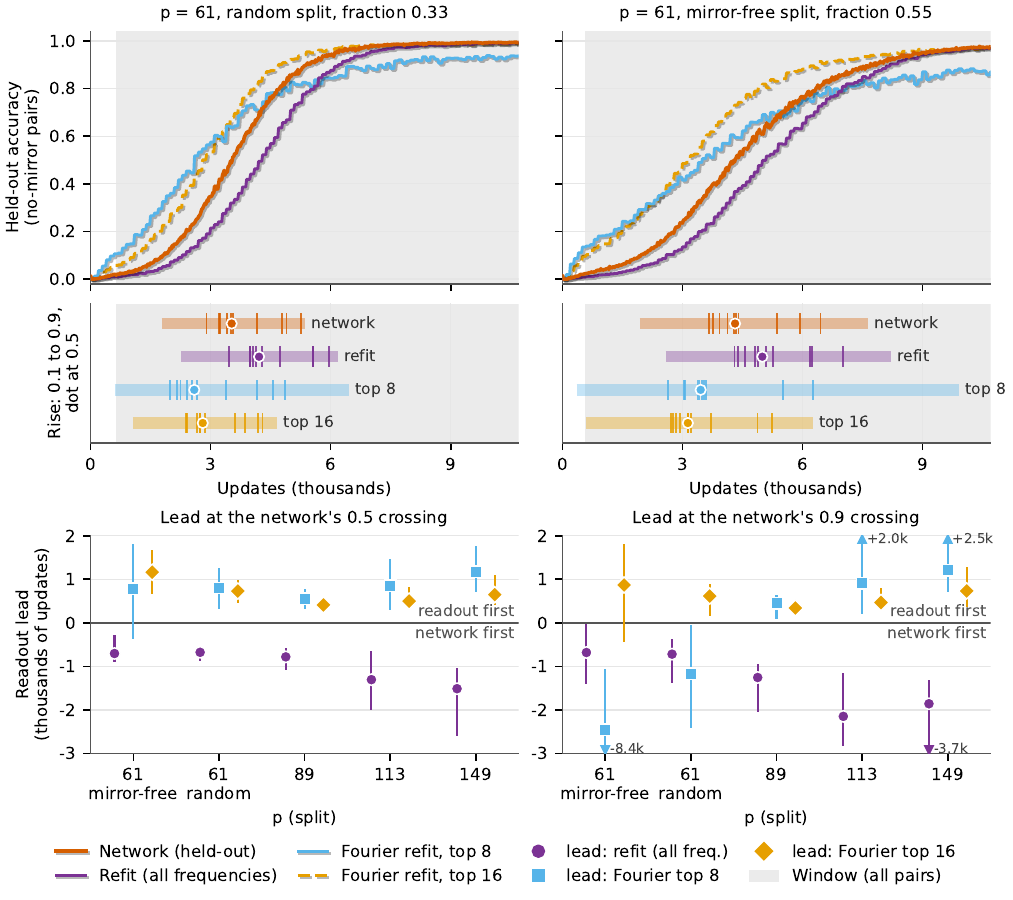}
\caption{\added{\textbf{Readouts against the network}, on the no-mirror
pairs. Top: $p=61$ with the random split ($\phi=0.33$) and the mirror-free
split ($\phi=0.55$): held-out accuracy of the network, of the refit (all
frequencies) and of the Fourier-filtered refit with the top 8 and the top
16 frequencies, medians over 10 runs; below, where each rises (bars and dots
as in Figure~\ref{fig:tf-mirror}). Bottom: each readout's lead over the
network at the network's $0.5$ and $0.9$ crossings (median over runs; bar
from the smallest to the largest run) at the settings of
Figure~\ref{fig:tf-primes} and at $p=61$ with the random split; above zero
the readout crosses first. The
top-16 readout is post hoc.}}
\label{fig:tf-readouts}
\end{figure}

\clearpage
\begin{proofrevision}
\section*{N\quad Cross-appendix notation}
\phantomsection
\makeatletter\def\@currentlabel{N}\makeatother
\label{app:notation}
\label{sec:notation-shared}\label{sec:notation-ridge}
\label{sec:notation-halfgrid}\label{sec:notation-random}
\addcontentsline{toc}{section}{N\ \ Cross-appendix notation}
\begingroup
Beyond Section~\ref{paper:setup}, the following notation is shared
across appendices; other notation is defined locally.
We use $j$ or $t$ for update indices, with $\mat{H}_j=\mat{H}(\mat{W}_j)$, $L_j=L(\mat{W}_j,\mat{V}_j)$
and $\mat{F}_j,\mat{G}_j$ the departing negative gradients of Section~\ref{paper:setup:updates}.
For an indexed pair $i=(a_i,b_i)$, $x_i=x_{a_i b_i}$ and $y_i=y_{a_i b_i}$;
$\mat{Y}[B]$ restricts rows and $\mat{Y}(x,:)$ selects one row. For matrices indexed on
the grid in both coordinates, subscripts such as $\mat{Y}_{DS}$ restrict rows
to $D$ and columns to $S$. Matrix inner products
are $\ip{\mat{U}}{\mat{V}}=\op{tr}(\mat{U}^*\mat{V})$, where $*$ is conjugate transpose
($\top$ for real matrices); $\one_B$ is the indicator of $B$.
$\norm{\cdot}_{\rm op}$, $\odot$, $\Var\paren{Z}$ and $S^c$
denote the operator norm, entrywise product, $\expct{|Z-\expct{Z}|^2}$ and the complement
in the stated ambient set. A positive semidefinite matrix's
supported inverse square root, written $\mat{S}^{+1/2}$
or $(\mat{S}^+)^{1/2}$, acts on its positive-eigenvalue support and is zero on its kernel.
For complete ridge comparison, let $\mat{U}$ be a full feature table, $S,D$ the
training and query rows, $\mat{T}$ the matrix of targets on $S$ (not the
full-grid target table), and $\gamma>0$; put
$q_U=\gamma\norm{\mat{U}_S}_F^2$ and, when $q_U>0$,
$\mat{Y}_U=\mat{U}_D(\mat{U}_S^*\mat{U}_S+q_U\Id)^{-1}\mat{U}_S^*\mat{T}$.
For a same-sized comparison table $\mat{V}$, define $q_V,\mat{Y}_V$ analogously and set
$e_S=\norm{\mat{V}_S-\mat{U}_S}_F$,
$e_D=\norm{\mat{V}_D-\mat{U}_D}_F$,
$t=\norm{\mat{T}}_{\rm op}$,
$N=\norm{\mat{T}}_F$ and
$L=\norm{\mat{U}_D(\mat{U}_S^*\mat{U}_S+q_U\Id)^{-1/2}}_{\rm op}$.
All columns enter these formulas unless stated otherwise; $t,N,L$ here are
not time, dictionary size or training CE.
The shared comparison is proved in
Appendix~\ref{sec:main}, Lemma~\ref{lem:shared-full-ridge-transfer}, and
instantiated in Appendix~\ref{sec:r62}, Lemma~\ref{lem:bwc-near}.

\readability{The two assembly results are the entry points for the
decay proofs in Appendix~\ref{app:decay}. In the half-grid construction,}
$\cB_M$ is the full-dimensional initial
ball at width $M$ (Definition~\ref{def:a-can});
Proposition~\ref{prop:aasm-continuation} selects this ball, the positive
head-step cap $\eta_0$ with $\eta_V=\eta_0/2$,
and the decay cap $d_0\in(0,1/2]$.
Appendices~\ref{sec:main} and~\ref{app:decay},
Corollary~\ref{cor:crd-halfgrid}, use the same ball and head step.
\readability{The positive-decay corollary uses part (a) of that
proposition, which preserves Theorem~\ref{paper:thm:half}(i)--(iii);
the late-refit part (b) assumes zero decay.}

For the random split, $h=2\alpha(1-\beta)>0$
is the fixed dictionary height,
$\cE_{\rm split}$ the initial-data event
(Definition~\ref{def:r62-split}), and
$\bar{\mat{V}}_0(\Train)$, $r_{\rm new}(\Train)$
the train-only center head and prescribed joint-ball radius
(\eqref{eq:r62head} and~\eqref{eq:r62ball}).
Appendix~\ref{sec:r62} constructs these
objects; Appendix~\ref{app:decay}, Definition~\ref{def:r62-decay} and
Corollary~\ref{cor:crd-random}, use them in the pre-sampling decay budget.
\readability{Here Proposition~\ref{lem:sc26-b-head-interface} transfers
all trajectory conclusions whenever each head increment has norm at
most $\eta$. The decay proof checks this hypothesis with data step
$\eta/2$, keeping $\eta$ in the prescribed radius and error budgets.}
Identical letters in the two constructions need not denote the same parameter.
Terminating decimals are exact rationals; existence, positivity and bounds
are proved where claimed, not inferred from defining formulas.
\endgroup
\end{proofrevision}

\section{Structured half-grid theorem: construction and estimates}\label{sec:main}

\begin{proofrevision}
\subsection{Statement and conventions}\label{sec:setup}
\end{proofrevision}
\proofrevtext{The half-grid data and common diagnostics are defined in the main paper;
Section~\ref{sec:notation-halfgrid} lists the interface reused in Appendix~\ref{app:decay};
the construction and proof quantities are defined as they enter below. The default range is prime $p\ge97$; Theorem~\ref{paper:thm:half}
uses $p\ge101$, and the late-refit module uses $p\ge1764$.
The split contains whole difference classes; all diagonal pairs are held out.
It is not an iid split or a guarantee for arbitrary balanced half-grids.
Here $\mat{P}_c=\mat{C}_p$ is class centering and $T_y=e_y-\one_p/p$ is its
centered class-$y$ target.}

\begin{proofrevision}
\paragraph{Public statement and appendix supplement.}
This appendix proves Theorem~\ref{paper:thm:half}. It also proves the following
additional bound, which is not part of that public statement: for $p\ge1764$,
the actual late complete train-only refit has $E_{\rm ref}(j)<121/400$ at every $j\in\cI_{\rm late}$.
The common theorem range is $p\ge101$, although some construction estimates
below hold for $p\ge97$; the public late-refit margin and accuracy guarantee
is on nonsquare differences, not diagonal inputs.
The scales are fixed in Definition~\ref{def:half-scales} and the clock and windows
in~\eqref{eq:clock}--\eqref{eq:acr-windows}; the
positive head step and full-dimensional initialization-ball radius are
existential, supplied by Proposition~\ref{prop:aasm-continuation}.
\end{proofrevision}

\begin{proofrevision}
\subsection{Claim I: construction and uniformly controlled reference dynamics}\label{sec:a-claim-dynamics}

\readability{The construction has two stages. We first bound the uniform source
and its Polar rows, using the mask counts and Fourier covariance blocks.
Those bounds then identify the static initializer and verify the exact
frozen-head trajectory in Lemma~\ref{lem:round15-core-only-bootstrap}.}

\subsubsection{The static construction}\label{sec:init}
\proofstepheading{Design, visible space and conventions}\label{sec:acr-design}
\label{sec:notation-halfgrid-data}
For the split~\eqref{eq:split}, let $Q$ be the nonzero squares in $\F_p$
and $d=(p-1)/2$ their number.
The visible input space has dimension $r=2p-1$, and the arithmetic dictionary
has size $N=p^2(p-1)$ and uniform mass $\varpi=1/N$.
\begin{definition}[Visible input space]\label{def:half-visible}
Set $E=(\one_{\rm left},-\one_{\rm right})^\perp$ and let $\mat{P}_E$ be its
orthogonal projector.
\end{definition}

The design identities needed below are:

\begin{lemma}[Training design]\label{lem:acr-design}
Every $x\in\Omega$ lies in $E$. Every $a\in\F_p$ is the left operand of exactly $d$
training inputs and the right operand of exactly $d$, and every label occurs exactly
$d$ times in $\Train$.
\end{lemma}
\begin{proof}
Each $x_{ab}$ has one unit entry in each half. The training inputs with left operand
$a$ are $(a,a+q)$, $q\in Q$, and similarly on the right. For a difference $q\in Q$ and a
label $y$ the unique pair is $((y-q)/2,(y+q)/2)$, so each label occurs once per $q$.
\end{proof}
\proofstepheading{The arithmetic core}\label{sec:acr-core}

\paragraph{Fixed construction scales.}
Write $s=\eta_W$ and $\lambda=\lambda_W$ for the
parameters of~\eqref{paper:optimizer}.
The half-grid learner has $\lambda_V=0$ except in Appendix~\ref{app:decay};
its positive head step is selected in Proposition~\ref{prop:aasm-continuation}.
\begin{definition}[Construction scales]\label{def:half-scales}\label{app:rates}
Fix a prime $p\ge101$, $\lambda>0$, and then choose $\epsilon>0$ sufficiently
small.
The construction scales are
\begin{equation}\label{eq:round15-scales}
\begin{gathered}
 R=(\lambda\sqrt r)^{-1},\qquad v=\epsilon/R,
 \qquad L_{\rm lift}=R\epsilon^{-3}.
\end{gathered}
\end{equation}
\end{definition}

\begin{definition}[Arithmetic masks and reflection pairs]\label{sec:notation-halfgrid-center}
Let $\mathsf I=\{0,1,\ldots,d-1\}\subset\F_p$. The core dictionary is
$\Lambda=\{(\ell,s,k):\ell\in\F_p^\times,\ s,k\in\F_p\}$, of size $N$.
For $l=(\ell,s,k)$ define
\begin{equation}\label{eq:acr-masks}
 A_{\ell,s}=\{a:\ell a+s\in\mathsf I\},\qquad
 B_{\ell,s,k}=k-A_{\ell,s}^{\,c},\qquad
 D_l(a,b)=\one_{A_{\ell,s}}(a)\one_{B_{\ell,s,k}}(b),\qquad k_l=k.
\end{equation}
Write $A_l,B_l$; the mask shift $s$ is unrelated to the step size $s$.
Every core row has mass $\pi_l=1/N$.
Let $\Lambda_c$ consist of the pairs
$\{(\ell,s,k),(-\ell,d-1-s,k)\}$ in $\Lambda$. Each pair $B\in\Lambda_c$
has the same mask and class; put $P_B=2/N$ and denote these shared objects
by $D_B,k_B$ (Lemma~\ref{lem:acr-dict}).
Let $\mat{C}_C\in\R^{\Lambda\times\Lambda_c}$ have entry $1/\sqrt2$ on the two
rows of each pair and zero elsewhere. Thus $\mat{C}_C^\top \mat{C}_C=\Id$.
\end{definition}

\begin{definition}[Inactive indicators and core heads]\label{def:half-heads}
For $l=(\ell,s,k)\in\Lambda$, let $\mathbf j_l\in\{0,1\}^{2p}$ indicate
the inactive left and right operand coordinates, and define the analytic
head matrix by
\begin{equation}\label{eq:acr-masses}
 \bar{\mat{V}}_{C,l}=-v\sqrt{\pi_l}\,T_{k_l}.
\end{equation}
Write $\iota(\ell,s,k)=(-\ell,d-1-s,k)$ for the reflection defining
the pairs in $\Lambda_c$. Thus the physical heads satisfy
$\bar{\mat{V}}_{C,c}=\bar{\mat{V}}_C\mat{C}_C$, with
$\bar{\mat{V}}_{C,c,B}=-v\sqrt{P_B}\,T_{k_B}$.
\end{definition}

\begin{lemma}[The rectangles]\label{lem:acr-dict}
For every $l\in\Lambda$:
\begin{enumerate}[label=(\roman*)]
\item $|A_l|=d$ and $|B_l|=(p+1)/2$; $\mathbf j_l$ has exactly $p$ ones, and for
$x\in\Omega$, $\mathbf j_l^\top x=0$ if $D_l(x)=1$ while $\mathbf j_l^\top x\ge1$ otherwise.
\item (Label exclusion) $D_l(x)=1$ implies $y^\star(x)\ne k_l$.
\item (Coverage) For $x\in\Omega$ and $k\ne y^\star(x)$, exactly $d(p+1)/2$ rows $l$ with
$k_l=k$ have $D_l(x)=1$.
\item (Class mass and head frame) $\sum_{l:k_l=k}\pi_l=1/p$ for every class $k$, and
$\bar{\mat{V}}_C\bar{\mat{V}}_C^\top=(v^2/p)\mat{P}_c$.
\item (Reflection pairs) $\iota$ is a fixed-point-free involution of $\Lambda$ with
$A_{\iota(l)}=A_l$, $B_{\iota(l)}=B_l$, $k_{\iota(l)}=k_l$ and
$\mathbf j_{\iota(l)}=\mathbf j_l$. Thus $\Lambda$ is the disjoint union of $N/2$
reflection pairs of rows with identical masks, classes, inactive indicators and masses.
\end{enumerate}
\end{lemma}
\begin{proof}
(i) $z\mapsto\ell z+s$ is a bijection of $\F_p$, so $|A_l|=|\mathsf I|=d$ and
$|B_l|=|A_l^c|=p-d$; $\mathbf j_l$ has $(p-d)+d=p$ ones, and
$\mathbf j_l^\top x_{ab}=\one\{a\notin A_l\}+\one\{b\notin B_l\}$.

(ii) If $a\in A_l$ and $a+b=k_l$, then $b\in k_l-A_l$, which is disjoint from $B_l$.

(iii) Let $x=(a,b)$; then $a\ne k-b$. By the definition~\eqref{eq:acr-masks} of the masks, $D_l(x)=1$ with $k_l=k$ means
$\ell a+s\in\mathsf I$ and $\ell(k-b)+s\notin\mathsf I$. An affine bijection
$z\mapsto\ell z+s$ is determined by the images of the two distinct points $a,k-b$, and
every pair of distinct images arises, so there are exactly $|\mathsf I|\,|\mathsf I^c|=d(p+1)/2$ such rows.

(iv) Each class has $p(p-1)$ rows of mass $1/N$. Hence
\begin{equation*}
 \bar{\mat{V}}_C\bar{\mat{V}}_C^\top=v^2\sum_l\pi_lT_{k_l}T_{k_l}^\top=(v^2/p)\sum_kT_kT_k^\top=(v^2/p)\mat{P}_c.
\end{equation*}

(v) $z\in\mathsf I$ if and only if $d-1-z\in\mathsf I$, since $z\mapsto d-1-z$ maps
$\{0,\ldots,d-1\}$ onto itself and $\{d,\ldots,p-1\}$ onto
$\{d-p,\ldots,-1\}\equiv\{d,\ldots,p-1\}$. As $-\ell a+d-1-s=d-1-(\ell a+s)$, the rows $l$
and $\iota(l)$ have the same $A$, hence the same $B$ and $\mathbf j$. Clearly
$\iota\circ\iota$ is the identity, and $\iota(l)=l$ would require $\ell=-\ell$, impossible
for odd $p$.
\end{proof}

\begin{definition}[Uniform source and mask tables]\label{def:half-uniform-source}
Let $\mat{X}_\Train$ and $\mat{X}_{\rm grid}$ have rows $x^\top$ on the indicated input sets.
Write $g_l=D_l\in\R^\Omega$ and let
$\mat{D}_{\rm grid}\in\{0,1\}^{\Omega\times\Lambda}$ have columns $g_l$;
$\mat{D}_\Train$ is its training-row restriction.
The uniform source and its fixed physical compression are
\begin{equation}\label{eq:acr-F0}
 b_l=\frac1{np}\sum_{x\in\Train}D_l(x)x,\qquad
 \mat{F}_0\in\R^{\Lambda\times2p},\quad \mat{F}_{0,l}=\sqrt{\pi_l}\,b_l^\top,
 \quad \mat{\Sigma}_0=\mat{F}_0^\top \mat{F}_0,\qquad \mat{F}_{0,c}=\mat{C}_C^\top \mat{F}_0.
\end{equation}
By~\eqref{eq:acr-F0}, put
$\mat{F}_{\rm unif}:=\mat{F}_0=(\sqrt\varpi/(np))\mat{D}_\Train^\top \mat{X}_\Train$
and $\mat{Q}_{\rm QR}=\pol(\mat{F}_{\rm unif})$.
The positive-pair source has the same nonzero singular values as $\mat{F}_0$.
\end{definition}

\begin{definition}[Uniform anchor]\label{def:afr-anchor}
With $\mat{Q}_{\rm QR}$ just defined, let
$\mat{H}=\mat{H}^{\rm QR}=\mat{D}_{\rm grid}\odot(\mat{X}_{\rm grid}\mat{Q}_{\rm QR}^{\top})$.
\end{definition}

By~\eqref{eq:r15-core-source}, at uniform probabilities the physical core source is
$v\mat{F}_{0,c}$ and its analytic lift is $v\mat{F}_0$. The table $\mat{H}^{\rm QR}$ will supply
the Fourier anchor in Section~\ref{sec:fourier}.

\proofstepheading{Counts, covariance blocks and inverse roots}\label{sec:afr-roots}
\label{sec:notation-halfgrid-roots}
For the interval calculations, write $b=d$ and $c=(p+1)/2$ for the left-
and right-operand mask sizes, so the rectangle size is $u=bc=(p^2-1)/4$.
The base interval is $I=\mathsf I$ and its reflected complement is
$I'=-I^c=\{1,\ldots,c\}$.
\begin{definition}[Character operators and Fourier modes]\label{def:half-character}
Let $\chi_2$ be the quadratic character, with $\chi_2(0)=0$.
Let
$(\mat{C}_-)_{a,z}=\chi_2(z-a)$ and $\mat{C}_+=\mat{C}_-^\top$. Put $e(y)=e^{2\pi iy/p}$
and $\hat f(t)=\sum_af(a)e(-ta)$. Define the Hermitian multiplier $\mat{H}_\chi$
to have eigenvalue $\chi_2(t)$ on the Fourier vector
$(p^{-1/2}e(ta))_{a\in\F_p}$, and put
$\tau=\sum_a\chi_2(a)e(-a)$. Write
$e^L_t=((p^{-1/2}e(ta))_a,0)$ and
$e^R_t=(0,(p^{-1/2}e(ta))_a)$ in $\mathbb C^{2p}$, and
$e_c=(\one,\one)/\sqrt{2p}$.
Put $\mat{\mathsf T}=\bigl[\begin{smallmatrix}0&\bar\tau\\\tau&0\end{smallmatrix}\bigr]$.
\end{definition}

\begin{lemma}[Character operators]\label{lem:afr-char}
For every odd prime $p$, $|\tau|^2=p$ and $\bar\tau=\chi_2(-1)\tau$ (see also~\cite{irelandrosen1990}). Moreover
$\mat{C}_-=\bar\tau \mat{H}_\chi$, $\mat{C}_+=\tau \mat{H}_\chi$, $\mat{H}_\chi^2=\mat{P}_c$, $\mat{C}_\pm\one=0$,
$\mat{C}_+=\chi_2(-1)\mat{C}_-$ and $\mat{C}_-\mat{C}_+=\mat{C}_+\mat{C}_-=p\mat{P}_c$. In particular
$\chi_2(y)=(\tau/p)\sum_{m\ne0}\chi_2(m)e(my)$ for all $y$.
\end{lemma}
\begin{proof}
Write $e_t(a)=p^{-1/2}e(ta)$ for the normalized Fourier vector.
For $t\ne0$, substitution $y=z-a$ gives
$\mat{C}_-e_t=\chi_2(t)\bar\tau e_t$; $\mat{C}_-$ annihilates constants.
Also $\bar\tau=\chi_2(-1)\tau$ by $a\mapsto-a$, and
\begin{equation*}
 |\tau|^2=\sum_u\chi_2(u)\sum_{b\ne0}e(b(1-u))
 =(p-1)-\sum_{u\ne1}\chi_2(u)=p.
\end{equation*}
Here $\sum\chi_2=0$ by nonsquare dilation. Transpose and Fourier inversion
give all remaining identities.
In particular, $\mat{\mathsf T}$ is Hermitian, $\mat{\mathsf T}^2=p\Id$ and $\|\mat{\mathsf T}\|=\sqrt p$.
\end{proof}
Thus $\tau$ is real for $p\equiv1\pmod4$ and imaginary for $p\equiv3\pmod4$;
neither its sign nor a reality assumption is used.

For a fixed mask, write $\alpha=\one_{A_l}$ and $\beta=\one_{B_l}$ for
the left and right indicators, with intersection indicator
$L=\alpha\odot\beta$ and size $h=|A_l\cap B_l|$.
The count calculation uses the signed character count
$\Xi=\alpha^\top \mat{C}_-\beta$ and the number $n_l=E_l$ of active training inputs.
\begin{definition}[Training counts and interval powers]\label{def:half-counts}
The QR training-count vector $\mathbf n=npb_l$ has operand blocks
\begin{align*}
 \mathbf n^L(a)&=\alpha(a)\#\{z\in B_l:(a,z)\in\Train\},\\
 \mathbf n^R(z)&=\beta(z)\#\{a\in A_l:(a,z)\in\Train\}.
\end{align*}
Expectations $\E$ average over $l$ uniform on $\Lambda$.
Put $\mat{\cB}=\sqrt\varpi \mat{D}_{\rm grid}$ and
\begin{equation*}
 \omega_v=|\hat\one_I(v)|^2,\qquad
 \omega'_v=|\hat\one_{I'}(v)|^2,\qquad
 \Sigma_2=\sum_{v\ne0}\omega_v^2.
\end{equation*}
\end{definition}

\begin{lemma}[Exact counts and mask moments]\label{lem:afr-counts}
Let $p\ge5$ be prime.
\begin{enumerate}[label=(\roman*)]
\item The QR training counts are
\begin{equation*}
\begin{aligned}
 \mathbf n^L&=\tfrac12(c\alpha-L+\alpha\odot \mat{C}_-\beta),\\
 \mathbf n^R&=\tfrac12(b\beta-L+\beta\odot \mat{C}_+\alpha).
\end{aligned}
\end{equation*}
The source row is $\mat{F}_{{\rm unif},l}=(\sqrt\varpi/(np))\mathbf n^\top$.
\item Both count blocks sum to the number of active training inputs:
$\one^\top\mathbf n^L=\one^\top\mathbf n^R=n_l=(u-h+\Xi)/2$,
with $|\Xi|\le u/\sqrt p$.
\item Write $l\leftrightarrow(\mu,t_1,t_2)$ with $\mu=\ell^{-1}$, $t_1=-s$ and
$t_2=\ell k+s$. This is a bijection
$\Lambda\to\F_p^\times\times\F_p\times\F_p$, and $A_l=\mu(I+t_1)$, $B_l=\mu(I'+t_2)$
with $I'=\{1,\ldots,c\}=-I^c$. For fixed $\mu$, $t_1,t_2$ are independent and uniform.
\item With $\Var\paren{h}$ the variance under $\E$: $\expct{h}=u/p$, $\expct{h^2}=u/3$,
$\Var\paren{h}=u(p^2+3)/(12p^2)$, $\expct{\Xi}=\expct{h\Xi}=0$ and $\expct{\Xi^2}=p\Var\paren{h}$. Hence
$\expct{n_l}=u(p-1)/(2p)$ and $\Var\paren{n_l}=(p+1)\Var\paren{h}/4$.
\end{enumerate}
\end{lemma}
\begin{proof}
\readability{Parts (i)--(iii) give the count formulas and a convenient
parameterization of the masks. For part (iv), we use that parameterization
to diagonalize the mask kernel and read off the moments of $h$ and $\Xi$.}

(i) Expand $1_Q(t)=[1-1_{t=0}+\chi_2(t)]/2$ in each QR count;
the source factor is~\eqref{eq:acr-F0}.

(ii) Sum the two count formulas,
using $\one^\top(\alpha\odot \mat{C}_-\beta)=\Xi=\one^\top(\beta\odot \mat{C}_+\alpha)$. By
Lemma~\ref{lem:afr-char}, $\Xi=(\mat{P}_c\alpha)^\top \mat{C}_-(\mat{P}_c\beta)$ with
$\norm{\mat{C}_-}_{\rm op}=\sqrt p$ and $\norm{\mat{P}_c\alpha}^2=\norm{\mat{P}_c\beta}^2=u/p$.

(iii) This
is direct: $A_l=\ell^{-1}(I-s)$ and $B_l=k-\ell^{-1}(I^c-s)=\ell^{-1}(I'+\ell k+s)$.

(iv) Put $\mat{\cB}=\sqrt\varpi \mat{D}_{\rm grid}$ and
$\mat{\cK}=\mat{\cB}\mat{\cB}^\top$. In the orthonormal grid basis
$\phi_{r,s}(a,z)=p^{-1}e(ra+sz)$, its eigenvalues are
\begin{equation}\label{eq:afr-mask-kernel}
 \lambda(r,s)=\frac1{p^2(p-1)}\sum_{\mu\ne0}\omega_{\mu r}\omega'_{\mu s},
 \qquad \omega_v=|\hat1_I(v)|^2,\quad \omega'_v=|\hat1_{I'}(v)|^2.
\end{equation}
Indeed, the coefficient of column $l$ is
$\sqrt\varpi\,p^{-1}\hat1_{A_l}(r)\hat1_{B_l}(s)$.
Conditional on $\mu$, the independent translations of $A_l,B_l$ from (iii)
annihilate all unequal frequency pairs.
The remaining diagonal average is~\eqref{eq:afr-mask-kernel}.
Here $\omega_0=b^2$, $\omega'_0=c^2$ and $\omega'_v=\omega_v$ for $v\ne0$.
Here $|y|$ is the absolute value of the representative of $y\in\F_p$
in $\{-(p-1)/2,\ldots,(p-1)/2\}$.
The cyclic autocorrelation $|I\cap(I+y)|=b-|y|$ gives, by Parseval,
\begin{readabilityblock}
\begin{equation}\label{eq:readable-e1-mask-fourth-moment}
 \Sigma_2=p\sum_y(b-|y|)^2-b^4=(p^2-1)(p^2+3)/48.
\end{equation}
\end{readabilityblock}
For the two grid functions $d_0(a,z)=\one[a=z]$ and
$c_\chi(a,z)=\chi_2(z-a)$, the Fourier coefficients at $(r,s)$ are
$\one[r+s=0]$ and $\one[r+s=0]\tau\chi_2(s)$. Since
$h=\ip{g_l}{d_0}$ and $\Xi=\ip{g_l}{c_\chi}$, their second moments are
the corresponding quadratic forms of $\mat{\cK}$.
Using $\lambda(0,0)=u^2/p^2$ and
$\lambda(-s,s)=\Sigma_2/[p^2(p-1)]$ for $s\ne0$ gives
\begin{readabilityblock}
\begin{equation}\label{eq:readable-e1-count-second-moments}
 \expct{h^2}=\frac{u^2+\Sigma_2}{p^2}=\frac u3,\qquad
 \expct{\Xi^2}=\frac{\Sigma_2}p,\qquad \expct{h\Xi}=0.
\end{equation}
\end{readabilityblock}
The first moments are $\expct{h}=u/p$, $\expct{\Xi}=0$, since $\expct{g_l}=u/p^2$.
\readability{Combining these first moments with
\eqref{eq:readable-e1-count-second-moments} gives
$\Var\paren{h}=\Sigma_2/p^2$, with $\Sigma_2$ evaluated in
\eqref{eq:readable-e1-mask-fourth-moment}. Substituting into
$n_l=(u-h+\Xi)/2$ gives the stated mean and variance of $n_l$.}
\end{proof}

\paragraph{Line-sum evaluations.}
\begin{definition}[Line-sum quantities]\label{def:half-linesums}
For $z\ne0,1$, define $\mathsf W(z)=\sum_{t\ne0}\omega_{tz}\omega_{t(1-z)}$
and $\mathsf J_\chi=\sum_{v\ne0}\chi_2(v)\omega_v$. In the following line sums,
every $z$ ranges over $\F_p\setminus\{0,1\}$:
\begin{equation*}
 \mathsf J=\frac{\sum_z\mathsf W(z)}{p^2(p-1)},\quad
 \mathsf R=\frac{\mathsf W(1/2)}{p^2(p-1)},\quad
 \mathsf U=\frac{\sum_z\chi_2(z)\mathsf W(z)}{p^2(p-1)},\quad
 \mathsf V=\frac{\chi_2(-1)\sum_z\chi_2(z(1-z))\mathsf W(z)}{p^2(p-1)}.
\end{equation*}
Put $\Sigma_1=\sum_{v\ne0}\omega_v$,
$g_0=4p/\sqrt{(p+1)(p^2+1)}$ and $\vartheta=2p/(p^2+1)$.
\end{definition}
Write $\mathsf j=g_0^2\mathsf J$, $\mathsf r_0=g_0^2\mathsf R$,
$\mathsf u=g_0^2\mathsf U$ and $\mathsf w=g_0^2\mathsf V$ for the normalized
total, central, one-character and two-character line sums, respectively.
Directly,
\begin{equation}\label{eq:afr-linesums}
\begin{aligned}
 \mathsf J&=\frac{(p+1)(p^2-3)}{24p^2},&
 \mathsf R&=\frac{\mathsf W(1/2)}{p^2(p-1)}=\frac{\Sigma_2}{p^2(p-1)}=\frac{(p+1)(p^2+3)}{48p^2},\\
 \mathsf V&=\frac{\chi_2(-1)\mathsf J_\chi^2-\Sigma_2}{p^2(p-1)}.
\end{aligned}
\end{equation}
The evaluations follow from $\sum_z\mathsf W(z)=\Sigma_1^2-\Sigma_2$ and
$\sum_z\chi_2(z(1-z))\mathsf W(z)=\mathsf J_\chi^2-\chi_2(-1)\Sigma_2$. The first is
obtained by writing the double sum
$\sum_{t\ne0}\sum_{m\ne0,t}\omega_m\omega_{t-m}=\sum_{m\ne0}\omega_m(\Sigma_1-\omega_m)$
with $m=tz$. The second is the same double sum with the weights
$\chi_2(m)\chi_2(t-m)$; the full convolution square of $\chi_2\omega$ is
$\mathsf J_\chi^2$, and its $t=0$ term is $\chi_2(-1)\Sigma_2$.

Consequently $\mathsf j=2(p^2-3)/(3(p^2+1))$ and
$\mathsf r_0=(p^2+3)/(3(p^2+1))$.

\begin{definition}[Covariance blocks]\label{def:half-covariance}
Set
\begin{equation}\label{eq:afr-scal}
\begin{aligned}
 \mathsf e_0&=\frac{(p+1)(5p^2-3)}{3(p-1)^2(p^2+1)},&
 \mathsf x&=\vartheta\Bigl(1-\frac{p+1}{(p-1)^2}\Bigr),&
 \mathsf e_1&=-\frac2{p-1}+\frac{\mathsf j+p\mathsf w}{(p-1)^2},\\
 \mathsf f&=\frac2{p-1}-\frac{2\mathsf u}{(p-1)^2},&
 \mathsf k_0&=\frac{2(1+\mathsf u)}{(p-1)^2},&
 \mathsf k_1&=\frac{2\vartheta}{(p-1)^2}.
\end{aligned}
\end{equation}
\begin{equation}\label{eq:afr-ED}
 \mat{\mathsf E}=\begin{bmatrix}\mathsf e_0+\mathsf x&\mathsf e_1\\
 \mathsf e_1&\mathsf e_0-\mathsf x\end{bmatrix},\qquad
 \mat{\mathsf D}=\begin{cases}\mathsf f\,\mat{\mathsf T},&\chi_2(-1)=-1,\\
 \tau\begin{bmatrix}-\mathsf k_0+\mathsf k_1&\mathsf f\\
 \mathsf f&-\mathsf k_0-\mathsf k_1\end{bmatrix},&\chi_2(-1)=1.\end{cases}
\end{equation}
Define $\mat{M}_\pm=g_0^{-2}(\Id+\mat{\mathsf E}\pm\mat{\mathsf D})$ and $c_0=\expct{n_l^2}/b^2$.
For the covariance computation, $\mat{\cP}_Q$ denotes multiplication by
$\one_Q(z-a)$.
\end{definition}

\begin{lemma}[Exact source covariance]\label{lem:afr-cov}
Let $p\ge5$ be prime. For $t\ne0$ with $\varsigma=\chi_2(t)$, the matrix of
$p^5\mat{F}_{\rm unif}^\top \mat{F}_{\rm unif}$ on the ordered pair $(e^L_t,e^R_t)$ is
\begin{equation}\label{eq:afr-Mform}
 \mat{M}_\varsigma=g_0^{-2}(\Id+\mat{\mathsf E}+\varsigma\mat{\mathsf D}).
\end{equation}
On constants, $p^5\mat{F}_{\rm unif}^\top \mat{F}_{\rm unif}e_c=2c_0e_c$ and
$\mat{F}_{\rm unif}(\one,-\one)=0$, where
\begin{equation}\label{eq:afr-c0}
 c_0=\frac{\expct{n_l^2}}{b^2}=\frac{u^2}{p^2}+\frac{(p+1)^2(p^2+3)}{48p^2(p-1)}\ \ge\ \frac{u^2}{p^2}.
\end{equation}
Equivalently, $\mat{F}_{\rm unif}^\top \mat{F}_{\rm unif}=p^{-5}[\mat{M}_+\otimes\mat{\Pi}_++\mat{M}_-\otimes\mat{\Pi}_-
+2c_0e_ce_c^\top]$ with $\mat{\Pi}_\pm=(\mat{P}_c\pm \mat{H}_\chi)/2$. Here $\mat{M}\otimes\mat{\Pi}_\pm$ acts by $\mat{M}$
on the operand index and by $\mat{\Pi}_\pm$ on the input coordinate.
\end{lemma}
\begin{proof}
\readability{We compute the covariance one Fourier plane at a time.
The mask-kernel eigenvalues in~\eqref{eq:afr-mask-kernel} supply the
weights; the two exceptional terms $s=0,t$ and the remaining line sums
then give the block in~\eqref{eq:afr-Mform}.}

Write $q=\one_Q$ and
$\mat{F}_{\rm unif}=(np)^{-1}\mat{\cB}^\top\mat{\cP}_Q\mat{X}_{\rm grid}$.
The indicator has $\hat q(0)=b$ and
$\hat q(s)=(-1+\tau\chi_2(s))/2$ for $s\ne0$.
The two masked operand modes have grid Fourier coefficients
$p^{-1/2}\one[r+s=t]\hat q(s)$ and
$p^{-1/2}\one[r+s=t]\hat q(s-t)$. By~\eqref{eq:afr-mask-kernel},
different $t$ do not mix, and their Gram is
\begin{equation}\label{eq:afr-source-gram}
 \mat{M}_t=\frac1{b^2}\sum_s\lambda(t-s,s)
 \begin{bmatrix}
 |\hat q(s)|^2&\overline{\hat q(s)}\hat q(s-t)\\
 \overline{\hat q(s-t)}\hat q(s)&|\hat q(s-t)|^2
 \end{bmatrix}.
\end{equation}
Here $p^5/(n^2p^3)=b^{-2}$ accounts for all normalizations.

Put $L=bc^3/[p^2(p-1)]$ and $R'=b^3c/[p^2(p-1)]$.
For $t\ne0$, the terms $s=0,t$ use $\lambda(t,0)=L$,
$\lambda(0,t)=R'$, with $L+R'=2g_0^{-2}$ and
$L-R'=2g_0^{-2}\vartheta$.
The remaining terms, indexed by $s=tz$, have weight
$\mathsf W(z)/[p^2(p-1)]$. Their character sums are precisely
$\mathsf J,\mathsf U,\mathsf V$ of Definition~\ref{def:half-linesums}.
For example, using $\mathsf W(z)=\mathsf W(1-z)$ and
$\tau\chi_2(-1)=\bar\tau$ gives
\begin{equation*}
 (\mat{M}_t)_{LR}=\frac{L+R'}{2b}(\chi_2(t)\bar\tau-1)
 +\frac{\mathsf J+p\mathsf V-2\chi_2(t)\bar\tau\mathsf U}{4b^2}.
\end{equation*}
For the diagonal entries use
$|\hat q(s)|^2=(p+1-2\op{Re}\tau\,\chi_2(s))/4$ for $s\ne0$.
Substitution of~\eqref{eq:afr-scal}--\eqref{eq:afr-ED} now gives
$\mat{M}_t=g_0^{-2}(\Id+\mat{\mathsf E}+\chi_2(t)\mat{\mathsf D})$ in both Gauss phases.
At $t=0$ the two columns agree, each Gram entry is
$\expct{n_l^2}/b^2=c_0$ by Lemma~\ref{lem:afr-counts}(iv), and the constant
block is $c_0\one\one^\top$. Its eigenvalues are $2c_0$ and zero.
\end{proof}

\begin{lemma}[Character-sum envelope]\label{lem:afr-scalars}
For every prime $p\ge5$, $|\mathsf u|\le\mathsf j<2/3$ and
$-\mathsf r_0\le\mathsf w\le\mathsf j$.
If $\chi_2(-1)=-1$, then $\mathsf J_\chi=0$ and $\mathsf V=-\mathsf R$, so
$\mathsf w=-\mathsf r_0$.
\end{lemma}
\begin{proof}
The nonnegative weights $\mathsf W(z)$ give $|\mathsf U|,|\mathsf V|\le\mathsf J$.
In the imaginary phase, pairing $x,-x$ in $\mathsf J_\chi$ gives zero since
$\omega_{-x}=\omega_x$; now use~\eqref{eq:afr-linesums}.
In the real phase the same identity gives
$\mathsf V=(\mathsf J_\chi^2-\Sigma_2)/[p^2(p-1)]\ge-\mathsf R$.
\end{proof}

\begin{definition}[Supported inverse-root quantities]\label{def:half-inverse-roots}
The inverse roots $\mat{M}_\pm^{-1/2}$ exist for $p\ge97$ by
Lemma~\ref{lem:afr-rootbounds}(i). Define
\begin{equation*}
 \mat{G}_{\rm avg}=\tfrac12\bigl(\mat{M}_+^{-1/2}+\mat{M}_-^{-1/2}\bigr),\qquad
 \mat{G}_{\rm diff}=\tfrac12\bigl(\mat{M}_+^{-1/2}-\mat{M}_-^{-1/2}\bigr),\qquad
 \mat{G}_{\rm const}=\frac{1}{2\sqrt{2c_0}}\begin{bmatrix}1&1\\1&1\end{bmatrix}.
\end{equation*}
Put $\mat{A}_e=\mat{G}_{\rm avg}/g_0$, $\mat{A}_o=\mat{G}_{\rm diff}/g_0$, and $f_p=\sqrt p\,g_0/b$.
\end{definition}

For $p\ge97$ both blocks $\mat{M}_\pm$ are positive definite
(Lemma~\ref{lem:afr-rootbounds}(i) below). The inverse root of
$\mat{F}_{\rm unif}^\top \mat{F}_{\rm unif}$ on its support $E$, with
$\mat{P}_0=\one\one^\top/p$, is therefore
\begin{equation}\label{eq:afr-root}
 (\mat{F}_{\rm unif}^\top \mat{F}_{\rm unif})^{+1/2}=p^{5/2}\bigl[\mat{G}_{\rm avg}\otimes \mat{P}_c+\mat{G}_{\rm diff}\otimes \mat{H}_\chi
 +\mat{G}_{\rm const}\otimes \mat{P}_0\bigr].
\end{equation}
The quantities $\mat{G}_{\rm avg},\mat{G}_{\rm diff},\mat{G}_{\rm const},\mat{A}_e,\mat{A}_o,f_p$ are defined in
Definition~\ref{def:half-inverse-roots}; only the structural constant-minus direction is null.
Multiplying the QR source row in Lemma~\ref{lem:afr-counts}(i) by
the supported inverse root~\eqref{eq:afr-root} gives
\begin{equation}\label{eq:afr-rows}
\begin{split}
 \mat{Q}_{{\rm QR},l}^\top&=\sqrt\varpi\tfrac{\sqrt p}{b}
 \bigl[\mat{G}_{\rm avg}\otimes \mat{P}_c+\mat{G}_{\rm diff}\otimes \mat{H}_\chi\\
 &\hspace{2em}+\mat{G}_{\rm const}\otimes \mat{P}_0\bigr]\mathbf n .
\end{split}
\end{equation}
These operators are real. If $\chi_2(-1)=-1$ then $\mat{M}_-=\overline{\mat{M}_+}$, so
$\mat{G}_{\rm avg}$ is real symmetric and $\mat{G}_{\rm diff}$ is imaginary Hermitian. Also
$\mat{\mathsf T}\otimes \mat{H}_\chi$ acts as $(w,v)\mapsto(\mat{C}_-v,\mat{C}_+w)$.

\begin{lemma}[Inverse-root bounds]\label{lem:afr-rootbounds}
Let $p\ge97$ be prime.
\begin{enumerate}[label=(\roman*)]
\item $\tfrac{11}{15}g_0^{-2}\Id\prec \mat{M}_\pm\prec\tfrac{19}{15}g_0^{-2}\Id$.
\item $\norm{\mat{A}_e-\Id}<3/p$. The matrix $\mat{A}_e=[a_{ij}]$ is real symmetric, with
$1-3/p<a_{11},a_{22}<1+3/p$ and $|a_{12}|=|a_{21}|<3/p$.
\item $\norm{\mat{A}_o+\mat{\mathsf T}/p}<15/(2p^{3/2})$; hence $\norm{\mat{A}_o}<(1+15/(2p))/\sqrt p$.
\item $f_p<8/(p-1)$ and $c_0\ge u^2/p^2$.
\end{enumerate}
\end{lemma}
\begin{proof}
\readability{First we bound the two covariance perturbations to prove (i)
and justify the inverse-root series. We then separate terms containing
even or odd numbers of $\mat{\mathsf D}$ factors: the even part gives (ii), while a cancellation
in the odd part gives the sharper approximation in (iii). Part (iv)
follows from the explicit scalar formulas.}

\readability{\emph{Scalar inputs and positivity, part (i).}}
Write $a=\mathsf e_0$, $h=\op{tr}\mat{\mathsf D}/2$,
$\mat{E}=\mat{\mathsf E}-a\Id$, $\mat{D}=\mat{\mathsf D}-h\Id$, $e^2=\op{tr}(\mat{E}^2)/2$,
$d^2=\op{tr}(\mat{D}^2)/2$ with $e,d\ge0$, and $c=\op{tr}(\mat{E}\mat{D})/2$.
Set $s=\|\mat{\mathsf E}\|$, $t=\|\mat{\mathsf D}\|$, $\ell=39/8$ and $m=21/10$.
These letters are local to this proof.
The traceless Hermitian matrices satisfy $\mat{E}^2=e^2\Id$, $\mat{D}^2=d^2\Id$
and $\mat{E}\mat{D}+\mat{D}\mat{E}=2c\Id$. Substituting the envelope of Lemma~\ref{lem:afr-scalars}
in~\eqref{eq:afr-scal} gives
\begin{readabilityblock}
\begin{equation}\label{eq:readable-e1-root-scalar-bounds}
\begin{gathered}
 \frac5{3p}<a<\frac7{4p},\quad
 e<\frac{25}{8p},\quad \frac4p<d^2<\frac{1089}{256p},\quad
 d<\frac{33}{16\sqrt p},\\
 |h|<\frac7{2p^{3/2}},\quad
 |c|<\frac{24}{5p^{3/2}}.
\end{gathered}
\end{equation}
\end{readabilityblock}
For explicit checks, $p\mathsf x<2$, $2<p\mathsf f<2+7/(2p)\le41/20$,
while $p^2\mathsf k_0<7/2$ and $p^3\mathsf k_1<25/6$ in the real phase.
The lower $a$ bound reduces to $15p^3-13p^2+7p-5>0$.
The envelope also gives $\mathsf e_1<0$ and
\begin{equation*}
 p|\mathsf e_1|\le\frac{2p}{p-1}+\frac13
       +\frac{1+8/p}{3(p-1)^2}<\frac{12}{5},
\end{equation*}
so $p^2e^2<4+(12/5)^2<(25/8)^2$.
Also $pd^2\le(41/20)^2+625/(36p^4)<(33/16)^2$.
In the imaginary phase $h=c=0$; otherwise $h=-\tau\mathsf k_0$ and
$c=\tau(\mathsf x\mathsf k_1+\mathsf e_1\mathsf f)$, whence
\begin{equation*}
 p^{3/2}|c|<\frac{25}{3p^2}
 +\left[\frac{2p}{p-1}+\frac13+\frac{1+8/p}{3(p-1)^2}\right]
  \left(2+\frac7{2p}\right)<\frac{24}{5}.
\end{equation*}
The last bounds hold at $97$ and improve thereafter. Thus
$\|\mat{\mathsf E}\|+\|\mat{\mathsf D}\|\le a+e+|h|+d
<39/(8p)+(33/16+7/(2p))/\sqrt p<4/15$,
proving (i); the last inequality holds at $97$ using $\sqrt{97}>49/5$ and improves thereafter.

\readability{\emph{Separating the inverse-root series.}
The bound $s+t<4/15$ just proved ensures convergence for both signs in
$(\Id+\mat{\mathsf E}\pm\mat{\mathsf D})^{-1/2}$.}
Put $z_n=\binom{2n}{n}/4^n$ and
$\mat{\cC}=3(\mat{\mathsf E}\mat{\mathsf D}+\mat{\mathsf D}\mat{\mathsf E})/8-5\mat{\mathsf D}^3/16$.
Define the convergent parity remainders
\begin{align*}
 \mat{R}_e&=\frac12\sum_{n\ge3}(-1)^nz_n
       [(\mat{\mathsf E}+\mat{\mathsf D})^n+(\mat{\mathsf E}-\mat{\mathsf D})^n],\\
 \mat{R}_o&=\frac12\sum_{n\ge4}(-1)^nz_n
       [(\mat{\mathsf E}+\mat{\mathsf D})^n-(\mat{\mathsf E}-\mat{\mathsf D})^n].
\end{align*}
The convergent binomial series for $(\Id+\mat{\mathsf E}\pm\mat{\mathsf D})^{-1/2}$ gives
\begin{readabilityblock}
\begin{equation}\label{eq:readable-e1-root-parity-expansion}
\begin{aligned}
 \mat{A}_e-\Id&=-\frac{\mat{\mathsf E}}{2}
       +\frac38(\mat{\mathsf E}^2+\mat{\mathsf D}^2)+\mat{R}_e,\\
 \mat{A}_o&=-\frac{\mat{\mathsf D}}{2}+\mat{\cC}
       -\frac5{16}(\mat{\mathsf E}^2\mat{\mathsf D}+\mat{\mathsf E}\mat{\mathsf D}\mat{\mathsf E}
                         +\mat{\mathsf D}\mat{\mathsf E}^2)+\mat{R}_o.
\end{aligned}
\end{equation}
\end{readabilityblock}
The even and odd combinations retain only words containing respectively
an even and an odd number of $\mat{\mathsf D}$ factors; no commutation is assumed.
The coefficients $z_n$ decrease, so geometric summation gives
\begin{align*}
 \|\mat{R}_e\|&\le\frac5{16}
 \frac{s^3+3st^2-s^4+t^4}{(1-s)^2-t^2},\\
 \|\mat{R}_o\|&\le\frac{35}{128}
 \frac{4s^3t+4st^3-3s^4t+2s^2t^3+t^5}{(1-s)^2-t^2}.
\end{align*}
Here $s<\ell/p$, $t<m/\sqrt p$, and
$(1-s)^2-t^2>1-(2\ell+m^2)/p>17/20$.
Dropping the negative numerator terms and evaluating at $97$ yields
\begin{readabilityblock}
\begin{equation}\label{eq:readable-e1-root-remainder-budgets}
\begin{aligned}
 p\|\mat{R}_e\|&<\frac{25}{68}
 \left[\frac{\ell^3}{97^2}+\frac{3\ell m^2+m^4}{97}\right]<\frac7{20},\\
 p^{3/2}\|\mat{R}_o\|&<\frac{175}{544}
 \left[\frac{4\ell^3m+2\ell^2m^3}{97^2}
       +\frac{4\ell m^3+m^5}{97}\right]<\frac45.
\end{aligned}
\end{equation}
\end{readabilityblock}
These bounds improve thereafter.

\readability{\emph{The even part, part (ii).}
We combine the first line of~\eqref{eq:readable-e1-root-parity-expansion}
with its remainder budget in~\eqref{eq:readable-e1-root-remainder-budgets}.}
Since $-a/2+3d^2/8>0$, the quadratic
even term satisfies
\begin{equation*}
 p\left\|-\frac{\mat{\mathsf E}}{2}+\frac38(\mat{\mathsf E}^2+\mat{\mathsf D}^2)\right\|
 <\frac{4681}{6144}+\frac{25}{16}+
   \frac{3\ell^2/8+693/128}{p}+\frac{147}{32p^2}<\frac{53}{20}.
\end{equation*}
\readability{Adding $p\|\mat{R}_e\|<7/20$ from
\eqref{eq:readable-e1-root-remainder-budgets} proves
$\|\mat{A}_e-\Id\|<3/p$.}
The real-phase roots are real symmetric and the imaginary-phase roots are
conjugates, proving the reality and entry bounds in (ii).

\readability{\emph{The odd part, part (iii).}
In the second line of~\eqref{eq:readable-e1-root-parity-expansion}, the
terms collected in $\mat{\cC}$ must be estimated together. The traceless
identities give the cancellation directly:}
\begin{readabilityblock}
\begin{equation}\label{eq:readable-e1-root-odd-cancellation}
\begin{aligned}
 \mat{\cC}&=\left(\frac{3ah}{4}+\frac{3c}{4}
                   -\frac{5h^3}{16}-\frac{15hd^2}{16}\right)\Id\\
 &\quad+\left(\frac{3a}{4}-\frac{5d^2}{16}-\frac{15h^2}{16}\right)\mat{D}
 +\frac{3h}{4}\mat{E}.
\end{aligned}
\end{equation}
\end{readabilityblock}
\readability{The scalar bounds~\eqref{eq:readable-e1-root-scalar-bounds}
imply $|3a/4-5d^2/16|<15/(128p)$. Applying them to the three coefficients
in~\eqref{eq:readable-e1-root-odd-cancellation} consequently gives}
\begin{equation*}
 p^{3/2}\|\mat{\cC}\|<\frac{18}{5}+\frac{495}{2048}
 +\frac{151977/8192+525/64}{p}+\frac{24255}{1024p^2}
 +\frac{1715}{128p^3}<\frac{33}{8}.
\end{equation*}
The mixed cubic terms contribute at most
$15s^2t/16<1/(2p^{3/2})$.
Finally the explicit form of $\mat{\mathsf D}$ gives
$p^{3/2}\|-\mat{\mathsf D}/2+\mat{\mathsf T}/p\|
\le p^2(2+1/p)/(p-1)^2<33/16$:
in the real phase use
$1+\mathsf u+\sqrt{\vartheta^2+(\mathsf u-(p-1)/p)^2}\le2+1/p$;
the imaginary-phase expression is smaller.
\readability{The second line of~\eqref{eq:readable-e1-root-parity-expansion}
now has four contributions: the leading-term error, $\mat{\cC}$, the
mixed cubic terms, and the remainder bounded in
\eqref{eq:readable-e1-root-remainder-budgets}. Adding them gives}
\begin{readabilityblock}
\begin{equation}\label{eq:readable-e1-root-odd-bound}
 p^{3/2}\|\mat{A}_o+\mat{\mathsf T}/p\|
 <\frac{33}{16}+\frac{33}{8}+\frac12+\frac45
 =\frac{599}{80}<\frac{15}{2}.
\end{equation}
\end{readabilityblock}
\readability{Adding $\|\mat{\mathsf T}\|/p=p^{-1/2}$ to
\eqref{eq:readable-e1-root-odd-bound} proves the remaining bound in (iii).}

(iv) $g_0^{-2}>p/16$ gives $g_0<4/\sqrt p$, so $f_p<8/(p-1)$. The bound on $c_0$ is
\eqref{eq:afr-c0}.
\end{proof}

For a source supported on $E$, write
$\sigma_{\min,E}(\mat{F})=\min_{z\in E,\ \|z\|=1}\|\mat{F}z\|$.

\begin{lemma}[Uniform source]\label{lem:acr-uniform}
\begin{enumerate}[label=(\roman*)]
\item $\sigma_{\min,E}(\mat{F}_0)>1/(5p^2)$.
\item For every $l\in\Lambda$: $\pol(\mat{F}_0)_l\,x>\frac{13}{60}\sqrt{\pi_l}$ whenever
$D_l(x)=1$, and $\pol(\mat{F}_0)_l\,x<\frac{17}2\sqrt{\pi_l}$ for every $x\in\Omega$.
\end{enumerate}
\end{lemma}
\begin{proof}
The two bounds use the notation and formulas of Appendix~\ref{sec:afr-roots}, where
$\mat{F}_{\rm unif}=\mat{F}_0$ and $\mat{Q}_{\rm QR}=\pol(\mat{F}_0)$.
\readability{The singular-value bound in (i) follows from the covariance
blocks. For the row bounds in (ii), we first control the training counts,
then insert them into~\eqref{eq:afr-rows}. The resulting row formula
separates a signed count contribution from the inverse-root errors.}

(i) On the nonconstant Fourier planes, $p^5\mat{\Sigma}_0$ has blocks $\mat{M}_\pm$;
on the constant line in $E$ it has eigenvalue $2c_0$.
The root bounds give
$\mat{M}_\pm\succ(11/15)g_0^{-2}\Id\succ(11p/240)\Id$,
while $2c_0\ge2u^2/p^2>g_0^{-2}>p/16$.
Thus $\sigma_{\min,E}(\mat{F}_0)>\sqrt{11/240}\,p^{-2}>1/(5p^2)$.

(ii) Fix a row and write $\alpha=\one_A$, $\beta=\one_B$,
$L=\alpha\odot\beta$, $h=|A\cap B|$ and $\mathbf n=(\mathbf n^L,\mathbf n^R)=np\,b_l$.
Set $E_-=\tfrac12(u-b-u/\sqrt p)$ and $E_+=\tfrac12(u+u/\sqrt p)$.
Lemma~\ref{lem:afr-counts}(i)--(ii), together with $0\le h\le b$, gives
\begin{equation*}
 E_-\le E_l\le E_+.
\end{equation*}

We first give the two count budgets used below. Put
$\mathsf M_p=\sqrt p[\log(p+1)+3-2\log2]/(2\sqrt2)$.
For every affine image $D$ of
$\mathsf I$ or $\mathsf I'=-\mathsf I^c$ and every $y$,
\begin{readabilityblock}
\begin{equation}\label{eq:readable-e1-interval-character-bound}
 \left|\sum_{w\in D}\chi_2(w-y)\right|
 \le \mathsf M_p<\frac{9p}{40}.
\end{equation}
\end{readabilityblock}
Indeed, Fourier inversion and $|\tau|=\sqrt p$ bound this sum by
$p^{-1/2}\sum_{m\ne0}|\widehat{\one_{\mathsf I}}(m)|$.
Put $\theta_m=\pi m/(2p)$ and $K=\lceil b/2\rceil$. The summand is
$(2\sin\theta_m)^{-1}$ for odd $m$ and $(2\cos\theta_m)^{-1}$ for even $m$.
Pair $m$ with $p-m$, use $\sin\theta\ge2\sqrt2\theta/\pi$ and
$\cos\theta\ge1/\sqrt2$ for $\theta\le\pi/4$, and use $K\le(p+1)/4$.
Then the sum is at most
\begin{equation*}
 \frac p{\sqrt2}\sum_{j=1}^K\frac1{2j-1}
       +\sqrt2\lfloor b/2\rfloor
 <\frac p{\sqrt2}(1+\tfrac12\log K)+\frac p{2\sqrt2}
 \le\sqrt p\,\mathsf M_p.
\end{equation*}
The harmonic bound follows by integrating $1/w$ on $[2j-2,2j]$ for $j\ge2$.
The ratio $\mathsf M_p/p$ decreases for $p\ge97$, and
$\mathsf M_{97}<21.7<9\cdot97/40$.
Because $A,B$ have the required affine forms (Lemma~\ref{lem:afr-counts}),
this bounds every entry of $\mat{C}_-\beta$ and $\mat{C}_+\alpha$.

\readability{The character-sum cap~\eqref{eq:readable-e1-interval-character-bound}
gives two count budgets: a coordinate bound for the even inverse-root
error and a norm bound for the odd error.} Set
$\nu_\infty=\max(\|\mat{P}_c\mathbf n^L\|_\infty,\|\mat{P}_c\mathbf n^R\|_\infty)$ and
$\nu_2^2=\|\mat{P}_c\mathbf n^L\|^2+\|\mat{P}_c\mathbf n^R\|^2$.
Each centered count lies between $-E_+/p$ and $(c+\mathsf M_p)/2-E_-/p$;
the latter is below $19p/80+\sqrt p/8+1/2+1/(8p)<p/3$, so $\nu_\infty<p/3$.
The two uncentered blocks have the same sum $E_l$ and respective entry caps
$(c+\mathsf M_p)/2,(b+\mathsf M_p)/2$. Therefore
\begin{readabilityblock}
\begin{equation}\label{eq:readable-e1-count-norm-budget}
 \nu_2^2\le(p/2+\mathsf M_p)E_l-2E_l^2/p
 <(29p/40)E_l-2E_l^2/p<p^3/16.
\end{equation}
\end{readabilityblock}
For the final inequality put $t=p^{-1/2}<1/9$. Since
$E_l\le E_+<p^2(1+t)/8<5p^2/36<29p^2/160$, the quadratic is increasing
through this interval, and its upper bound divided by $p^3$ is
$(19+9t-10t^2)/320<1/16$.

\readability{\emph{Decomposing a Polar row.}
We now apply $\nu_\infty<p/3$ and
\eqref{eq:readable-e1-count-norm-budget} to the inverse-root row.}
Put $\gamma_l=2E_l/(pg_0\sqrt{2c_0})$ and, for $x=x_{az}$, define
\begin{align*}
 S_l(x)&=(\mat{P}_c\mathbf n^L)(a)+(\mat{P}_c\mathbf n^R)(z)
 -[(\mat{C}_-\mathbf n^R)(a)+(\mat{C}_+\mathbf n^L)(z)]/p,\\
 \Delta_e(x)&=x^\top((\mat{A}_e-\Id)\otimes \mat{P}_c)\mathbf n,\\
 \Delta_o(x)&=x^\top((\mat{A}_o+\mat{\mathsf T}/p)\otimes \mat{H}_\chi)\mathbf n.
\end{align*}
The inverse-root row formula~\eqref{eq:afr-rows} now reads
\begin{readabilityblock}
\begin{equation}\label{eq:readable-e1-polar-row-decomposition}
 \frac{\pol(\mat{F}_0)_lx}{\sqrt{\pi_l}}
 =f_p\,[S_l(x)+\Delta_e(x)+\Delta_o(x)+\gamma_l].
\end{equation}
\end{readabilityblock}
Lemma~\ref{lem:afr-rootbounds}(ii)--(iii) and $\mat{H}_\chi^2=\mat{P}_c$ give
\begin{readabilityblock}
\begin{equation}\label{eq:readable-e1-polar-row-errors}
\begin{gathered}
 |\Delta_e|\le(3/p)2\sqrt2\nu_\infty<2\sqrt2,\qquad
 |\Delta_o|\le\frac{15\sqrt2}{2p^{3/2}}\nu_2<\frac{15\sqrt2}8,\\
 |\Delta_e|+|\Delta_o|<\frac{11}2.
\end{gathered}
\end{equation}
\end{readabilityblock}
The formulas for $g_0,c_0$ give
\begin{readabilityblock}
\begin{equation}\label{eq:readable-e1-polar-row-scales}
\begin{gathered}
 \frac8{p+2}<f_p<\frac8{p-1},\qquad
 (u/p)^2\le c_0<(u/p)^2(1+1/(2p)),\\
 \frac{\sqrt p}8<\gamma_l<\frac{\sqrt p}4.
\end{gathered}
\end{equation}
\end{readabilityblock}

The two sides of the $\gamma_l$ estimate follow from the count and covariance
bounds as follows.
Indeed $2E_l\ge u(1-1/c-p^{-1/2})>3u/4$ and
$pg_0\sqrt{2c_0}<6u/\sqrt p$ prove the lower bound on $\gamma_l$.
For the upper bound, substitution of $E_+$ gives
$\gamma_l/\sqrt p\le(1+p^{-1/2})\sqrt{1+2/p}/(4\sqrt2)<1/4$ for $p\ge97$.

\readability{\emph{Lower bound on active inputs.}
With the errors in~\eqref{eq:readable-e1-polar-row-decomposition}
controlled, it remains to bound $S_l(x)$.}
On an active mask, $a\in A,z\in B$. Put $\epsilon_x=\beta(a)+\alpha(z)$.
Substitution of the count formulas yields
\begin{readabilityblock}
\begin{equation}\label{eq:readable-e1-signed-row-expansion}
 \begin{split}
 S_l(x)={}&(p-\epsilon_x)/2-2E_l/p
 +[c(\mat{C}_-\beta)(a)+b(\mat{C}_+\alpha)(z)]/(2p)\\
 &+[(\mat{C}_-L)(a)+(\mat{C}_+L)(z)]/(2p)\\
 &-[(\mat{C}_-\diag(\beta)\mat{C}_+\alpha)(a)
       +(\mat{C}_+\diag(\alpha)\mat{C}_-\beta)(z)]/(2p).
 \end{split}
\end{equation}
\end{readabilityblock}

\readability{In~\eqref{eq:readable-e1-signed-row-expansion}, we bound the
middle-line intersection term first, then apply a projection estimate
to each quadratic term in the last line.}
The middle line is at least $(\epsilon_x-2h)/(2p)$.
For an orthogonal projection $\mat{P}$,
$|v^\top \mat{P}w|\le(|v^\top w|+\|v\|\|w\|)/2$, since $2\mat{P}-\Id$ is orthogonal.
Apply this with $(v,w,\mat{P})=(\mat{C}_+e_a,\mat{C}_+\alpha,\diag\beta)$ and then with
$(\mat{C}_-e_z,\mat{C}_-\beta,\diag\alpha)$.
The two quadratic terms are at most
$(c+\sqrt{(p-1)u})/2$ and $(b+\sqrt{(p-1)u})/2$.
The $h$ terms cancel, giving
\begin{readabilityblock}
\begin{equation}\label{eq:readable-e1-signed-row-lower}
\begin{aligned}
 S_l(x)&\ge p/2-u/p-1+1/p-u/p^{3/2}\\
 &\qquad-\mathsf M_p/2-1/4-\sqrt{(p-1)u}/(2p)\\
 &>11p/80-\sqrt p/2-5/4.
\end{aligned}
\end{equation}
\end{readabilityblock}

\readability{\emph{Upper bound on all inputs.}
Here the count budgets suffice, without assuming $D_l(x)=1$.}
For arbitrary $x$, Cauchy--Schwarz on the character rows and
$\|\mat{P}_c\mathbf n^L\|+\|\mat{P}_c\mathbf n^R\|\le\sqrt2\nu_2<3p^{3/2}/8$ give
\begin{readabilityblock}
\begin{equation}\label{eq:readable-e1-signed-row-upper}
 S_l(x)\le p/2+\mathsf M_p-2E_-/p+\frac{\sqrt{p-1}}p\sqrt2\nu_2
 <17p/20+\sqrt p/4+1/2.
\end{equation}
\end{readabilityblock}

\readability{\emph{Assembling the row margins.}
Insert~\eqref{eq:readable-e1-signed-row-lower} and
\eqref{eq:readable-e1-signed-row-upper} into
\eqref{eq:readable-e1-polar-row-decomposition}, using the error budget
\eqref{eq:readable-e1-polar-row-errors} and the scales
\eqref{eq:readable-e1-polar-row-scales}.}
The bracket in the row formula is consequently above
$11p/80-3\sqrt p/8-27/4>0$ on the mask and below
$17p/20+\sqrt p/2+6$ everywhere. Thus
\begin{readabilityblock}
\begin{equation*}
\begin{aligned}
 \frac{\pol(\mat{F}_0)_lx}{\sqrt{\pi_l}}
 &>\frac{11p-30\sqrt p-540}{10(p+2)}>\frac{13}{60}
 \quad(D_l(x)=1),\\
 \frac{\pol(\mat{F}_0)_lx}{\sqrt{\pi_l}}
 &<\frac{34p/5+4\sqrt p+48}{p-1}<\frac{17}2.
\end{aligned}
\end{equation*}
\end{readabilityblock}
For the upper bound a nonpositive bracket is immediate. The final scalar
inequalities are $53p-180\sqrt p-3266>0$ and $17p-40\sqrt p-565>0$;
both hold at $97$ using $\sqrt{97}<10$, and increase thereafter.
\end{proof}
\proofstepheading{Helpers and training-invisible queries}\label{sec:acr-groups}
\label{sec:notation-halfgrid-state}
\begin{definition}[Fixed active bank and width padding]
\label{def:acr-groups}
The active physical bank consists of $N/2$ core pair rows $B\in\Lambda_c$,
$p$ ordinary helpers $u\in\F_p$, and $p$ training-invisible queries
$u\in\F_p$, with the roles described in Section~\ref{paper:theory:half}.
At every admitted width $M\ge M_{\min}$, append $M-(N/2+2p)$ dormant rows;
this is at least $n$ and equals $n$ at $M_{\min}$. The canonical incoming row of each dormant unit is
$-\one_{2p}^\top$ and its head column is zero. Its grid preactivation is
$-2$. No coordinate is tied in the full-dimensional initialization ball.
\end{definition}
\begin{definition}[Row helpers]\label{def:half-helpers}
Define
\begin{equation}\label{eq:acr-row-helper}
 \mat{J}_{\rm row}\in\{0,1\}^{p^2\times p},\qquad
 (\mat{J}_{\rm row})_{(a,b),u}=\one[a=u].
\end{equation}
The incoming rows are
$\mat{W}^{\rm nom}_{N,u}=\mat{W}^0_{N,u}=L_{\rm lift}(2e_u-\one;0)$, and their heads are
\begin{equation}\label{eq:round15-init-helper}
 \mat{V}^0_{N,u}=0.
\end{equation}
\end{definition}

Their nominal feature table is $\mat{H}_N^{\rm nom}=L_{\rm lift}\mat{J}_{\rm row}$.
At the nominal helper rows the helper features are
$\ReLU(\mat{W}^0_{N,u}x_{ab})=L_{\rm lift}\one\{a=u\}$, so exactly one helper is active
at each input.

\begin{definition}[Training-invisible query rows]\label{def:half-query}
For $u\in\F_p$, let $\Test_u=\{(u,w)\in\Test\}$ and $c(u)=2u+1$.
Put $A_*=21R/10$ and $\upsilon=\epsilon$, and set
\begin{equation*}
 \mat{W}_{Q,u}^{\rm wit}=\mat{W}^0_{Q,u}
 =A_*\bigl(2e_u-\one;-2\one_{u+Q}\bigr).
\end{equation*}
The query heads are
\begin{equation}\label{eq:round15-init-query}
 \mat{V}_{Q,u}^{\rm wit}=\mat{V}^0_{Q,u}
 =\frac{24\epsilon}{p(21R/10)}T_{c(u)}
 =\frac{24\upsilon}{pA_*}T_{c(u)}.
\end{equation}
Let $\mat{H}_Q^{\rm wit}$ denote the query ReLU table.
\end{definition}

\begin{lemma}[Query witness]\label{lem:acr-query}
\begin{enumerate}[label=(\roman*)]
\item $c(u)\ne y^\star(x)$ for every $x\in\Test_u$, and $c(u)=y_i$ for the training input
$i=(u,u+1)$.
\item The witness just defined
has preactivations equal to $A_*$ on $\Test_u$ and in $\{-A_*,-3A_*\}$ elsewhere; in
particular they are negative on $\Train$. Its feature table $\mat{H}_{\rm Q}^{\rm wit}$ has column
$u$ equal to $A_*\one_{\Test_u}$, which is nonnegative and supported in one first-operand
row, and $\norm{\mat{H}_{\rm Q}^{\rm wit}}_F^2=p(p+1)A_*^2/2<A_*^2p^2$. Its logits vanish on
$\Train$ and equal $(24\epsilon/p)T_{c(u)}$ on $\Test_u$.
\end{enumerate}
\end{lemma}
\begin{proof}
(i) $c(u)=u+w$ means $w-u=1\in Q$, so $(u,w)\in\Train$; and $(u,u+1)\in\Train$ has label
$2u+1$.
(ii) The left part of the preactivation at $(a,w)$ is $\pm A_*$ according to $a=u$ or
$a\ne u$, and the right part is $-2A_*$ if $w-u\in Q$ and $0$ otherwise. On row $u$ this
gives $A_*$ at held-out and $-A_*$ at training inputs, and on the other rows $-A_*$ or
$-3A_*$. The energy is $p\cdot\frac{p+1}2A_*^2$. On $\Test_u$ only row $u$ is active.
\end{proof}

\paragraph{Actual source identities.}
Write $P(f)_c=e^{f_c}/\sum_k e^{f_k}$ for softmax and
$r_x=e_{y^\star(x)}-P(f(x))$ for the negative derivative of cross-entropy
with respect to the example logits.
The sources $\mat{F}=-\nabla_{\mat{W}}L$ and $\mat{G}=-\nabla_{\mat{V}}L$ have, for a
hidden row with incoming row $w$ and head column $c$, the row and column
\begin{equation}\label{eq:aind-sources}
 \mat{F}_{\rm row}=\frac1n\sum_{x\in\Train}\one\{wx>0\}\,(c^\top r_x)\,x^\top,\qquad
 \mat{G}_{\rm col}=\frac1n\sum_{x\in\Train}r_x\ReLU(wx),
\end{equation}
since $\ReLU'(0)=0$. Rows of $\mat{F}$ lie in $E$ (Lemma~\ref{lem:acr-design}), so $\mat{F}\mat{P}_E=\mat{F}$ and
$\rank \mat{F}\le r$ at every state.

\proofstepheading{Physical realization and positive compression}\label{sec:acr-phys}

\begin{definition}[Fixed positive pair isometry]
\label{def:acr-compress}
Put $M_{\rm act}=N/2+2p$. In core/helper/query order define
\begin{equation*}
 \mat{C}_{\rm act}=\diag(\mat{C}_C,\Id_p,\Id_p)\in\R^{(N+2p)\times M_{\rm act}},
 \qquad \mat{C}_{\rm act}^\top \mat{C}_{\rm act}=\Id.
\end{equation*}
It lifts an active physical state by
$\cE_{\rm act}(\mat{W},\mat{V})=(\mat{C}_{\rm act}\mat{W},\mat{V}\mat{C}_{\rm act}^\top)$.
The paired rows have identical masks, so positive homogeneity gives the
analytic observer table $\mat{H}^{\rm ana}=\mat{H}^{\rm act}\mat{C}_{\rm act}^\top$ and
preserves logits and the example-row Gram matrix $\mat{H}\mat{H}^\top$. Dormant rows are ordinary
independent physical coordinates, not part of this isometry.
\end{definition}

\begin{lemma}[Positive lifts and full-column observers]
\label{lem:acr-isometry}\label{lem:awd-observers}
Let $\mat{C}\in\R^{m\times r}$ satisfy $\mat{C}^\top \mat{C}=\Id$, with each row a positive
multiple of a coordinate row. For $(\mat{W},\mat{V})=(\mat{C}\mat{W}_c,\mat{V}_c\mat{C}^\top)$, the feature
table is $\mat{H}=\mat{H}_c\mat{C}^\top$ and the logits agree. For every example-row map $\mat{P}$,
$\|\mat{P}\mat{H}\|_F=\|\mat{P}\mat{H}_c\|_F$, also on each feature block preserved by $\mat{C}$.
The example-row Gram matrix and prescribed ridge penalty agree. When that penalty
is positive, the complete train-only refit predictions, errors and
accuracies agree, and its coefficient matrix is $\mat{C}\widehat{\mat{B}}_c$.
Appending zero feature columns preserves these same quantities, with
the refit assertions again restricted to a positive penalty.
For every compatible matrix $\mat{F}_c$,
\begin{equation*}
 (\mat{C}\mat{F}_c)^\top(\mat{C}\mat{F}_c)=\mat{F}_c^\top \mat{F}_c,\qquad
 \pol(\mat{C}\mat{F}_c)=\mat{C}\pol(\mat{F}_c),
\end{equation*}
including the zero-source case.
\end{lemma}
\begin{proof}
Positive homogeneity of ReLU gives $\mat{H}=\mat{H}_c\mat{C}^\top$ rowwise, including at
zero preactivations. The identity $\mat{C}^\top \mat{C}=\Id$ then gives equality of
logits and example-row Gram matrices, and
\begin{equation*}
 \|\mat{P}\mat{H}\|_F^2
 =\op{tr}(\mat{P}\mat{H}_c\mat{C}^\top \mat{C} \mat{H}_c^\top \mat{P}^\top)
 =\|\mat{P}\mat{H}_c\|_F^2.
\end{equation*}
The same calculation applies blockwise when the lift preserves that
partition. Lemma~\ref{lem:shared-full-ridge-transfer}(i) gives the complete
refit assertions and the zero-column assertion.

If $\mat{F}_c=\mat{U}\mat{\Sigma} \mat{Z}^\top$ is a thin positive SVD, then
$\mat{C}\mat{F}_c=(\mat{C}\mat{U})\mat{\Sigma} \mat{Z}^\top$ is one because $(\mat{C}\mat{U})^\top \mat{C}\mat{U}=\Id$.
This proves both source identities; if $\mat{F}_c=0$, both Polar factors vanish.
\end{proof}

\subsubsection{Explicit initialization and uniformly controlled reference dynamics}\label{sec:a-ref}

Use the construction scales of Definition~\ref{def:half-scales}.
Put $\rho_*=\epsilon^4$ and $\rho_g=\rho_*/8$. The fixed geometric clock
and the certified windows are
\begin{equation}\label{eq:clock}
\begin{aligned}
 q_{\rm clk}&=(9/10)^{1/32},& s&=(1-q_{\rm clk})/\lambda,\\
 J&=\left\lceil\frac{\log(4/\rho_*)}{-\log q_{\rm clk}}\right\rceil,
 &\rho_j&=q_{\rm clk}^{\,j}\quad(0\le j\le J).
\end{aligned}
\end{equation}
The certified windows are
\begin{equation}\label{eq:acr-windows}
 \cI_{\rm early}=\{1\le j\le J:\rho_j\ge9/10\},\qquad
 \cI_{\rm late}=\{1\le j\le J:\rho_*/2\le\rho_j\le\rho_*\}.
\end{equation}

\begin{lemma}[Clock and windows]\label{lem:acr-windows}
For $0<\epsilon<1/200$ and the prescribed windows~\eqref{eq:acr-windows}, $0<\lambda s<1/8$,
$\rho_j\ge\rho_J>\rho_g$ for $0\le j\le J$, and $\rho_J<\rho_*/2$.
Moreover $\cI_{\rm early}=\{1,\ldots,32\}$,
$\cI_{\rm late}$ contains at least sixteen integers, and
$\max\cI_{\rm early}<\min\cI_{\rm late}$.
\end{lemma}
\begin{proof}
\readability{We first locate the last arrival relative to $\rho_*$, then
count late arrivals starting at the first crossing of $\rho_*$.}
Let $a=-\log q_{\rm clk}=\log(10/9)/32$. Since
$q_{\rm clk}>7/8$, $\lambda s=1-q_{\rm clk}<1/8$. With
$J=\lceil\log(4/\rho_*)/a\rceil$,
\begin{readabilityblock}
\begin{equation}\label{eq:readable-e1-clock-terminal}
 q_{\rm clk}\rho_*/4<\rho_J\le\rho_*/4,
\end{equation}
\end{readabilityblock}
so $\rho_J>\rho_g$ and $\rho_J<\rho_*/2$. Also $q_{\rm clk}^{32}=9/10$.
Since $\rho_J\le\rho_*/4<9/10=\rho_{32}$, we have $J>32$ and hence
$\cI_{\rm early}=\{1,\ldots,32\}$.
Let $j_0$ be the first index with $\rho_{j_0}\le\rho_*$; it exists
\readability{by~\eqref{eq:readable-e1-clock-terminal}}, and $j_0\ge1$ since $\rho_0=1>\rho_*$.
Minimality gives, for $0\le k\le15$,
\begin{equation*}
 \rho_*\ge q_{\rm clk}^{j_0+k}
 >q_{\rm clk}^{k+1}\rho_*
 \ge q_{\rm clk}^{16}\rho_*>\rho_*/2,
\end{equation*}
since $q_{\rm clk}^{16}=\sqrt{9/10}>1/2$.
As $\rho_J<\rho_*/2$, these indices are strictly below $J$.
Thus $j_0,\ldots,j_0+15$ belong to $\cI_{\rm late}$. Finally
$\epsilon<1/200$ gives $\rho_*<9/10$, so every late index follows 32.
\end{proof}

Here $\eta=\eta_V$ is the head data step and $d_V=\eta_V\lambda_V$
is its per-step decoupled shrinkage; the continuation below allows both
parameters to vary about zero.
\begin{definition}[Actual update map]\label{def:half-update}
For head parameters $\eta,d_V\in\R$, write $\cU_{\eta,d_V}(\mat{W},\mat{V})$ for
\begin{equation}\label{eq:a-update}
 \bigl(q_{\rm clk}\mat{W}+s\,\cN_{\rm pol}(\mat{F}),\ (1-d_V)\mat{V}+\eta\,\cN_F(\mat{G})\bigr),
\end{equation}
with $\mat{F}=\mat{F}(\mat{W},\mat{V})$, $\mat{G}=\mat{G}(\mat{W},\mat{V})$ and the zero conventions of~\eqref{paper:optimizer}.
Use the departure conventions of
Section~\ref{paper:setup:updates}; the frozen-head reference sets $\eta=d_V=0$.
\end{definition}

For each pair $B$, let $\mathbf j_B$ be its shared inactive-coordinate vector.
Put $\mat{D}_C(x)=\diag((D_B(x))_{B\in\Lambda_c})$,
$\mat{J}_C=(\mathbf j_B^\top)_{B\in\Lambda_c}$,
$g=4\sqrt2R/\rho_g$ and $g_c=\sqrt2g$. The physical anchor is
\begin{equation*}
 \mat{Q}_{{\rm QR},c}=\pol(\mat{F}_{0,c})=\mat{C}_C^\top \mat{Q}_{\rm QR},\qquad
 \mat{H}_C^{{\rm QR},c}(x)=\mat{D}_C(x)\mat{Q}_{{\rm QR},c}x
 =\bigl(\mat{H}^{\rm QR}_{x,:}\mat{C}_C\bigr)^\top.
\end{equation*}
Here $\mat{H}_C^{{\rm QR},c}(x)$ is a feature column; its full table is
$\mat{H}_C^{{\rm QR},c}=\mat{H}^{\rm QR}\mat{C}_C$.
The uniform-source row bounds in Lemma~\ref{lem:acr-uniform}(ii) apply to
the physical pair rows with $\sqrt{P_B}=\sqrt2\sqrt{\pi_l}$.

\begin{definition}[Static convex initialization]\label{def:half-static}
Before initializing the network, solve the following fixed-mask program:
\begin{equation*}
 \begin{aligned}
 \cA&=\{\mat{A}\in\R^{(N/2)\times2p}:\mat{A}\mat{P}_E=\mat{A},\ \|\mat{A}\|_{\rm op}\le1\},\\
 \Phi_\epsilon(\mat{A})&=\frac1n\sum_{x\in\Train}
       \CE(R\bar{\mat{V}}_{C,c}\mat{D}_C(x)\mat{A}x,y^\star(x)),\\
 \mat{Q}_\epsilon&=\argmin_{\mat{A}\in\cA}\Phi_\epsilon(\mat{A}).
 \end{aligned}
\end{equation*}
Lemma~\ref{lem:round15-core-only-bootstrap} proves existence and uniqueness
for small $\epsilon$. This is a one-time training-data-dependent initializer,
not a projection or an additional operation in the trained optimizer.
The analytic lift is $\mat{Q}_\epsilon^{\rm ana}=\mat{C}_C\mat{Q}_\epsilon$.
For this static program only, $P_A(\cdot\mid x)$ denotes the softmax of
$R\bar{\mat{V}}_{C,c}\mat{D}_C(x)\mat{A}x$, and define
\begin{equation}\label{eq:r15-core-source}
 \widehat{\mat{F}}(\mat{A})_B=\frac{\sqrt{P_B}}n
    \sum_{x\in\Train}D_B(x)P_A(k_B\mid x)x^\top.
\end{equation}
The formal masks are used only to define this convex program; its minimizer
will be shown to have precisely these masks in the actual ReLU network.

The analytic core rows
(used only as the positive lift of the physical pair bank) are
\begin{equation}\label{eq:round15-init-core}
 \mat{W}^{0,\rm ana}_{C,l}=R \mat{Q}_{\epsilon,l}^{\rm ana}-g\mathbf j_l^\top,
 \qquad \mat{V}^{0,\rm ana}_{C,l}=-v\sqrt{\pi_l}\,T_{k_l}.
\end{equation}
The physical incoming rows are
$\mat{W}^0_{C,B}=R\mat{Q}_{\epsilon,B}-g_c\mathbf j_B^\top$, with heads
\begin{equation}\label{eq:round15-compressed-core}
 \mat{V}^0_{C,B}=-v\sqrt{P_B}\,T_{k_B}.
\end{equation} They equal
$\mat{C}_C^\top \mat{W}_C^{0,\rm ana}$ and $\mat{V}_C^{0,\rm ana}\mat{C}_C$, respectively.
Starting from these rows, the helper/query rows and the dormant rows of
Definition~\ref{def:acr-groups}, keep the heads frozen and update by
$\cU_{0,0}$. Let $\mat{W}_{C,j}^\circ$ be the physical core incoming block at
arrival $j$, so
$\mat{H}_{C,j}^{\rm phys}=\ReLU(\mat{X}_{\rm grid}(\mat{W}_{C,j}^\circ)^\top)$.
\end{definition}

\begin{definition}[Discrete reference feature tables]
The reference trajectory is constructed in
Lemma~\ref{lem:round15-core-only-bootstrap}; only its discrete arrivals
are used. The physical active feature table and its padded width-$M$
table are
\begin{equation}\label{eq:acr-actual-table}
 \mat{H}_j^{\rm act}=[\mat{H}_{C,j}^{\rm phys}\mid \mat{H}_{N,j}^{\rm phys}\mid \mat{H}_{Q,j}^{\rm phys}],
 \qquad \mat{H}_j^M=[\mat{H}_j^{\rm act}\mid0_{p^2\times(M-(N/2+2p))}].
\end{equation}
Each block is the ReLU feature table of the corresponding physical
rows at arrival $j$.
The analytic observer, a lift of this same active trajectory, is
\begin{equation}\label{eq:acr-ana-table}
 \mat{H}_j^{\rm ana}=[\mat{H}_{C,j}^{\rm phys}\mat{C}_C^\top\mid
                 \mat{H}_{N,j}^{\rm phys}\mid \mat{H}_{Q,j}^{\rm phys}].
\end{equation}
At every reference arrival the helper and query blocks are respectively
$\rho_jL_{\rm lift}\mat{J}_{\rm row}$ and $\rho_j\mat{H}_Q^{\rm wit}$; the dormant
feature columns vanish.
The complete refit with its prescribed penalty is computed from any of
$\mat{H}_j^M,\mat{H}_j^{\rm act},\mat{H}_j^{\rm ana}$; these tables have the same example-row
Gram matrix $\mat{H}\mat{H}^\top$ and Frobenius norms after the fixed isometric lift
and zero padding.
The point feature $\mat{H}_{C,j}^{\rm phys}(x)=h_{C,j}^{\rm phys}(x)$ is a
column vector, and $H_{B,j}(x)$ is its component indexed by $B$.
\end{definition}
\paragraph{Feature-table order.}
Analytic observer tables use the fixed order core/helper/query. Physical
width-$M$ tables append dormant zero columns at the right. The analytic lift
in~\eqref{eq:acr-ana-table} preserves the example-row Gram matrix; zero
padding preserves it as well.

\begin{lemma}[Static convex core and exact reference dynamics]
\label{lem:round15-core-only-bootstrap}
Fix a prime $p\ge101$ and $\lambda>0$. For all sufficiently small $\epsilon>0$,
the static program just defined has a unique
minimizer $\mat{Q}_\epsilon$. It is a supported partial isometry,
$\mat{Q}_\epsilon^\top \mat{Q}_\epsilon=\mat{P}_E$, and $\mat{Q}_\epsilon\to \mat{Q}_{{\rm QR},c}$
as $\epsilon\downarrow0$.

The actual frozen-head reference initialized above has, for every
$0\le j\le J$, the exact core weights and feature vector
\begin{equation}\label{eq:r15-core-table}
 \mat{W}_{C,j}^\circ=R\mat{Q}_\epsilon-\rho_jg_c\mat{J}_C,
 \qquad \mat{H}_{C,j}^{\rm phys}(x)=R \mat{D}_C(x)\mat{Q}_\epsilon x.
\end{equation}
Every departure source has rank $r$. The grid gates are the prescribed ones:
core row $B$ is active exactly on $D_B$, helper row $u$ exactly when the
first operand is $u$, query row $u$ exactly on $\Test_u$, and dormant rows never.
Uniformly in $j\le J$, the physical core table divided by $R$ converges
to $\mat{H}_C^{{\rm QR},c}$; on each active mask its values satisfy
\begin{equation*}
 \tfrac15 R\sqrt{P_B}<H_{B,j}(x)<9R\sqrt{P_B}.
\end{equation*}
The reference training logits, and hence its training CE, are constant
throughout this horizon.
\end{lemma}
\begin{proof}
\readability{The proof has three steps. Uniform closeness to the source
$\mat{F}_{0,c}$ gives full supported rank throughout the static program.
First-order optimality then identifies its minimizer as a Polar factor
and proves uniqueness. Finally, strict gate signs let us substitute that
same factor into the actual update and verify~\eqref{eq:r15-core-table}
by induction.}

\paragraph{The static program.}
The objective $\Phi_\epsilon$ is continuous and convex on the compact convex
set $\cA$, so it has a minimizer. Lemma~\ref{lem:acr-dict}(iv)
and the fixed pair lift give $\bar{\mat{V}}_{C,c}\bar{\mat{V}}_{C,c}^\top=(v^2/p)\mat{P}_c$.
For every $\mat{A}\in\cA$, therefore,
\begin{readabilityblock}
\begin{equation}\label{eq:readable-e1-static-logit-bound}
 \|R\bar{\mat{V}}_{C,c}\mat{D}_C(x)\mat{A}x\|_2\le\epsilon\sqrt{2/p}.
\end{equation}
\end{readabilityblock}
Label exclusion gives $-\nabla_{\mat{A}}\Phi_\epsilon(\mat{A})=\epsilon\widehat{\mat{F}}(\mat{A})$,
with the fixed-mask source~\eqref{eq:r15-core-source}.
\readability{Each softmax coordinate has gradient $\ell_1$ norm at most
$1/2$, so~\eqref{eq:readable-e1-static-logit-bound} controls its deviation
from $1/p$. Using $\sum_BP_BD_B(x)\le1$ and $\|x\|_2=\sqrt2$ gives}
\begin{readabilityblock}
\begin{equation}\label{eq:readable-e1-static-source-bound}
 \|\widehat{\mat{F}}(\mat{A})-\mat{F}_{0,c}\|_F
 \le\sqrt2\sup_{x,k}|P_A(k\mid x)-1/p|
 \le\epsilon/\sqrt p.
\end{equation}
\end{readabilityblock}
\readability{All sources are supported on $E$. The uniform bound
\eqref{eq:readable-e1-static-source-bound}, together with the singular-value
floor in Lemma~\ref{lem:acr-uniform}(i) and the pair isometry, therefore
gives full supported rank for every $\mat{A}\in\cA$ when $\epsilon$ is small.}

\paragraph{Optimality identifies the Polar factor.}
At a minimizer, the first-order inequality says that $\mat{A}$ maximizes
$\ip{\widehat{\mat{F}}(\mat{A})}{\mat{B}}_F$ over $\mat{B}\in\cA$.
For a source of rank $r$ on $E$, the unique maximizer is its supported
Polar factor: in an SVD, each positive singular value forces
$\mat{B}v_i=u_i$ in the inequality $\ip{u_i}{\mat{B}v_i}\le1$.
\readability{Thus every minimizer satisfies the fixed-point and isometry identities}
\begin{readabilityblock}
\begin{equation}\label{eq:readable-e1-static-polar-fixedpoint}
 \mat{A}=\pol(\widehat{\mat{F}}(\mat{A})),\qquad
 \mat{A}^\top\mat{A}=\mat{P}_E.
\end{equation}
\end{readabilityblock}
If $\mat{A},\mat{B}$ are minimizers, convexity makes $\Phi_\epsilon$ constant on
their segment. Thus $\ip{\widehat{\mat{F}}(\mat{A})}{\mat{B}-\mat{A}}_F=0$, so $\mat{B}$ also
maximizes $\ip{\widehat{\mat{F}}(\mat{A})}{\cdot}_F$ on $\cA$.
The unique-maximizer statement above gives $\mat{B}=\mat{A}$.
\readability{Apply~\eqref{eq:readable-e1-static-polar-fixedpoint} at this
unique minimizer. Continuity of the supported Polar map at
$\mat{F}_{0,c}$ and~\eqref{eq:readable-e1-static-source-bound} give
$\mat{Q}_\epsilon\to\mat{Q}_{{\rm QR},c}$.}

\paragraph{Verification of the reference trajectory.}
The anchor row bounds of Lemma~\ref{lem:acr-uniform}(ii) transfer to
$\mat{Q}_\epsilon$, giving the stated active values for small $\epsilon$.
For an inactive mask, $\mat{J}_Cx\ge1$ in that row while
$|(\mat{Q}_\epsilon x)_B|\le\sqrt2$. Lemma~\ref{lem:acr-windows} and the shield give
$\rho_jg_c\ge\rho_Jg_c>16q_{\rm clk}R>\sqrt2R$.
Thus every candidate core preactivation in~\eqref{eq:r15-core-table}
has the prescribed strict sign.

Helpers have zero heads; queries are strictly inactive on training;
dormant rows are negative everywhere. Their sources are zero and they
decay exactly. A zero source row has zero Polar row, and appending such
rows preserves source rank and normalization. Noncore training logits
therefore vanish. The core training logits are exactly the fixed-mask
logits at $\mat{Q}_\epsilon$, with source $v\widehat{\mat{F}}(\mat{Q}_\epsilon)$.
\readability{By~\eqref{eq:readable-e1-static-polar-fixedpoint} and rank $r$,
its normalized Polar factor is $\mat{Q}_\epsilon/\sqrt r$.}

Finally $s/\sqrt r=(1-q_{\rm clk})R$, so substituting the proposed incoming
weights into the actual feature update gives
\begin{readabilityblock}
\begin{equation}\label{eq:readable-e1-reference-update}
 q_{\rm clk}(R\mat{Q}_\epsilon-\rho_jg_c\mat{J}_C)
       +s\mat{Q}_\epsilon/\sqrt r
 =R\mat{Q}_\epsilon-\rho_{j+1}g_c\mat{J}_C.
\end{equation}
\end{readabilityblock}
\readability{The initialization gives the case $j=0$, and
\eqref{eq:readable-e1-reference-update} supplies the induction step.
The gate and source checks above therefore prove the exact trajectory,
gates and rank through $J$.}
The feature convergence now follows directly from $\mat{Q}_\epsilon\to
\mat{Q}_{{\rm QR},c}$, uniformly in all arrivals, without a drift accumulation.
The constant core logits prove the final CE assertion.
\end{proof}
\end{proofrevision}

\begin{proofrevision}
\subsection{Claim II: useful-representation consequences on the same trajectory}\label{sec:a-claim-representation}
\end{proofrevision}

\begin{proofrevision}
\subsubsection{The Fourier-fraction bound}\label{sec:fourier}
\end{proofrevision}
\begin{proofrevision}
We derive the supporting source and energy estimates for $p\ge97$, then prove
the $p\ge101$ anchor bound and net gain needed by Theorem~\ref{paper:thm:half}.
The transfer uses the discrete actual-arrival limits of Lemma~\ref{lem:acr-table-limits}.
All terminating decimals below are exact rationals.
\readability{The argument first isolates the anchor's common component in
\eqref{eq:orthoden}--\eqref{eq:numerator}, then controls its remainder by
comparing the two Polar factors. Proposition~\ref{prop:afr-netgain} combines
these estimates and transfers the resulting Fourier bound to the arrivals.}

\proofstepheading{The anchor and common-component orthogonality}\label{sec:afr-anchor}

\label{sec:notation-halfgrid-energy}
\begin{definition}[Anchor energy and masked incoming map]\label{def:half-anchor-energy}
Put $K_0=\sqrt{u(1-u/p^2)}$, $a_p=(p^2+3)/[3(3p^2+1)]$ and
$\mat{\cJ}_l=\sqrt\varpi g_l$. Let
$\bar{\mat{F}}=\sqrt\varpi \mat{D}_{\rm grid}^\top \mat{X}_{\rm grid}$, $\bar{\mat{Q}}=\pol(\bar{\mat{F}})$,
$\mat{G}=p^3\mat{F}_{\rm unif}$, $\bar\alpha=(p+1)/(4p)$ and
$m=\ip{\mat{Q}_{\rm QR}}{\bar{\mat{F}}}_F/u$.
The masked incoming map is
$\cL(\mat{T})=\mat{D}_{\rm grid}\odot(\mat{X}_{\rm grid}\mat{T}^\top)$ for an incoming matrix $\mat{T}$.
The mask indicators and the full-bank average $\E$ are as above,
equivalently evaluated on the positive pair lift with the original weights.
\end{definition}
Here $\mat{\cJ}$ is a table, distinct from the scalar line sum $\mathsf J$.
For row $l$ write $\mat{Q}_{{\rm QR},l}=(w_l^\top,v_l^\top)$ in the two operand blocks.
Then its masked column is
\begin{equation}\label{eq:afr-column}
 H_l(a,z)=1_{A_l}(a)1_{B_l}(z)[w_l(a)+v_l(z)].
\end{equation}
\begin{lemma}[Full-mask common-component identity]\label{lem:afr-shaped}
For the uniform-source anchor $\mat{H}=\mat{H}^{\rm QR}$,
\begin{equation}\label{eq:orthoden}
 \|\Cgrid \mat{H}\|_F^2=m^2K_0^2+\|\Cgrid(\mat{H}-m\mat{\cJ})\|_F^2,
 \qquad \|\Cgrid\mat{\cJ}\|_F^2=K_0^2.
\end{equation}
Moreover,
\begin{equation}\label{eq:numerator}
 \|\Padd \mat{H}\|_F^2=a_pK_0^2m^2+\|\Padd(\mat{H}-m\mat{\cJ})\|_F^2,
 \qquad \|\Padd\mat{\cJ}\|_F^2=a_pK_0^2.
\end{equation}
For every incoming matrix $\mat{T}$, the centered masked map satisfies
\begin{equation*}
 \|\Cgrid\cL(\mat{T})\|_F^2\le(p-u/p)\|\mat{T}\|_F^2.
\end{equation*}
\end{lemma}
\begin{proof}
\readability{The rectangle calculation gives the indicator energies and
the masked-map bound. The remaining task is to identify the two cross
moments that make the energy decompositions orthogonal.}

\emph{Rectangle geometry.}
Every rectangle has size $u$, so its centered indicator has squared norm
$K_0^2$. Common dilation and independent translations reduce addition counts
to $t(y)=y+1$ for $0\le y<b$, $t(y)=2b-y$ for $b\le y<2b$, and $t(2b)=0$.
Thus $\sum_y t(y)^2=up/3$ and
$\Padd g(a,z)=t(a+z)/p-u/p^2$, proving the two indicator norms.

For one incoming row, restrict its operand profiles to $A,B$, and decompose
them into active means $w_0,v_0$ and mean-zero parts $w^o,v^o$.
The centered masked squared norm is
\begin{equation*}
 c\|w^o\|^2+b\|v^o\|^2+K_0^2(w_0+v_0)^2.
\end{equation*}
Its input squared norm is at least
$\|w^o\|^2+\|v^o\|^2+bw_0^2+cv_0^2$.
Weighted Cauchy--Schwarz gives the row-map norm squared
$\max\{b,c,K_0^2(1/b+1/c)\}=p-u/p$. Summing rows proves the map bound.

\emph{The full-mask cross moment.}
Let $\cL^*$ be the Frobenius adjoint of $\cL$, and put
$\mat{Z}=\cL^*\Padd\mat{\cJ}-a_p(1-u/p^2)\bar{\mat{F}}$.
For each mask write $\alpha=\one_A$, $\beta=\one_B$ and let $\psi_A$ be
the centered restriction to $A$ of $a\mapsto\sum_{z\in B}(\Padd g_l)(a,z)$,
extended by zero outside $A$; $\psi_B$ reverses operands.
Set $\mat{M}_A=\expct{\alpha\psi_A^\top}$ and $\mat{M}_B=\expct{\beta\psi_B^\top}$.
The adjoint map gives
\begin{readabilityblock}
\begin{equation}\label{eq:readable-e2-adjoint}
 \cL^*\Cgrid\mat{\cJ}=(1-u/p^2)\bar{\mat{F}},
 \qquad
 \mat{Z}_l=\sqrt\varpi(\psi_A^\top,\psi_B^\top).
\end{equation}
\end{readabilityblock}
Indeed each active restriction has mean $a_pK_0^2/b$ or $a_pK_0^2/c$,
because its total is $\|\Padd g_l\|^2=a_pK_0^2$.
For fixed slope and $A$,
\begin{equation*}
 \sum_{z\in B}(\Padd g_l)(a,z)
 =\frac1p\sum_{a'\in A}|B\cap(B+a'-a)|-\frac{cu}{p^2},
\end{equation*}
so $\psi_A$ is independent of the translation of $B$.
The analogous statement holds for $\psi_B$.

Full-mask averaging is invariant under every common affine map of the field.
Under $(a,z)\mapsto(\sigma a+t,\sigma z+t)$, the addition coordinate becomes
$\sigma(a+z)+2t$. Thus $\Padd$ commutes with the action and the two
restricted profiles transform covariantly.
Consequently $\mat{M}_A$ and $\mat{M}_B$ commute with its permutation matrices.
The affine group is two-transitive, so each matrix has the form
$r\Id+s\one\one^\top$. For $\mat{M}_A$, both $\mat{M}_A\one=0$ and
$\op{tr}\mat{M}_A=\expct{\ip{\alpha}{\psi_A}}=0$.
These equations force $r=s=0$. The same argument gives $\mat{M}_B=0$.
Since $\one^\top \mat{M}_A=b\expct{\psi_A^\top}$ and
$\one^\top \mat{M}_B=c\expct{\psi_B^\top}$, both profile expectations also vanish.

The exact QR count formulas, averaged over the translation of $B$, give
\begin{equation*}
 \expctover{B}{\mathbf n^L}=\tfrac c2(1-p^{-1})\alpha,
 \qquad
 \expctover{B}{\mathbf n^R}=\tfrac c{2p}(b\one-\alpha+\mat{C}_+\alpha).
\end{equation*}
For the opposite conditioning they give
\begin{equation*}
 \expctover{A}{\mathbf n^R}=\tfrac b2(1-p^{-1})\beta,
 \qquad
 \expctover{A}{\mathbf n^L}=\tfrac b{2p}(c\one-\beta+\mat{C}_-\beta).
\end{equation*}
The source definition gives
$\mat{G}_l=(p/b)\sqrt\varpi(\mathbf n^{L\top},\mathbf n^{R\top})$.
Since $\mat{M}_A=\mat{M}_B=0$ and $\mat{C}_\pm$ are fixed linear operators, the four
averaged count identities therefore imply $\mat{G}^*\mat{Z}=0$.
Writing the supported Polar factor as
$\mat{Q}_{\rm QR}=\mat{G}(\mat{G}^*\mat{G})^{\dagger/2}$ therefore gives
$\ip{\mat{Q}_{\rm QR}}{\mat{Z}}_F=0$.

\readability{Using \eqref{eq:readable-e2-adjoint} and
$\ip{\mat{Q}_{\rm QR}}{\mat{Z}}_F=0$ in the definition of $\mat{Z}$ gives
the required cross moments:}
\begin{readabilityblock}
\begin{equation}\label{eq:readable-e2-cross-moments}
 \ip{\Cgrid \mat{H}}{\Cgrid\mat{\cJ}}=mK_0^2,
 \qquad
 \ip{\Padd \mat{H}}{\Padd\mat{\cJ}}=a_pmK_0^2.
\end{equation}
\end{readabilityblock}
\readability{Together with the indicator norms, \eqref{eq:readable-e2-cross-moments}
identifies $m\Cgrid\mat{\cJ}$ and $m\Padd\mat{\cJ}$ as the respective
orthogonal projections. Pythagoras gives \eqref{eq:orthoden} and \eqref{eq:numerator}.}
\end{proof}

\proofstepheading{Same-source energy estimates}\label{sec:afr-energy}

\begin{definition}[Binary source comparison]\label{def:half-comparison}
For $t\ne0$, subscripts $t$ indicate the ordered input pair $(e_t^L,e_t^R)$.
Put $\varepsilon_p=(p^2+3)/[3(p-1)^3]$ and
$\delta_Q^2=\|\mat{Q}_{\rm QR}-\bar{\mat{Q}}\|_F^2$.
The common active value used below is
$h_\circ=2\sqrt{p/(p+1)}+\sqrt{2/p}$.
For $t\ne0$, use the ordered Fourier pair $(e_t^L,e_t^R)$ and put
\begin{equation*}
 \eta_t=\frac{-1+\chi_2(t)\tau}{p-1},\qquad
 \mat{S}_t=\sqrt{\bar\alpha}\diag(c,b),\qquad \mat{K}_t=\bar{\mat{Q}}_t^*\mat{G}_t.
\end{equation*}
For the constant mode put $\nu_c=\sqrt2u/\sqrt p$,
$\bar{\mat{Q}}_c=\one/\sqrt N$ and $\ell_c=\mat{G}e_c-\nu_c\bar{\mat{Q}}_c$.
\end{definition}

\begin{lemma}[Polar distance from the binary comparison]\label{lem:afr-den}
For every prime $p\ge97$:
\begin{enumerate}[label=(\roman*)]
\item $\bar{\mat{F}}_t=\bar{\mat{Q}}_t\mat{S}_t$ and $\mat{K}_t=\sqrt{\bar\alpha}
\bigl[\begin{smallmatrix}c&c\bar\eta_t\\b\eta_t&b\end{smallmatrix}\bigr]$.
On constants, $\bar{\mat{F}}e_c=\nu_c\bar{\mat{Q}}_c$ and
$\bar{\mat{F}}(\one,-\one)=0$. On every active rectangle the corresponding row of
$\bar{\mat{Q}}$ has the same value $\sqrt\varpi\,[\bar\alpha^{-1/2}+\sqrt{2/p}]$.
\item $\mat{G}e_c=\nu_c\bar{\mat{Q}}_c+\ell_c$, with $\ell_c\perp\bar{\mat{Q}}_c$ and
$\norm{\ell_c}^2=\varepsilon_p\nu_c^2$.
\item $\delta_Q^2<2$.
\end{enumerate}
\end{lemma}
\begin{proof}
(i) Apply~\eqref{eq:afr-mask-kernel} with
$\bar{\mat{F}}=\mat{\cB}^\top \mat{X}_{\rm grid}$ and
$\mat{G}=(p/b)\mat{\cB}^\top\mat{\cP}_Q\mat{X}_{\rm grid}$.
The unmasked input modes have sole Fourier coefficients $\sqrt p$ at $(t,0)$
and $(0,t)$; the masked coefficients in~\eqref{eq:afr-source-gram} give
\begin{equation*}
 \bar{\mat{F}}_t^*\bar{\mat{F}}_t=\bar\alpha\diag(c^2,b^2),\qquad
 \bar{\mat{F}}_t^*\mat{G}_t=\frac1{2p}
 \begin{bmatrix}c^3&c^3\bar\eta_t\\b^2c\eta_t&b^2c\end{bmatrix}=\mat{S}_t\mat{K}_t.
\end{equation*}
Every row of $\bar{\mat{F}}$ is $\sqrt\varpi(c\one_A,b\one_B)$, and its constant
coordinate is $\sqrt\varpi\sqrt{2/p}\,u$. Consequently
\begin{equation*}
 \bar{\mat{Q}}=\bar{\mat{F}}[\bar\alpha^{-1/2}\diag(c^{-1},b^{-1})\otimes \mat{P}_c+
 \nu_c^{-1}e_ce_c^\top],\qquad
 \bar{\mat{Q}}_lx_{az}=\sqrt\varpi[\bar\alpha^{-1/2}(\one_A(a)+\one_B(z)-1)+\sqrt{2/p}].
\end{equation*}
This proves all claims in (i), including the common active value.

(ii) The constant coordinate of row $l$ of $\mat{G}$ is
$\sqrt\varpi(p/b)\sqrt{2/p}\,n_l$. Lemma~\ref{lem:afr-counts}(iv) gives its
mean and variance, so $\norm{\ell_c}^2=(2p/b^2)\Var\paren{n_l}=\varepsilon_p\nu_c^2$.

\readability{(iii) We bound the nonconstant Polar blocks by a source-error
estimate and lower singular-value bounds, then add the constant block
from (ii).} Write $\op{Herm}(\mat{A})=(\mat{A}+\mat{A}^*)/2$ and set
\begin{equation*}
 \mat{S}'_t=\sqrt{\bar\alpha}
 \begin{bmatrix}c&p\bar\eta_t/2\\p\eta_t/2&b\end{bmatrix}.
\end{equation*}
By covariance and (i), $\mat{G}_t^*\mat{G}_t=p\mat{M}_{\chi_2(t)}$,
$\bar{\mat{Q}}_t^*\mat{G}_t=\mat{K}_t$, and $\mat{S}'_t=\op{Herm}(\mat{K}_t)$.
Set $\mat{E}_t=\mat{G}_t-\bar{\mat{Q}}_t\mat{S}'_t$ and $\Gamma=p\sqrt{\bar\alpha}/2$.
Thus $\|\mat{E}_t\|_F^2=\op{tr}(\mat{G}_t^*\mat{G}_t)-\op{tr}((\mat{S}'_t)^2)$.
Equal phase counts, $\mat{M}_++\mat{M}_-=2g_0^{-2}(\Id+\mat{\mathsf E})$, and
$\sum_{t\ne0}|\eta_t|^2=(p+1)/(p-1)$ give the exact trace sums
\begin{equation*}
 \sum_{t\ne0}\op{tr}(\mat{G}_t^*\mat{G}_t)=2p(p-1)g_0^{-2}(1+\mathsf e_0),\quad
 \sum_{t\ne0}\op{tr}((\mat{S}'_t)^2)
 =\bar\alpha\left[(p-1)(b^2+c^2)+\frac{p^2(p+1)}{2(p-1)}\right].
\end{equation*}
The Gauss-sum identity $|\tau|^2=p$ and $\sum_{t\ne0}\chi_2(t)=0$
give the stated $\eta$ sum, so
\begin{readabilityblock}
\begin{equation}\label{eq:readable-e2-source-error}
\begin{aligned}
 \sum_{t\ne0}\frac{\|\mat{E}_t\|_F^2}{\Gamma^2}
 &=2(p-1)\left[(1+p^{-2})\mathsf e_0-\frac{p+1}{(p-1)^2}\right]\\
 &=\frac{2(p+1)(2p^2-3)}{3p^2(p-1)}
 =\frac43\frac{p+1}{p-1}\left(1-\frac3{2p^2}\right)\\
 &<\frac43\cdot\frac{49}{48}=\frac{49}{36}.
\end{aligned}
\end{equation}
\end{readabilityblock}
For full-column-rank $\mat{X}=\mat{Q}\mat{T}$, $\mat{Y}=\mat{U}\mat{S}$ with $\mat{T},\mat{S}\succ0$ and $\mat{Q},\mat{U}$ isometries,
\begin{equation*}
 2\op{Re}\ip{\mat{X}-\mat{Y}}{\mat{Q}-\mat{U}}_F
 =2\op{tr}\!\left[\mat{T}\bigl(\Id-\op{Herm}(\mat{Q}^*\mat{U})\bigr)
 +\mat{S}\bigl(\Id-\op{Herm}(\mat{U}^*\mat{Q})\bigr)\right]
 \ge(\sigma_{\min}(\mat{T})+\sigma_{\min}(\mat{S}))\|\mat{Q}-\mat{U}\|_F^2,
\end{equation*}
because each $\Id-\op{Herm}(\mat{Q}^*\mat{U})\succeq0$ and its trace is
$\|\mat{Q}-\mat{U}\|_F^2/2$. Cauchy--Schwarz therefore yields
$\|\mat{Q}-\mat{U}\|_F\le2\|\mat{X}-\mat{Y}\|_F/(\sigma_{\min}(\mat{X})+\sigma_{\min}(\mat{Y}))$.
Lemma~\ref{lem:afr-rootbounds}(i) and $g_0^2\bar\alpha=4p/(p^2+1)$ give
\begin{readabilityblock}
\begin{equation}\label{eq:readable-e2-source-singular}
 \frac{\sigma_{\min}(\mat{G}_t)}{\Gamma}
 >\sqrt{\frac{11}{15}(1+p^{-2})}>\frac{17}{20}.
\end{equation}
\end{readabilityblock}
Also $\mat{S}'_t=\sqrt{\bar\alpha}
\left[\begin{smallmatrix}c&p\bar\eta_t/2\\p\eta_t/2&b\end{smallmatrix}\right]$
and $|\eta_t|\le(\sqrt p+1)/(p-1)$ imply
\begin{readabilityblock}
\begin{equation}\label{eq:readable-e2-comparison-singular}
 \frac{\lambda_{\min}(\mat{S}'_t)}{\Gamma}
 \ge1-\frac1p-|\eta_t|
 \ge1-\frac1p-\frac1{\sqrt p-1}>\frac78
 \qquad(p\ge97).
\end{equation}
\end{readabilityblock}
The last inequality follows from $1/p+1/(\sqrt p-1)<1/97+5/44<1/8$.
In particular $\mat{S}'_t\succ0$. Apply the Polar inequality to
$\mat{X}=\mat{G}_t$, $\mat{Y}=\bar{\mat{Q}}_t\mat{S}'_t$, whose polar factors are $\mat{Q}_{{\rm QR},t}$ and
$\bar{\mat{Q}}_t$. \readability{The two lower bounds
\eqref{eq:readable-e2-source-singular}--\eqref{eq:readable-e2-comparison-singular}
give the denominator $69\Gamma/40$; summing and using
\eqref{eq:readable-e2-source-error} bounds the nonconstant squared Polar distance by}
\begin{equation*}
 \frac{4\sum_{t\ne0}\|\mat{E}_t\|_F^2}{(69\Gamma/40)^2}
 <\frac{78400}{42849}<\frac{19}{10}.
\end{equation*}
By (ii), the constant block contributes exactly
$2[1-(1+\varepsilon_p)^{-1/2}]\le\varepsilon_p$,
where $\varepsilon_p=(p^2+3)/(3(p-1)^3)<1/[2(p-1)]\le1/192$.
Thus $\delta_Q^2<19/10+1/192<2$.

\end{proof}

\subsubsection{Actual-arrival table and Fourier transfer}
\label{sec:round15-p3-transfer}\label{sec:a-table-transfer}

Use the scales in Definition~\ref{def:half-scales} and clock~\eqref{eq:clock},
and write $M_{\rm act}=N/2+2p$ for the active physical width.
The analytic observer table is the fixed active-bank lift
\eqref{eq:acr-ana-table}; it is not a second trajectory.  The physical
table at width $M$ is $[\mat{H}_j^{\rm act}\mid0_{M-M_{\rm act}}]$; these zero
columns are appended to the physical active table, not to its analytic lift.

\begin{lemma}[Actual-arrival table limits]
\label{lem:acr-table-limits}
Fix a prime $p\ge101$ and $\lambda>0$.  As $\epsilon\downarrow0$, the
actual arrivals of the frozen-head reference path satisfy
\begin{align}
 \sup_{\substack{0\le j\le J\\ \rho_j\ge9/10}}
 \left\|\frac{\mat{H}_j^{\rm ana}}{\rho_jL_{\rm lift}}
       -[0\mid \mat{J}_{\rm row}\mid0]\right\|_F&\longrightarrow0,
 \label{eq:acr-early-table-limit}\\
 \sup_{\substack{0\le j\le J\\ \rho_j\le\epsilon^4}}
 \left\|\frac{\mat{H}_j^{\rm ana}}R
       -[\mat{H}^{\rm QR}\mid0\mid0]\right\|_F&\longrightarrow0.
 \label{eq:acr-late-table-limit}
\end{align}
Both limits are uniform over the indicated discrete arrival sets.
\end{lemma}

\begin{proof}
\readability{Both limits follow from the same three-block table. The early
normalization keeps the helper block at unit scale; the late normalization
keeps the core at unit scale.}
The exact trajectory from Lemma~\ref{lem:round15-core-only-bootstrap} and
$\mat{H}^{\rm QR}\mat{C}_C\mat{C}_C^\top=\mat{H}^{\rm QR}$ (reflection-paired feature columns coincide) give
\begin{readabilityblock}
\begin{equation}\label{eq:readable-e2-arrival-blocks}
\begin{aligned}
 \mat{H}_j^{\rm ana}&=
 [\mat{H}_{C,0}^{\rm phys}\mat{C}_C^\top\mid
     \rho_jR\epsilon^{-3}\mat{J}_{\rm row}\mid\rho_j\mat{H}_Q^{\rm wit}],\\
 R^{-1}\mat{H}_{C,0}^{\rm phys}\mat{C}_C^\top&\longrightarrow \mat{H}^{\rm QR}.
\end{aligned}
\end{equation}
\end{readabilityblock}
\readability{For \eqref{eq:acr-early-table-limit}, divide
\eqref{eq:readable-e2-arrival-blocks} by $\rho_jR\epsilon^{-3}$:}
the core block is at most
$(10/9)\epsilon^3\|\mat{H}_{C,0}^{\rm phys}/R\|_F$ and the query block is less
than $(21p/10)\epsilon^3$. Both vanish uniformly on the early arrival set;
the helper block is exactly $\mat{J}_{\rm row}$.

\readability{For \eqref{eq:acr-late-table-limit}, divide
\eqref{eq:readable-e2-arrival-blocks} by $R$.} Since $\rho_j\le\epsilon^4$ and
$\|\mat{J}_{\rm row}\|_F=p$, the total error is bounded by
\begin{equation*}
 \left\|\mat{H}_{C,0}^{\rm phys}\mat{C}_C^\top/R-\mat{H}^{\rm QR}\right\|_F
       +p\epsilon+(21p/10)\epsilon^4\longrightarrow0.
\end{equation*}
This bound is uniform on the late arrival set. Positive lift and zero
padding preserve all row Gram matrices and Fourier energies, so the same
limits serve the physical observers at every admitted width.
\end{proof}

\begin{proposition}[Fourier gain on the frozen-head reference path]
\label{prop:afr-netgain}
For every prime $p\ge101$, all sufficiently small positive $\epsilon$, and
every actual arrival $0\le j\le J$, the centered table has positive energy.
The reference path satisfies
\begin{equation}\label{eq:afr-anchor101}
 \AH(\mat{H}^{\rm QR})>\frac{1275}{11879}>0.107,
\end{equation}
and, for sufficiently small $\epsilon$,
\begin{equation}\label{eq:afr-gain}
 \AH(\mat{H}^{\rm ana}_0)<0.004,
 \qquad
 \AH(\mat{H}^{\rm ana}_j)>0.104\quad(j\in\cI_{\rm late}).
\end{equation}
Consequently
\begin{equation*}
 \AH(\mat{H}^{\rm ana}_j)-\AH(\mat{H}^{\rm ana}_0)>0.100>0.0998
 \qquad(j\in\cI_{\rm late}).
\end{equation*}
\end{proposition}
\begin{proof}
\readability{First we lower-bound the anchor's Fourier fraction by its
angle with the common component. We then use the two table limits to
obtain the early and late bounds on the reference path.}
Lemma~\ref{lem:afr-den}(i) gives $\cL(\bar{\mat{Q}})=h_\circ\mat{\cJ}$,
where $h_\circ>2$: use $2\sqrt{p/(p+1)}>2-2/(p+1)$ and
$\sqrt{2/p}>2/(p+1)$.
Put $x=\Cgrid \mat{H}^{\rm QR}$ and $y=h_\circ\Cgrid\mat{\cJ}$.
The map bound in Lemma~\ref{lem:afr-shaped}, the bound $\delta_Q^2<2$
from Lemma~\ref{lem:afr-den}(iii), and
$(p-u/p)/K_0^2=4p/(p^2-1)$ give
\begin{readabilityblock}
\begin{equation}\label{eq:readable-e2-relative-anchor-error}
 \frac{\|x-y\|_F^2}{\|y\|_F^2}
 <\frac{8p}{(p^2-1)h_\circ^2}
 <\frac{2p}{p^2-1}<1.
\end{equation}
\end{readabilityblock}
Thus $x\ne0$. Let $\theta$ be the angle between $x$ and $y$.
\readability{The cross moments \eqref{eq:readable-e2-cross-moments} and
the numerator decomposition \eqref{eq:numerator} give}
$\ip{x}{y}=h_\circ mK_0^2$ and
$\|\Padd \mat{H}^{\rm QR}\|_F^2\ge a_pm^2K_0^2$, so
$\AH(\mat{H}^{\rm QR})\ge a_p\cos^2\theta$.
The distance from $y$ to $\R x$ is at most $\|x-y\|_F$, hence
$\sin^2\theta\le\|x-y\|_F^2/\|y\|_F^2$.
\readability{Using \eqref{eq:readable-e2-relative-anchor-error} and $a_p>1/9$ therefore gives}
\begin{equation*}
 \AH(\mat{H}^{\rm QR})>
 \frac19\left(1-\frac{2p}{p^2-1}\right)
 \ge\frac{4999}{45900}>\frac{1275}{11879}>0.107.
\end{equation*}
Here $p/(p^2-1)$ decreases for $p\ge101$.

View the tables as vectors in their Frobenius Hilbert space, with
$\Cgrid$ and $\Padd$ acting columnwise. Thus $\AH$ is the Rayleigh
quotient $\|\Padd \mat{H}\|_F^2/\|\Cgrid \mat{H}\|_F^2$ wherever its denominator is positive.
For the early limit, $\Padd \mat{J}_{\rm row}=0$ and
$\|\Cgrid \mat{J}_{\rm row}\|_F^2=p(p-1)>0$. Thus
\eqref{eq:acr-early-table-limit} and continuity give $\AH(\mat{H}^{\rm ana}_0)\to0$.
For late arrivals,~\eqref{eq:acr-late-table-limit} gives uniform convergence
to the anchor, so $\AH(\mat{H}^{\rm ana}_j)$ converges uniformly to
$\AH(\mat{H}^{\rm QR})>0.107$. This proves both strict bounds in~\eqref{eq:afr-gain}.

Finally $\Cgrid \mat{H}^{\rm QR}\ne0$ was proved above.
The uniform core limit in Lemma~\ref{lem:acr-table-limits} therefore keeps
the centered core energy positive at every arrival for small enough $\epsilon$.
The full table contains that core block, so its centered energy is positive
as well. Zero padding does not change it.
\end{proof}

\subsubsection{Margins, cross-entropy, and prediction chronology}\label{sec:a-margins-chronology}

Write $m_{j,k}(x)=f_{j,y^\star(x)}(x)-f_{j,k}(x)$ for the individual
true-versus-rival contrast, so $m_j(x)=\min_{k\ne y^\star(x)}m_{j,k}(x)$.
For a reference arrival and a rival $k\ne y^\star(x)$, put
\begin{equation*}
 m^C_{j,k}(x)=(e_{y^\star(x)}-e_k)^\top
       \bar{\mat{V}}_{C,c}\,h_{C,j}^{\rm phys}(x).
\end{equation*}
For each rival the uniform active mass is
\begin{equation*}
 \sum_{B\in\Lambda_c}P_BD_B(x)\one[k_B=k]
 =m_0,\qquad m_0=\frac{p+1}{4p^2}.
\end{equation*}

\begin{proposition}[Margins and CE on the frozen-head reference path]
\label{prop:aex-margins}
Fix $p\ge101$.  For all sufficiently small $\epsilon>0$, every actual
arrival $0\le j\le J$, grid input $x$, and rival $k\ne y^\star(x)$ has positive
core margin $m^C_{j,k}(x)$, with
\begin{equation}\label{eq:r15-core-margin}
 \frac{\epsilon m_0}{5}<m^C_{j,k}(x)<9\epsilon m_0<\frac{3\epsilon}{p},
 \qquad m_0=\frac{p+1}{4p^2}.
\end{equation}
On training inputs, the reference path has training accuracy one and
\begin{equation}\label{eq:r15-ce-interval}
 \log p-0.0201<L_j<\log p,
 \qquad
 \frac{\max_{0\le j\le J}L_j}{\min_{0\le j\le J}L_j}-1<0.0045.
\end{equation}
On every held-out input, its actual accuracy is zero when $\rho_j\ge9/10$
and one when $\rho_j\le\epsilon^4$.  In particular these conclusions
hold on the prescribed early and late clock windows, including diagonal
held-out inputs.
\end{proposition}

\begin{proof}
\readability{We first bound each core contrast. On training inputs this
also bounds the full contrast and hence CE; on held-out inputs we compare
it with the decaying query contribution.}
For each rival class, exactly $u/2=(p^2-1)/8$ active reflection-pair core
blocks contribute.  Their masses sum to
$P_Bu/2=(2/N)(u/2)=u/N=m_0$.  By the active-row bounds in
Lemma~\ref{lem:round15-core-only-bootstrap}, their normalized active values
lie in $(1/5,9)$.  The frozen negative-class heads and $vR=\epsilon$ give
the exact contrast formula
\begin{readabilityblock}
\begin{equation}\label{eq:readable-e2-core-contrast}
 m^C_{j,k}(x)=v\sum_{B:\,k_B=k,\,D_B(x)=1}
                  \sqrt{P_B}\,H_{B,j}(x).
\end{equation}
\end{readabilityblock}
Label exclusion means no active core unit has excluded class $y^\star(x)$; the
common centered offset in each $T_k$ head cancels from a true-versus-rival
contrast. \readability{Inserting the active-row bounds and the mass sum into
\eqref{eq:readable-e2-core-contrast} gives \eqref{eq:r15-core-margin}.}
Helpers have zero heads, and the query rows
are inactive on training inputs, so these are also the training margins
of the full network.

If $\epsilon<1/200$, then $9\epsilon m_0<3\epsilon/p<0.0201$.
The CE contrast identity~\eqref{paper:ce-contrast-main} and
$0<m_{j,k}<9\epsilon m_0$ yield
\begin{equation*}
 \log p-9\epsilon m_0<L_j<\log p.
\end{equation*}
This implies the CE interval in~\eqref{eq:r15-ce-interval}.  Its relative
variation is at most
$0.0201/(\log p-0.0201)<0.0045$ for $p\ge101$.

For $x=(a,b)\in\Test$, only the query indexed by $a$ is active, and its
named class $c(a)=2a+1$ is a rival.  In the early window its contribution
to the $c(a)$-versus-true score difference is at least
$24(9/10)\epsilon/p=21.6\epsilon/p$.  The core can oppose this by less
than $3\epsilon/p$, so the named rival wins by more than
$18\epsilon/p$.  Hence every held-out point is misclassified.

In the late window, the adverse query contribution to any margin is at
most $24\epsilon^5/p$.  If $\epsilon^4<1/480$, then for $p\ge101$
\begin{equation*}
 24\epsilon^5/p<\epsilon m_0/5,
\end{equation*}
because $m_0p/120=(p+1)/(480p)>1/480$.  The core lower bound
\eqref{eq:r15-core-margin} therefore leaves every held-out margin strictly
positive.  The same core-plus-query formulas apply to diagonal inputs.
\end{proof}

Proposition~\ref{prop:arf-early} uses
\eqref{eq:acr-early-table-limit} to retain zero held-out refit accuracy and
normalized error greater than one.  For $p\ge1764$,
Proposition~\ref{prop:lx-late} uses~\eqref{eq:acr-late-table-limit} and the
QR-anchor refit bound; it gives held-out accuracy at least
$(p-1)/(p+1)$, nonsquare-difference margins greater than $11/50$, and full
held-out normalized error below $121/400$.  The latter accuracy statement deliberately makes
no diagonal-correctness claim; the actual network's late accuracy-one
conclusion above does.
\end{proofrevision}

\begin{proofrevision}
\subsubsection{The complete refit at early arrivals}\label{sec:refit}
The diagnostic uses every current hidden column and the prescribed ridge penalty
\eqref{paper:refit-penalties}. We first record the normalization
and perturbation facts also used by the random-split proof. All matrices below retain
their full column sets.

In this generic matrix lemma, the matrix letters, their SVD factors and
eigenvalues, and differences such as $\mat{V}-\mat{U}$ are freshly quantified objects
of compatible dimensions, unrelated to construction objects denoted by
the same letter.
For the following lemma, write
$\mat{B}_U=\mat{U}_S^*\mat{U}_S+q_U\Id$ and $\mat{w}_U=\mat{B}_U^{-1}\mat{U}_S^*\mat{T}$, with analogous definitions
for $\mat{B}_V,\mat{w}_V$; put $\mat{r}_V=\mat{T}-\mat{V}_S\mat{w}_V$.
For a real scalar $a$, $(a)_+=\max(a,0)$.

\begin{lemma}[Full ridge: normalization and perturbation]\label{lem:shared-full-ridge-transfer}
Use the full-ridge notation of Appendix~\ref{app:notation} with real matrices,
$\gamma>0$ and $q_U>0$.
\begin{enumerate}[label=(\roman*)]
\item If $\mat{V}=c\mat{U}\mat{C}^\top$, $c>0$ and $\mat{C}^\top \mat{C}=\Id$, then
$q_V=c^2q_U$, $\mat{V}_A\mat{V}_B^\top=c^2\mat{U}_A\mat{U}_B^\top$ for any row sets $A,B$,
$\mat{w}_V=c^{-1}\mat{C}\mat{w}_U$, and $\mat{Y}_V=\mat{Y}_U$.
Appending columns which vanish on $S$ changes neither the penalty nor any prediction;
their fitted coefficients are zero. If $\mat{T}\one=0$, every predicted row is centered.

The complete prediction $\mat{U}\mapsto \mat{Y}_U$ is continuous on $q_U>0$.
Any finite collection of strict inequalities in its scores and squared errors
to fixed targets, with fixed positive error denominators, therefore persists near $\mat{U}$.
\item For any table $\mat{V}$ of the same size with $q_V>0$, the comparison quantities of
Appendix~\ref{app:notation} satisfy
\begin{equation}\label{eq:shared-ridge-perturb}
\norm{\mat{Y}_V-\mat{Y}_U}_F\le
\frac{t e_D}{2\sqrt{q_V}}+
L\left[t e_S\left(\frac1{\sqrt{q_U}}+\frac1{2\sqrt{q_V}}\right)
+\frac{|q_V-q_U|N}{2\sqrt{q_Uq_V}}\right].
\end{equation}
For the prescribed penalties,
\begin{equation}\label{eq:shared-ridge-rms}
q_V\ge\gamma(\norm{\mat{U}_S}_F-e_S)_+^2,\qquad
|q_V-q_U|\le\gamma e_S(2\norm{\mat{U}_S}_F+e_S).
\end{equation}

\end{enumerate}
\end{lemma}
\begin{proof}
\readability{Part (i) supplies the invariances and continuity needed to
transfer both limiting refits. Part (ii) records a quantitative comparison:
subtract the normal equations, control the resulting coefficient error,
and then pass to query predictions.}
The dual formula $\mat{w}_U=\mat{U}_S^\top(\mat{U}_S\mat{U}_S^\top+q_U\Id)^{-1}\mat{T}$ proves the
normalization, zero-column and centering assertions.
On $q_U>0$, continuity of the inverse gives continuity of the scores and
squared errors. Finitely many strict inequalities have a common neighborhood.

For (ii), write $\mat{\Delta}=\mat{V}-\mat{U}$ and
$\mat{r}_V=\mat{T}-\mat{V}_S\mat{w}_V$. The normal equations give the exact identity
\begin{equation}\label{eq:shared-ridge-residual}
\mat{B}_U(\mat{w}_V-\mat{w}_U)=\mat{\Delta}_S^\top \mat{r}_V-\mat{U}_S^\top\mat{\Delta}_S\mat{w}_V-(q_V-q_U)\mat{w}_V.
\end{equation}
The singular-value filters give
$\norm{\mat{r}_V}_{\rm op}\le t$,
$\norm{\mat{w}_V}_{\rm op}\le t/(2\sqrt{q_V})$,
$\norm{\mat{w}_V}_F\le N/(2\sqrt{q_V})$,
$\norm{\mat{B}_U^{-1/2}}_{\rm op}\le q_U^{-1/2}$ and
$\norm{\mat{B}_U^{-1/2}\mat{U}_S^\top}_{\rm op}\le1$.
Multiply \eqref{eq:shared-ridge-residual} by $\mat{B}_U^{-1/2}$ and use
$\mat{Y}_V-\mat{Y}_U=\mat{\Delta}_D\mat{w}_V+\mat{U}_D(\mat{w}_V-\mat{w}_U)$ to obtain \eqref{eq:shared-ridge-perturb}.
The reverse triangle inequality and
$q_V-q_U=\gamma(2\ip{\mat{U}_S}{\mat{\Delta}_S}+\norm{\mat{\Delta}_S}_F^2)$ prove
\eqref{eq:shared-ridge-rms}.

\end{proof}

\begin{proposition}[Early complete-refit failure for the frozen-head reference]
\label{prop:arf-early}
For all sufficiently small $\epsilon>0$, every prescribed arrival
$0\le j\le J$ with $\rho_j\ge9/10$, and every held-out input $x=(a,b)$,
let $\widehat{\mat{Y}}_j$ be the prediction of the complete train-only
refit. Its true-class score is negative, the named rival $2a+1$ exceeds it
by more than
\begin{equation*}
 \frac{1}{2(d+1/32)},
\end{equation*}
and its normalized held-out error is greater than one. In particular, its
accuracy is zero, including on diagonal inputs.
\end{proposition}
\begin{proof}
\readability{We compute the complete fit of the limiting helper table
explicitly, then use its strict wrong-class margins and error reserve
to transfer the conclusion to the early arrivals.}
Put $\mat{H}^\circ=[0\mid \mat{J}_{\rm row}\mid0]$ and write
$\mat{J}_{\Train}=\mat{J}_{\rm row,\Train}=\mat{J}_{\rm row}|_{\Train}$.
Each first operand occurs in exactly $d$ training rows, so
$\mat{J}_{\Train}^{\top}\mat{J}_{\Train}=d\Id_p$ and $\|\mat{J}_{\Train}\|_F^2=n=pd$.
Thus $\mat{H}^\circ$ has training penalty $1/32$, and its complete refit
prediction is
\begin{readabilityblock}
\begin{equation}\label{eq:readable-e2-early-prediction}
 \mat{Y}^\circ(a,b)=\frac{1}{d+1/32}\sum_{q\in Q}T_{2a+q}.
\end{equation}
\end{readabilityblock}
For $x=(a,b)\in\Test$, the difference $b-a$ is either zero or a nonsquare,
so the true label $y^\star(x)=2a+(b-a)$ is absent from $2a+Q$. Hence
\begin{equation*}
 \mat{Y}^\circ(x)_{y^\star(x)}=-\frac{d}{p(d+1/32)}<0.
\end{equation*}
Since $1\in Q$, the class $2a+1$ occurs once in the sum and its score exceeds
the true score by $1/(d+1/32)$.

Every ridge prediction is centered because the training targets are centered.
\readability{Expanding the squared error and using the negative true-class
score from \eqref{eq:readable-e2-early-prediction} gives}
\begin{readabilityblock}
\begin{equation}\label{eq:readable-e2-early-error}
 \frac{\|\mat{Y}^\circ(x)-T_{y^\star(x)}\|_2^2}{(p-1)/p}
 =1+\frac{\|\mat{Y}^\circ(x)\|_2^2-2\mat{Y}^\circ(x)_{y^\star(x)}}{(p-1)/p}>1.
\end{equation}
\end{readabilityblock}
By~\eqref{eq:acr-early-table-limit}, the full-column tables
$\mat{H}_j^{\rm ana}/(\rho_jL_{\rm lift})$ converge uniformly to $\mat{H}^\circ$ over
the indicated actual arrivals. \readability{The limiting penalty is positive,
so Lemma~\ref{lem:shared-full-ridge-transfer}(i) retains the strict score
inequalities and the conclusion from \eqref{eq:readable-e2-early-error}
that normalized error exceeds one, with the rival reserve halved.}
Positive lift and zero-column invariance transfer the complete refit to
the physical network at every admitted width.
\end{proof}
\end{proofrevision}

\begin{proofrevision}
\subsubsection{The complete refit at late arrivals}\label{sec:lx-late}
This subsection proves the static QR-anchor estimate and transfers it to every late
arrival for each prime $p\ge1764$. We use the uniform normalized-table limit
of Lemma~\ref{lem:acr-table-limits}; no quantitative transport estimate is needed.
The Fourier lemmas below serve only this subsection. The complete fit always
uses every feature column and the prescribed ridge penalty.
\readability{The proof has three stages. Lemma~\ref{lem:lx-alias} reduces
the anchor error to two representative Fourier problems through
\eqref{eq:lx-EU}. The binary comparison and sparse teacher then control the
complete fit in Proposition~\ref{prop:lx-frame}. Finally,
Proposition~\ref{prop:lx-late} converts this error bound into margins and
uses \eqref{eq:acr-late-table-limit} to reach every late arrival.}

\paragraph{The late-refit anchor and its symmetries.}\label{sec:notation-halfgrid-refit}
The actual-arrival table limits are the discrete pair
\eqref{eq:acr-early-table-limit}--\eqref{eq:acr-late-table-limit}; no
continuous-$\rho$ table is defined.

\begin{definition}[Normalized late-refit tables]\label{def:half-late-tables}
Let $Q^{\rm c}=\F_p\setminus Q$ and fix a nonsquare $\nu$.
Write $x=(a,a+t)$ and $\mat{\mathsf F}_{s,t}=p^{-1/2}e(-st)$.
With $h_\circ$ from Definition~\ref{def:half-comparison}, put
\begin{equation}\label{eq:lx-norm}
 \widetilde{\mat{H}}=\frac{\mat{H}^{\rm QR}}{h_\circ},\quad
 \widetilde{\mat{B}}=\sqrt\varpi \mat{D}_{\rm grid},\quad
 \widetilde{\mat{E}}=\widetilde{\mat{H}}-\widetilde{\mat{B}},\quad
 \ell=\frac{\norm{\widetilde{\mat{H}}_\Train}_F^2}{32n}.
\end{equation}
\end{definition}

\begin{definition}[Translation blocks and complete ridge fit]\label{def:half-late-blocks}
For common translations, $r\cdot(l_1,l_2,l_3)=(l_1,l_2-l_1r,l_3+2r)$
and $\Lambda_0=\{l:l_3=0\}$.
A table is equivariant if $\mat{Y}(x+(r,r),r\cdot l)=\mat{Y}(x,l)$.
The target-aligned Fourier blocks are
\begin{equation}\label{eq:lx-blocks}
 \mat{Y}_k(t,l)=\sum_{a\in\F_p}e(-k(2a+t))\mat{Y}((a,a+t),l)
 \qquad(k,t\in\F_p,\ l\in\Lambda_0).
\end{equation}
Let $\mat{H}_k,\mat{B}_k,\mat{E}_k=\mat{H}_k-\mat{B}_k$ be the blocks of
$\widetilde{\mat{H}},\widetilde{\mat{B}},\widetilde{\mat{E}}$. For $k\ne0$ put $\mat{A}_k=\mat{H}_k[Q]$,
$\mat{K}_k=\mat{A}_k^*\mat{A}_k+\ell \Id$, $\hat c_k=\mat{K}_k^{-1}\mat{A}_k^*\one_Q$ and
$g_k=\mat{H}_k\hat c_k\in\mathbb C^p$.
Let $\mat{Y}^{\rm QR}$ be the complete-ridge prediction of $\mat{H}^{\rm QR}$ and
$E^{\rm QR}$ its held-out error~\eqref{paper:refiterror-main}; common rescaling means
$\mat{Y}^{\rm QR}$ is also the prediction of $\widetilde{\mat{H}}$.
\end{definition}

Common feature rescaling leaves complete-ridge predictions unchanged, so
we work with $\widetilde{\mat{H}}$ and its actual penalty $\ell>1/129$.
Common translation maps each mask to $r\cdot l$; its action is free and
$\Lambda_0$ meets each orbit once. Since the split is translation invariant,
the source and its Polar factor are equivariant, hence so are
$\widetilde{\mat{H}},\widetilde{\mat{B}},\widetilde{\mat{E}}$. Positive ridge penalties give the
unique minimizers specified by their normal equations.

For the alias reduction, let $Q_-$ be the nonsquares and define
\begin{equation*}
 \cU f(k,t)=p^{-1/2}\sum_a e(-k(2a+t))f(a,a+t),\qquad
 \mat{\mathfrak F}_{y,k}=p^{-1/2}e(ky).
\end{equation*}
For $\sigma\in Q$ put
$d_\sigma(l_1,l_2,l_3)=(l_1/\sigma,l_2,\sigma l_3)$.

\begin{lemma}[Reduction to two alias problems]\label{lem:lx-alias}
Let $p\ge1764$ be prime. For the complete ridge fit of $\mat{H}^{\rm QR}$, let
$g_k$ be its prediction blocks from~\eqref{eq:lx-blocks}.
\begin{enumerate}[label=(\roman*)]
\item The map $\cU$ is unitary and sends training and test supports onto
$\bigoplus_k\mathbb C^Q$ and $\bigoplus_k\mathbb C^{Q^{\rm c}}$, respectively.
For every equivariant $\mat{Y}$,
\begin{equation*}
 \cU\mat{Y}\mat{Y}^*\cU^*=\bigoplus_k\mat{Y}_k\mat{Y}_k^*,\qquad
 \norm{\mat{Y}}_F^2=\sum_k\norm{\mat{Y}_k}_F^2.
\end{equation*}
\item For every input $x=(a,a+t)$,
\begin{equation*}
 \norm{\mat{Y}^{\rm QR}(x)-T_{y^\star(x)}}^2
 =p^{-1}\sum_{k\in\F_p^\times}\lvert1-g_k(t)\rvert^2.
\end{equation*}
\item For every $\sigma\in Q$ and $k\in\F_p^\times$: $g_{\sigma k}(t)=g_k(\sigma t)$ and
$\norm{\mat{E}_{\sigma k}}_F=\norm{\mat{E}_k}_F$.
\item The anchor error is
\begin{equation}\label{eq:lx-EU}
 (p+1)E^{\rm QR}=\sum_{k\in\{1,\nu\}}\norm{\one_{Q^{\rm c}}-g_k[Q^{\rm c}]}^2,
\end{equation}
and every held-out input $x$ whose difference is a nonsquare satisfies
$\norm{\mat{Y}^{\rm QR}(x)-T_{y^\star(x)}}^2\le(p+1)E^{\rm QR}/p$.
\end{enumerate}
\end{lemma}
\begin{proof}
\readability{Common translations first separate the complete ridge
problem into Fourier blocks. Square dilations then identify all nonzero
frequencies with the two representatives $1$ and $\nu$, which yields both
the total error formula and the pointwise nonsquare bound.}
Since $2$ is invertible, $\cU$ is a Fourier transform followed by unit phases.
Equivariance gives, for $l\in\Lambda_0$,
\begin{equation*}
 (\cU\mat{Y})((k,t),r\cdot l)
 =p^{-1/2}e(-2kr)\mat{Y}_k(t,l).
\end{equation*}
The characters $p^{-1/2}e(2kr)$ are orthonormal on each free column orbit.
This proves (i) and supplies a unitary full-column decomposition of ridge,
with the same penalty $\ell \Id$ in every block.

The centered class-$k$ target has Fourier coefficient
$p^{-1/2}e(k(2a+t))$ for $k\ne0$ and zero for $k=0$.
Under $\cU$ it is $\one$ in block $k$ and zero in the others.
The block normal equations are therefore precisely those for $\hat c_k$.
Inverting $\cU$ gives
\begin{equation*}
 (\mat{Y}^{\rm QR}(x)^*\mat{\mathfrak F})_k
 =p^{-1/2}e(k(2a+t))g_k(t),
\end{equation*}
so class Parseval proves (ii).

For $\sigma\in Q$, dilation permutes training inputs, masks by $d_\sigma$,
and target classes. It intertwines the source and Polar factor and preserves
the actual penalty $\ell$. Uniqueness of full ridge thus gives the stated
identity for $g_{\sigma k}$. The block definition also gives
$\mat{E}_{\sigma k}(t,l)=\mat{E}_k(\sigma t,d_\sigma l)$ by the substitution $a\mapsto\sigma a$.
These row and column permutations preserve Frobenius norm, proving (iii).

Writing each nonzero frequency uniquely as $\sigma k$ gives
\begin{readabilityblock}
\begin{equation}\label{eq:readable-e2-alias-pointwise}
 p\norm{\mat{Y}^{\rm QR}(x)-T_{y^\star(x)}}^2
 =\sum_{k\in\{1,\nu\}}\sum_{\sigma\in Q}|1-g_k(\sigma t)|^2.
\end{equation}
\end{readabilityblock}
Every $\sigma$ permutes $Q^{\rm c}$, including $0$. Summing over all held-out
inputs and using $|\Test|(p-1)/p=(p+1)|Q|$ gives~\eqref{eq:lx-EU}.
\readability{For a nonsquare $t$, $\sigma t$ runs once through $Q_-$.
Thus \eqref{eq:readable-e2-alias-pointwise} is bounded by the two full
$Q^{\rm c}$ norms in \eqref{eq:lx-EU}, proving the pointwise assertion in (iv).}
\end{proof}

\paragraph{The binary alias geometry.}
\begin{definition}[Binary alias spectrum]\label{def:half-binary-spectrum}
Write $a_v=|\hat\one_I(v)|^2$ (so $a_0=b^2$),
$Q_4=\sum_{v\ne0}a_v^2$, and, for $q\ne0$, $A_q(v)=\hat\one_I(v/q)$ and
$B_q(v)=p\one[v=0]-\overline{A_q(v)}$.
Set $a'_0=c^2$ and $a'_v=a_v$ for $v\ne0$. Define
$\lambda_k(s)=\varpi\sum_{q\ne0}a_{(k-s)/q}a'_{(k+s)/q}$,
$\lambda_+=(p-1)^2(p+1)/(16p^2)$ and $\lambda_-=(p+1)^3/(16p^2)$.
These are the Gram eigenvalues of Lemma~\ref{lem:lx-binary}.
Put $\lambda_0=(p+1)(p^2+3)/(48p^2)$.
\end{definition}

Thus $A_q(0)=b$, $B_q(0)=c$ and $|A_q(v)|^2=|B_q(v)|^2=a_{v/q}$ for $v\ne0$.
For $l=(q,l_2,l_3)$, translation of the masks gives
$\hat\one_{A_l}(v)=e(vl_2/q)A_q(v)$ and
$\hat\one_{B_l}(v)=e(-v(l_3+l_2/q))B_q(v)$.

\begin{lemma}[Binary alias geometry]\label{lem:lx-binary}
Let $p\ge1764$ be prime and $k\ne0$.
Then
\begin{equation}\label{eq:lx-spectrum}
 \mat{B}_k\mat{B}_k^*=\mat{\mathsf F}^*\diag(\lambda_k)\mat{\mathsf F},\qquad \lambda_k(s)>0.
\end{equation}
Moreover $\lambda_k(s)=\lambda_1(s/k)$,
\begin{equation*}
 \lambda_k(0)=\varpi Q_4=\lambda_0,\qquad
 \lambda_k(k)=\lambda_+,\qquad \lambda_k(-k)=\lambda_-,\qquad
 \sum_s\lambda_k(s)=\frac{p+1}{6}.
\end{equation*}
\end{lemma}
\begin{proof}
The block definition and the mask transforms give, for $l=(q,l_2,0)$,
\begin{readabilityblock}
\begin{equation}\label{eq:readable-e2-binary-fourier}
 (\mat{\mathsf F}\mat{B}_k)(s,l)=\sqrt{\varpi/p}\,A_q(k-s)B_q(k+s)e(-2sl_2/q).
\end{equation}
\end{readabilityblock}
\readability{The phases in \eqref{eq:readable-e2-binary-fourier} make distinct
Fourier rows orthogonal after summing over $l_2$; their squared norms
are exactly $\lambda_k(s)$, proving \eqref{eq:lx-spectrum}.}
The geometric-series formula for $\hat\one_I(v)$ is nonzero at every $v$:
for $v\ne0$ both $1-e(-bv)$ and $1-e(-v)$ are nonzero. This proves positivity.
Substitute $q=kr$ for scaling. Parseval and Lemma~\ref{lem:afr-counts}(iv) give
$\sum_{v\ne0}a_v=u$ and $Q_4=\Sigma_2=p^2\Var\paren{h}=u(u+1)/3$,
hence all three special eigenvalues.
To obtain the trace, substitute $n=k/q\ne0$, $w=s/q$ and then
$(i,j)=(n-w,n+w)$. This is a bijection onto $i+j\ne0$, so
\begin{equation*}
 \sum_s\lambda_k(s)=\varpi\left[(pb)(pc)-\sum_i a_i a'_{-i}\right]
 =\varpi[p^2u-u^2-Q_4]=\frac{p+1}{6}.
\end{equation*}
\end{proof}

\paragraph{Source-defect and penalty bounds.}
\begin{definition}[Defect and penalty envelopes]\label{def:half-defect}
Put
\begin{equation}\label{eq:lx-scalars}
\begin{aligned}
 \kappa_p&=\sqrt{\tfrac p{p-1}},&
 \Delta_p^2&=\frac{4p}{(p-1)h_\circ^2},\\
 \ell_{\rm bin}&=\frac{p^2-1}{128p^2},&
 \ell_+&=\Bigl(1+\frac{2p}{(p-1)\sqrt{p+1}}\Bigr)^2\ell_{\rm bin}.
\end{aligned}
\end{equation}
\end{definition}

\begin{lemma}[Defect and penalty]\label{lem:lx-defect}
Let $p\ge1764$ be prime.
\begin{enumerate}[label=(\roman*)]
\item $\mat{D}_{\rm grid}\odot(\mat{X}_{\rm grid}\bar{\mat{Q}}^\top)=h_\circ\widetilde{\mat{B}}$; that is, every row of
$\bar{\mat{Q}}$ has the constant value $h_\circ\sqrt\varpi$ on its active rectangle.
\item $h_\circ>2$ and $\norm{\widetilde{\mat{E}}}_F^2<2p/h_\circ^2<p/2$.
\item $\norm{\mat{E}_1}_F^2+\norm{\mat{E}_\nu}_F^2<\Delta_p^2<\kappa_p^2$. In particular
$\norm{\mat{E}_k}_{\rm op}<\kappa_p$ for every $k\ne0$.
\item $\ell_{\rm bin}\le\ell\le\ell_+$.
\end{enumerate}
\end{lemma}
\begin{proof}
\readability{Parts (i)--(iii) turn the earlier Polar comparison into a
bound on the two alias defects. Part (iv) controls the penalty computed
from the actual QR table, so the subsequent ridge estimates use that
table's prescribed normalization.}
(i) Since $h_\circ=2\sqrt{p/(p+1)}+\sqrt{2/p}
=\bar\alpha^{-1/2}+\sqrt{2/p}$, this is exactly the active-rectangle value
proved in Lemma~\ref{lem:afr-den}(i).

(ii) The inequalities $2\sqrt{p/(p+1)}>2-2/(p+1)$ and
$\sqrt{2/p}>2/(p+1)$ give $h_\circ>2$. By (i),
$h_\circ\widetilde{\mat{E}}=\mat{D}_{\rm grid}\odot(\mat{X}_{\rm grid}(\mat{Q}_{\rm QR}-\bar{\mat{Q}})^\top)$.
For each incoming row $(w,v)$, weighted Cauchy--Schwarz gives
\begin{equation*}
 |w(a)+v(z)|^2\le\frac pc|w(a)|^2+\frac pb|v(z)|^2,
 \qquad
 \sum_{a\in A_l,z\in B_l}|w(a)+v(z)|^2\le p\norm{(w,v)}^2.
\end{equation*}
Summing over rows and using Lemma~\ref{lem:afr-den}(iii) yields
$\norm{\widetilde{\mat{E}}}_F^2\le p\delta_Q^2/h_\circ^2<2p/h_\circ^2<p/2$.

(iii) Lemma~\ref{lem:lx-alias}(i),(iii) gives
\begin{equation*}
 d(\norm{\mat{E}_1}_F^2+\norm{\mat{E}_\nu}_F^2)
 \le\sum_{k\in\F_p}\norm{\mat{E}_k}_F^2
 =\norm{\widetilde{\mat{E}}}_F^2<2p/h_\circ^2.
\end{equation*}
This is the first claim; the second follows from $h_\circ>2$ and
$\norm{\mat{E}_k}_{\rm op}\le\norm{\mat{E}_k}_F$.

(iv) Write $\|\mat{M}\|_*$ for the sum of the singular values of $\mat{M}$.
Let $\mat{M}=\widetilde{\mat{B}}_\Train^\top \mat{X}_\Train$, a positive multiple of $\mat{F}_{\rm unif}$, so that
$\mat{Q}_{\rm QR}=\pol(\mat{M})$ and $\ip{\mat{Q}_{\rm QR}}{\mat{M}}_F=\norm{\mat{M}}_*$. Since $\mat{D}_{\rm grid}$ is binary,
\begin{equation*}
 \ip{\widetilde{\mat{H}}_\Train}{\widetilde{\mat{B}}_\Train}_F=h_\circ^{-1}\ip{\mat{Q}_{\rm QR}}{\mat{M}}_F
 =h_\circ^{-1}\norm{\mat{M}}_*\ge h_\circ^{-1}\ip{\bar{\mat{Q}}}{\mat{M}}_F=\norm{\widetilde{\mat{B}}_\Train}_F^2 .
\end{equation*}
The inequality holds because $\norm{\bar{\mat{Q}}}_{\rm op}\le1$; the last equality is (i).
Expanding $\widetilde{\mat{H}}_\Train=\widetilde{\mat{B}}_\Train+\widetilde{\mat{E}}_\Train$ gives
$\norm{\widetilde{\mat{H}}_\Train}_F^2\ge\norm{\widetilde{\mat{B}}_\Train}_F^2+\norm{\widetilde{\mat{E}}_\Train}_F^2$.
For fixed $(l_1,l_2)$, a pair $(a,a')\in A_l\times A_l^{\rm c}$ yields the input
$(a,l_3-a')\in A_l\times B_l$, and this input is a training input for exactly $d$ values of
$l_3$. Hence
$\norm{\widetilde{\mat{B}}_\Train}_F^2=\varpi\,p(p-1)\,d^2(d+1)=d^2(d+1)/p=32n\,\ell_{\rm bin}$,
and $\ell\ge\ell_{\rm bin}$. For the upper bound,
$\norm{\widetilde{\mat{H}}_\Train}_F\le
\norm{\widetilde{\mat{B}}_\Train}_F+\norm{\widetilde{\mat{E}}}_F<\norm{\widetilde{\mat{B}}_\Train}_F+\sqrt{p/2}$ by
(ii). Dividing by $\norm{\widetilde{\mat{B}}_\Train}_F=\sqrt{d^2(d+1)/p}$ gives
$\sqrt{\ell/\ell_{\rm bin}}<1+2p/((p-1)\sqrt{p+1})$, which is precisely the stated
$\ell\le\ell_+$.
\end{proof}

\begin{definition}[Adjacent interval correlation]\label{def:half-adjacent}
Set $C_1=\sum_{v\notin\{0,-1\}}a_va_{v+1}$ and $t_I=C_1/a_1$.
\end{definition}

\begin{lemma}[Adjacent interval bound]\label{lem:lx-adjacent}
For every odd prime $p\ge1764$,
\begin{equation*}
 a_1>p^2/\pi^2,\qquad a_2<1,\qquad C_1/a_1<5/2.
\end{equation*}
\end{lemma}
\begin{proof}
For canonical representatives $|v|\le b$, the geometric sum gives
\begin{equation*}
 a_v=\frac1{4\cos^2(\pi v/(2p))}\le\frac12\quad(v\ne0\text{ even}),
 \qquad
 a_v=\frac1{4\sin^2(\pi v/(2p))}\quad(v\text{ odd}).
\end{equation*}
Thus $a_1>p^2/\pi^2$ and $a_2<1$.
Every non-wrap neighbor pair in $C_1$ has opposite parity. Each odd
coefficient meets at most two even neighbors, whose total is at most one.
The sole wrap pair has product $a_ba_{-b}=a_b^2<1$: if $b$ is even use
the displayed bound, while if $b$ is odd then $\pi b/(2p)>\pi/6$.
Parseval gives $\sum_{v\ne0}a_v=u$, so $C_1\le u+1$.
Consequently
\begin{equation*}
 \frac{C_1}{a_1}<\frac{\pi^2}{4}(1+3/p^2)
 <\frac{121}{49}(1+3/1764^2)<\frac52,
\end{equation*}
using $\pi<22/7$.
\end{proof}

\paragraph{A two-slope comparison teacher.}
The complete fit still uses every column. Only its comparison coefficient
is supported on the two slopes $q=\pm k$, so the adjacent interval
correlation controls its target error.

\begin{definition}[Sparse comparison teacher and correction]\label{def:half-teacher}
For $k\ne0$ put $\mat{G}_k=\mat{B}_k\mat{B}_k^*$ and $w_k=\mat{B}_k^*\one/\lambda_0$.
The sparse comparison coefficient is
\begin{equation*}
 w'_k(q,l_2,0)=-\frac{\one[q\in\{k,-k\}]}{2\sqrt\varpi a_1}.
\end{equation*}
Set $\zeta_k=w'_k-w_k$ and $f'_k=\mat{H}_kw'_k$.
For its regularized binary correction put
$z'_k=\mat{B}_k^*(\mat{G}_k+\ell \Id)^{-1}(f'_k-\one)$ and
$r'_k=\ell(\mat{G}_k+\ell \Id)^{-1}(f'_k-\one)$.
Only the comparison coefficient is sparse; $\hat c_k$ still minimizes over every column.
\end{definition}

\begin{lemma}[Sparse regularized target energy]\label{lem:lx-teacher}
For every prime $p\ge1764$ and $k\ne0$, the sparse comparison
coefficient satisfies
\begin{equation*}
 \mat{B}_kw'_k=\one,\qquad
 \zeta_k=w'_k-w_k\in\ker \mat{B}_k\perp\op{range}(\mat{B}_k^*),\qquad
 \|\zeta_k\|^2<\frac45.
\end{equation*}
Its actual QR teacher $f'_k=\mat{H}_kw'_k$ obeys
\begin{equation*}
 (f'_k-\one)^*(\mat{G}_k+\tfrac1{129}\Id)^{-1}(f'_k-\one)<\frac{200}{p}.
\end{equation*}
\end{lemma}
\begin{proof}
\readability{First we check the binary interpolation and coefficient-norm
claims. We then contract the QR source against this coefficient and bound
its nonzero and zero Fourier modes separately. Their weighted Parseval
sum gives the regularized target-energy estimate.}
For a slope-indexed expression write
$\langle F_q\rangle_k=[F_k+F_{-k}]/(2a_1)$.
\readability{Summing the Fourier rows \eqref{eq:readable-e2-binary-fourier}
against $w'_k$ over the $p$ shifts $l_2$ gives}
\begin{equation*}
 (\mat{\mathsf F}\mat{B}_kw'_k)(s)
 =-\sqrt p\,\one[s=0]\langle A_q(k)B_q(k)\rangle_k
 =\sqrt p\,\one[s=0].
\end{equation*}
Here $A_q(k)B_q(k)=-a_1$ for $q=\pm k$. Thus $\mat{B}_kw'_k=\one$.
Since $w_k=\mat{B}_k^*\one/\lambda_0$ is the minimum-norm binary solution,
the asserted nullspace orthogonality follows. The $2p$ nonzero entries of
$w'_k$, Lemma~\ref{lem:lx-adjacent}, and $\pi<22/7$ give
\begin{equation*}
 \|\zeta_k\|^2=\frac{p}{2\varpi a_1^2}-\frac p{\lambda_0}
 <\frac{1880}{2401}(1-p^{-1})+\frac{96}{p^2}<\frac45.
\end{equation*}
Here the exact difference
$48(1-p^{-1})-p/\lambda_0=48(2p^2-3)/[p(p+1)(p^2+3)]$
is less than $96/p^2$. Uniformly, $1880/2401+96/1764^2<4/5$.

\emph{Source contraction.}
Lift the comparison coefficient by
$W'_k(q,l_2,l_3)=e(kl_3)w'_k(q,l_2,0)$.
For this proof put $x=\sqrt p$, $\alpha=(x-1)^{-1}$,
$\beta_*=f_p/h_\circ$ and $\gamma=(g_0\sqrt{2c_0})^{-1}$.
For $h\ne0$ put $\mat{R}_h=\mat{G}_{\rm avg}+\chi_2(h)\mat{G}_{\rm diff}$ and
$\mat{\Delta}_h=\mat{R}_h/g_0-\Id$, with $\mat{\Delta}_0=0$.
For this contraction, write $\mat{R}_0=\mat{G}_{\rm const}$ for the constant inverse-root
block and $\mathbf n=npb_l$ for the QR count row, with operand blocks
$\mathbf n^L,\mathbf n^R$.
Set
\begin{equation*}
 \eta_h=\frac{-1+\tau\chi_2(h)+p\one[h=0]}{p-1}.
\end{equation*}
The lifted coefficient satisfies $W'_k((-a)\cdot l)=e(-2ka)W'_k(l)$,
so equivariance gives
\begin{equation*}
 f'_k(t)=e(-kt)\sum_l\widetilde{\mat{H}}((0,t),l)W'_k(l).
\end{equation*}
Insert the source rows~\eqref{eq:afr-rows} and the count convolutions
\begin{equation*}
 \widehat{\mathbf n}^L(h)=\frac bp\sum_i\eta_{-i}\hat\one_{A_l}(h-i)\hat\one_{B_l}(i),
 \qquad
 \widehat{\mathbf n}^R(h)=\frac bp\sum_i\eta_i\hat\one_{B_l}(h-i)\hat\one_{A_l}(i).
\end{equation*}
For fixed slope $q$, a product with total left and right frequencies $I_L,I_R$
has shift phase $e(l_2(I_L-I_R)/q)e((k-I_R)l_3)$.
Summation over the two shifts forces $I_L=I_R=k$.
Define the two inner convolutions by
\begin{readabilityblock}
\begin{equation}\label{eq:readable-e2-inner-convolutions}
 \begin{aligned}
 \Sigma_q^L(s)&=\sum_i A_q(k-i)A_q(i)
   [\eta_s(\mat{R}_{i-s}/g_0)_{11}+\eta_i(\mat{R}_{i-s}/g_0)_{12}],\\
 \Sigma_q^R(s)&=\sum_i B_q(k-i)B_q(i)
   [\eta_{-i}(\mat{R}_{i+s}/g_0)_{21}+\eta_s(\mat{R}_{i+s}/g_0)_{22}].
 \end{aligned}
\end{equation}
\end{readabilityblock}
Then, with $\widehat f'_k(s)=\sum_t e(-st)f'_k(t)$, the contraction is
\begin{readabilityblock}
\begin{equation}\label{eq:readable-e2-teacher-contraction}
 \widehat f'_k(s)=-\frac{\beta_*b}{p}
 \left\langle B_q(k+s)B_q(-s)\Sigma_q^L(s)
       +A_q(k-s)A_q(s)\Sigma_q^R(s)\right\rangle_k.
\end{equation}
\end{readabilityblock}
Indeed the source row and Fourier inversions contribute
$g_0/(h_\circ\sqrt p)=\beta_*b/p$ after the two shift sums.
We estimate these contractions in norm without evaluating each teacher coefficient.

\emph{Nonzero modes: a whole-convolution operator bound.}
\readability{For $s\ne0$ and $q=\pm k$, first bound the inner convolutions
\eqref{eq:readable-e2-inner-convolutions}; then apply the outer factors
in \eqref{eq:readable-e2-teacher-contraction} in the weighted Fourier norm.} Put
$d_0=3/p+15/(2x^3)$, $d_1=1/x+d_0$, $D=\alpha(d_0+d_1)$ and
\begin{equation*}
 j_{\rm kin}=p\alpha+c(d_1+\alpha d_0+D+\alpha)+D\sqrt{(p-2)t_I}.
\end{equation*}
Lemma~\ref{lem:afr-rootbounds}(ii)--(iii) gives
$|(\mat{\Delta}_h)_{jj}|\le d_0$ and
$|(\mat{\Delta}_h)_{12}|,|(\mat{\Delta}_h)_{21}|\le d_1$ for $h\ne0$.
Since $|\eta_h|\le\alpha$ for $h\ne0$, the root-error response is at
most $d_1+\alpha d_0$ at $i=0$, and at most $D$ otherwise.
The identity part uses
$\sum_iA_q(k-i)A_q(i)=pA_q(k)$ and
$\sum_iB_q(k-i)B_q(i)=pB_q(k)$.
The products at $i=0,k$ have moduli $b\sqrt{a_1}$ on the left
and $c\sqrt{a_1}$ on the right; the remaining squared sum is $C_1$.
Cauchy--Schwarz therefore bounds each inner sum, except its
source-frequency-zero jump, by
\begin{equation*}
 [p\alpha+c(d_1+\alpha d_0+D)
       +D\sqrt{(p-2)t_I}]\sqrt{a_1}.
\end{equation*}

The actual constant block is $\mat{R}_0/g_0=(\gamma/2)\one\one^\top$.
At $i=s$ on the left or $i=-s$ on the right, its response differs
from the identity response by $(\gamma-1)\eta_s$.
Here $0<\gamma<1$: Lemma~\ref{lem:afr-cov} and the formula for $g_0$ give
$2c_0g_0^2\ge2(p-1)^2(p+1)/(p^2+1)>1$.
The multiplying mask product is at most
$c\sqrt{a_1}$: at an endpoint this is immediate, and otherwise its
squared modulus is at most $C_1=t_Ia_1<c^2a_1$.
Thus both whole inner sums have modulus at most $j_{\rm kin}\sqrt{a_1}$.
For $x\ge42$ the decreasing scalar expressions satisfy
\begin{equation*}
 pD=\frac{x+6+15/x}{x-1}<\frac65,\qquad
 (x+2)d_0<\frac1{10},\qquad
 \frac{x}{x-1}<\frac{21}{20},\qquad
 \frac{c\alpha}{x}<\frac{11}{20}.
\end{equation*}
Also $pD<6/5$ and $t_I<5/2$ give
$D\sqrt{(p-2)t_I}/x<(6/(5p))\sqrt{5/2}<48/(25p)<1/100$. Hence
\begin{readabilityblock}
\begin{equation}\label{eq:readable-e2-inner-bound}
\begin{aligned}
 \frac{j_{\rm kin}}x&=\frac{x}{x-1}
  +\frac{c\alpha}{x}[2+(x+2)d_0]
  +\frac{D\sqrt{(p-2)t_I}}x\\
 &<\frac{21}{20}+\frac{11}{20}\frac{21}{10}+\frac1{100}
  =\frac{443}{200}<\frac94.
\end{aligned}
\end{equation}
\end{readabilityblock}

Now apply the outer multiplication operator in the weighted norm
$\|v\|_\omega^2=\sum_{s\ne0}|v(s)|^2/(\lambda_k(s)+1/129)$.
At $s=k,-k$, its direct factors have moduli
$b\sqrt{a_1},c\sqrt{a_1}$ and its cross factors are at most
$\sqrt{a_1a_2}<\sqrt{a_1}$.
Off these axes each outer factor has squared sum at most $C_1$.
\readability{Minkowski for the two slopes and the two operand blocks,
followed by \eqref{eq:readable-e2-inner-bound}, gives}
\begin{readabilityblock}
\begin{equation}\label{eq:readable-e2-nonzero-teacher}
\begin{aligned}
 \|\widehat f'_k\|_\omega
 &\le\frac{\beta_*b}{p}j_{\rm kin}
 \sqrt{\frac{(b+1)^2}{\lambda_+}
       +\frac{(c+1)^2}{\lambda_-}+516t_I}\\
 &<\frac{6j_{\rm kin}}{x}<\frac{27}{2}<14.
\end{aligned}
\end{equation}
\end{readabilityblock}
For the strict middle bound, the axis terms divided by $p$ sum to
\begin{equation*}
 \begin{aligned}
 \frac{4p(p+1)}{(p-1)^2}+\frac{4p(p+3)^2}{(p+1)^3}
 &=8+\frac{12}{p-1}+\frac8{(p-1)^2}+\frac{12}{p+1}-\frac{16}{(p+1)^3}\\
 &<8+\frac{24}{1763}+\frac8{1763^2}<\frac{81}{10},
 \end{aligned}
\end{equation*}
while $516t_I/p<1290/1764<4/5$, and $\beta_*<4/(p-1)$.

\emph{Zero frequency.}
Put $a_*=[g_0(b^2+c^2)+2u/\sqrt{2c_0}]/(h_\circ p^{3/2})$.
For $h\ne0$ the two normalized row factors
$(\mat{R}_h)_{11}/g_0+\eta_h(\mat{R}_h)_{12}/g_0$ and
$(\mat{R}_h)_{22}/g_0+\eta_{-h}(\mat{R}_h)_{21}/g_0$
differ from one by at most $d_0+\alpha d_1<5/p$, since
\begin{equation*}
 p(d_0+\alpha d_1)=4+\frac{15}{2x}
       +\frac{4+15/(2x)}{x-1}
 \le4+\frac5{28}+\frac{4+5/28}{41}<5.
\end{equation*}
The absolute sparse left and right convolutions are bounded by
$b+\sqrt{(p-2)t_I}$ and $c+\sqrt{(p-2)t_I}$, both below $p$
because $(p-1)^2>10(p-2)$.
\readability{Applying the contraction \eqref{eq:readable-e2-teacher-contraction}
at $s=0$ with these bounds gives}
\begin{readabilityblock}
\begin{equation}\label{eq:readable-e2-zero-teacher}
 |p^{-1}\widehat f'_k(0)-a_*|
 <\frac{g_0}{h_\circ p^{3/2}}\frac5p(cp+bp)<\frac{10}{p}.
\end{equation}
\end{readabilityblock}
The exact source covariance expresses $a_*$ as
\begin{equation*}
 a_*=\frac{2\sqrt{p/(p+1)}}{h_\circ}\sqrt{1+p^{-2}}
 +\frac{\sqrt{2/p}}{h_\circ}(1+\varepsilon_p)^{-1/2}.
\end{equation*}
The coefficients sum to one, and $0<\varepsilon_p<1/[2(p-1)]$.
Both endpoints differ from one by less than $1/p$, so
$|p^{-1}\widehat f'_k(0)-1|<11/p$
\readability{by \eqref{eq:readable-e2-zero-teacher}}.
Since $\lambda_0>p/48$, this mode contributes less than $5808/p^2$.
\readability{By \eqref{eq:lx-spectrum}, the weighted Parseval sum has
weights $(\lambda_k(s)+1/129)^{-1}$ and an overall factor $1/p$ for the
unnormalized transform $\widehat f'_k$. Combining the zero mode with
\eqref{eq:readable-e2-nonzero-teacher} proves the asserted energy bound:}
\begin{equation*}
 (f'_k-\one)^*(\mat{G}_k+\tfrac1{129}\Id)^{-1}(f'_k-\one)
 <\frac{14^2}{p}+\frac{5808}{p^2}<\frac{200}{p}.
\end{equation*}
Since $\ell>1/129$, the regularized correction in
Definition~\ref{def:half-teacher} consequently has
$\mat{B}_kz'_k+r'_k=f'_k-\one$ and
$\|z'_k\|^2+\|r'_k\|^2/\ell<200/p$.
\end{proof}
\paragraph{A target-sensitive normal-residual frame.}
\begin{proposition}[Complete-error frame]\label{prop:lx-frame}
For every prime $p\ge1764$, with $x=\sqrt p$,
\begin{equation}\label{eq:lx-frame}
 \frac{p+1}{p}E^{\rm QR}
 <\frac{2}{5x^3}\left[129(153/10)^2+
 \left(\sqrt{400+8x^2/5}+101\right)^2\right]<\frac3{10}.
\end{equation}
\end{proposition}
\begin{proof}
\readability{We compare the full ridge minimizer with a trial obtained by
correcting the sparse comparison coefficient. The proof bounds the trial's
residual, derives a sampling estimate
for arbitrary coefficient vectors, and controls the difference between
the trial and the minimizer through the normal equations. The alias
identity \eqref{eq:lx-EU} then converts these bounds into held-out error.}
Here $x=\sqrt p\ge42$. The scalar definitions and Lemma~\ref{lem:lx-defect}(iv) give
\begin{equation*}
 \kappa_p<1001/1000,\qquad
 1/129<\ell\le\ell_+<(21/20)^2/128<1/116,
\end{equation*}
using $2p/((p-1)\sqrt{p+1})<84/1763<1/20$.

\emph{The sparse target trial.}
Set $z_k=z'_k-\zeta_k$, $c_k=w_k-z_k$ and $e_k=\one-\mat{H}_kc_k$.
Lemma~\ref{lem:lx-teacher} and $\ell>1/129$ give
\begin{equation*}
 \mat{B}_k\zeta_k=0,\quad z'_k\perp\zeta_k,\quad \|\zeta_k\|^2<4/5,
 \quad \mat{B}_kz'_k+r'_k=f'_k-\one,\quad
 \|z'_k\|^2+\|r'_k\|^2/\ell<200/p.
\end{equation*}
Since $c_k=w'_k-z'_k=w_k-z_k$, its full-grid residual is
$e_k=\mat{E}_kz'_k-r'_k$. The norm of $[\mat{E}_k,-\sqrt\ell \Id]$, the nullspace
orthogonality and Lemma~\ref{lem:lx-defect}(iii) therefore give
\begin{readabilityblock}
\begin{equation}\label{eq:readable-e2-trial-bounds}
 \cR:=\sum_k\|e_k\|^2<\frac{200}{p}(\Delta_p^2+2\ell),
 \qquad \cZ:=\sum_k\|z_k\|^2<400/p+8/5.
\end{equation}
\end{readabilityblock}
All paired sums run over $k\in\{1,\nu\}$.

\emph{Sampling all Fourier modes.}
For fixed $k$ write $\mat{B}=\mat{B}_k$, $\mat{H}=\mat{H}_k$, $\mat{A}=\mat{H}[Q]$ and $\mat{K}=\mat{A}^*\mat{A}+\ell \Id$.
Put $\mat{M}=\mat{\mathsf F}\diag(\one_Q)\mat{\mathsf F}^*$,
$\alpha=(p-1)/(2p)$, $m=(\sqrt p+1)/(2p)$ and $\theta=\alpha+m$.
Set $T=(p+1)/6$ and $L=\ell+(\sqrt{T/x}+\sqrt2\kappa_p)^2$.
The matrix $\mat{M}$ has diagonal $\alpha$ and off-diagonal moduli at most $m$.
Thus $\psi^*\mat{M}\psi\ge\theta\|\psi\|^2-m\|\psi\|_1^2$,
where $\theta=(x+1)/(2x)>1/2$.
For $\psi=\mat{\mathsf F}\mat{B}\xi$, Fourier-row orthogonality gives
$\|\psi\|_1^2\le T\|\xi\|^2$.
Since $m/\theta=1/x$,
$\|\mat{B}[Q^{\rm c}]\xi\|^2\le\|\mat{B}[Q]\xi\|^2+(T/x)\|\xi\|^2$.
Using $\mat{H}-\mat{B}=\mat{E}_k$ and its orthogonal training/held-out restrictions gives
\begin{readabilityblock}
\begin{equation}\label{eq:readable-e2-sampling}
\begin{aligned}
 \|\mat{H}[Q^{\rm c}]\xi\|
 &\le\|\mat{A}\xi\|+(\sqrt{T/x}+\sqrt2\kappa_p)\|\xi\|\\
 &\le\sqrt{L/\ell}\,\|\mat{K}^{1/2}\xi\|.
\end{aligned}
\end{equation}
\end{readabilityblock}
For $x\ge42$, $2\ell/x<1/1000$, $\sqrt{T/p}<41/100$ and
$\sqrt2\kappa_p/\sqrt x<11/50$. The last inequality follows by
squaring and using $2x/(x^2-1)\le84/1763<(11/50)^2$. Hence
\begin{readabilityblock}
\begin{equation}\label{eq:readable-e2-sampling-constant}
 (L+\ell)/x<1/1000+(41/100+11/50)^2<2/5.
\end{equation}
\end{readabilityblock}

\emph{A target-matched normal residual.}
Put $s=p/(\lambda_0d)$ and
\begin{equation*}
 u_k=\mat{B}_k[Q]^*\one_Q-w_k/s,\qquad v_k=\mat{E}_k[Q]^*\one_Q,\qquad
 \rho_k=\ell s\one_Q,\qquad \delta_k=\hat c_k-c_k.
\end{equation*}
Over $k\in\{1,\nu\}$, put $U=\sum_k\|u_k\|^2$,
$\cK=\sum_k\delta_k^*\mat{K}_k\delta_k/\ell$ and
$e_{\rm ho}^2=\sum_k\|e_k[Q^{\rm c}]\|^2$.
The subtraction in $u_k$ removes exactly the constant Fourier component.
Since $|(\mat{\mathsf F}\one_Q)(a)|\le(1+1/x)/2<21/40$ for $a\ne0$ and $T<p/5$,
\begin{equation*}
 U=\sum_k\sum_{a\ne0}\lambda_k(a)|(\mat{\mathsf F}\one_Q)(a)|^2
 <2(21/40)^2p/5<p/9.
\end{equation*}
Also $\sum_k\|v_k\|^2<d\Delta_p^2$. The normal equations give
\begin{readabilityblock}
\begin{equation}\label{eq:readable-e2-normal-residual}
\begin{aligned}
 \mat{K}_k\delta_k&=\mat{A}_k^*(e_k[Q]-\rho_k)
 -\ell(c_k-\mat{A}_k^*\rho_k/\ell),\\
 c_k-\mat{A}_k^*\rho_k/\ell&=-(z_k+s(u_k+v_k)).
\end{aligned}
\end{equation}
\end{readabilityblock}
The operator $\mat{K}_k^{-1/2}[\mat{A}_k^*,-\sqrt\ell \Id]$ is a coisometry.
\readability{Apply it to \eqref{eq:readable-e2-normal-residual} jointly over
the two aliases. Adding the held-out residual pairs $e_k[Q]$ with
$e_k[Q^{\rm c}]$, so \eqref{eq:readable-e2-trial-bounds} controls their
full-grid norm. The triangle inequality and $1/\ell<129$ then give}
\begin{readabilityblock}
\begin{equation}\label{eq:readable-e2-residual-frame}
\begin{aligned}
 \cK+e_{\rm ho}^2/\ell
 &<129\left(\sqrt{\frac{200}{p}(\Delta_p^2+2\ell)}+\ell s\sqrt{2d}\right)^2\\
 &\quad+\left(\sqrt{400/p+8/5}+s\sqrt U+s\sqrt d\Delta_p\right)^2.
\end{aligned}
\end{equation}
\end{readabilityblock}
Here $sp=96p^4/[(p^2-1)(p^2+3)]<96<961/10$, since
$(p^2-1)(p^2+3)=p^4+2p^2-3>p^4$.
Also $h_\circ>2$ and $\ell<1/116$ imply
$\Delta_p^2+2\ell<1+1/1763+1/58<(101/100)^2$, and
$\sqrt d\Delta_p<\sqrt{p/2}<(71/100)\sqrt p$.
After multiplication by $x$, the first parenthesis is below
\begin{equation*}
 (101/100)(71/5)+(961/10)/116<153/10.
\end{equation*}
The second is below
$\sqrt{400+8x^2/5}+(961/10)(1/3+71/100)<\sqrt{400+8x^2/5}+101$.

\readability{Finally, $\one-g_k=e_k-\mat{H}_k\delta_k$.
Applying \eqref{eq:readable-e2-sampling} to $\delta_k$ and using the alias
identity \eqref{eq:lx-EU} gives}
\begin{readabilityblock}
\begin{equation}\label{eq:readable-e2-final-frame}
 \sqrt{(p+1)E^{\rm QR}}
 \le e_{\rm ho}+\sqrt{L\cK}
 \le\sqrt{L+\ell}\sqrt{\cK+e_{\rm ho}^2/\ell}.
\end{equation}
\end{readabilityblock}
\readability{Substitute the parenthesis bounds into
\eqref{eq:readable-e2-residual-frame}, square \eqref{eq:readable-e2-final-frame},
and use \eqref{eq:readable-e2-sampling-constant}. Dividing by $p=x^2$
gives the first inequality in \eqref{eq:lx-frame}.}
The quotient
$[\sqrt{400+8x^2/5}+101]/x=\sqrt{400/x^2+8/5}+101/x$
decreases for $x>0$, and $400+8\cdot42^2/5=16112/5<57^2$.
Thus $\sqrt{400+8x^2/5}+101<158x/42$ for $x\ge42$.
After this bound, both terms are decreasing in $x$, so the right side is at most
\begin{equation*}
 \frac{2[129(153/10)^2+158^2]}{5\cdot42^3}<\frac3{10}.
\end{equation*}
\end{proof}

\begin{proposition}[Late refit for every prime $p\ge1764$]
\label{prop:lx-late}
For each fixed prime $p\ge1764$ and $\lambda>0$, all sufficiently small
$\epsilon>0$ have the following property. At every prescribed late arrival
$j\in\cI_{\rm late}$, the complete train-only refit has margin
greater than $11/50$ against every rival at each held-out input with a
nonsquare difference, normalized held-out error less than $121/400$, and
accuracy at least $(p-1)/(p+1)$. The same statements hold for the compressed
physical table.
\end{proposition}
\begin{proof}
\readability{At the QR anchor, the pointwise bound of
Lemma~\ref{lem:lx-alias}(iv) and the complete-error estimate
\eqref{eq:lx-frame} in Proposition~\ref{prop:lx-frame} give}
\begin{equation*}
 \|\mat{Y}^{\rm QR}(x)-T_y\|^2<3/10
\end{equation*}
on every nonsquare-difference input. Thus every rival $c\ne y$ has margin
\begin{readabilityblock}
\begin{equation}\label{eq:readable-e2-anchor-margin}
 \mat{Y}^{\rm QR}(x)_y-\mat{Y}^{\rm QR}(x)_c
 >1-\sqrt{3/5}>11/50.
\end{equation}
\end{readabilityblock}
The same frame gives $E^{\rm QR}<3/10<121/400$, with diagonals included.

The full-column anchor $[\mat{H}^{\rm QR}\mid0\mid0]$ has positive prescribed
penalty because $\ell>1/129$ for $\mat{H}^{\rm QR}/h_\circ$.
\readability{The late table limit \eqref{eq:acr-late-table-limit} of
Lemma~\ref{lem:acr-table-limits} makes $\mat{H}_j^{\rm ana}/R$ converge
uniformly to this full-column anchor over the prescribed late arrivals.
Lemma~\ref{lem:shared-full-ridge-transfer}(i) therefore retains margins
greater than $11/50$, using the reserve in \eqref{eq:readable-e2-anchor-margin},
and the strict error bound for all sufficiently small $\epsilon>0$.}
The error uses the full held-out denominator, including diagonal inputs.

Every nonsquare-difference held-out input is thus correctly classified,
giving accuracy at least $pd/[p(d+1)]=(p-1)/(p+1)$; no diagonal correctness
is asserted. Positive lift and zero-column invariance transfer the complete
refit predictions and errors to the physical table.
\end{proof}
\end{proofrevision}

\begin{proofrevision}
\subsection{Assembly: positive head step and an initial ball at every width}\label{sec:width}

\subsubsection{Finite-horizon assembly at a fixed width}
\label{sec:aasm-local}\label{sec:aasm-cont}\label{sec:a-width-assembly}
\readability{We now transfer the reference conclusions to a neighborhood
of actual training runs. The proof checks continuity of the update at
each reference departure, then uses the finite horizon to preserve all
strict observer and prediction bounds on one initial ball.}
\begin{definition}[Canonical state and discrete reference]\label{def:a-can}
At width $M$, let $\theta_M^{\rm can}$ consist of the physical core rows
constructed before Lemma~\ref{lem:round15-core-only-bootstrap}, the helper
and query rows in Section~\ref{sec:acr-groups}, their heads in
\eqref{eq:round15-init-helper}, \eqref{eq:round15-init-query} and
\eqref{eq:round15-compressed-core}, and the dormant rows of
Definition~\ref{def:acr-groups}. Define the frozen-head arrivals by
\begin{equation}\label{eq:a-ref}
 \theta_{M,0}^\circ=\theta_M^{\rm can},\qquad
 \theta_{M,j+1}^\circ=\cU_{0,0}(\theta_{M,j}^\circ),\quad 0\le j<J.
\end{equation}
For a radius $\delta>0$, write
$\cB_M=\{\theta_0:\|\theta_0-\theta_M^{\rm can}\|_F<\delta\}$.
\end{definition}

\begin{proposition}[Finite-horizon assembly at every width]
\label{prop:aasm-continuation}\label{prop:awd-width}
Fix a prime $p\ge101$, a width $M\ge M_{\min}(p)$ and $\lambda>0$.
For every sufficiently small $\epsilon>0$ as in
Definition~\ref{def:half-scales}, there are
$\delta\in(0,1/2]$, $\eta_0>0$ and $d_0\in(0,1/2]$ such that every
$\theta_0\in\cB_M$ and all $0\le\eta<\eta_0$, $0\le d_V<d_0$
satisfy:
\begin{enumerate}[label=(\alph*)]
\item The width-$M$ trajectory satisfies Theorem~\ref{paper:thm:half}(i)--(iii)
on the prescribed windows. If $d_V=0$ and $\eta>0$, its actual head moves
at every departure $j<J$.
\item If $d_V=0$ and $p\ge1764$, it satisfies part (iv) and
$E_{\rm ref}(j)<121/400$ at every $j\in\cI_{\rm late}$.
\end{enumerate}
The constants may depend on the fixed width $M$. In particular,
$\eta=\eta_0/2$, $d_V=0$ gives the public theorem and its late-refit
error supplement, while the same ball and head step serve part (a) with
any fixed $0<d_V<d_0$.
\end{proposition}
\begin{proof}
Fix $p\ge101$, $\lambda>0$ and an admitted width $M$.
Choose and fix a sufficiently small $\epsilon>0$ satisfying the component
certificates above, including $\epsilon<1/200$. This fixes the rates,
windows and finite horizon $J$ before the initialization is chosen.
Use the width-$M$ canonical state and frozen-head reference
$\theta_{M,j}^\circ$ of Definition~\ref{def:a-can}.

\paragraph{Reference gates, rank and observers.}
Lemma~\ref{lem:round15-core-only-bootstrap} supplies the strict core
signs and full supported source rank at every reference departure.
The noncore rows have zero incoming sources: helpers have zero frozen
heads, queries are inactive on training, and dormant rows are negative.
They decay by $q_{\rm clk}$ and retain their strict training signs.
Their zero source rows append zeros to the core source, preserving its
rank and supported Polar factor. Thus the same active frozen-reference
trajectory serves this fixed width.

The fixed positive pair lift and the appended zero feature columns
preserve the complete refit and Fourier observers by
Lemma~\ref{lem:acr-isometry}. Hence the reference conclusions apply at
this width: Proposition~\ref{prop:aex-margins} supplies strict training
margins, the CE interval and variation bound, and the early and late
actual prediction signs; Proposition~\ref{prop:afr-netgain} supplies the
late Fourier gain and positive centered-table energy;
Proposition~\ref{prop:arf-early} supplies strict wrong-class margins for
the complete early refit. If $p\ge1764$, Proposition~\ref{prop:lx-late}
supplies the strict late error bound and $11/50$ nonsquare refit margins.
The latter margins count $pd$ correct rows among $p(d+1)$ held-out rows;
the full error denominator still includes the diagonal rows.

\paragraph{Joint continuity of the actual update.}
For $\theta=(\mat{W},\mat{V})$ write $\zeta=(\theta_0,\eta,d_V)$,
$\zeta^0=(\theta_M^{\rm can},0,0)$ and
$\Theta_j(\zeta)=\cU_{\eta,d_V}^{\circ j}(\theta_0)$.
Let $\mat{K}_E$ have orthonormal columns spanning $E$, and put
$\mat{\Gamma}(\theta)=\mat{K}_E^\top \mat{F}(\theta)^\top \mat{F}(\theta)\mat{K}_E$.
At a frozen reference departure write $\mat{G}=-\nabla_{\mat{V}}L$, and let
$P(c\mid x_i)$ be its softmax probabilities. Strict training margins give
\begin{readabilityblock}
\begin{equation}\label{eq:readable-e2-head-source}
 \ip{\mat{V}}{\mat{G}}_F
 =\frac1n\sum_{i\in\Train}\sum_{c\ne y_i}
 P(c\mid x_i)\bigl(f_{y_i}(x_i)-f_c(x_i)\bigr)>0.
\end{equation}
\end{readabilityblock}
\readability{Equation~\eqref{eq:readable-e2-head-source} ensures that the
head normalization has a nonzero denominator.} On a neighborhood with the same training gates, $\mat{F},\mat{G}$
are continuous and $\mat{F}=\mat{F}\mat{P}_E$ by Lemma~\ref{lem:acr-design}.
The supported source rank supplied by the bootstrap makes
$\mat{\Gamma}(\theta)$ positive definite, so it stays positive definite nearby.
There
\begin{readabilityblock}
\begin{equation}\label{eq:readable-e2-update-normalizers}
 \cN_{\rm pol}(\mat{F})=
 \frac{\mat{F}\mat{K}_E\mat{\Gamma}(\theta)^{-1/2}\mat{K}_E^\top}{\sqrt r},\qquad
 \cN_F(\mat{G})=\mat{G}/\|\mat{G}\|_F
\end{equation}
\end{readabilityblock}
\readability{are continuous. Substituting \eqref{eq:readable-e2-update-normalizers}
into the update \eqref{eq:a-update} proves joint continuity of
$\cU_{\eta,d_V}$ in state, $\eta$ and $d_V$ at $(\theta_{M,j}^\circ,0,0)$.}
Finite composition makes every arrival $\Theta_j(\zeta)$, $0\le j\le J$,
continuous at $\zeta^0=(\theta_M^{\rm can},0,0)$.

\paragraph{One full-coordinate ball and the two decay scopes.}
At every reference arrival the training feature table is nonzero,
because vanishing features would give zero logits and zero training
margins. Its prescribed refit penalty is therefore positive.
The ridge formula makes the full refit scores and normalized errors
continuous at that table. The positive centered-table energy likewise
makes the Fourier fraction continuous. Held-out ReLU gates need not
stay fixed for these continuity statements.

The strict reference conclusions include every training contrast in
$(0,0.0201)$, not only the aggregate CE interval. They involve finitely
many arrivals, inputs and classes. Joint continuity therefore supplies a
full-coordinate
$\delta$-ball about $\theta_M^{\rm can}$ and positive parameter caps
$\eta_0,d_0$ preserving part (a); shrink them so
$\delta\le1/2$ and $d_0\le1/2$. For $p\ge1764$, continuity on the slice
$d_V=0$ also preserves the late refit's strict error and nonsquare margins.
Shrink $\delta$ and $\eta_0$, keeping $d_0$ fixed, so that
$\cB_M\times[0,\eta_0)\times\{0\}$ lies in this slice neighborhood.
Shrinking preserves part (a), and proves part (b) on the same final ball
and head-step interval, without asserting a late-refit conclusion
for positive decay.

Accuracy is preserved through the stated strict correct or wrong score
margins, not through continuity of the accuracy function. The same
positive training margins give $\mat{G}_j\ne0$ at every actual departure in
the selected neighborhood, so the head increment has norm $\eta$ when
$d_V=0$ and $\eta>0$. All $3pM$ initial coordinates may vary in this ball.
\end{proof}

\begin{proof}[Proof of Theorem~\ref{paper:thm:half}]
Fix $p\ge101$, $M\ge M_{\min}(p)$ and $\lambda>0$.
Choose $\epsilon>0$ sufficiently small for all component certificates in
Proposition~\ref{prop:aasm-continuation}, including the late-refit
certificate when $p\ge1764$. This fixes $s,J$ and the two windows in
\eqref{eq:clock}--\eqref{eq:acr-windows} before the initialization is chosen.
Lemma~\ref{lem:acr-windows} gives
$\cI_{\rm early}=\{1,\ldots,32\}$, $|\cI_{\rm late}|\ge16$, and
$\max\cI_{\rm early}<\min\cI_{\rm late}$, with both windows contained
in $\{1,\ldots,J\}$.
Take the final ball and constants of Proposition~\ref{prop:aasm-continuation}
and set $(\eta_W,\lambda_W,\eta_V)=(s,\lambda,\eta_0/2)$, $\lambda_V=0$.
The ball $\cB_M$ is nonempty and open in $\R^{3pM}$.
Part (a) of that proposition gives (i)--(iii), and part (b) gives (iv)
for $p\ge1764$, with all guarantees on this same ball and trajectory.
\end{proof}

\paragraph{Scope.}
The positive constants $\delta,\eta_0,d_0$ may depend on $p$, $M$,
$\lambda$, and the chosen witness $\epsilon$ and its finite $J$.
No quantitative lower bound or uniformity over widths is asserted.
For each fixed width, the same selected ball, head step, horizon and
windows serve the undecayed theorem and the half-grid decay conclusion.
\end{proofrevision}

\section{A random-split low-CE theorem with refit-first recovery}\label{sec:r62}
\proofrevtext{The task, network and Fourier observable are those of
\eqref{paper:model} and~\eqref{paper:AH}. This is a different construction
from Theorem~\ref{paper:thm:half}: its width, readout penalty and asymptotic
regime are specified in Theorem~\ref{paper:thm:random} and
Section~\ref{paper:setup:refit}. Its dictionary, split event, calibrated head,
reference path and initial ball are defined in the corresponding
subsections below. Symbols in this appendix are local to it except for
the shared interface listed in Section~\ref{sec:notation-random}.}

\begin{proofrevision}
\subsection{Statement and construction}\label{sec:b-prelude}
\label{sec:r62-statement}\label{sec:r62-init}\label{sec:r62-refit}
The elementary construction formulas below apply to odd primes $p\ge5$;
the final theorem uses the larger common threshold.
Throughout this appendix, $m=M=p^2(p-1)^2$ denotes the hidden width.
With $h$ from Section~\ref{sec:notation-random} and $k,a$ from
Definition~\ref{def:r62-clock}, prescribe
\[
 \lambda=k/h,\qquad s=h(1-a)/k,\qquad \eta=\theta/(hp).
\]
These are fixed before sampling. The actual learner is
\begin{equation}\label{eq:r62optimizer}
 (\eta_W,\lambda_W,\eta_V,\lambda_V)=(s,\lambda,\eta,0)
\end{equation}
in~\eqref{paper:optimizer}, with mean CE over the realized split.
For the refit, write $\gamma=\kappa/[4p(p-1)]$ and
$q_{H_j}=\tau_j=\gamma\norm{\mat{H}_{j,\Train}}_F^2$.
Its target-normalized error is also denoted $E_{\rm obs}(\mat{H}_j)=E_{\rm ref}(j)$;
Definition~\ref{def:r62-refit} gives the same own-penalty error for any
nonzero training-feature table. We write $\Acc_{\rm ref}(\mat{U})$ for the
held-out accuracy of this refit of a feature table $\mat{U}$, with the fixed tie rule.
\begin{theorem}[Quantitative random-split realization]\label{thm:r62}
Theorem~\ref{paper:thm:random} holds on the event $\cE_{\rm split}$ of Definition~\ref{def:r62-split}
with the fixed updates~\eqref{eq:r62optimizer} and the training-data-built
center~\eqref{eq:r62dictionary}--\eqref{eq:r62head}. Its initial region
$\cR^{\rm random}_{p,\Train}$ is the nonempty ordinary open ball of
radius~\eqref{eq:r62ball} in all $3pm$ weight coordinates about that center.
The refit is~\eqref{eq:ridgekernel}, and $E_{\rm ref}(j)=E_{\rm obs}(\mat{H}_j)$.
In addition, the training margins in
clause~(i) are strictly positive, and the Fourier fraction in clause~(iv)
remains $O(1/p)$ throughout $0,\ldots,J$. The error gain compares each
trajectory with its own initial error. No accuracy after $J$ is asserted.
\end{theorem}

The rates are fixed before sampling, and both weight matrices remain trainable
and untied. The label-blind rectangular dictionary contains no added memories.
The head~\eqref{eq:r62head} uses two one-time centered train-only calibrations:
$b_{\rm base}$ for the canonical base, whose normalized residual is
$\widehat {\mat{R}}$, and $b_*$ for the mixed head. Proposition~\ref{prop:bse-cal}
proves their existence, uniqueness and bounds. This engineered warm start
is neither a zero head nor a proved outcome of earlier training.
Neither the actual initial head nor the separate complete diagnostic correctly
predicts any held-out input: Proposition~\ref{prop:bse-cal} and
Lemma~\ref{lem:bwc-obs0} establish the center claims, and
Lemma~\ref{lem:bfr-induction} and
Proposition~\ref{lem:sc26-b-head-interface} extend them to every initial
state in the prescribed ball. Refitted heads are never installed in the learner.

\paragraph{Proof strategy.}
\begin{readabilityblock}
The construction first fixes the split event and clock, then calibrates the
initial head in Proposition~\ref{prop:bse-cal}. The reference kernels in
Section~\ref{sec:bfr-ref} describe how this center evolves along the
fixed-direction comparison path. Lemma~\ref{lem:bfr-induction} then controls
actual states on the prescribed $r_6$ ball, uniformly over the event and
every head sequence with increments at most $\eta$. Its Frobenius estimate
gives mean loss and held-out accuracy; its rank and operator estimates
give strict training margins. These are the outputs of Claim~I.

Claim~II, in Section~\ref{sec:b-claim-representation}, proves the full-refit
and Fourier conclusions separately on those same states. Initial
actual-head failure comes from Lemma~\ref{lem:bfr-induction}, while initial
refit failure uses the signed estimate in Lemma~\ref{lem:bwc-obs0}.
Proposition~\ref{lem:sc26-b-head-interface} assembles the conclusions with
one pre-sampling threshold. The actual trajectory is not required to
retain its initial gates.
\end{readabilityblock}
\end{proofrevision}

\begin{proofrevision}
\subsection{Claim I: actual dynamics, plateau and delayed network recovery}\label{sec:b-claim-dynamics}
\subsubsection{Split model, dictionary and the initial-data event}\label{sec:bse}

\paragraph{Grid notation.}
Use the task and centered targets of Sections~\ref{paper:setup:model}
and~\ref{paper:setup:refit}, and the shared height scale $h$ of
Section~\ref{sec:notation-random}.
For $x=(a,b)$, write $\bx_x=x_{ab}$, $y_x=y^\star(x)$ and $t_x=T_x$.
Here $\mat{P}=\mat{C}_p$ is class centering, so $\mat{T}_\Train=\mat{Y}_\Train \mat{P}$, and
$n_\Test=|\Test|=p^2-n$.
\begin{definition}[Grid tables and construction scale]\label{def:r62-task}
Let $\mat{X}\in\R^{\Omega\times2p}$ be the grid input table with rows $\bx_x^\top$,
and let $\mat{Y}_\Train$ be the one-hot training-label table.
The visible space $E_{\rm vis}=\{(v,v')\in\R^p\times\R^p:\one^\top v=\one^\top v'\}$
contains every $\bx_x$ and has dimension $r_v=\dim E_{\rm vis}=2p-1$;
write $A=\log(p-1)+a_{\rm conf}$ for the confidence scale.
The fixed theorem parameters and Bernoulli split are those of
Theorem~\ref{paper:thm:random}.
\end{definition}

\paragraph{Split model.}
\begin{definition}[Sampled masks and counts]\label{def:r62-counts}
The mask $\mat{M}\in\{0,1\}^{p\times p}$ has independent Bernoulli$(\rho)$ entries,
and $\Train=\{(a,b):M_{ab}=1\}$. We write $M_x=M_{ab}$ for $x=(a,b)$,
$n^{\rm row}_a=\sum_bM_{ab}$, $n^{\rm col}_b=\sum_aM_{ab}$ and $n_c=\sum_aM_{a,c-a}$
for the row, column and class counts, and
\begin{equation*}
\hat n^{\rm row}_{a,c}=M_{a,c-a}-\frac{n^{\rm row}_a}p,\qquad
\hat n^{\rm col}_{b,c}=M_{c-b,b}-\frac{n^{\rm col}_b}p,\qquad
\hat n_c=n_c-\frac np=\sum_a\hat n^{\rm row}_{a,c}=\sum_b\hat n^{\rm col}_{b,c}
\end{equation*}
for their class-centered versions. 
The centered masks are $\bar {\mat{M}}=\mat{M}-\rho\one\one^\top$,
$\bar M^{\rm row}_{a,c}=M_{a,c-a}-\rho$ and
$\bar M^{\rm col}_{b,c}=M_{c-b,b}-\rho$.
For $x=(a,b)$ set $N_{xc}=M_{a,c-a}+M_{c-b,b}$.

Write $\hat n=\mat{T}_\Train^\top\one$, with coordinates $\hat n_c=n_c-n/p$.
\end{definition}

For fixed $x$, rival arm pairs use distinct Bernoulli entries, so the $N_{xc}$,
$c\ne y_x$, are independent Binomial$(2,\rho)$ variables, independent of $M_x$.
The three permuted centered mask matrices have independent entries individually;
independence between those matrices is never used.

\paragraph{Dictionary and forces.}
\begin{definition}[Dictionary and full-grid force]\label{def:r62-dictionary}
Let $b_0=(p+1)/2$, $I_0=\{0,\ldots,b_0-1\}$, $\mu=b_0/p$ and $q=\mu/2=(p+1)/(4p)$.
The neurons of~\eqref{eq:r62dictionary} are indexed by
$(u,v,u',v')\in(\F_p^\times\times\F_p)^2$, with $l$ and $r$ the indicator vectors
of $uI_0+v$ and $u'I_0+v'$. Write $\E_{l,r}$ for the uniform average over these
$m$ neurons and $g_x=l_ar_b$ for $x=(a,b)$. Let $\mat{O}_{\rm row}$ and $\mat{O}_{\rm col}$ be
the $\Omega\times\Omega$ matrices with entries $\one\{a=a'\}$ and $\one\{b=b'\}$,
and $\mat{B}_0=\Id+\mat{O}_{\rm row}+\mat{O}_{\rm col}+\one\one^\top$.

Denote the prescribed center by $(\bar {\mat{W}}_0,\bar {\mat{V}}_0)$ and its feature
table by $\bar {\mat{H}}_0=\mat{H}(\bar {\mat{W}}_0)$; unbarred initial states may be any
point in the certified ball. For the two indicator masks $l,r\in\{0,1\}^p$,
the dictionary row is
\begin{equation}\label{eq:r62dictionary}
 (\bar {\mat{W}}_0)_{(l,r)}=\frac{\alpha}{\sqrt m}
 (2l-\one-\beta\one,\;2r-\one-\beta\one).
\end{equation}

The full-grid rule force is $c_F=\mat{X}^\top\diag(g)\mat{T}\mat{T}^\top g$.
Forces use the neuron norm $\norm{c}_{L^2}=(\expctover{l,r}{\norm{c(l,r)}^2})^{1/2}$.
The interaction source is $f_x=(l_a-\mu)(r_b-\mu)$.

For the center dictionary put $d=h^2q^2$, $\mat{H}=\bar {\mat{H}}_{0,\Train}$,
$\mat{B}=\mat{H}\mat{H}^\top/d$ and $w=\one_m/\sqrt m$.
\end{definition}

\begin{lemma}[Dictionary identities]\label{lem:bse-dict}
For the dictionary and kernels in Definition~\ref{def:r62-dictionary}:
\begin{enumerate}[label=(\roman*)]
\item $\Pr(l_a=1)=\mu$ and $\Pr(l_a=l_{a'}=1)=q$ for $a\ne a'$. Hence
$\expct{ll^\top}=q(\Id+\one\one^\top)$ and the covariance of $l$ is $q\mat{P}$; the same holds
for $r$, which is independent of $l$.
\item $\expctover{l,r}{gg^\top}=q^2\mat{B}_0$, and the interaction source satisfies
$\expctover{l,r}{f_xf_{x'}}=q^2P_{aa'}P_{bb'}$.
\item $\bar H_{0,x,(l,r)}=hg_x/\sqrt m$. Hence $\bar {\mat{H}}_0\bar {\mat{H}}_0^\top=d\mat{B}_0$ with
$\bar {\mat{H}}_0w=h\mu^2\one_\Omega$, and the normalized training Gram matrix $\mat{B}$
of~\eqref{eq:r62head} equals $\mat{B}_{0,\Train\Train}$.
\item $\mat{B}_0\mat{T}=\mat{T}$, $\mat{T}^\top \mat{T}=p\mat{P}$, $\norm{\mat{T}}_{\rm op}=\sqrt p$, $\norm{\mat{X}}_{\rm op}=\sqrt{2p}$,
and $\expctover{l,r}{\norm{\mat{T}^\top g}^2}=q^2p(p-1)$.
\end{enumerate}
\end{lemma}

\begin{proof}
(i) Fix $a$. For each $k\in I_0$ and $u\ne0$ exactly one $v$ gives $uk+v=a$, so
$a\in uI_0+v$ for $b_0(p-1)$ of the $p(p-1)$ pairs. For $a\ne a'$, the map
$(k,k')\mapsto(u,v)=((a-a')/(k-k'),a-uk)$ is a bijection from ordered pairs of
distinct elements of $I_0$ onto the pairs with $a,a'\in uI_0+v$, so the
probability is $b_0(b_0-1)/(p(p-1))=q$. As $\mu=2q$ and $\mu^2=q(1+1/p)$, the
covariance is $q(\Id+\one\one^\top)-\mu^2\one\one^\top=q\mat{P}$. (ii) follows from (i)
and the independence of $l$ and $r$. (iii) At $x$, the preactivation of neuron
$(l,r)$ is $\alpha(2l_a+2r_b-2-2\beta)/\sqrt m$, which equals $h/\sqrt m$ if
$l_a=r_b=1$ and is at most $-2\alpha\beta/\sqrt m<0$ otherwise. Thus
$\bar {\mat{H}}_0\bar {\mat{H}}_0^\top=h^2\expctover{l,r}{gg^\top}$ and $(\bar {\mat{H}}_0w)_x=h\expctover{l,r}{g_x}=h\mu^2$.
(iv) Every operand row and column of $\Omega$ meets each class exactly once, so
$\mat{O}_{\rm row}\mat{T}=\mat{O}_{\rm col}\mat{T}=0$ and $\one^\top \mat{T}=0$; $\mat{T}^\top \mat{T}=p\mat{P}$, and
$\mat{X}^\top \mat{X}=\bigl(\begin{smallmatrix}p\Id&\one\one^\top\\\one\one^\top&p\Id\end{smallmatrix}\bigr)$
has largest eigenvalue $2p$. Finally
$\expct{\norm{\mat{T}^\top g}^2}=q^2\op{tr}(\mat{T}^\top \mat{B}_0\mat{T})=q^2\op{tr}(\mat{T}^\top \mat{T})$.
\end{proof}

For a neuron $(l,r)$, a source $\varphi\in\R^\Omega$ (which may depend on
the neuron) and $\Train'\subseteq\Omega$, define the gated force
\begin{equation}\label{eq:bse-force}
c_\varphi(\Train')=\mat{X}_{\Train'}^\top\diag(g_{\Train'})\,\mat{T}_{\Train'}\mat{T}_{\Train'}^\top\varphi_{\Train'}
=\sum_{i,j\in\Train'}g_i\ip{t_i}{t_j}\varphi_j\,\bx_i\in\R^{2p}.
\end{equation}
In particular $c_F=c_g(\Omega)$. Since $g-f=\mu l_a+\mu r_b-\mu^2$
is additive, $\mat{T}^\top(g-f)=0$; hence $c_F=c_f(\Omega)$ and $c_{\one}(\Omega)=0$.

\paragraph{Count, design and arm events.}
\begin{definition}[Count thresholds and design/arm events]\label{def:r62-count-events}
Use the pre-sampling mixing coefficient $\xi$ of Definition~\ref{def:r62-clock}.
Before sampling define
\begin{equation}\label{eq:bse-arm}
r_\xi=e^{-A\xi},\qquad Q_\xi=(p-1)(\rho+(1-\rho)r_\xi)^2,\qquad
\varsigma=\rho(1-\rho).
\end{equation}

For $\delta\in(0,1)$ put $\ell_\delta=\log(32p^2/\delta)$ and
\begin{equation}\label{eq:bse-Dstar}
D_*=\sqrt{p\ell_\delta/2},\quad t_*=\sqrt{4p\log9+2\log(24/\delta)},\quad
Q_*=\sqrt{2p\ell_\delta},\quad K_*=4D_*+t_*.
\end{equation}
Put $n_{\max}=\rho p+D_*$ and $n_-=p(\rho p-D_*)$. Set
\begin{equation*}
D_H=\sqrt{\frac{3p}2\log\frac{2p(p-1)}\delta},\quad
D_L=\sqrt{\frac{p-1}2\log\frac{4p^2}\delta},\quad
D_W=\sqrt{\frac{p(p-1)}2\log\frac{2p}\delta}.
\end{equation*}
Define $\cE_{\rm des}(\delta)$ by
\begin{align}
&\max_a|n^{\rm row}_a-\rho p|,\ \max_b|n^{\rm col}_b-\rho p|,\ \max_c|n_c-\rho p|\le D_*;\label{eq:bse-des1}\\
&\norm{\bar {\mat{M}}}_{\rm op},\ \norm{\bar {\mat{M}}^{\rm row}}_{\rm op},\ \norm{\bar {\mat{M}}^{\rm col}}_{\rm op}\le t_*;\label{eq:bse-des2}\\
&\max_{a,c}\Bigl|\sum_b\bar M_{ab}\hat n^{\rm col}_{b,c}\Bigr|\le Q_*,\qquad
\max_{b,c}\Bigl|\sum_a\bar M_{ab}\hat n^{\rm row}_{a,c}\Bigr|\le Q_*.\label{eq:bse-des3}
\end{align}

For $y\ne c$ put $\mathsf H_{yc}=\sum_aM_{a,y-a}M_{a,c-a}M_{a+c-y,y-a}$.
Let $N_x^{(2)}=\#\{c\ne y_x:N_{xc}=2\}$ and put
$\mathsf W_c=\sum_{i\in\Train,\,y_i\ne c}r_\xi^{2-N_{ic}}$.
The other three events are
\begin{equation*}
\begin{aligned}
\cE_H(\delta)&=\{|\mathsf H_{yc}-\rho^3p|\le D_H\ \forall y\ne c\},\\
\cE_L(\delta)&=\{|N_x^{(2)}-(p-1)\rho^2|,\\
 &\qquad |\#\{c\ne y_x:N_{xc}=0\}-(p-1)(1-\rho)^2|\le D_L\ \forall x\},\\
\cE_W(\delta)&=\bigl\{\bigl|\mathsf W_c-\rho\bigl[n_c(n_c-1)
 +2r_\xi n_c(p-n_c)\\
 &\qquad+r_\xi^2(p-n_c)(p-n_c-1)\bigr]\bigr|\le D_W\ \forall c\bigr\}.
\end{aligned}
\end{equation*}

The count certificates used in the bounds are $\mathsf H_{\min}=\rho^3p-D_H$,
$\zeta_x^0=\sum_{c\ne y_x}r_\xi^{2-N_{xc}}$ and $\Delta_W=D_W+(3p+1)D_*+p$.
\end{definition}

For the count calculation, put
\[
 u_\xi=\rho+(1-\rho)r_\xi,\qquad
 N_x^{(0)}=\#\{c\ne y_x:N_{xc}=0\},
\]
and
\[
 \mathsf F(k)=k(k-1)+2r_\xi k(p-k)+r_\xi^2(p-k)(p-k-1).
\]
\begin{lemma}[Count and arm events]\label{lem:bse-W}
For $\delta\in(0,1)$, use the four count/arm events in
Definition~\ref{def:r62-count-events}. Each has probability at least $1-\delta$.
On $\cE_{\rm des}$, $|\hat n_c|\le2D_*$; on $\cE_H$,
$\mathsf H_{yc}\ge\mathsf H_{\min}$; on $\cE_L$,
$|\zeta_x^0-Q_\xi|\le(1-r_\xi^2)D_L\le D_L$; and on
$\cE_{\rm des}\cap\cE_W$,
\begin{equation*}
 |\mathsf W_c-n_cQ_\xi|\le\Delta_W.
\end{equation*}
\end{lemma}
\begin{proof}
\emph{Tail bound and scalar-event families.}

We use one bounded-range estimate. For a variable of range length $\ell$,
the second derivative of its centered log moment-generating function is a
tilted variance, at most $\ell^2/4$ by comparison with the range midpoint.
Integrating twice bounds it by $t^2\ell^2/8$. For sums with $r$ colors and
independence within each color, H\"older with exponent $r$ across colors,
followed by exponential Markov and optimization, gives
\[
\Pr\{|S-\expct{S}|\ge u\}\le2\exp[-2u^2/(r\sum_i\ell_i^2)].
\]
The argument also holds conditionally. The five scalar applications have
the following numbers of tests, effective range-square sums, and union costs:
\[
\begin{array}{c|c|c|c}
\text{family}&\text{tests}&r\sum_i\ell_i^2&\text{cost}\\ \hline
\text{counts}&3p&p&3\delta/(16p)\\
\text{cross sums}&2p^2&4p&\delta/8\\
\mathsf H&p(p-1)&3p&\delta\\
N^{(0)},N^{(2)}&2p^2&p-1&\delta\\
\mathsf W&p&p(p-1)&\delta
\end{array}
\]
Here the counts are sums of Bernoulli variables. The cross-sum terms
$\bar M_{ab}\hat n^{\rm col}_{b,c}$ lie in $[-1,1]$, use independent columns,
and have means $\varsigma(\one\{a=c-b\}-1/p)$ summing to zero; transpose for
the other family. For $\Delta=c-y$, write
$\mathsf H_{yc}=\sum_aY_aC_aC_{a+\Delta}$ with
$Y_a=M_{a,y-a}$, $C_a=M_{a,c-a}$. These mean-$\rho^3$ indicators have the
three-colorable $p$-cycle as dependency graph. Rival arm pairs at fixed $x$
are disjoint, giving independent indicators of means $\rho^2,(1-\rho)^2$.
Finally, condition on the class-$c$ line, with $k=n_c$ sampled entries.
Write each remaining input as $(a,c-a')$, $a\ne a'$. These $p(p-1)$ mask bits
are independent; their weights are
$k(k-1)$ copies of $1$, $2k(p-k)$ copies of $r_\xi$, and
$(p-k)(p-k-1)$ copies of $r_\xi^2$, giving conditional mean $\rho\mathsf F(k)$.
Substituting $D_*,Q_*,D_H,D_L,D_W$ gives the displayed costs.

\emph{Matrix-norm events and union accounting.}

For each of the three matrices in~\eqref{eq:bse-des2}, a unit bilinear form
has range-square sum $1$. A maximal $1/4$-separated sphere net has at most $9^p$
points, by comparing disjoint radius-$1/8$ balls with a radius-$9/8$ ball;
approximation in both arguments
bounds the matrix norm by twice the largest net bilinear form. Thus these
tests cost $6\cdot9^{2p}e^{-t_*^2/2}=\delta/4$. The total design cost is
$3\delta/(16p)+3\delta/8<\delta$, proving all four probability claims.

\emph{Consequences for counts, arms and weights.}

The centered-count bound follows because $n/p$ is the mean class count.
For $r\in[0,1]$,
$\sum_{c\ne y_x}r^{2-N_{xc}}=r(p-1)+(1-r)N_x^{(2)}-r(1-r)N_x^{(0)}$,
whose mean is $(p-1)(\rho+(1-\rho)r)^2$; its event error is at most
$(1-r^2)D_L$. For $r=r_\xi$,
$\mathsf F'(k)=2k(1-r)^2+2r(1-r)p-(1-r^2)\in[-1,2p]$ and
$\rho\mathsf F(\rho p)-\rho pQ_\xi=-\rho^2(1-\rho)(1-r)^2p$.
Together with $Q_\xi\le p-1$ and $|n_c-\rho p|\le D_*$ these prove the
stated bound on $\mathsf W_c$.
\end{proof}

\begin{definition}[Sampled-force events]\label{def:r62-force-events}
Let $c_f=c_f(\Train)$, $c_g=c_g(\Train)$ and $c_1=c_{\one}(\Train)$
be the forces~\eqref{eq:bse-force} with the interaction, gate and constant sources.
For $\delta\in(0,1)$ define the events
$\cE_f(\delta)=\bigl\{\norm{c_f-\rho^2c_F}_{L^2}\le t_f\bigr\}$ and
$\cE_g(\delta)=\bigl\{\norm{c_g-a_gc_1-\rho^2c_F}_{L^2}\le t_g\bigr\}$, with $t_f=p^{3/2}(2\delta)^{-1/2}$, $t_g=(2/\delta)^{1/2}p^{3/2}$ and the
neuron-dependent coefficient $a_g=\one_\Train^\top g_\Train/n\in[0,1]$ when $n>0$
(and $a_g=0$ when $n=0$).
\end{definition}

\begin{lemma}[Sampled initial forces]\label{lem:bse-forces}
For $p\ge5$ and $\delta\in(0,1)$, each sampled-force event
$\cE_f(\delta),\cE_g(\delta)$ of Definition~\ref{def:r62-force-events}
has probability greater than $1-4\delta$.
\end{lemma}
\begin{proof}
For $\varphi\in\{f,g,\one\}$, define the local finite-difference quantities
\[
\begin{aligned}
 \partial_iC&=C|_{M_i=1}-C|_{M_i=0},\qquad
 u=\mat{T}^\top\diag\bigl((M_x)_{x\in\Omega}\bigr)\varphi,\\
 v_c&=\sum_jM_jg_j(\one\{y_j=c\}-1/p)\bx_j,\qquad
 d_i=(1-1/p)g_i\varphi_i\bx_i.
\end{aligned}
\]
For the finite difference $\partial_iC$, orthogonality of products of $M_i-\rho$ gives
$\Var\paren{C}\le\varsigma\sum_i\expct{\norm{\partial_iC}^2}$: each nonconstant
monomial's variance is counted by its degree on the right. Apply this to
$C=c_\varphi$, $\varphi\in\{f,g,\one\}$, in the neuron $L^2$ space.
The quantities $u,v_c,d_i$ satisfy, by direct differentiation,
\[
\partial_i c_\varphi=g_i u_{y_i}\bx_i+\varphi_i v_{y_i}+(1-2M_i)d_i;
\]
the $M_i$ terms cancel. Here $|\varphi_i|\le1$ and
$\expctover{l,r}{\norm{\mat{T}^\top\varphi}^2}\le p^2$, so averaging over masks and neurons gives
$\expct{\norm{u}^2}\le(\rho^2+\varsigma)p^2=\rho p^2$.
At the full mask, each coordinate of $v_c$ has magnitude at most one
(on the left it is $l_a(r_{c-a}-\mu)$); also
$\sum_{c,j}\|g_j(\one\{y_j=c\}-1/p)\bx_j\|^2\le2p^2$.
Thus $\expct{\sum_c\norm{v_c}^2}\le2\rho p^2$.
Each class has $p$ inputs and $\sum_i\|d_i\|^2\le2p^2$, so the three-term
squared triangle inequality bounds the variance by $3\varsigma(4\rho p^3+2p^2)$.
Moreover $\expctover{M}{c_\varphi}-\rho^2c_\varphi(\Omega)=\varsigma\sum_i d_i$,
whose squared norm is at most $2\varsigma^2p^3$. Hence
\[
\expct{\norm{c_\varphi-\rho^2c_\varphi(\Omega)}_{L^2}^2}
\le3\varsigma(4\rho p^3+2p^2)+2\varsigma^2p^3
\le\frac{1112}{625}p^3<2p^3.
\]
Markov gives $\Pr(\cE_f^c)<4\delta$ and
$\Pr\{\|c_g-\rho^2c_F\|_{L^2}^2+\|c_1\|_{L^2}^2>p^3/\delta\}<4\delta$.
For every neuron-dependent $a\in[0,1]$,
$\|c_g-ac_1-\rho^2c_F\|_{L^2}^2\le
2(\|c_g-\rho^2c_F\|_{L^2}^2+\|c_1\|_{L^2}^2)$, proving $\cE_g$.
\end{proof}

\begin{definition}[Complete initial-data split event]\label{def:r62-split}
Let
\[
\begin{aligned}
\cE_{\rm split}={}&\cE_{\rm des}(p^{-1/2}/7)\cap\cE_H(p^{-1/2}/7)
 \cap\cE_L(p^{-1/2}/7)\\
 &\cap\cE_W(p^{-1/2}/7)\cap\cE_f(p^{-1/2}/7)\cap\cE_g(p^{-1/2}/7).
\end{aligned}
\]
Unless its argument is displayed, a split threshold uses $\delta=p^{-1/2}/7$.
The nonempty proper-split and denominator conditions are proved in
Proposition~\ref{prop:bse-event} and the assembly of Appendix~\ref{sec:r62}.
The event $\cE_H$ and its certificate $\mathsf H_{\min}$ are retained in
this event; the calibration argument below uses a direct bound from
$\cE_{\rm des}$ instead.
\end{definition}

For the following failure-budget calculation, write $\delta_1=p^{-1/2}/7$.
\begin{proposition}[Initial-data split event]\label{prop:bse-event}
The event $\cE_{\rm split}$ of Definition~\ref{def:r62-split} satisfies:
\begin{enumerate}[label=(\roman*)]
\item $\Pr(\cE_{\rm split})\ge1-2/\sqrt p$.
\item The thresholds are $t_f=(7/2)^{1/2}p^{7/4}$ and $t_g=14^{1/2}p^{7/4}$, and, as $p\to\infty$,
$t_*=O(\sqrt p)$, $D_W=O(p\sqrt{\log p})$ and $D_*,Q_*,K_*,D_H,D_L=O(\sqrt{p\log p})$.
\item For all sufficiently large $p$, $K_*<\rho p$ and $D_*<\min(\rho,1-\rho)p=p/5$.
Then, on $\cE_{\rm split}$, every operand row, operand column and class contains
both training and held-out inputs; in particular the split is proper and
$n_c\ge1$ for every class.
\end{enumerate}
\end{proposition}

\begin{proof}
The four count/arm failures cost at most $4\delta_1$ and the two force failures less than
$8\delta_1$, so their sum is less than $12/(7\sqrt p)<2/\sqrt p$.
At $\delta=\delta_1$ every logarithm in the count thresholds is $O(\log p)$,
which gives the orders in (ii), and $t_f,t_g$ follow by substitution.
Hence (iii) holds for large $p$, and then all row, column and class counts
lie strictly between $0$ and $p$.
\end{proof}

\paragraph{Standing assumptions and the sampled design.}
From here on the theorem parameters, including $J$, are fixed, the split
lies in $\cE_{\rm split}$, and $p$ is a sufficiently large odd prime.
All estimates are uniform over this event. One common threshold
$p_0\ge2048^2$ is chosen in Proposition~\ref{lem:sc26-b-head-interface};
the initial-data bounds at the beginning of the proof of
Lemma~\ref{lem:bfr-induction} establish every
denominator and sign condition used below, independently of any trajectory.
\begin{definition}[Sampled design]\label{def:r62-design}
Let $\mat{\Upsilon}=\diag(\Id,\Id,(p+2)^{-1})$ on
$\one^\perp\oplus\one^\perp\oplus\R$,
$z_x=(\mat{P}e_a,\mat{P}e_b,p^{-1/2})$ for $x=(a,b)$, and let
$\mat{\cZ}$ have rows $z_x$. Put $\mat{\cK}=\mat{\cZ}_\Train^\top\mat{\cZ}_\Train$, $\mat{\Psi}=\mat{\cZ}_\Train^\top \mat{T}_\Train$,
$s_0=\rho p-K_*$, $\kappa_*=K_*/(\rho p)$, $N_\Train=\norm{\mat{T}_\Train}_F$, $N_\Test=\norm{\mat{T}_\Test}_F$ and
\begin{equation}\label{eq:bwc-design}
\begin{gathered}
N_\Test^-=\sqrt{p((1-\rho)p-D_*)(1-1/p)},\quad N_+=\sqrt{p(p-1)}.
\end{gathered}
\end{equation}
\end{definition}

Use the sampled-design quantities~\eqref{eq:bwc-design}, and put
\[
 E_\circ=\one^\perp\oplus\one^\perp\oplus\R,\qquad
 n_\Test^-=p((1-\rho)p-D_*).
\]
\begin{lemma}[Sampled design]\label{lem:bwc-design}
$\mat{\cZ}^\top\mat{\cZ}=p\Id_{E_\circ}$ and $\mat{\cZ}\mat{\Upsilon}^{-1}\mat{\cZ}^\top=\mat{O}_{\rm row}+\mat{O}_{\rm col}+\one\one^\top$, so
$\mat{B}_0=\Id+\mat{\cZ}\mat{\Upsilon}^{-1}\mat{\cZ}^\top$. Moreover $n\ge n_-\ge\rho p^2/2$, $n_\Test\ge n_\Test^->0$,
$N_\Test^-\le N_\Test\le N_+$, $N_\Train\le N_+$ and $\norm{\mat{\cK}-\rho p\Id}_{\rm op}\le K_*$; hence
$s_0\Id\preceq\mat{\cK}\preceq(\rho p+K_*)\Id$ and $\kappa_*\le1/2$. Also
$\norm{\mat{X}_\Train}_{\rm op}\le\sqrt{2n_{\max}}$, $\norm{\mat{T}_\Train}_{\rm op}\le\sqrt{n_{\max}}\le\sqrt p$,
$\norm{\mat{\cZ}_\Test}_{\rm op}\le\sqrt p$, $\norm{\mat{\Psi}}_{\rm op}\le\sqrt3t_*$,
$\norm{\mat{\Psi}}_F\le\sqrt{3(p-1)}t_*$, $\mat{\Psi} \mat{P}=\mat{\Psi}$, $\mat{T}_\Train^\top \mat{T}_\Train=\mat{P}\diag(n_c)\mat{P}$, and
$\norm{\mat{\cZ}_\Train c}^2\ge(s_0/p)\norm{\mat{\cZ} c}^2$ for $c\in E_\circ$.
\end{lemma}
\begin{proof}
The identities for $\mat{\cZ}$ follow from
$z_x\mat{\Upsilon}^{-1}z_{x'}^\top=(\one\{a=a'\}-\frac1p)+(\one\{b=b'\}-\frac1p)+\frac{p+2}p$
and from summing $z_xz_x^\top$ over $\Omega$. Sum the split count bounds for $n,n_\Test$; each
target row has squared norm $1-1/p$ and each class has at most $p$ points. In block form, with
$\delta n^{\rm row}=n^{\rm row}-\rho p\one$ and $\delta n^{\rm col}=n^{\rm col}-\rho p\one$,
\begin{equation}\label{eq:ba-Kblock}
\mat{\cK}-\rho p\Id=\begin{pmatrix}\mat{P}\diag(\delta n^{\rm row})\mat{P}&\mat{P}\bar {\mat{M}}\mat{P}&\mat{P}\delta n^{\rm row}/\sqrt p\\
\mat{P}\bar {\mat{M}}^\top \mat{P}&\mat{P}\diag(\delta n^{\rm col})\mat{P}&\mat{P}\delta n^{\rm col}/\sqrt p\\
\delta n^{\rm row\top}\mat{P}/\sqrt p&\delta n^{\rm col\top}\mat{P}/\sqrt p&(n-\rho p^2)/p\end{pmatrix}.
\end{equation}
On $\cE_{\rm des}$ the upper-left $2\times2$ block has norm at most $D_*+t_*$, the
border vector has norm at most $\sqrt2D_*$, and $|n-\rho p^2|/p\le D_*$; so
$\norm{\mat{\cK}-\rho p\Id}_{\rm op}\le K_*$, and $K_*\le\rho p/2$ gives $\kappa_*\le1/2$.
The nonnegative matrix $\mat{X}_\Train^\top \mat{X}_\Train$ has row sums $2n^{\rm row}_a$ and $2n^{\rm col}_b$;
$\mat{T}_\Train=\mat{Y}_\Train \mat{P}$ with $\mat{Y}_\Train^\top \mat{Y}_\Train=\diag(n_c)$; and $n_{\max}<p$ as $D_*<p/5$.
The three blocks of $\mat{\Psi}$ are $\mat{P}\bar {\mat{M}}^{\rm row}\mat{P}$, $\mat{P}\bar {\mat{M}}^{\rm col}\mat{P}$ and
$\one^\top\bar {\mat{M}}^{\rm row}\mat{P}/\sqrt p$, each of norm at most $t_*$; since $\mat{\Psi}\one=0$, its rank is
at most $p-1$. Finally $\norm{\mat{\cZ}}_{\rm op}=\sqrt p$ and
$\norm{\mat{\cZ}_\Train c}^2=c^\top\mat{\cK} c\ge s_0\norm c^2$.
\end{proof}

\subsubsection{The dictionary coefficient of the rates}\label{sec:bse-cR}

\begin{definition}[Force covariance and pre-sampling coefficient]\label{def:r62-rates}
Write $c_F=(c_F^{\rm L},c_F^{\rm R})$ and define
\begin{equation*}
 \tau_c^2=\frac{\expctover{l,r}{\norm{\mat{P}c_F^{\rm L}}^2}}{p-1},
 \qquad \tau_0^2=\frac{2\expctover{l,r}{(\one^\top c_F^{\rm L})^2}}p,
 \qquad \mathbf u_0=\frac{(\one,\one)}{\sqrt{2p}}.
\end{equation*}
Let $\mat{\Pi}_\circ$ be the orthogonal projector
onto $\one^\perp\oplus\one^\perp$; $\tau_c,\tau_0$ denote the nonnegative square roots.

Let $\mat{\Sigma}_F=\expctover{l,r}{c_Fc_F^\top}$. The coefficient
in~\eqref{eq:r62rates} is
\begin{equation}\label{eq:bse-cR}
c_R(p)=\frac{\op{tr}\mat{\Sigma}_F^{1/2}}{q^2p(p-1)\sqrt{2p-1}},\qquad q=\frac{p+1}{4p},
\end{equation}
The comparison direction $\mat{U}_F$ is the normalized supported Polar factor
of the matrix whose row for neuron $(l,r)$ is $c_F(l,r)^\top/\sqrt m$;
its corresponding row is $\mat{U}_F(l,r)$. Both $c_R(p)$ and $\mat{U}_F$ are fixed
before the training split is sampled.
\end{definition}

\begin{definition}[Coherent-force quantities]\label{def:r62-scoped}
For an interval mask $I$, $\op{ac}_I(t)=|I\cap(I+t)|$;
$I_l=\{a:l_a=1\}$ and $I_r=\{b:r_b=1\}$. The coherent sum is
\[
 E_{\rm coh}=\sum_{a=0}^{(p-1)/2}
 \left[a\left(\frac{p-1}{2}-a\right)+\frac{b_0(p-1)}{4p}\right].
\]
\end{definition}

The pre-sampling coefficient is~\eqref{eq:bse-cR}.
Let $\mat{\Phi}_F\in\R^{m\times2p}$ have rows $c_F(l,r)^\top/\sqrt m$;
then $\mat{\Phi}_F^\top\mat{\Phi}_F=\mat{\Sigma}_F$ and $\mat{U}_F=\cN_{\rm pol}(\mat{\Phi}_F)$.

\begin{lemma}[Own covariance and bounds on $c_R(p)$]\label{lem:bse-cR}
For the force-covariance quantities of Definition~\ref{def:r62-rates}:
\begin{enumerate}[label=(\roman*)]
\item $\mat{\Sigma}_F=\tau_c^2\mat{\Pi}_\circ+\tau_0^2\mathbf u_0\mathbf u_0^\top$; hence
$\op{tr}\mat{\Sigma}_F^{1/2}=2(p-1)\tau_c+\tau_0$ and
\begin{equation*}
c_R(p)=\frac{2(p-1)\tau_c+\tau_0}{q^2p(p-1)\sqrt{2p-1}}.
\end{equation*}
\item For $p\ge64$, $\tau_c>p^{3/2}/32$ and $\tau_0\ge p^2/24$. Thus $\mat{\Sigma}_F$ has
rank $2p-1$ and range $E_{\rm vis}$, its smallest positive eigenvalue is
$\min(\tau_c,\tau_0)^2>p^3/1024$, and
$\mat{U}_F(l,r)=c_F(l,r)^\top(\mat{\Sigma}_F^+)^{1/2}/\sqrt{m(2p-1)}$.
\item For $p\ge64$, $2/3<c_R(p)\le4p/\sqrt{p^2-1}<7$.
\end{enumerate}
\end{lemma}

\begin{proof}
Write $(l*r)_c=\sum_a l_ar_{c-a}$, $n_0=(p-1)/2$, and
$f_{\rm coh}(a)=a(n_0-a)+b_0n_0/(2p)$ on $I_0$, zero off $I_0$;
then $E_{\rm coh}=\one^\top f_{\rm coh}$.
Since $(\mat{T}^\top g)_c=(l*r)_c-b_0^2/p$,
\[
c_F^{\rm L}(a)=l_a\sum_{b\in I_r}[(l*r)_{a+b}-b_0^2/p]
=l_a\Big[\sum_{v\in I_l}{\rm ac}_{I_r}(a-v)-b_0^3/p\Big].
\]
Summing over $a$ gives $\one^\top c_F^{\rm L}=\|\mat{P}(l*r)\|^2$;
the right block is symmetric.

(i) Left translations translate $\mat{P}c_F^{\rm L}$ and fix $c_F^{\rm R}$
and the common coordinate sum. Averaging these translations makes their
cross second moments zero. Simultaneous dilations act transitively on
nonzero coordinate differences, so the left centered second moment is
$\tau_c^2\mat{P}$; exchanging the masks gives the identical right block.
Together with $c_F\in E_{\rm vis}$ and
$\mathbf u_0^\top c_F=\sqrt{2/p}\,\one^\top c_F^{\rm L}$ this gives (i).

(ii) The slope event $u'/u=\pm1$ has probability $2/(p-1)$.
Translations and simultaneous dilation reduce to $I_l=I_r=I_0$.
Summing $b_0-|a-v|$ on $I_0$ identifies the coherent block with the
function $f_{\rm coh}$ defined above and gives
\[
E_{\rm coh}=\frac{(p^2-1)(p^2+3)}{48p}\in[p^3/48,p^3/48+p/24].
\]
Using $\sum_{a=0}^{n_0}a^2(n_0-a)^2=(n_0^5-n_0)/30$ and
$(1-1/p)^5\ge1-5/p$, for $p\ge64$ we obtain
\[
\frac{\|\mat{P}f_{\rm coh}\|^2}{p^5}
\ge\frac{1-5/64}{960}-\frac1{60\cdot64^4}
     -\frac{(1+2/64^2)^2}{2304}>\frac1{2048}.
\]
Restricting the two nonnegative second moments to this slope event gives
$\tau_c^2>p^3/1024$ and
$\tau_0^2\ge4E_{\rm coh}^2/[p(p-1)]\ge p^4/576$.
The rank, range and supported Polar formula follow from (i).

(iii) The lower bound follows from (ii), discarding $\tau_0$ and using
$q\le65/256$ and $q^2<21/320$. For each neuron, $\diag(g)\mat{X}$ is
the incidence matrix of the complete bipartite graph on its two masks,
each of size $b_0$, so its squared operator norm is $2b_0=p+1$.
Since $(\mat{T}\mat{T}^\top)^2=p\mat{T}\mat{T}^\top$, Lemma~\ref{lem:bse-dict}(iv) gives
\[
\op{tr}\mat{\Sigma}_F\le(p+1)p\expct{\norm{\mat{T}^\top g}^2}
=q^2p^2(p^2-1).
\]
Scalar Cauchy--Schwarz on the $2p-1$ positive eigenvalues and
division by~\eqref{eq:bse-cR} give
$c_R\le\sqrt{p+1}/(q\sqrt{p-1})=4p/\sqrt{p^2-1}<7$.
\end{proof}

\begin{definition}[Geometric clock and pre-sampling choices]\label{def:r62-clock}
Set $Z=9/50$.
The geometric rate parameters are
\begin{equation}
\begin{aligned}
 k&=\rho c_R(p),& a&=(1+kZ)^{-1/J},& \nu_j&=a^j,\\
 z_j&=(a^{-j}-1)/k,&
 \xi&=k(z_{J-1}+z_J)/4.\label{eq:r62rates}
\end{aligned}
\end{equation}

These are pre-sampling choices. The associated clock aliases are
\begin{equation*}
\begin{gathered}
\vartheta_j=1-4z_j,\quad g_-=(1-a)/(2a),\quad g_+=(1-a)/2,\\
\mat{W}_j^{\rm ref}=\nu_j(\bar {\mat{W}}_0+hz_j\mat{U}_F),\quad \mat{H}_j^{\rm ref}=\mat{H}(\mat{W}_j^{\rm ref}).
\end{gathered}
\end{equation*}
The clock identities and signs are proved in Lemma~\ref{lem:bfr-clock}.

These construction parameters are separate from those of Definition~\ref{def:half-scales}.
\end{definition}

For the clock comparison put $g_Z=1+kZ$ and $a_J=(137/125)^{-1/J}$.
\begin{lemma}[Geometric clock and reference recursion]\label{lem:bfr-clock}
Let $J\ge30$ and $p\ge64$, with the pre-sampling rates and clock of
Definition~\ref{def:r62-clock} and~\eqref{eq:r62optimizer}.
\begin{enumerate}[label=(\roman*)]
\item $0<a<a_J$, $1-s\lambda=a$ and $s/h=(1-a)/k$.
For $0\le j\le J$, $\frac{125}{251}<\nu_j\le1$,
$\nu_j(1+kz_j)=1$, $0=z_0<z_1<\cdots<z_J=Z=\frac9{50}$,
and $\vartheta_j\ge\frac7{25}$.
\item $z_{J-1}>\frac16$, and
$z_j-z_{j-1}>\frac{45}{502J}>\frac1{12J}$ for $1\le j\le J$.
\item $0<\xi<kZ/2$ and $\nu_j\xi>\frac{11}{540}$ for $0\le j\le J$.
The quantity $1-\nu_j(1+2\xi)$ increases strictly with $j$ and equals
$-g_-$ at $J-1$ and $g_+$ at $J$, where both positive gaps exceed
$(1-a_J)/2$.
\item The reference path has $\mat{W}_0^{\rm ref}=\bar {\mat{W}}_0$ and
$\mat{W}_{j+1}^{\rm ref}=a\mat{W}_j^{\rm ref}+s\mat{U}_F$ for $0\le j<J$.
\end{enumerate}
\end{lemma}

\begin{proof}
\medskip\noindent\textbf{Clock identities and recursion.}\par\smallskip
Lemma~\ref{lem:bse-cR}(iii) gives $8/15<k<28/5$, hence
$137/125<g_Z<251/125$ and $0<a<a_J<1$.
The prescribed rates imply $1-s\lambda=a$, $s/h=(1-a)/k$, and
\begin{equation*}
 \nu_j=a^j,\qquad \nu_jkz_j=1-a^j,\qquad
 1+2\xi=g_Z(1+a)/2.
\end{equation*}
The first two identities give the increasing clock from $0$ to $Z$,
$\nu_j\ge\nu_J=g_Z^{-1}>125/251$, and $\vartheta_j\ge1-4Z=7/25$.
They also rewrite the reference path as
\begin{equation*}
 \mat{W}_j^{\rm ref}=a^j\bar {\mat{W}}_0+(h/k)(1-a^j)\mat{U}_F,
\end{equation*}
which proves its initial state and recursion because $s=h(1-a)/k$.

\medskip\noindent\textbf{Positive geometric increments.}\par\smallskip
The positive geometric increments increase with $j$, sum to $Z$, and
have maximum-to-minimum ratio
\begin{equation*}
 \frac{z_J-z_{J-1}}{z_1-z_0}=a^{1-J}=ag_Z<g_Z.
\end{equation*}
Comparison with their arithmetic mean $Z/J$ therefore gives
\begin{equation*}
 \frac{Z}{Jg_Z}<z_j-z_{j-1}<\frac{g_ZZ}{J}.
\end{equation*}
The lower bound exceeds $45/(502J)>1/(12J)$. Since $J\ge30$, the upper bound gives
\begin{equation*}
 z_{J-1}>Z(1-g_Z/J)>\frac9{50}\left(1-\frac{251}{3750}\right)>\frac16.
\end{equation*}

\medskip\noindent\textbf{Midpoint mixing and signs.}\par\smallskip
The midpoint identity and $a>g_Z^{-1}$, since $J>1$, give directly
\begin{equation*}
 \nu_j\xi\ge\nu_J\xi
 =\frac{1+a-2/g_Z}{4}
 >\frac{1-g_Z^{-1}}4>\frac3{137}>\frac{11}{540}.
\end{equation*}
Also $\xi<kZ/2$ because $z_{J-1}<Z$.
Since $\nu_j$ decreases, the signed quantity increases strictly.
The same midpoint identity gives its endpoint values
$-(1-a)/(2a)=-g_-$ and $(1-a)/2=g_+$.
Both gaps exceed $(1-a_J)/2>0$ because $a<a_J<1$.
\end{proof}

Thus the reference path is the decayed fixed-direction recursion of the learner
with $\cN_{\rm pol}(\mat{F}_j)$ replaced by $\mat{U}_F$. It is used only in the proof.

\subsubsection{Train-only calibration of the initial head}\label{sec:bse-cal}
\readability{The clock is now fixed by Definition~\ref{def:r62-clock}.
This subsection builds the head that will be compared along that clock.
Lemma~\ref{lem:bse-kernel} first controls held-out kernel predictions;
Proposition~\ref{prop:bse-cal} then uses the two train-only calibrations
to obtain the center-margin identity~\eqref{eq:ba-center-margin} with
uniformly small error.}

\paragraph{Local scales and oscillation.}\label{sec:notation-softmax}
Write $\omega=\sqrt{\log p/p}$ for the asymptotic split scale and
$\osc(f)=\max_k f_k-\min_k f_k$ for $f\in\R^p$.

\begin{definition}[Shifted inverses]\label{def:r62-shifted}
Fix $c_\circ\ge1$ and put $\mat{\cA}_{c_\circ}=c_\circ\mat{\Upsilon}+\mat{\cK}$, $\mat{C}^{(c_\circ)}=\mat{\cA}_{c_\circ}^{-1}\mat{\Psi}$,
$\Lambda_1=c_\circ+\rho p$ and
$\iota=\frac1{p(\rho p+c_\circ/(p+2))}-\frac2{p\Lambda_1}$, and set
\begin{equation}\label{eq:bse-epspred}
\begin{aligned}
\epsilon_{\rm pred}={}&\frac{4(D_*+Q_*)}{\Lambda_1^2}
+\frac{8|\iota|D_*^2}{\Lambda_1}\\
&+\frac{K_*\sqrt{2p+2\ell_\delta}}{\rho p\sqrt p\,(\rho p+c_\circ/(p+2))}
+\frac{\sqrt2\kappa_*^2\sqrt{2p+2\ell_\delta}}{(1-\kappa_*)\rho p}.
\end{aligned}
\end{equation}
These depend on $c_\circ$; without an argument they refer to $c_\circ=1$,
and $\mat{\cA}=\mat{\cA}_1$. 

For the constant function use the coordinate vector $d_{\one}=(0,0,\sqrt p)$,
$h_x=\mat{B}_{0,x\Train}\mat{B}^{-1}\one_\Train$ and
$\epsilon_h=\sqrt{2p}/[(p+2)s_0]$.
Thus $\mat{\cZ} d_{\one}=\one_\Omega$; the scalar $d_j$ below instead denotes
an actual-state distance.
\end{definition}

The held-out behavior of the canonical head and of the center refit is governed by the
realized inverses of shifted kernels, in the coordinates of Lemma~\ref{lem:bwc-design}.

For the shifted-kernel estimate, put
\[
 \Lambda_g=\rho p+\frac{c_\circ}{p+2},\qquad
 B_*=\sqrt{2p+2\ell_\delta}.
\]
\begin{lemma}[Kernel predictions]\label{lem:bse-kernel}
For the shifted inverse quantities in Definition~\ref{def:r62-shifted},
for every $c_\circ\ge1$ and $x\in\Test$,
\begin{equation}\label{eq:bse-pred}
\mat{B}_{0,x\Train}[\mat{B}+(c_\circ-1)\Id]^{-1}
=z_x\mat{\cA}_{c_\circ}^{-1}\mat{\cZ}_\Train^\top,\qquad
h_x=1-z_x\mat{\cA}^{-1}\mat{\Upsilon} d_{\one} .
\end{equation}
Also $|h_x-1|\le\epsilon_h$ and
$\mat{\cA}_{c_\circ}\succeq\mat{\cK}\succeq s_0\Id$. For $x=(a,b)\in\Test$, $y=a+b$ and $c\ne y$,
\begin{equation}\label{eq:bse-fine}
\Bigl|z_x\mat{C}^{(c_\circ)}(e_y-e_c)+\frac{N_{xc}}{\Lambda_1}\Bigr|
\le2|\iota|D_*+2\epsilon_{\rm pred}.
\end{equation}
\end{lemma}

\begin{proof}
For $x\notin\Train$ the identity part of $\mat{B}_0$ does not contribute, so
$\mat{B}_{0,x\Train}=z_x\mat{\Upsilon}^{-1}\mat{\cZ}_\Train^\top$. Since
\[
\mat{\cA}_{c_\circ}\mat{\Upsilon}^{-1}\mat{\cZ}_\Train^\top
=\mat{\cZ}_\Train^\top[\mat{B}+(c_\circ-1)\Id],
\]
multiplication by both inverses gives the first formula
of~\eqref{eq:bse-pred}. At $c_\circ=1$, use $\one_\Train=\mat{\cZ}_\Train d_{\one}$,
$\mat{\cK}=\mat{\cA}-\mat{\Upsilon}$ and $z_xd_{\one}=1$ for the second. The bound on $h_x$ follows from
$\mat{\cA}\succeq s_0\Id$, $\norm{z_x}\le\sqrt2$ and $\norm{\mat{\Upsilon} d_{\one}}=\sqrt p/(p+2)$.

For $x=(a,b)\in\Test$, $y=a+b$ and rival $c$, put
\[
\begin{gathered}
 v=e_y-e_c,\quad r_d=M_{d,y-d}-M_{d,c-d},\quad
 s_d=M_{y-d,d}-M_{c-d,d},\quad \delta_c=n_y-n_c,\\
 \psi=\mat{\Psi} v,\qquad \bar{\mat{\cA}}=c_\circ\mat{\Upsilon}+\rho p\Id,\qquad
 \theta_0=\bar{\mat{\cA}}^{-1}\psi,\\
 \mat{\mathsf E}=\mat{\cK}-\rho p\Id,\qquad \mat{\mathsf B}=\bar{\mat{\cA}}^{-1}\mat{\mathsf E}.
\end{gathered}
\]
For these contrast vectors, $|r_d|,|s_d|\le1$, $|\delta_c|\le2D_*$ and
$\psi=\mat{\Psi} v=(\mat{P}r,\mat{P}s,\delta_c/\sqrt p)$ has norm at most $B_*$.
For these $\bar{\mat{\cA}},\theta_0$,
$z_x\theta_0=-N_{xc}/\Lambda_1+\iota\delta_c$, since $x\notin\Train$.
The perturbation satisfies $\norm{\mat{\mathsf B}}_{\rm op}\le\kappa_*<1$ and
\[
\mat{C}^{(c_\circ)}v=\theta_0-\mat{\mathsf B}\theta_0
+\mat{\mathsf B}^2(\Id+\mat{\mathsf B})^{-1}\theta_0.
\]
Since $\norm{\theta_0}\le B_*/(\rho p)$, the remainder and global coordinate
of the first-order term contribute at most the fourth and third terms of
\eqref{eq:bse-epspred}, respectively. By~\eqref{eq:ba-Kblock}, its row block
before applying $\bar{\mat{\cA}}^{-1}$ is
$\mat{P}[(\delta n^{\rm row}\odot r+\bar {\mat{M}} s)/\Lambda_1+\iota\,\delta n^{\rm row}\delta_c]$.
Here $\norm{\bar {\mat{M}} s}_\infty\le2Q_*$ by subtracting two instances of
\eqref{eq:bse-des3}. Applying $\mat{P}$ at most doubles a coordinate bound;
the row and column blocks therefore contribute at most
$4(D_*+2Q_*)/\Lambda_1^2+8|\iota|D_*^2/\Lambda_1$.
This is bounded by twice the first two terms of~\eqref{eq:bse-epspred}.
Adding the global and remainder bounds, then $|\iota\delta_c|\le2|\iota|D_*$,
proves~\eqref{eq:bse-fine}.
\end{proof}

\begin{definition}[Train-only calibrations and center]\label{def:r62-calibration}
For $A>0$ and $n_c\ge1$ for every $c$, choose the unique balance
scalar $s_{\rm bal}>0$ for which
\[
 \sum_c\frac{\sqrt{1+4(e^A-1)n_c/s_{\rm bal}}-1}{2(e^A-1)}=1,
\]
and put $q_{{\rm base},c}=(\sqrt{1+4(e^A-1)n_c/s_{\rm bal}}-1)/(2(e^A-1))$.
Existence and uniqueness are proved in Proposition~\ref{prop:bse-cal}. Set
\begin{equation}\label{eq:bse-base}
b_{{\rm base},c}=\log q_{{\rm base},c}-\frac1p\sum_d\log q_{{\rm base},d},\qquad
\alpha_c=\frac{1+(e^A-1)/p}{1+(e^A-1)q_{{\rm base},c}},\qquad
\widehat {\mat{R}}_i=\alpha_{y_i}(e_{y_i}-q_{\rm base})\quad(i\in\Train).
\end{equation}

Let $b_*$ be the unique centered training-CE offset for the premixed
training logits
\[
 A[(1-\xi)\mat{T}_\Train+\xi \mat{B}\widehat {\mat{R}}]+(1-\xi)\one b_{\rm base}^\top;
\]
Proposition~\ref{prop:bse-cal} establishes this train-only calibration.
The center head is
\begin{equation}\label{eq:r62head}
 \begin{aligned}
 \bar {\mat{V}}_0={}&\frac A d[(1-\xi)\mat{B}^{-1}\mat{T}_\Train+\xi\widehat {\mat{R}}]^\top \mat{H}\\
 &+\frac{1-\xi}{d}\,b_{\rm base}\one_\Train^\top \mat{B}^{-1}\mat{H}
   +\frac{b_*w^\top}{h\mu^2}.
 \end{aligned}
\end{equation}
Existence and the claimed properties of both calibrations are proved in
Proposition~\ref{prop:bse-cal}.

Define $u_{0,xc}$ by the center-margin identity
\begin{equation}\label{eq:ba-center-margin}
(\bar {\mat{H}}_0\bar {\mat{V}}_0^\top)_{xy_x}-(\bar {\mat{H}}_0\bar {\mat{V}}_0^\top)_{xc}
=A[M_x+\xi(2M_x-N_{xc})]+Au_{0,xc},
\end{equation}
and set $B_{\rm cal}:=\osc(b_*)$, $e_0:=\max_{x\in\Omega,\ c\ne y_x}|u_{0,xc}|$ and
$E_{\rm cal}:=Ae_0$. These are proof certificates of the already prescribed center, not
inputs to either train-only minimizer, the rates or the radius.
\end{definition}

\paragraph{Local train-only calibration notation.}
For a training-logit matrix $\mat{F}\in\R^{n\times p}$ with labels $y_i$
and $n_c\ge1$ for every class, set $\cC=\{b\in\R^p:\one^\top b=0\}$ and
\[
 \cL_F(b)=\frac1n\sum_i
 \left[\log\sum_c e^{F_{ic}+b_c}-F_{iy_i}-b_{y_i}\right].
\]
Write $b^F=\argmin_{b\in\cC}\cL_F(b)$ and
$\mat{R}=\mat{Y}_\Train-\op{softmax}(\mat{F}+\one(b^F)^\top)$.
Existence and uniqueness of the centered minimizer are proved below.
The softmax Jacobian at a probability vector $\pi$ is
$\diag(\pi)-\pi\pi^\top$.

For $A>0$, put $k_A=e^A-1$ and
$q_c(t)=(\sqrt{1+4k_An_c/t}-1)/(2k_A)$ for $t>0$.
The balance scalar satisfies $\sum_cq_c(s_{\rm bal})=1$, and
$q_{\rm base}=q(s_{\rm bal})$. Set
$\varpi_c=(1+k_Aq_{{\rm base},c})^{-1}$.
For the premixed head define
\[
\begin{aligned}
 \mat{V}_{\rm can}&=\frac1d(A\mat{T}_\Train+\one b_{\rm base}^\top)^\top \mat{B}^{-1}\mat{H},\\
 \mat{V}_{\rm raw}&=\frac Ad\widehat {\mat{R}}^\top \mat{H},\qquad
 \mat{V}_\xi=(1-\xi)\mat{V}_{\rm can}+\xi \mat{V}_{\rm raw}.
\end{aligned}
\]
Put $\mat{F}_x^\xi=\bar {\mat{H}}_{0,x}\mat{V}_\xi^\top$, and use $\mat{F}=\mat{F}_\Train^\xi$
and $b_*=b^F$ for its calibration. The actual center residual is
$\mat{R}_{\rm cal}=\mat{Y}_\Train-\op{softmax}(\mat{F}_\Train^\xi+\one b_*^\top)$.

\begin{proposition}[Train-only calibrated initial head]\label{prop:bse-cal}
For any training-logit table $\mat{F}$ with $n_c>0$, the centered minimizer $b^F$
of $\cL_F$ exists uniquely and is characterized by
\begin{equation*}
 \sum_{i\in\Train}\op{softmax}(\mat{F}_i+b^F)= (n_c)_c.
\end{equation*}
Its calibrated residual $\mat{R}$ has $\mat{R}^\top\one_n=0$ and $\mat{R}\one_p=0$.

For $\mat{F}=A\mat{T}_\Train$, the prescribed $s_{\rm bal}$ exists uniquely,
$q_{\rm base}=q(s_{\rm bal})$, and
$b_{\rm base}=b^{AT_\Train}$. The normalized canonical residual satisfies
$\widehat {\mat{R}}^\top\one_n=0$, $\widehat {\mat{R}}\one_p=0$. Uniformly on
$\cE_{\rm split}$,
\begin{gather*}
 q_{{\rm base},c}=1/p+O(\omega/p),\qquad \alpha_c=1+O(\omega),\\
 \osc(b_{\rm base}),\ \norm{b_{\rm base}}_\infty=O(\omega),\qquad
 \norm{\widehat {\mat{R}}-\mat{T}_\Train}_{\rm op}=O(\sqrt p\,\omega).
\end{gather*}
Every canonical training margin is $A+O(\omega)$.

On $\cE_{\rm split}$, the prescribed head~\eqref{eq:r62head} has logits
$\mat{F}_x^\xi+b_*$ on every input, and $\mat{R}_{\rm cal}$ has both
zero sums. The center certificates from
Definition~\ref{def:r62-calibration} satisfy
\begin{equation*}
\begin{gathered}
 B_{\rm cal}=O(A\omega),\qquad e_0=O(\omega),\qquad
 E_{\rm cal}=O(A\omega).
\end{gathered}
\end{equation*}
Consequently every center training margin is at least $A(1-e_0)>0$, and
every held-out input has a rival margin at most $-A(\xi-e_0)<0$ for all
sufficiently large $p$.
\end{proposition}
\begin{proof}
\readability{We first establish the canonical calibration and its residual
bounds. We then estimate the premixed margins and their class imbalance,
and use the balance Jacobian to bound the final offset $b_*$.}

\emph{(1) Convexity and the canonical root.}

For
$\pi_i(b)=\op{softmax}(\mat{F}_i+b)$, the Hessian quadratic form is
$n^{-1}\sum_i\sum_{c<d}\pi_{ic}(b)\pi_{id}(b)(v_c-v_d)^2$; it is positive
on $\cC$. Also
$\cL_F(b)\ge\min_c(n_c/n)\osc(b)+
n^{-1}\sum_i(\min_c F_{ic}-F_{iy_i})$, and
$\osc(b)\ge\norm b_2/\sqrt p$ on $\cC$, so the objective is coercive.
Its gradient lies in $\cC$ and vanishes exactly at class-count balance,
proving the first assertion.

For canonical logits, write $q=q_{\rm base}$. The factor $\varpi$ gives
a class-$y$ row residual $\varpi_y(e_y-q)$, so balance is equivalent to
$n_c\varpi_c=s_{\rm bal}q_c$ for $s_{\rm bal}=\sum_y n_y\varpi_y$, or
$n_c=s_{\rm bal}q_c(1+k_Aq_c)$. Solving gives the stated $q_c(t)$; each coordinate
decreases continuously from $+\infty$ to $0$ as $t$ increases from $0$ to
$+\infty$, so $\sum_cq_c(t)=1$ has exactly one root. Centering $\log q$
therefore gives the unique calibrated offset and the residual formulas.

On $\cE_{\rm split}$, the canonical balance identity gives, for $q_i\ge q_j$,
\[
 \frac{q_i}{q_j}\le\frac{n_i}{n_j}
 =1+O(\omega).
\]
Together with $\sum_cq_c=1$, this gives $q_c=1/p+O(\omega/p)$ uniformly.
Centering $\log q$ gives $\osc(b_{\rm base}),\norm{b_{\rm base}}_\infty=O(\omega)$.
Also
\[
 |\alpha_c-1|
 =\frac{k_A|1/p-q_c|}{1+k_Aq_c}
 \le\frac{|1/p-q_c|}{q_c}=O(\omega).
\]
Finally, with $\mat{P}=\Id-\one\one^\top/p$,
\[
 \widehat {\mat{R}}-\mat{T}_\Train
 =\mat{Y}_\Train\{(\diag(\alpha)-\Id)\mat{P}-\alpha(q-\one/p)^\top\}.
\]
The bracket has operator norm $O(\omega)$, since
$\norm\alpha=O(\sqrt p)$ and $\norm{q-\one/p}=O(\omega/\sqrt p)$.
Using $\norm{\mat{Y}_\Train}_{\rm op}\le\sqrt{n_{\max}}=O(\sqrt p)$ proves the
residual bound; the margin claim follows from the offset oscillation.

\emph{(2) Premixed pointwise margins and initial imbalance.}
Put
\[
 \mathsf D_x^\alpha=\sum_{b'}M_{ab'}\alpha_{a+b'}
 +\sum_{a'}M_{a'b}\alpha_{a'+b}\quad(x=(a,b)).
\]
For the mixed logits $\mat{F}=\mat{F}_\Train^\xi$, write
$\pi_i=\pi_i(0)$ with $\pi_i(b)=\op{softmax}(\mat{F}_i+b)$, and set
\[
\begin{gathered}
 \tau=\frac{e^{-A}}{1+e^{-A}Q_\xi},\qquad
 t_{ic}=r_\xi^{2-N_{ic}}\quad(c\ne y_i),\qquad
 \mathsf g_c=\sum_i\pi_{ic}-n_c,\\
 w_{cd}(b)=\sum_i\pi_{ic}(b)\pi_{id}(b),\qquad
 w_{cd}=w_{cd}(0),\qquad w_0=\min_{c\ne d}w_{cd}.
\end{gathered}
\]
The event gives
$s_0\asymp p$. The canonical training margins are as above; on $\Test$,
Lemma~\ref{lem:bse-kernel} gives the canonical margin estimate
\begin{equation*}
 O(A/p)+A[2|\iota|D_*+2\epsilon_{\rm pred}]
 +(1+\epsilon_h)O(\omega)=O(A\omega).
\end{equation*}
For $q=q_{\rm base}$, the constant block in $\mat{B}_{0,x\Train}\widehat {\mat{R}}$
vanishes by $\widehat {\mat{R}}^\top\one=0$, and its remaining blocks give
\begin{equation}\label{eq:bse-rawvec}
 (\mat{B}_{0,x\Train}\widehat {\mat{R}})_c
 =M_x\alpha_y(\one\{c=y\}-q_c)+\alpha_cN_{xc}-q_c\mathsf D_x^\alpha.
\end{equation}
Hence its normalized $y$-versus-$c$ margin is
$3M_x\alpha_y-N_{xc}\alpha_c-(q_y-q_c)(M_x\alpha_y+\mathsf D_x^\alpha)
=3M_x-N_{xc}+O(\omega)$: here $\mathsf D_x^\alpha=O(p)$,
$|\alpha_c-1|=O(\omega)$ and $|q_y-q_c|=O(\omega/p)$.
Mixing the two components therefore gives, uniformly on the entire grid,
\begin{equation}\label{eq:bse-premix}
 F_{xy_x}^\xi-F_{xc}^\xi
 =A[M_x+\xi(2M_x-N_{xc})]+O(A\omega).
\end{equation}

The quantities $Q_\xi,\tau,t_{ic}$ have $Q_\xi\asymp p$ and
$\tau\asymp p^{-1}$. Equation~\eqref{eq:bse-premix}, the arm-count event
and $\sum_{c\ne y_i}t_{ic}=Q_\xi+O(D_L)$ give
$\pi_{ic}=\tau t_{ic}[1+O(A\omega)]$ for $c\ne y_i$ and
$1-\pi_{iy_i}=\tau Q_\xi[1+O(A\omega)]$. Thus the weighted-count event yields
\begin{equation*}
 \mathsf g_c=\tau(\mathsf W_c-n_cQ_\xi)
 +O\!\left(\tau A\omega(\mathsf W_c+n_cQ_\xi)\right)=O(pA\omega),
\end{equation*}
since $|\mathsf W_c-n_cQ_\xi|\le\Delta_W=O(p^2\omega)$ and
$\mathsf W_c+n_cQ_\xi=O(p^2)$.

For $c\ne d$, the three mask factors defining $\mathsf H_{cd}$ have sums
$n_c,n_d,n_d$. The binary inequality $uvw\ge u+v+w-2$ therefore gives
\[
\mathsf H_{cd}\ge n_c+2n_d-2p
\ge(3\rho-2)p-3D_*\asymp p.
\]
On these two-arm rows, the true probability is bounded below and the rival
probability is $\gtrsim e^{-A}\asymp p^{-1}$, proving $w_0\gtrsim1$.

\emph{(3) A uniform balance-Jacobian bound.}
Apply part (1) to $\mat{F}=\mat{F}_\Train^\xi$. Recall $B_{\rm cal}=\osc(b_*)$, and let $i,k$
maximize/minimize $b_*$. For $0\le t\le1$,
$\pi_{jc}(tb_*)\ge e^{-tB_{\rm cal}}\pi_{jc}$, hence
$w_{cd}(tb_*)\ge e^{-2tB_{\rm cal}}w_{cd}$.
Integrating the balance Jacobian from $0$ to $b_*$ gives
\begin{equation*}
 \begin{aligned}
 \mathsf g_k-\mathsf g_i
 &=\int_0^1\sum_d\bigl[w_{id}(tb_*)(b_{*,i}-b_{*,d})\\
 &\qquad+w_{kd}(tb_*)(b_{*,d}-b_{*,k})\bigr]\,dt\\
 &\ge pw_0B_{\rm cal}\int_0^1e^{-2tB_{\rm cal}}\,dt
 =\frac{pw_0}{2}(1-e^{-2B_{\rm cal}}).
 \end{aligned}
\end{equation*}
For each $d$ the coefficient differences are nonnegative and sum to $B_{\rm cal}$;
a diagonal edge has zero coefficient. The integral formula also holds at
$B_{\rm cal}=0$. Since $\norm{\mathsf g}_\infty=O(pA\omega)$ and $w_0\gtrsim1$,
$1-e^{-2B_{\rm cal}}=O(A\omega)=o(1)$, so $B_{\rm cal}=O(A\omega)$.

The exact identity $\bar {\mat{H}}_0w=h\mu^2\one$ realizes this same offset at every
input. The balance equations give both zero sums for $\mat{R}_{\rm cal}$, and
adding $b_*$ to~\eqref{eq:bse-premix} changes any margin by at most
$\osc(b_*)=O(A\omega)$. Thus $e_0=O(\omega)$ and
$E_{\rm cal}=O(A\omega)$. On training rows, $M_x=1$ and $N_{xc}\le2$,
so the ideal margin in~\eqref{eq:bse-premix} is at least $A$.
For every held-out $x=(a,b)$, its operand row contains a training input on
the split event, giving some $c\ne y_x$ with $N_{xc}\ge1$; then its ideal
margin is at most $-A\xi$. Since $e_0=o(1)<\xi$, the two strict signs follow.
\end{proof}

\subsubsection{The reference direction and kernels}\label{sec:bfr-ref}
\readability{Proposition~\ref{prop:bse-cal} supplies the starting margins.
We next quantify the direction that moves the features: Lemma~\ref{lem:bfr-cap}
bounds the reference coordinates and supplies a singular-value floor
for the force anchor, Lemma~\ref{lem:bfr-polar}
controls perturbations of that direction, and Lemma~\ref{lem:bfr-kernel}
identifies the reference kernel's rule component and leakage.
These estimates feed both the actual-state comparison and the later
refit calculation.}

The prescribed center, rates and radius depend only on the fixed problem, fixed parameters and initial training data; full-grid auxiliary errors are proof-only. The comparison path
$\mat{W}_j^{\rm ref}=\nu_j(\bar {\mat{W}}_0+hz_j\mat{U}_F)$ is not an optimizer state. The proof allows any head
sequence with $\norm{\mat{V}_{j+1}-\mat{V}_j}_F\le\eta$, always using the actual departing-state CE gradient
in the feature update.

\begin{definition}[Reference force normalization]\label{def:r62-direction}
Write $\bar c_F=\norm{\mat{T}^\top g}^2/p$,
$w_F=\sqrt{mr_v}\,\mat{U}_F(l,r)^\top$, and let $\mat{J}_{\rm vis}$ be an orthonormal
basis matrix for $E_{\rm vis}$. Let $\mat{B}_F$ have rows
$\rho^2c_F(l,r)^\top/\sqrt m$.
Set $G_F=\max_{l,r,x}|\bx_x^\top w_F|/\sqrt{r_v}$ and
$\sigma_p=\rho^2p^{3/2}/32$.
\end{definition}

\begin{lemma}[Reference direction]\label{lem:bfr-cap}
For $p\ge64$, every neuron and coordinate satisfy
\begin{gather*}
(w_F)_i=((c_F)_i-\bar c_F)/\tau_c+\bar c_F/\tau_0,
\qquad |(w_F)_i|<\tfrac83\sqrt p+1,\\
G_F<4,\qquad
\sigma_{\min}(\mat{B}_F\mat{J}_{\rm vis})>\sigma_p.
\end{gather*}
Moreover $\mat{B}_F=\mat{B}_F\mat{J}_{\rm vis}\mat{J}_{\rm vis}^\top$ and
$\cN_{\rm pol}(\mat{B}_F)=\mat{U}_F$.
\end{lemma}
\begin{proof}
Lemma~\ref{lem:bse-cR} gives
$(\mat{\Sigma}_F^+)^{1/2}=\mat{\Pi}_\circ/\tau_c+\mathbf u_0\mathbf u_0^\top/\tau_0$;
the two blocks of $c_F$ have mean $\bar c_F$, proving the formula.
For an active left coordinate,
$(c_F)_a=\sum_{v\in I_l}{\rm ac}_{I_r}(a-v)-b_0^3/p$.
The nonzero-shift autocorrelations of an affine half interval are the
multiset $\{1,1,\ldots,b_0-1,b_0-1\}$. Summing its $b_0-1$ smallest or
largest elements gives
\[
-p^2/16+p/4-1/16-1/(8p)\le(c_F)_a
 \le p^2/16+1/16-1/(8p).
\]
Inactive coordinates vanish. Cauchy--Schwarz applied to the two
autocorrelation vectors gives
$0\le p\bar c_F\le\sum_t{\rm ac}_{I_0}(t)^2-b_0^4/p=E_{\rm coh}$,
hence $0\le\bar c_F\le p^2/48+1/24$. Thus
$|(c_F)_i-\bar c_F|\le p^2/12$; the bounds
$\tau_c>p^{3/2}/32$, $\tau_0\ge p^2/24$ prove the coordinate cap.
Consequently
$G_F<(16/3+2/\sqrt p)/\sqrt{2-1/p}\le134/(3\sqrt{127})<4$.
Finally every force row lies in $E_{\rm vis}$, and
$\mat{B}_F^\top \mat{B}_F=\rho^4\mat{\Sigma}_F$ has restricted eigenvalues
$\rho^4\tau_c^2,\rho^4\tau_0^2$. Its thin singular decomposition proves
both the floor and the asserted Polar identity.
\end{proof}

In the following matrix lemma, all matrix names, SVD factors and
eigenvalues are freshly quantified objects of compatible dimensions.
They are unrelated to construction quantities denoted by the same letters.

\begin{lemma}[Anchored Polar estimate]\label{lem:bfr-polar}
If $\mat{B}^\top \mat{B}\succeq\sigma^2\Id_r$, $\sigma>0$, for $\mat{A},\mat{B}\in\R^{N\times r}$ and
$\epsilon=\norm{\mat{A}-\mat{B}}_F$, then
\[
\norm{\cN_{\rm pol}(\mat{A})-\cN_{\rm pol}(\mat{B})}_F
\le\frac{2\epsilon}{\sqrt r\sigma},\qquad
r-\rank \mat{A}\le\frac{\epsilon^2}{\sigma^2}.
\]
For matrices with rows in $E_{\rm vis}$, replace $r$ by $r_v$ and
$\mat{B}$ by $\mat{B}\mat{J}_{\rm vis}$ in the floor. In particular $\mat{A}\ne0$ when
$\epsilon<\sqrt r\sigma$.
\end{lemma}
\begin{proof}
Let $\mat{Q}_A,\mat{Q}_B$ be the unnormalized supported polar factors and $j=\rank \mat{A}$.
Projection onto $\ker \mat{A}$ gives $(r-j)\sigma^2\le\epsilon^2$.
If $j=0$, then $\epsilon\ge\sqrt r\sigma$ and the direction gap is one.
Otherwise put
$a=\sqrt r$, $b=\sqrt j$, $D=\norm{\mat{Q}_B-\mat{Q}_A}_F$ and
$d=\norm{\mat{Q}_B/a-\mat{Q}_A/b}_F\le2$.
The SVD gives $\ip {\mat{A}}{\mat{Q}_A-\mat{Q}_B}\ge0$. Writing $\mat{B}=\mat{Q}_B\mat{H}$, use
$\mat{H}\succeq\sigma \Id$ and $\Id-\op{sym}(\mat{Q}_B^\top \mat{Q}_A)\succeq0$ to obtain
\[
\epsilon D\ge\ip{\mat{B}-\mat{A}}{\mat{Q}_B-\mat{Q}_A}
\ge\tfrac\sigma2(D^2+a^2-b^2)
\ge\frac{\sigma aD^2}{a+b},
\]
where the last step uses $D\le a+b$ and $a\ge b$.
If $D=0$ the claim is immediate; otherwise combine this with
$4D^2-(a+b)^2d^2=(a-b)^2(4-d^2)\ge0$ to get
$d\le2D/(a+b)\le2\epsilon/(\sigma a)$.
Right multiplication by $\mat{J}_{\rm vis}$ preserves rank, nonzero singular
values and Frobenius norms, proving the visible-space assertion.
\end{proof}

The Fourier characters are
$\chi_\omega(x)=p^{-1}e^{2\pi i\omega\cdot x/p}
=\psi_{\omega_1,\omega_2}(x)$ from~\eqref{paper:fourier-main}.
Write $\mathsf L_R=\{(k,k):k\ne0\}$ and
$\mat{P}_R=\mat{P}_{\mathsf L_R}=\Padd=\mat{T}\mat{T}^\top/p$ for the rule projector.
\begin{definition}[Reference kernels]\label{def:r62-kernels}
Set $\mat{F}_0=\bar {\mat{H}}_0/(hq)$ and
$(\mat{P}_F)_{x,(l,r)}=g_{(l,r)}(x)\bx_x^\top \mat{U}_F(l,r)^\top/q$;
write $\mat{F}(z)=\mat{F}_0+z\mat{P}_F$, $\mat{C}=\mat{P}_F\mat{F}_0^\top$ and $\mat{B}_2=\mat{P}_F\mat{P}_F^\top$.
Put $\mat{P}_0=\one\one^\top/p^2$, and let $\mat{P}_{\rm ax}$ be the sum of the
orthogonal projectors onto the two punctured Fourier axes. Set
$\mat{P}_{\rm mix}=\Id-\mat{P}_0-\mat{P}_{\rm ax}-\mat{P}_R$ and $\mat{P}_{\rm int}=\mat{P}_R+\mat{P}_{\rm mix}$.

The scalar coefficients and trace definitions are
\begin{gather*}
c_{\rm ax}=\frac{4q^2p(p-1)}{\sqrt{r_v}}
 \left[\frac{1-\mu}{\tau_c}+\frac\mu{\tau_0}\right],\quad
c_0=(p+1)c_{\rm ax},\quad k_w=\frac{c_{\rm ax}}{p+1},\quad \mat{C}_{\rm mix}=\mat{C}\mat{P}_{\rm mix},\\
R_R=\frac{\op{tr}(\mat{P}_R\mat{B}_2)}{p-1},\qquad
\lambda_R(z)=1+2c_Rz+R_Rz^2.
\end{gather*}
The identification of these coefficients with the kernel eigenvalues is proved
in Lemma~\ref{lem:bfr-kernel}.
\end{definition}

Use the reference tables of Definition~\ref{def:r62-kernels} and put
$\mat{H}_U=q\mat{P}_F$. For a punctured frequency line
$\mathsf L=\{k(u,v):k\ne0\}$ with $(u,v)\ne(0,0)$, let $\mat{P}_{\mathsf L}$
be the real orthogonal projector onto the corresponding character span.
For a diagonal subset mask $\mat{\mathsf M}$, put
\[
 \widehat{\mathsf M}(w)=p^{-1}\sum_x\mathsf M_{xx}e^{2\pi i w\cdot x/p}.
\]

\begin{lemma}[Cross kernel and leakage]\label{lem:bfr-kernel}\label{lem:bwc-leak}
For $p\ge64$, the symmetric matrices $\mat{C},\mat{B}_2$ act as scalars on
each invariant subspace projected onto by $\mat{P}_{\mathsf L}$ or $\mat{P}_0$, and
\begin{gather*}
\mat{C}=c_R\mat{P}_R+c_{\rm ax}\mat{P}_{\rm ax}+c_0\mat{P}_0+\mat{C}_{\rm mix},\quad
c_{\rm ax}\in(0,4),\quad \mat{H}_Uw=\mu^2k_w\one,\\
\norm{\mat{P}_F}_F,\ \norm{\mat{C}_{\rm mix}}_F\le\sqrt{2p}/q,
\qquad \norm{\mat{C}_{\rm mix}}_{\rm op}<8.
\end{gather*}
On the rule line, $\mat{B}_2$ has eigenvalue $R_R\ge c_R^2$, and
$\mat{F}(z)\mat{F}(z)^\top$ has eigenvalue
$\lambda_R(z)\ge(1+c_Rz)^2$.
For a subset mask $\mat{\mathsf M}$ with $n'$ selected points and any
$\mat{E}=\sum_{\mathsf L\ne\mathsf L_R}e_{\mathsf L}\mat{P}_{\mathsf L}$,
\[
\norm{\mat{E}\mat{\mathsf M}\mat{T}}_F^2\le\frac{n'}{p(p-1)}\norm {\mat{E}}_F^2.
\]
\end{lemma}
\begin{proof}
The entries of $\mat{C}$ and $\mat{B}_2$ are respectively
$\expct{gg'\bx_x^\top w_F}/(q^2\sqrt{r_v})$ and
$\expct{gg'\bx_x^\top w_F\bx_{x'}^\top w_F}/(q^2r_v)$.
Translation of either mask translates its force block; simultaneous
dilation and exchange of the masks dilate and exchange the blocks.
The dictionary law is invariant under these operations. Thus both
kernels are translation invariant, symmetric (use dilation by $-1$),
and constant on frequency lines, with equal eigenvalues on the axes.
The rule eigenvalue follows from
\[
\op{tr}(\mat{T}^\top \mat{C}\mat{T})
=\frac{\expct{c_F^\top(\mat{\Sigma}_F^+)^{1/2}c_F}}{q^2\sqrt{r_v}}
=c_Rp(p-1).
\]
On active entries $\mat{P}_{\rm ax}g=2\mu(1-\mu)$ and
$(l,r)^\top w_F=2p\bar c_F[(1-\mu)/\tau_c+\mu/\tau_0]$.
Using $\expct{p\bar c_F}=q^2p(p-1)$ in
$\op{tr}(\mat{P}_{\rm ax}\mat{C})=2(p-1)c_{\rm ax}$ gives the formula.
Its upper bound follows from
$c_{\rm ax}<4q^2(16+24/\sqrt p)/\sqrt{2-1/p}<3.5$ at $p\ge64$.
Averaging $\mat{H}_Uw$ and $\mat{C}\one$ over the grid gives the two constant-mode
identities. Also
$\norm{\mat{P}_F}_F\le\norm{\mat{X}\mat{U}_F^\top}_F/q\le\sqrt{2p}/q$.
Since $\mat{F}_0\mat{F}_0^\top=\Id$ on $\mat{P}_{\rm mix}$,
$\norm{\mat{C}_{\rm mix}}_F\le\norm{\mat{P}_F}_F$; each line has dimension $p-1$,
so its operator norm is at most $\sqrt{2p/(q^2(p-1))}<8$.
For a unit rule vector $v$, $\norm{\mat{F}_0^\top v}=1$ and
$c_R=\ip{\mat{P}_F^\top v}{\mat{F}_0^\top v}\le\sqrt{R_R}$.

For the mask Fourier coefficients defined above, Parseval gives
$\sum_w|\widehat{\mathsf M}(w)|^2=n'$. On a nonrule line with generator
$u$, the map $(k,l)\mapsto(l,l)-ku$ is injective for $k,l\ne0$.
Thus $\norm{\mat{P}_{\mathsf L}\mat{\mathsf M}\mat{P}_R}_F^2\le n'/p^2$.
Orthogonality of the line ranges and $\mat{T}\mat{T}^\top=p\mat{P}_R$ give the claim.
\end{proof}

\subsubsection{Uniform finite-horizon comparison}\label{sec:bfr-init}\label{sec:bfr-path}\label{sec:bfr-ind}\label{sec:bwc}\label{sec:bwc-ce}\label{sec:bwc-exact}
All theorem parameters, including $J$, are fixed before $p$ and the split.
The constants used for the exact initial ball and the complete refit
are fixed next, before any actual-state estimates.
\readability{The radius~\eqref{eq:r62ball} is chosen using the initial
refit's signed reserve~\eqref{eq:bwc-Emargin} and the perturbation budget
\eqref{eq:bwc-Pp}. Its positivity is checked at the start of
Lemma~\ref{lem:bfr-induction}; its protection of initial refit failure
is verified in Section~\ref{sec:bwc-radii}.}

\begin{definition}[Complete refit and reference-kernel budgets]\label{def:r62-refit}
For any feature table $\mat{U}$ with $\norm{\mat{U}_\Train}_F>0$, set
$q_U=\gamma\norm{\mat{U}_\Train}_F^2$,
$\mat{Y}_U=\mat{U}_\Test \mat{U}_\Train^\top(\mat{U}_\Train \mat{U}_\Train^\top+q_U\Id)^{-1}\mat{T}_\Train$ and
$E_{\rm obs}(\mat{U})=\|\mat{Y}_U-\mat{T}_\Test\|_F^2/N_\Test^2$.

For $0\le z\le1/5$, put
$\mat{B}(z)=\mat{F}(z)\mat{F}(z)^\top$ and $\theta_R(z)=\lambda_R(z)-1$. Define
\begin{gather*}
\mat{E}_z=\mat{P}_{\rm mix}(\mat{B}(z)-\Id)\mat{P}_{\rm mix},\\
a_{\min}=1/5,\quad q_{\min}=\kappa\rho/200,\\
C_{\max}=\max\left\{4/\rho,(1+4/\rho)\frac{169}{25q_{\min}}\right\}.
\end{gather*}

Put $\Gamma(z)=\rho\theta_R(z)/(1+q_{F(z)}+\rho\theta_R(z))$.
\begin{equation}\label{eq:bwc-Pp}
\begin{gathered}
\mathsf P_p(e)=\mathsf A_pe+\mathsf B_pe^2,\quad
\mathsf A_p=\frac{\sqrt p(1+3\sqrt{C_{\max}})}{2\sqrt{q_{\min}}N_\Test^-}
+\frac{2\sqrt{C_{\max}}\gamma(9/5)pN_+}{q_{\min}N_\Test^-},\quad
\mathsf B_p=\frac{\sqrt{C_{\max}}\gamma N_+}{2q_{\min}N_\Test^-}.
\end{gathered}
\end{equation}

Put $\lambda_{\rm ob}=1+4\gamma n$, let $\Lambda_1,\iota,\epsilon_{\rm pred}$ be those of
Definition~\ref{def:r62-shifted} at $c_\circ=\lambda_{\rm ob}$, and set
\begin{equation}\label{eq:bwc-Emargin}
\mu_{\rm obs0}=\frac2{\Lambda_1}-2|\iota|D_*-2\epsilon_{\rm pred}.
\end{equation}
For the Fourier transfer, put $A_0=1/(3p+1)$.
\end{definition}

\begin{definition}[Prescribed initial ball]\label{def:r62-ball}
Define, from the initial data and fixed parameters,
\begin{equation}\label{eq:r62ball}
r_{\rm new}=r_6:=\frac{\mu_{\rm obs0}}
 {2\sqrt2N_+\mathsf P_p(\sqrt{2p}/(hq))}.
\end{equation}

The corresponding initial region is
$\cR^{\rm random}_{p,\Train}=\{(\mat{W},\mat{V}):\|(\mat{W}-\bar {\mat{W}}_0,\mat{V}-\bar {\mat{V}}_0)\|_F<r_{\rm new}\}$.
Its positivity and eventual validity are proved in
Proposition~\ref{lem:sc26-b-head-interface}.
\end{definition}

\paragraph{Actual-state comparison notation.}
Write $d_j=\norm{\mat{W}_j-\mat{W}_j^{\rm ref}}_F/h$ for the normalized feature distance
and $r=\|(\mat{W}_0-\bar {\mat{W}}_0,\mat{V}_0-\bar {\mat{V}}_0)\|_F$ for the joint initial distance.
For this comparison put $\phi=p^{-1/4}+A\omega$ and
\[
\begin{aligned}
 \mat{U}&=(1-\xi)\mat{B}^{-1}(\mat{T}_\Train+\one b_{\rm base}^\top/A)+\xi\widehat {\mat{R}},\\
 s(g)&=\mat{U}^\top g_\Train+b_*/(4A),\qquad
 \mat{C}_T(g)=\mat{X}_\Train^\top\diag(g_\Train)\mat{T}_\Train.
\end{aligned}
\]
For each time $j$, set
\[
\begin{gathered}
 t_j=e^{-A\nu_j\xi},\qquad Q_j=(p-1)(\rho+(1-\rho)t_j)^2,\\
 \mathsf w_j=\frac{e^{-A}Q_j}{1+e^{-A}Q_j},\qquad
 \mathsf k_j=\frac{A\mathsf w_j}{hq^2n},\qquad
 \delta_j^F=\frac{\norm{\mat{F}_j-\mathsf k_j\mat{B}_F}_F}
 {\mathsf k_j\sigma_p\sqrt{r_v}}.
\end{gathered}
\]

\begin{lemma}[Two-output finite-horizon comparison]
\label{lem:bfr-src}\label{lem:bfr-refgates}\label{lem:bfr-refmargin}
\label{lem:bfr-induction}\label{lem:bwc-contrast}\label{lem:bwc-rank}
For all sufficiently large $p$, the train-only center is defined on
$\cE_{\rm split}$ and its prescribed radius satisfies $0<r_6<1$.
Initialize at joint distance $r<r_6$ from the center and use the feature
update in~\eqref{eq:r62optimizer} with any head sequence whose increments
have norm at most the pre-sampling budget $\eta=\theta/(hp)$.
Uniformly over these masks, initial states and head sequences, for every
fixed $\ell\ge0$,
\[
A^\ell\max_{j\le J}d_j\longrightarrow0,\qquad
A^\ell\max_{j<J}\delta_j^F\longrightarrow0.
\]
Moreover $L_j=L_*+o(1)$, $\Acc_\Test(j)\to0$ for $j<J$,
$\Acc_\Test(J)\to1$, and $\Acc_\Test(0)=0$ exactly for any fixed tie rule.
Every training margin satisfies, eventually,
\begin{equation}\label{eq:bwc-liminf}
(\mat{H}_j\mat{V}_j^\top)_{i,y_i}-(\mat{H}_j\mat{V}_j^\top)_{i,c}\ge\nu_jA/25
\quad(i\in\Train, c\ne y_i, 0\le j\le J).
\end{equation}
\end{lemma}
\begin{proof}
All orders and limits below are uniform over the family in the statement.
Their constants may depend only on the fixed theorem parameters.
\readability{The proof has three stages. The static and source estimates
control the calibrated head and the reference margins. The departing-state
estimate then closes a finite induction for $d_j$ and $\delta_j^F$.
Finally, the resulting logit error gives mean loss and held-out accuracy,
while the rank estimate gives the separate pointwise training bound
\eqref{eq:bwc-liminf}.}

\medskip\noindent\textbf{Static event, center and exact radius.}\par\smallskip
The event bounds give $t_*=O(\sqrt p)$,
$D_*,Q_*,K_*,D_H,D_L=O(p\omega)$,
$D_W=O(p\sqrt{\log p})$ and $t_f,t_g=O(p^{7/4})$.
Thus $n,n_\Test\asymp p^2$, $s_0,\mathsf H_{\min}\asymp p$ and
$Q_\xi-D_L,Q_j-\mathsf w_jD_L\gtrsim p$.
In particular $K_*\le\rho p/2$, $D_*<p/5$ and
$D_L<\rho^2(p-1)$ eventually, and every class has positive training count.
These facts precede any use of a calibrated inverse.
The count bounds give $\Delta_W=O(p^2\omega)$.
Proposition~\ref{prop:bse-cal}, using only these initial-data bounds
and the kernel identities, gives
$\norm{b_{\rm base}}_\infty,\norm{\alpha-\one}_\infty=O(\omega)$,
$B_{\rm cal},E_{\rm cal}=O(A\omega)$ and $e_0=O(\omega)<\xi$.
These certificates establish the prescribed center's signs independently
of any trajectory; take also $p\ge2048^2$.

Uniformly for $1\le c_\circ\le2$, the kernel formulas give
$\Lambda_1=\rho p+O(1)$, $|\iota|=O(p^{-2})$ and
$\epsilon_{\rm pred}=O(p^{-3/2}\log p)$.
Since $1<\lambda_{\rm ob}=1+\kappa n/[p(p-1)]<2$,
\[
p\mu_{\rm obs0}\longrightarrow2/\rho,\qquad
\mathsf A_p=O(p^{-1/2}),\quad \mathsf B_p=O(p^{-2}).
\]
Thus $r_6>0$. The first positive summand of $\mathsf A_p$,
$\mu_{\rm obs0}\le2/(\rho p)$, and
$N_\Test^-/N_+=\sqrt{1-\rho-D_*/p}$ give
\[
r_6\le\frac{hq\sqrt{q_{\min}(1-\rho)}}
 {\rho(1+3\sqrt{C_{\max}})p^2}=O(p^{-2})<1.
\]

\medskip\noindent\textbf{Canonical and centered source bounds.}\par\smallskip
Put
\[
\begin{aligned}
 d_{\rm atom}&=(\mu(l-\mu\one),\mu(r-\mu\one),\sqrt p\mu^2),\\
 e_{\rm can}&=\mat{T}_\Train^\top \mat{B}^{-1}g_\Train-\mat{T}_\Train^\top f_\Train.
\end{aligned}
\]
The atom coordinates give $g=f+\mat{\cZ} d_{\rm atom}$.
The additive inverse identity gives
\[
e_{\rm can}=\mat{\Psi}^\top\mat{\cA}^{-1}(\mat{\Upsilon} d_{\rm atom}-\mat{\cZ}_\Train^\top f_\Train).
\]
Here $\norm{\mat{\Psi}}_{\rm op}=O(\sqrt p)$,
$\norm{\mat{\Upsilon} d_{\rm atom}}=O(\sqrt p)$ and
$\norm{\mat{\cZ}_\Train^\top f_\Train}^2
\le t_*^2p[2\mu(1-\mu)+\mu^2(1-\mu)^2]=O(p^2)$.
Thus $\sup_g\norm{e_{\rm can}}=O(\sqrt p)$.
Also $\one^\top \mat{B}^{-1}g_\Train
=d_{\one}^\top\mat{\Upsilon}\mat{\cA}^{-1}\mat{\cZ}_\Train^\top g_\Train=O(1)$,
using $\norm{\mat{\Upsilon} d_{\one}}=\sqrt p/(p+2)$ and $\norm{g_\Train}\le b_0$.

\medskip\noindent\textbf{Calibrated coefficients and force comparison.}\par\smallskip
Set
\[
 v_g=\mat{T}_\Train^\top g_\Train-a_g\hat n,\qquad
 \mat{\cJ}_\alpha=(\Id-q_{\rm base}\one^\top)\diag(\alpha),\qquad
 \mat{E}=\widehat {\mat{R}}-\mat{T}_\Train.
\]
The centered coefficients give $\widehat {\mat{R}}^\top g_\Train=\mat{\cJ}_\alpha v_g$
and $\mat{C}_Tv_g=c_g-a_gc_1$. Since $\one^\top v_g=0$,
$\norm{\mat{\cJ}_\alpha-\mat{P}}_{\rm op}=O(\omega)$ implies
$\norm{(\mat{\cJ}_\alpha-\Id)v_g}\le O(\omega)\norm{v_g}$.
Dictionary covariance and $\norm{\hat n}\le2D_*\sqrt p$ give
\[
\norm{\mat{T}_\Train^\top g_\Train}_{L^2}^2
\le q^2(3n+4pD_*^2),\qquad
\norm{v_g}_{L^2}=O(p\sqrt{\log p}).
\]
Expanding $\mat{C}_Ts-\rho^2c_F$, the canonical and raw force errors
are bounded by $t_f+t_g=O(p^{7/4})$.
The remaining terms are $\mat{C}_T$ applied to
$(1-\xi)e_{\rm can}$, $\xi(\mat{\cJ}_\alpha-\Id)v_g$ and the two offsets.
Their $L^2$ norms are respectively $O(\sqrt p)$,
$O(\sqrt p\log p)$ and $O(\sqrt{\log p})$.
Since $\norm{\mat{C}_T}_{\rm op}\le2^{1/2}n_{\max}=O(p)$,
their total is $O(p^{3/2}\log p)=O(p^{7/4})$.

\medskip\noindent\textbf{One head Gram budget.}\par\smallskip
\label{def:r62-budgets}
Put $\mat{S}=(1-\xi)\mat{B}^{-1}\mat{T}_\Train+\xi\widehat {\mat{R}}$ and, in this paragraph,
$v=e_y-e_c$ for $y\ne c$.
The maps $\mat{I}_{\rm row}$ and $\mat{I}_{\rm col}$ send a train vector to its
sums over each first and second operand, respectively; $\mat{I}_{\rm op}$
denotes either map, with $\mathrm{op}\in\{\mathrm{row},\mathrm{col}\}$.
Define the strict-training contrast budget
\begin{equation}\label{eq:bwc-Kp}
 K_p=\max_{y\ne c}\frac{hq}{A}\norm{\bar {\mat{V}}_{0,y}-\bar {\mat{V}}_{0,c}}_2.
\end{equation}
The coefficient formulas give $\norm{s(g)}_\infty=O(p)$, since
$|(\mat{T}_\Train^\top f)_c|\le n_{\max}+n/p$ and
$|\widehat {\mat{R}}_c^\top g_\Train|\le(\max_d\alpha_d)n_{\max}=O(p)$.
Indeed $g\in\{0,1\}$, and the zero-sum residual column has positive mass
$n_c\alpha_c(1-q_{{\rm base},c})\le n_c\alpha_c$.
The canonical factorization gives
$\norm{\mat{\cJ}_\alpha^\top-\mat{P}}_{\rm op}=O(\omega)$.
The exact factorization $\widehat {\mat{R}}=\mat{Y}_\Train\mat{\cJ}_\alpha^\top$, balance and
the centered-mask event imply, for either operand map,
\[
\mat{I}_{\rm op}\widehat {\mat{R}}=\mat{P}\bar {\mat{M}}^{\rm op}\mat{\cJ}_\alpha^\top,
\qquad \norm{\mat{I}_{\rm op}\widehat {\mat{R}}}_{\rm op}=O(\sqrt p).
\]

For the mixed coefficient $\mat{S}$, the dictionary Gram and
$\widehat {\mat{R}}^\top\one=0$ give the exact identity
\begin{readabilityblock}
\begin{equation}\label{eq:readable-f-head-gram}
\begin{aligned}
\mat{S}^\top \mat{B}\mat{S}={}&(\mat{T}_\Train+\xi \mat{E})^\top(\mat{T}_\Train+\xi \mat{E})
+\xi^2\sum_{\mathrm{op}\in\{\mathrm{row},\mathrm{col}\}}
 (\mat{I}_{\rm op}\widehat {\mat{R}})^\top(\mat{I}_{\rm op}\widehat {\mat{R}})\\
&-(1-\xi)^2\mat{T}_\Train^\top(\Id-\mat{B}^{-1})\mat{T}_\Train.
\end{aligned}
\end{equation}
\end{readabilityblock}
The last term is negative semidefinite; the others have operator norm
$O(p)$, using $\norm {\mat{E}}_{\rm op}=O(\sqrt p\omega)$.
For $v=e_y-e_c$, every coordinate of $\mat{I}_{\rm op}\mat{T}_\Train v$ is in
$\{-1,0,1\}$. The preceding operand identity and count bounds therefore give
\[
\norm{\mat{I}_{\rm op}\widehat {\mat{R}}v}\le\sqrt p+O(\sqrt{\log p}),\qquad
\norm{(\mat{T}_\Train+\xi \mat{E})v}^2=2\rho p+O(\sqrt{p\log p}).
\]
Thus $v^\top \mat{S}^\top \mat{B}\mat{S}v\le2p(\rho+\xi^2)+O(\sqrt{p\log p})$.

Since $\mat{B}\succeq \Id+\one\one^\top$, the two head offsets satisfy
\[
\norm{\frac{hq}{A}\bar {\mat{V}}_0-\frac{\mat{S}^\top \mat{H}}{hq}}_{\rm op}
\le\frac{(1-\xi)\norm{b_{\rm base}}}{A}
+\frac{q\norm{b_*}}{A\mu^2}=O(\sqrt p\omega).
\]
For a contrast $v$, the same bound uses the two oscillations instead of
the vector norms and is $O(\omega)$. Since $\mat{H}\mat{H}^\top=h^2q^2\mat{B}$, these give
\begin{readabilityblock}
\begin{equation}\label{eq:readable-f-head-budgets}
\norm{\bar {\mat{V}}_0}_{\rm op}=O(A\sqrt p),\qquad
K_p^2\le2p(\rho+\xi^2)+O(\sqrt{p\log p}).
\end{equation}
The Gram identity~\eqref{eq:readable-f-head-gram} therefore supplies two
bounds: the head operator norm for the departing-state comparison, and
the contrast budget $K_p$ for the final strict-training argument.
\end{readabilityblock}
The center head column is $As(g)/(hq^2\sqrt m)$, so its entrywise bound
is $O(Ap/\sqrt m)$; this is separate from its operator norm.

\medskip\noindent\textbf{Reference features and their table error.}\par\smallskip
In this paragraph set $v=\mat{\cA}^{-1}\mat{\Upsilon} d_{\one}$.
The initial preactivation is $h/\sqrt m$ on an active entry and
at most $-h/\sqrt m$ otherwise. Lemma~\ref{lem:bfr-cap} bounds
the increment by $4hz/\sqrt m$. Thus for $0\le z\le1/5$ and $\nu>0$,
$\nu(\bar {\mat{W}}_0+hz\mat{U}_F)$ retains its initial signs with modulus at least
$\nu h(1-4z)/\sqrt m$, and its features are $\nu hq\mat{F}(z)$.
Since $\norm{\mat{F}_0}_{\rm op}=p+1$ and $\norm{\mat{P}_F}_F=O(\sqrt p)$,
$\norm{\mat{H}_j^{\rm ref}}_{\rm op}=O(p)$.
The coefficient errors satisfy the exact identities
\[
\mat{U}-\mat{T}_\Train=(1-\xi)\mat{\cZ}_\Train(vb_{\rm base}^\top/A-\mat{\cA}^{-1}\mat{\Psi})+\xi \mat{E}.
\]
The constant-mode identity is
\[
\one^\top \mat{U}=(1-\xi)[v^\top\mat{\Psi}+(\one^\top\mat{\cZ}_\Train v)b_{\rm base}^\top/A].
\]
Together these give $\norm{\mat{U}-\mat{T}_\Train}_F=O(p\omega)$,
$\norm{\mat{T}_\Train^\top \mat{U}/p-\rho \mat{P}}_{\rm op}=O(\omega)$ and
$\norm{\one^\top \mat{U}}=O(p^{-1}+\sqrt{\log p}/A)$.
Here use $\norm v=O(p^{-3/2})$, $0\le\one^\top\mat{\cZ}_\Train v<1$,
and the calibration's $\norm {\mat{E}}_{\rm op}=O(\sqrt p\omega)$.

\medskip\noindent\textbf{Reference margins.}\par\smallskip
Let $\widetilde {\mat{U}}$ be $\mat{U}$ extended by zero off $\Train$, and put
\[
 \mat{\Delta}_j=z_j(\mat{C}\widetilde {\mat{U}}-k\mat{T}),\qquad
 u_{j,xc}=u_{0,xc}+z_jk_w(b_{*,y_x}-b_{*,c})/A.
\]
For $\mat{\Delta}_j$, the rule part is $O(p\omega)$ in
Frobenius norm. The axes and constant-mode bounds are
$c_{\rm ax}[\sqrt{2n/p}+\norm{\mat{U}-\mat{T}_\Train}_F]$ and
$c_0\norm{\one^\top \mat{U}}/p$.
The mixed part is at most
$\sqrt{2n/(q^2(p-1))}+8\norm{\mat{U}-\mat{T}_\Train}_F$
by Lemma~\ref{lem:bwc-leak}. Thus $\norm{\mat{\Delta}_j}_F=O(p\omega)$.
The reference logits are exactly
\[
\nu_j[\bar {\mat{H}}_0\bar {\mat{V}}_0^\top+Az_jk\mat{T}+A\mat{\Delta}_j
+z_jk_w\one b_*^\top].
\]
Use the center margins~\eqref{eq:ba-center-margin},
$k_w=c_{\rm ax}/(p+1)$ and $\osc(b_*)\le B_{\rm cal}$.
Every reference margin is therefore
\begin{readabilityblock}
\begin{equation}\label{eq:readable-f-reference-margin}
\nu_jA\{M_x+\xi(2M_x-N_{xc})+kz_j
+u_{j,xc}+\Delta_{j,xy_x}-\Delta_{j,xc}\},
\end{equation}
\end{readabilityblock}
with $\sup|u_{j,xc}|=O(\omega)$, $\mat{\Delta}_0=0$ and $|u_{0,xc}|\le e_0$.
Since $\nu_j(1+kz_j)=1$, the ideal training margin is
$A[1+\nu_j\xi(2-N_{xc})]\ge A$.

\medskip\noindent\textbf{Departing-state comparison.}\par\smallskip
\readability{We now compare one actual departing state with the reference
margin formula~\eqref{eq:readable-f-reference-margin}. The resulting force
error will be the input to the finite induction below.}

At a departing state $(\mat{W},\mat{V})$ set
$d=\norm{\mat{W}-\mat{W}_j^{\rm ref}}_F/h$ and $v=\norm{\mat{V}-\bar {\mat{V}}_0}_F$.
The ideal rival vector $a_i$ has true entry zero and rival entries
$t_j^{2-N_{ic}}$; put $\zeta_i=\one^\top a_i-Q_j$.
Let $\mat{R}_i^*$ be the CE residual of the ideal margins
$A[1+\nu_j\xi(2-N_{ic})]$.
Write $\mat{D},\mat{D}^0$ for the actual and reference training gates and define
\[
 \mat{F}^{\rm anc}=\mathsf w_j n^{-1}
 ((\mat{T}_\Train\bar {\mat{V}}_0)\odot \mat{D}^0)^\top \mat{X}_\Train,\qquad
 b=\nu_jh\vartheta_j/\sqrt m.
\]
The clock and sampled design give $\mathsf w_j\asymp1$,
$\mathsf k_j\asymp A/p^2$ and
$\mathsf k_j\sigma_p\sqrt{r_v}\gtrsim A$.
At any departing state with $d\le1$, the head distance obeys
$v\le r+J\eta=O(p^{-1})$.
ReLU contraction and $\norm{\mat{X}_E}_{\rm op}\le\sqrt{2p}$ give, without
any actual-gate assumption, on every subset $E$,
\begin{equation}\label{eq:bfr-local-logits}
\norm{\mat{H}_E\mat{V}^\top-\mat{H}^{\rm ref}_{j,E}\bar {\mat{V}}_0^\top}_F
\le O(pAd)+(O(p)+O(\sqrt p\,d))v.
\end{equation}

For the ideal-rival vector $a_i$, the arm identity gives
$\one^\top a_i=Q_j+\zeta_i$, $|\zeta_i|\le D_L$, and the exact residual is
\[
\mat{R}_i^*/\mathsf w_j-\mat{T}_i
=\frac{(1-\mathsf w_j)\zeta_i\mat{T}_i-\mat{P}a_i}
 {Q_j+\mathsf w_j\zeta_i}.
\]
Since $Q_j-\mathsf w_jD_L\gtrsim p$ and
$\norm{\mat{P}a_i}\le\sqrt{p-1}$, its norm is $O(\omega)$.
Softmax is $1/2$-Lipschitz in Euclidean norm. Margin errors bounded by
$b$ change the probabilities by at most $b$ in $\ell^1$: after matching
true logits their range is at most $2b$, and the Jacobian's $\ell^1$
action is their mean absolute deviation. The reference bounds and
\eqref{eq:bfr-local-logits} now give
\[
\norm{\mat{R}-\mathsf w_j\mat{T}_\Train}_F/\sqrt n
=O(A\omega+Ad+v+dv/\sqrt p).
\]

The reference gate gap is $b\gtrsim m^{-1/2}$, so entrywise
\[
|\mat{D}-\mat{D}^0|\le|\mat{X}_\Train(\mat{W}-\mat{W}_j^{\rm ref})^\top|/b.
\]
Expand $\mat{F}=n^{-1}((\mat{R}\mat{V})\odot \mat{D})^\top \mat{X}_\Train$ about the anchor $\mat{F}^{\rm anc}$.
Use $\norm {\mat{R}}_F\le\sqrt{2n}$, $\norm{\mat{X}_\Train}_{\rm op}=O(\sqrt p)$,
\readability{the center head operator bound in~\eqref{eq:readable-f-head-budgets}}
and the separate entrywise bound
$|(\mat{T}_\Train\bar {\mat{V}}_0)_{ig}|=O(Ap/\sqrt m)$ to obtain
\[
\norm{\mat{F}-\mat{F}^{\rm anc}}_F
\le O(v/\sqrt p)+O(A)\norm{\mat{R}-\mathsf w_j\mat{T}_\Train}_F/\sqrt n+O(Ad).
\]
The source comparison adds $\mathsf k_jO(p^{7/4})$.
\begin{readabilityblock}
Dividing by $\mathsf k_j\sigma_p\sqrt{r_v}\gtrsim A$ gives
\begin{equation}\label{eq:readable-f-force-error}
\delta_j^F\le C(\phi+Ad).
\end{equation}
\end{readabilityblock}

\medskip\noindent\textbf{Fixed-horizon closure.}\par\smallskip
\begin{readabilityblock}
Every force row lies in $E_{\rm vis}$. Lemma~\ref{lem:bfr-polar} therefore gives
the following bound without an actual-rank premise:
\begin{equation}\label{eq:readable-f-polar-error}
 \norm{\cN_{\rm pol}(\mat{F}_j)-\mat{U}_F}_F\le2\delta_j^F.
\end{equation}
\end{readabilityblock}
Hence, whenever $d_j\le1$,
\begin{readabilityblock}
\begin{equation}\label{eq:readable-f-distance-recurrence}
d_{j+1}\le a d_j+(2s/h)\delta_j^F.
\end{equation}
Together, \eqref{eq:readable-f-force-error}, \eqref{eq:readable-f-polar-error} and
\eqref{eq:readable-f-distance-recurrence} propagate the feature distance
and force error from one departing state to the next.
\end{readabilityblock}
The quantity $\phi$ obeys $A^\ell\phi\to0$ for every fixed $\ell$.
Induct on $j$ with the assertion $A^\ell d_j\to0$ for every fixed $\ell\ge0$.
At $j=0$ it follows from $d_0\le r_6/h=O(p^{-2})$.
At the inductive step, $d_j\to0$ licenses the departing bound $d_j\le1$
eventually, and
\[
A^\ell\delta_j^F\le C(A^\ell\phi+A^{\ell+1}d_j)\longrightarrow0.
\]
\readability{The distance recurrence~\eqref{eq:readable-f-distance-recurrence}
proves the assertion at $j+1$.}
All estimates are event-, ball- and head-uniform. For each fixed power
there are finitely many times, so every threshold needed for the theorem
can be chosen before the split.
This establishes the claimed limits without a future tube or rank premise.

\medskip\noindent\textbf{Mean loss and held-out accuracy.}\par\smallskip
Put $u'_{j,xc}=\nu_jAu_{j,xc}$ and
\[
 \mat{G}_j=\nu_jA\mat{\Delta}_j+\mat{H}_j\mat{V}_j^\top-\mat{H}_j^{\rm ref}\bar {\mat{V}}_0^\top.
\]
This paragraph's $\mat{G}_j$ is a logit-error table, distinct from the
departing-state head source in Appendix~\ref{app:decay}.
The actual margins consequently have the form
\begin{readabilityblock}
\begin{equation}\label{eq:readable-f-actual-margin}
\nu_jA[M_x+\xi(2M_x-N_{xc})+kz_j]
+u'_{j,xc}+G_{j,xy_x}-G_{j,xc},
\end{equation}
\end{readabilityblock}
where $|u'_{j,xc}|=O(A\omega)$ and
$\norm{\mat{G}_j}_F/p=O(A\omega+Ad_j+p^{-1})=o(1)$.
At zero, $|u'_{0,xc}|\le Ae_0$ and
$\norm{\mat{G}_0}_F=O(pAr)=o(A)$, because $\mat{\Delta}_0=0$.
The ideal training CE differs from $L_*$ by at most
$e^{-a_{\rm conf}}[2t_j+D_L/(p-1)]=o(1)$,
using $\nu_j\xi>11/540$.
CE is $1$-Lipschitz in the maximum margin error and
$\sqrt2$-Lipschitz in the logit vector. \readability{Averaging the errors
in~\eqref{eq:readable-f-actual-margin} and using}
$\sum_{i\in\Train}\norm{\mat{G}_{j,i}}\le\sqrt n\norm{\mat{G}_j}_F$
proves $L_j=L_*+o(1)$.

\readability{For held-out accuracy, apply the same margin
decomposition~\eqref{eq:readable-f-actual-margin} to the clock signs in
Lemma~\ref{lem:bfr-clock}(iii).}
Every held-out row has an $N_{xc}=2$ rival since
$\rho^2(p-1)>D_L$. Its pre-$J$ ideal margin is at most $-Ag_-$;
at $J$ all ideal margins are at least $Ag_+$.
The clock gaps are uniformly bounded below by $(1-a_J)/2>0$.
Adverse classification therefore
requires a row norm of $\mat{G}_j$ at least a fixed positive multiple of $A$;
counting such rows by $\norm{\mat{G}_j}_F^2$ proves both accuracy limits.
At zero the chosen rival has ideal margin $-2A\xi$; $e_0<\xi$ and
$\norm{\mat{G}_0}_F=o(A)$ make it strictly negative in every row, proving
the exact initial claim independently of the refit.

\medskip\noindent\textbf{Operator motion and every training margin.}\par\smallskip
Lemma~\ref{lem:bse-cR}(iii) and $\xi<\rho Zc_R/2$ give
$\limsup\xi\le36/125$ for the prescribed rates.
The anchored Polar rank estimate gives
$\rank \mat{F}_j\ge r_v[1-(\delta_j^F)^2]$.
Thus the forces are nonzero eventually, and their normalized Polar
directions have operator norm at most
$[r_v(1-\max_{t<J}(\delta_t^F)^2)]^{-1/2}$.
Exact unrolling of the feature update, with
$s\sum_{t<j}a^{-t-1}=hz_j$, gives
\begin{readabilityblock}
\begin{equation}\label{eq:readable-f-operator-motion}
\frac{\norm{\mat{W}_j-\nu_j\bar {\mat{W}}_0}_{\rm op}}{\nu_j}
\le r+\frac{hz_j}{\sqrt{r_v(1-\max_{t<J}(\delta_t^F)^2)}}.
\end{equation}
\end{readabilityblock}
Positive homogeneity and ReLU contraction multiply this bound by at most
$\sqrt2$ for a row-feature error. Using the center training margin
$A(1-e_0)$, the head distance $r+J\eta$, and
\readability{the contrast budget~\eqref{eq:readable-f-head-budgets}},
each actual margin divided by $\nu_jA$ is at least
\[
1-e_0-\frac{\sqrt2K_p}{hq}
 \left(r+\frac{hz_j}{\sqrt{r_v}}(1+o(1))\right)
-\frac{\sqrt2(r+J\eta)}A
 \left[2hq+\sqrt2\left(r+\frac{hz_j}{\sqrt{r_v}}(1+o(1))\right)\right].
\]
All errors except feature motion vanish uniformly. Since $z_j\le9/50$,
$q\to1/4$ and $r_v=2p-1$, the squared limsup of that motion term is
\[
\limsup\left(\frac{\sqrt2K_p(9/50)}{q\sqrt{r_v}}\right)^2
\le32\left(\frac9{50}\right)^2
 \left[\frac45+\left(\frac{36}{125}\right)^2\right]
=\frac{8939808}{9765625}<\left(\frac{24}{25}\right)^2.
\]
Thus the normalized margins have uniform liminf greater than $1/25$,
which proves~\eqref{eq:bwc-liminf}. This operator argument is separate
from the Frobenius comparison used for mean loss and held-out accuracy.
\end{proof}

\paragraph{Completion of Claim I.}
\readability{Lemma~\ref{lem:bfr-induction} now gives both required forms
of control: the logit decomposition~\eqref{eq:readable-f-actual-margin}
gives mean loss and delayed held-out recovery, while the operator-motion
bound~\eqref{eq:readable-f-operator-motion} gives every training margin.
These dynamics conclusions are established before the refit and Fourier
estimates of Claim~II.}
\subsection{Claim II: refit-first recovery and Fourier growth}\label{sec:b-claim-representation}
\subsubsection{Full refit and Fourier fraction}\label{sec:bwc-refit}
\readability{We first analyze the reference table $\mat{F}(z)$:
Lemma~\ref{lem:bwc-ref} bounds its kernel,
Lemma~\ref{lem:bwc-refit} identifies its refit predictions, and
Lemma~\ref{lem:bwc-fourier} proves its Fourier growth.
Lemma~\ref{lem:bwc-near} controls changes in the complete refit, including
its feature-dependent penalty. Section~\ref{sec:bwc-radii} then applies
these estimates to the actual tables using Lemma~\ref{lem:bfr-induction}.}

The complete rule~\eqref{eq:ridgekernel}, written $\mat{Y}_U$ for a table $\mat{U}$,
is invariant under positive scaling by
Lemma~\ref{lem:shared-full-ridge-transfer}(i); $\AH(\mat{U})$ has the same
invariance directly from its ratio definition~\eqref{paper:AH}.

\begin{lemma}[Reference-kernel bounds]\label{lem:bwc-ref}
For all sufficiently large $p$, uniformly over the split event and
$0\le z\le1/5$, $\mat{B}(z)$ has a constant diagonal and
\begin{gather*}
\norm{\mat{P}_{\rm int}\mat{F}_0}_{\rm op}=1,\quad
\norm{\mat{P}_{\rm int}\mat{P}_F}_{\rm op}<8,\quad
\norm{\mat{P}_{\rm int}\mat{F}(z)}_{\rm op}^2\le169/25,\\
(1+2z/3)^2-1\le\theta_R(z),\quad\lambda_R'(z)>4/3,\\
\mat{B}(z)=(\mat{P}_0+\mat{P}_{\rm ax})\mat{B}(z)(\mat{P}_0+\mat{P}_{\rm ax})
+\mat{P}_{\rm int}+\theta_R(z)\mat{P}_R+\mat{E}_z,\\
\norm{\mat{E}_z}_{\rm op}=O(1),\quad\norm{\mat{E}_z}_F=O(\sqrt p),
\quad \mat{E}_z\mat{P}_R=\mat{E}_z\mat{\cZ}=0.
\end{gather*}
Moreover $\norm{\mat{F}(z)_\Train}_F/\sqrt n=2+O(p^{-1/2})$ and
\begin{gather*}
q_{F(z)}=\kappa\rho+o(1),\qquad q_{F(z)}\ge q_{\min},\\
\norm{\mat{F}(z)_\Test v}^2\le C_{\max}
 [\norm{\mat{F}(z)_\Train v}^2+q_{F(z)}\norm v^2].
\end{gather*}
\end{lemma}
\begin{proof}
\readability{Lemma~\ref{lem:bfr-kernel} gives}
$\mat{B}(z)=\mat{B}_0+2z\mat{C}+z^2\mat{B}_2$ and translation/dilation invariance. Thus the
diagonal is constant, the additive, rule and mixed spaces reduce $\mat{B}(z)$,
and $\norm{\mat{F}(z)_\Train}_F^2=(n/p^2)\norm{\mat{F}(z)}_F^2$.
Each eigenvalue of $\mat{B}_2$ on a line is at most
$\norm{\mat{P}_F}_F^2/(p-1)\le2p/(q^2(p-1))<64$.
Since $\mat{B}_0=\Id$ on the interaction space, this proves the operator bounds
and the displayed decomposition. The rule lower bound and derivative
follow from $c_R>2/3$ and $R_R\ge c_R^2$.
On the mixed space, $\mat{E}_z=2z\mat{C}_{\rm mix}+z^2\mat{P}_{\rm mix}\mat{B}_2\mat{P}_{\rm mix}$,
so its operator norm is bounded and
\[
\norm{\mat{E}_z}_F\le(2z+8z^2)\norm{\mat{P}_F}_F=O(\sqrt p).
\]

The constant diagonal, $\norm{\mat{F}_0}_F=2p$ and
$\norm{\mat{P}_F}_F\le\sqrt{2p}/q$ give
$\norm{\mat{F}(z)_\Train}_F/\sqrt n=2+O(p^{-1/2})$.
Hence $q_{F(z)}=4\gamma n[1+O(p^{-1/2})]=\kappa\rho+o(1)$
uniformly, so $q_{F(z)}\ge q_{\min}$ eventually.
For a comparison vector $v$, put $f=\mat{F}(z)v$ and decompose
$f=f_{\rm add}+f_{\rm int}$, with
$f_{\rm add}=(\mat{P}_0+\mat{P}_{\rm ax})f$ and $f_{\rm int}=\mat{P}_{\rm int}f$.
The sampled design gives
$\norm{f_{\rm add}}^2\le(2/\rho)\norm{f_{{\rm add},\Train}}^2$
and $\norm{f_{\rm int}}^2\le169\norm v^2/25$.
Using $\norm{f_{{\rm add},\Train}}^2\le
2\norm{f_\Train}^2+2\norm{f_{\rm int}}^2$ proves the leverage bound.
\end{proof}

For the reference-refit proof put $c_z=1+q_{F(z)}$ and
\[
\begin{gathered}
 \mat{U}_\perp=\mat{T}_\Train-\mat{\cZ}_\Train\mat{\cK}^{-1}\mat{\Psi},\qquad
 \mat{\Delta}_\perp=\mat{T}_\Train^\top \mat{U}_\perp-\rho p\mat{P},\\
 \mat{\alpha}_z=\mat{U}_\perp/(c_z+\rho\theta_R(z)),\qquad
 \mat{Q}=(\mat{B}(z)_{\Train\Train}+q_{F(z)}\Id)^{-1}.
\end{gathered}
\]
\begin{lemma}[Full reference predictions]\label{lem:bwc-refit}
Uniformly over the split event and $z\in[0,1/5]$, for all sufficiently
large $p$,
\begin{equation}\label{eq:bwc-signed}
\norm{\mat{Y}_{F(z)}-\Gamma(z)\mat{T}_\Test}_F/N_\Test=O(\omega).
\end{equation}
At $z=0$,
\begin{equation}\label{eq:bwc-initial-exact}
\mat{Y}_{F_0}=\mat{\cZ}_\Test(\lambda_{\rm ob}\mat{\Upsilon}+\mat{\cK})^{-1}\mat{\Psi}.
\end{equation}
\end{lemma}
\begin{proof}
\readability{The trial coefficient $\mat{\alpha}_z$ removes the additive
part of the training targets. We estimate its normal-equation residual
and its held-out prediction, then use the bound in
Lemma~\ref{lem:bwc-ref} to control the correction to the exact ridge
solution.}

The sampled design gives $\mat{\cK}\succeq s_0\Id$, $s_0\asymp p$,
$\norm{\mat{\cZ}_\Train}_{\rm op}=O(\sqrt p)$,
$\norm{\mat{\Psi}}_F=O(p)$ and $N_\Test\asymp p$.
The orthogonal residual satisfies
\begin{gather*}
\mat{\cZ}_\Train^\top \mat{U}_\perp=0,\qquad
\norm{\mat{T}_\Train-\mat{U}_\perp}_F=O(\sqrt p),\\
\mat{\Delta}_\perp=\mat{T}_\Train^\top \mat{T}_\Train-\rho p\mat{P}-\mat{\Psi}^\top\mat{\cK}^{-1}\mat{\Psi},
\qquad\norm{\mat{\Delta}_\perp}_{\rm op}
\le D_*+3t_*^2/s_0=O(p\omega).
\end{gather*}
Extend $\mat{U}_\perp$ by zero off $\Train$. Its orthogonality annihilates
the entire additive block of $\mat{B}(z)$, irrespective of that block's
eigenvalues. \readability{The nonrule-line leakage estimate of
Lemma~\ref{lem:bfr-kernel}} and $\norm{\mat{E}_z}_F=O(\sqrt p)$
bound $(\mat{E}_z)_{\Omega\Train}\mat{T}_\Train$ by $O(\sqrt p)$.
The correction $\mat{T}_\Train-\mat{U}_\perp$
has the same bound, and $\norm{\mat{E}_z}_{\rm op}=O(1)$; thus
$\norm{(\mat{E}_z)_{\Train\Train}\mat{U}_\perp}_F+
\norm{(\mat{E}_z)_{\Test\Train}\mat{U}_\perp}_F=O(\sqrt p)$.

Since $\mat{T}_\Train \mat{P}=\mat{T}_\Train$, the trial normal residual and test prediction are
\begin{readabilityblock}
\begin{gather}
(\mat{B}(z)_{\Train\Train}+q_{F(z)}\Id)\mat{\alpha}_z-\mat{T}_\Train
=\frac{c_z(\mat{U}_\perp-\mat{T}_\Train)+
\theta_R(z)\mat{T}_\Train\mat{\Delta}_\perp/p+(\mat{E}_z)_{\Train\Train}\mat{U}_\perp}
{c_z+\rho\theta_R(z)},\label{eq:readable-f-trial-residual}\\
\mat{B}(z)_{\Test\Train}\mat{\alpha}_z=\Gamma(z)\mat{T}_\Test+
\frac{\theta_R(z)\mat{T}_\Test\mat{\Delta}_\perp/p+
(\mat{E}_z)_{\Test\Train}\mat{U}_\perp}{c_z+\rho\theta_R(z)}.\label{eq:readable-f-trial-prediction}
\end{gather}
\end{readabilityblock}
Both errors have Frobenius norm $O(p\omega)$.
\readability{Applying the reference bound of Lemma~\ref{lem:bwc-ref}
to the range of $\mat{F}(z)_\Train^\top$ gives}
\[
\mat{Q} \mat{B}(z)_{\Train\Test}\mat{B}(z)_{\Test\Train}\mat{Q}
\preceq C_{\max}\mat{Q} \mat{B}(z)_{\Train\Train}\preceq C_{\max}\Id.
\]
Here $\mat{Q} \mat{B}(z)_{\Train\Train}=\Id-q_{F(z)}\mat{Q}\preceq \Id$.
\readability{The correction to~\eqref{eq:readable-f-trial-prediction} is
the residual~\eqref{eq:readable-f-trial-residual} acted on by
$-\mat{B}(z)_{\Test\Train}\mat{Q}$, whose operator norm is at most
$\sqrt{C_{\max}}$. Dividing the resulting prediction error by
$N_\Test\asymp p$ proves~\eqref{eq:bwc-signed}.}

At $z=0$, the constant diagonal gives $q_{F_0}=4\gamma n$.
Equation~\eqref{eq:bse-pred} at $c_\circ=\lambda_{\rm ob}=1+q_{F_0}$,
multiplied by $\mat{T}_\Train$, gives~\eqref{eq:bwc-initial-exact}.
\end{proof}

For the centered reference trace, put
\[
 B_c=\|\Cgrid \mat{P}_F\|_F/\sqrt{(p-1)(3p+1)},\qquad
 d_c(z)=\|\Cgrid \mat{F}(z)\|_F^2/[(p-1)(3p+1)].
\]
\begin{lemma}[Fourier growth and transfer]\label{lem:bwc-fourier}
For all sufficiently large $p$, uniformly on $0\le z\le1/5$,
\begin{equation*}
\AH(\mat{F}(z))=A_0\lambda_R(z)/d_c(z),\quad
\frac d{dz}\AH(\mat{F}(z))>A_0,\quad
\AH(\mat{F}(z))\le(34/5)A_0.
\end{equation*}
For any family with $z\in[0,1/5]$, $e=\norm{\mat{V}-\mat{F}(z)}_F$ and
$\sup e=o(\sqrt p)$, uniformly
\begin{equation*}
|\AH(\mat{V})-\AH(\mat{F}(z))|=O(e/p^{3/2})=o(A_0).
\end{equation*}
\end{lemma}
\begin{proof}
The initial identity $4p^2-(p+1)^2=(p-1)(3p+1)$ gives $d_c(0)=1$.
Triangle/Cauchy and $B_c=O(p^{-1/2})$ give
$d_c=1+O(p^{-1/2})$ and $d_c'=O(p^{-1/2})$.
The rule-line energy is $(p-1)\lambda_R$, giving the fraction. Since
$\lambda_R'>4/3$ and $\lambda_R\le\norm{\mat{P}_{\rm int}\mat{F}(z)}_{\rm op}^2\le169/25$,
\begin{equation*}
\frac{1}{A_0}\frac d{dz}\AH(\mat{F}(z))
=\frac{\lambda_R'}{d_c}-\frac{\lambda_Rd_c'}{d_c^2}
\ge\frac43-O(p^{-1/2})>1,
\qquad \frac{\AH(\mat{F}(z))}{A_0}\le169/25+O(p^{-1/2})<34/5.
\end{equation*}
For $N_X=\norm{\Padd\Cgrid \mat{X}}_F$, $D_X=\norm{\Cgrid \mat{X}}_F$,
the reference has $N_F=O(\sqrt p)$ and $D_F=\Theta(p)$.
Contraction gives $|N_V-N_F|,|D_V-D_F|\le e$ and hence
$N_V=O(\sqrt p)$, $D_V=\Theta(p)$.
\begin{equation*}
|\AH(\mat{V})-\AH(\mat{F}(z))|
\le\frac{(N_V+N_F)e}{D_V^2}
+\frac{\AH(\mat{F}(z))(D_V+D_F)e}{D_V^2}
=O(e/p^{3/2})=o(A_0).
\end{equation*}
\end{proof}

For the following nearby-table estimate use $b_{\max}=9/5$.
\begin{lemma}[Nearby full feature tables]\label{lem:bwc-near}
For all sufficiently large $p$, let $0\le z\le1/5$, $\mat{U}=\mat{F}(z)$ and
$e=\norm{\mat{V}-\mat{U}}_F\le a_{\min}\sqrt{n_-}$, where
$a_{\min}$ is fixed in Definition~\ref{def:r62-refit} and $b_{\max}=9/5$. Then $q_U,q_V\ge q_{\min}$,
$|q_V-q_U|\le\gamma(4b_{\max}pe+e^2)$ and
$\norm{\mat{Y}_V-\mat{Y}_U}_F\le\mathsf P_p(e)N_\Test$,
with $\mathsf P_p$ as in~\eqref{eq:bwc-Pp}.
\end{lemma}
\begin{proof}
\readability{The reference training-norm bounds in Lemma~\ref{lem:bwc-ref} give}
$2a_{\min}\sqrt n\le\norm{\mat{U}_\Train}_F\le2b_{\max}\sqrt n$;
thus $\norm{\mat{V}_\Train}_F\ge a_{\min}\sqrt n$ and
$q_V\ge\gamma a_{\min}^2n\ge q_{\min}$.
Apply Lemma~\ref{lem:shared-full-ridge-transfer}(ii) with
$e_S,e_D\le e$, $t\le\sqrt p$, $N\le N_+$ and
$L\le\sqrt{C_{\max}}$. Its penalty-change estimate gives the displayed
penalty change, and substitution gives exactly \eqref{eq:bwc-Pp}.
\end{proof}

\begin{lemma}[Initial refit failure]\label{lem:bwc-obs0}
For the signed margin reserve~\eqref{eq:bwc-Emargin},
every held-out input has a rival whose center-refit margin is at most $-\mu_{\rm obs0}$. If a
feature table $\mat{V}$ has $e=\norm{\mat{V}-\mat{F}_0}_F\le a_{\min}\sqrt{n_-}$ and
$\sqrt2N_+\mathsf P_p(e)<\mu_{\rm obs0}$, then every held-out input has a strictly negative refit
margin for $\mat{V}$, and $\Acc_{\rm ref}(\mat{V})=0$.
\end{lemma}
\begin{proof}
Equation~\eqref{eq:bwc-initial-exact} identifies the center-refit row
as $z_x\mat{C}^{(\lambda_{\rm ob})}$.
Every held-out $x$ has a rival $c$ with $N_{xc}=2$, since
$N_x^{(2)}\ge\rho^2(p-1)-D_L>0$.
Equation~\eqref{eq:bse-fine} bounds this margin by
$-2/\Lambda_1+2|\iota|D_*+2\epsilon_{\rm pred}=-\mu_{\rm obs0}$.
Lemma~\ref{lem:bwc-near} with $z=0$ moves each prediction row by at most
$\mathsf P_p(e)N_\Test\le\mathsf P_p(e)N_+$ and each margin by at most
$\sqrt2$ times this quantity, proving the claim.
\end{proof}

\subsection{Assembly: the prescribed ball and one common threshold}\label{sec:bwc-radii}
\readability{The dynamics conclusions are already supplied by
Lemma~\ref{lem:bfr-induction}. The remaining task is to transfer the
reference refit and Fourier estimates to those trajectories, use the
prescribed ball for exact initial refit failure, and choose a common
threshold for all conclusions.}

Use the exact prescribed radius~\eqref{eq:r62ball} from
Definition~\ref{def:r62-ball}.
This section proves a sufficiently-large-prime statement with all
theorem parameters, including $J$, fixed first; no conclusion is
asserted for smaller primes merely because they pass finite tests.

For observer transfer put
\begin{readabilityblock}
\begin{equation}\label{eq:readable-f-observer-normalization}
 \widetilde {\mat{H}}_j=\mat{H}_j/(\nu_jhq),\qquad
 e_j^{\rm obs}=\|\widetilde {\mat{H}}_j-\mat{F}(z_j)\|_F,\qquad
 d_p=\frac{\sqrt{2p}}{hq}.
\end{equation}
\end{readabilityblock}
\begin{proposition}[Uniform eventual window]\label{lem:sc26-b-head-interface}
For the fixed parameters of Theorem~\ref{thm:r62}, there is a finite
$p_0$, chosen before the split, such that for every odd prime $p\ge p_0$
and every split in $\cE_{\rm split}$, the center and all static
quantities used above are defined and $0<r_{\rm new}<1$.
Let $\norm{(\mat{W}_0-\bar {\mat{W}}_0,\mat{V}_0-\bar {\mat{V}}_0)}_F<r_{\rm new}$, and use
\[
\mat{W}_{j+1}=a\mat{W}_j+s\cN_{\rm pol}(-\nabla_{\mat{W}}L(\mat{W}_j,\mat{V}_j)),\qquad
\mat{V}_{j+1}=\mat{V}_j+\mat{D}_j,\qquad \norm{\mat{D}_j}_F\le\eta\quad(j<J).
\]
Then all trajectory conclusions of Theorem~\ref{thm:r62} hold on
$0,\ldots,J$, including strict training margins, the comparison with
this trajectory's own initial refit error, and $\AH(\mat{H}_j)<7/(3p+1)$.
The head budget is $\eta=\theta/(hp)$, prescribed before sampling
in~\eqref{eq:r62optimizer}.
\end{proposition}
\begin{proof}
All orders below are uniform over the event, exact prescribed ball and
head sequences in the statement. Lemma~\ref{lem:bfr-induction} proves
the static certificates, $0<r_{\rm new}=r_6<1$, $A\max_jd_j\to0$ and
strict training. Its loss and actual-head accuracy limits, with $L_*>0$
and the fixed tolerances, give all required CE and accuracy conclusions,
including exactly zero initial held-out accuracy for any fixed tie rule.

\emph{Positive-time refit accuracy.}

\begin{readabilityblock}
For the normalized observer table~\eqref{eq:readable-f-observer-normalization},
ReLU contraction and Lemma~\ref{lem:bfr-induction} give
\begin{equation}\label{eq:readable-f-observer-transfer}
e_j^{\rm obs}\le\frac{\sqrt{2p}\,d_j}{\nu_jq}
=o(\sqrt p),\qquad \mathsf P_p(e_j^{\rm obs})=o(1).
\end{equation}
\end{readabilityblock}
Both the nearby-refit hypothesis $e_j^{\rm obs}\le a_{\min}\sqrt{n_-}$
and the uniform Fourier-transfer hypothesis $e_j^{\rm obs}=o(\sqrt p)$
hold eventually. The uniform reference error in
Lemma~\ref{lem:bwc-refit} is $O(\omega)$. \readability{Combining its
prediction estimate~\eqref{eq:bwc-signed} with Lemma~\ref{lem:bwc-near}
and~\eqref{eq:readable-f-observer-transfer} gives}
\begin{readabilityblock}
\begin{equation}\label{eq:readable-f-actual-refit}
\frac{\norm{\mat{Y}_{\widetilde H_j}-\Gamma(z_j)\mat{T}_\Test}_F}{N_\Test}=o(1).
\end{equation}
\end{readabilityblock}
\readability{For $j\ge1$, Lemma~\ref{lem:bfr-clock}(ii) gives
$z_j>1/(12J)$, so}
$\Gamma(z_j)$ is uniformly bounded below by a positive constant.
A row without a strict correct maximum costs at least
$\Gamma(z_j)^2/2$ in squared prediction error. As
$N_\Test^2=(1-1/p)n_\Test$, the erroneous fraction tends to zero,
giving the required refit accuracy. At $j=0$, $\Gamma(0)=0$,
so this trajectory's own $E_{\rm obs}(\mat{H}_0)\to1$.

\emph{Same-trajectory error reduction.}

\readability{Equation~\eqref{eq:readable-f-actual-refit} also determines
the normalized error along this trajectory. The following lower bound on
$\Gamma$ makes the late-time gain quantitative.}

For $j=J-1,J$, $z_j>1/6$ and the uniform limit $q_{F(z)}\to\kappa\rho$
in Lemma~\ref{lem:bwc-ref} give
\[
\liminf_{p\to\infty}
\frac{\rho\theta_R(z_j)}{1+q_{F(z_j)}}
\ge\frac{19\rho}{81(1+\kappa\rho)}
\ge\frac{152}{891}>\frac16.
\]
Hence $\Gamma(z_j)\ge1/7$ eventually and
\[
E_{\rm obs}(\mat{H}_0)-E_{\rm obs}(\mat{H}_j)
\ge1-(6/7)^2-o(1)=13/49-o(1)>401/1568.
\]
This comparison is with the same initialization, not the center.

\emph{Exact initial refit failure.}

The asymptotic error limit above does not establish zero accuracy at $j=0$;
that conclusion uses the exact prescribed radius.
Here $e_0^{\rm obs}\le d_pr$ and nonnegative coefficients and $r<r_{\rm new}<1$ give
\begin{readabilityblock}
\begin{equation}\label{eq:readable-f-initial-reserve}
\sqrt2N_+\mathsf P_p(e_0^{\rm obs})
\le\sqrt2N_+\mathsf P_p(d_pr)
\le r\sqrt2N_+\mathsf P_p(d_p)<\mu_{\rm obs0}/2.
\end{equation}
The radius choice~\eqref{eq:r62ball} yields
\eqref{eq:readable-f-initial-reserve}, leaving the signed margin reserve
required by Lemma~\ref{lem:bwc-obs0}, which proves zero initial refit accuracy.
\end{readabilityblock}

\emph{Fourier fraction and stepwise gain.}

\readability{Finally, with $A_0=1/(3p+1)$, apply
Lemma~\ref{lem:bwc-fourier} using the table error
in~\eqref{eq:readable-f-observer-transfer}. This gives}
$|\AH(\mat{H}_j)-\AH(\mat{F}(z_j))|=o(A_0)$.
\readability{Lemma~\ref{lem:bfr-clock}(ii) and the reference derivative
in Lemma~\ref{lem:bwc-fourier} give increments}
$>A_0/(12J)$, so actual increments exceed $A_0/(24J)$ eventually;
the same transfer gives
$\AH(\mat{H}_j)\le(34/5)A_0+o(A_0)<7A_0$.
Positive scaling preserves the full refit by
Lemma~\ref{lem:shared-full-ridge-transfer}(i), and preserves $\AH$
by~\eqref{paper:AH}.
There are finitely many states and every bound is uniform, so one
finite $p_0$ establishes all assertions simultaneously.
\end{proof}

\begin{proof}[Proof of Theorem~\ref{thm:r62}]
Proposition~\ref{prop:bse-event} gives the proper split event with probability
at least $1-2/\sqrt p$. Proposition~\ref{lem:sc26-b-head-interface} supplies,
uniformly on that event, a finite pre-sampling threshold and a positive
radius, hence the prescribed nonempty open ball in all $3pm$ coordinates.
The actual update has $1-s\lambda=a$ and head increments of norm $\eta$
or zero, so the same proposition combines Claims~I and~II on its trajectory.
No accuracy after $J$ is asserted.
\end{proof}
\end{proofrevision}

\begin{proofrevision}
\section{Small decoupled readout decay}\label{app:decay}
\begin{readabilityblock}
This appendix gives two positive readout-decay certificates, each using the
assembly result for its own construction. For the half-grid,
Proposition~\ref{prop:aasm-continuation}(a) gives joint finite-horizon
continuity: the same ball, head step and windows preserve
Theorem~\ref{paper:thm:half}(i)--(iii). Its late-refit clause (iv) is outside
this positive-decay conclusion. For the random split,
Proposition~\ref{lem:sc26-b-head-interface} allows any head sequence whose
increments have norm at most $\eta$. We verify that bound using data step
$\eta/2$ and a decay cap chosen before sampling. This preserves the
specified event, ball and $\eta$-based budgets.

In both constructions, the raw training-CE sources are evaluated at the
departing state, decay is outside normalization, and $\cN_F(0)=0$.
The decayed trajectories need not equal the undecayed ones.
Throughout, $d_V=\eta_V\lambda_V$ denotes per-step readout shrinkage;
$\eta_V$ is the actual data step of the particular extension, not
necessarily its undecayed motion budget $\eta$.
\end{readabilityblock}

\subsection{The structured half-grid construction}\label{sec:sc26-c-half}
Fix a prime $p\ge101$ and a width $M\ge p(p^2+3)/2$. Let $s$, $\lambda$, $J$, the
windows~\eqref{eq:acr-windows}, the head step $\eta_V=\eta_0/2$ and the ball $\cB_M$ be those of
Proposition~\ref{prop:aasm-continuation}, and let $d_0$ be its decay bound;
all are fixed before the initialization is chosen.

\begin{corollary}[Half-grid readout decay]\label{cor:crd-halfgrid}
For every fixed $d_V$ with $0<d_V<\min\{d_0,1/(20J)\}$, every initialization in $\cB_M$,
trained by the feature update of~\eqref{paper:optimizer} and the head update
$\mat{V}_{j+1}=(1-d_V)\mat{V}_j+\eta_V\cN_F(\mat{G}_j)$, satisfies Theorem~\ref{paper:thm:half}(i)--(iii), with the
same arrival windows~\eqref{eq:acr-windows}. Moreover $Jd_V<1/20$.
\end{corollary}
\begin{proof}
\begin{readabilityblock}
Apply Proposition~\ref{prop:aasm-continuation}(a) on its final ball with
\begin{equation}\label{eq:readable-g-half-parameters}
 \eta=\eta_V=\eta_0/2<\eta_0,\qquad
 0<d_V<d_0,\qquad Jd_V<J/(20J)=1/20.
\end{equation}
It gives clauses (i)--(iii) on the same arrival
windows~\eqref{eq:acr-windows}. Part (b), which supplies the late-refit
conclusion, assumes $d_V=0$ and is not invoked here.
\end{readabilityblock}
\end{proof}
The cap $d_0$ is a fixed positive number from finite-horizon continuity; no formula is given
for it. The same ball and the same head step serve the undecayed theorem.

\subsection{The random-split construction}\label{sec:sc26-c-random}
Use Appendix~\ref{sec:r62}'s event, width $m=p^2(p-1)^2$,
step $\eta=\theta/(hp)$, common threshold $p_0$ of
Proposition~\ref{lem:sc26-b-head-interface}, and specified centers and
radii $\bar{\mat{V}}_0(\Train),r_{\rm new}(\Train)$.
\readability{Definition~\ref{def:r62-split} supplies the event and
Definition~\ref{def:r62-ball} supplies the prescribed ball; their
definitions are retained when decay is added.}
\begin{definition}[Pre-sampling readout-decay budgets]\label{def:r62-decay}
For the random-split decay extension, define before sampling
\[
 B_p=\max_{\Train\in\cE_{\rm split}}
 \{\norm{\bar{\mat{V}}_0(\Train)}_F+r_{\rm new}(\Train)\},\qquad
 Q_p=B_p+J\eta/2.
\]
The maximum is over all event masks, not the realized split.
All Appendix~\ref{sec:r62} radius and error formulas retain the undecayed
step budget $\eta$.
\end{definition}
Using these pre-sampling bounds, choose
\begin{equation}\label{eq:sc26-c-random-cap}
 0<d_V\le\min\{1/2,\eta/(2Q_p)\}
\end{equation}
before sampling, keep the feature update~\eqref{eq:r62optimizer}, and use
\begin{equation}\label{eq:sc26-c-random-rule}
 \mat{V}_{j+1}=(1-d_V)\mat{V}_j+\frac\eta2\cN_F(\mat{G}_j).
\end{equation}
\begin{corollary}[Random-split readout decay]\label{cor:crd-random}
Fix the parameters of Theorem~\ref{thm:r62} and an odd prime $p\ge p_0$.
Choose the pre-sampling decay~\eqref{eq:sc26-c-random-cap} and use
the update~\eqref{eq:sc26-c-random-rule}.
On $\cE_{\rm split}$, of probability at least $1-2/\sqrt p$, every initialization in
the ball~\eqref{eq:r62ball} satisfies all trajectory conclusions of Theorem~\ref{thm:r62} at states $0,\ldots,J$.
The width, split law, event, common threshold and prescribed radii are those
of Appendix~\ref{sec:r62}; every radius and error formula retains $\eta$,
although the actual data step is $\eta/2$.
\end{corollary}
\begin{proof}
\begin{readabilityblock}
\noindent\textbf{Step 1: obtain a bound valid before sampling.}
The finite nonempty event makes $B_p,Q_p$ in
Definition~\ref{def:r62-decay} finite before sampling. After a split in
that event is realized, the initial ball~\eqref{eq:r62ball} gives
\begin{equation}\label{eq:readable-g-initial-head}
 \|\mat{V}_0\|_F
 <\|\bar{\mat{V}}_0(\Train)\|_F+r_{\rm new}(\Train)
 \le B_p.
\end{equation}

\noindent\textbf{Step 2: bound the head along the decayed trajectory.}
The normalized data source has norm at most one, including when
$\mat{G}_j=0$. Since $0<d_V\le1/2$, the update
\eqref{eq:sc26-c-random-rule} therefore gives
\begin{equation}\label{eq:readable-g-head-recursion}
 \|\mat{V}_{j+1}\|_F
 \le(1-d_V)\|\mat{V}_j\|_F+\eta/2
 \le\|\mat{V}_j\|_F+\eta/2.
\end{equation}
Induction from~\eqref{eq:readable-g-initial-head} yields
\begin{equation}\label{eq:readable-g-head-bound}
 \|\mat{V}_j\|_F<B_p+j\eta/2\le Q_p,
 \qquad 0\le j\le J.
\end{equation}

\noindent\textbf{Step 3: verify the original motion budget.}
Combine the head bound~\eqref{eq:readable-g-head-bound} with the
pre-sampling cap~\eqref{eq:sc26-c-random-cap}. At every departure $j<J$,
\begin{equation}\label{eq:readable-g-increment-budget}
 \|\mat{V}_{j+1}-\mat{V}_j\|_F
 \le d_V\|\mat{V}_j\|_F+\eta/2
 <d_VQ_p+\eta/2
 \le\eta.
\end{equation}
This inequality also covers $\mat{G}_j=0$.
Proposition~\ref{lem:sc26-b-head-interface} now applies with the
undecayed step budget $\eta$ and gives all the stated trajectory
conclusions.
\end{readabilityblock}
\end{proof}
Any deterministic finite upper bound for $B_p$ suffices. Using only the
realized split would instead give post-sampling quantifiers.

\begin{remark}[Readout-decay extensions]\label{rem:readoutdecay}
For the random split the data step decreases from $\eta$ to $\eta/2$, reserving motion
budget for decay; the half-grid keeps its head step and takes its decay cap from joint
continuity. The resulting trajectories need not coincide
with the undecayed ones. No practical decay magnitude, equal two-layer shrinkage,
or positive decay lower bound uniform in $p$ is proved.
In the main-text convention $1-\eta_V\lambda_V$, the half-grid has
$\eta_V\lambda_V=d_V$, so $\lambda_V=d_V/\eta_V$ and $J\eta_V\lambda_V<1/20$; the random split has
$\eta_V=\eta/2$, $\eta_V\lambda_V=d_V$ and
$J\eta_V\lambda_V\le J\eta_V/(B_p+J\eta_V)<1$.
These are consequences, not substitutes for the prescriptions.
Software's $\omega_V$ replaces $\lambda_V$ in these equalities;
normalizing $\mat{G}-\omega_V\mat{V}$ as one source is a different, uncovered optimizer.
\end{remark}
\end{proofrevision}

\end{document}